\documentclass{article}

\usepackage{hyperref}
\usepackage{url}
\usepackage{microtype}
\usepackage{graphicx}
\usepackage{subfigure}
\usepackage{booktabs} % for professional tables
\usepackage{amsfonts}
\usepackage{nicefrac}
\usepackage{microtype}
\usepackage{comment}
\usepackage{color}
\usepackage{tikz}   
\usepackage{bm}
\usepackage{wrapfig} 
\usepackage{makecell}
\usepackage{amsmath}
\usepackage{amssymb}
\usepackage{mathtools}
\usepackage{amsthm}
\usepackage{enumitem}
\usepackage{caption}
\usepackage[capitalize,noabbrev]{cleveref}

\theoremstyle{plain}
\newtheorem{thm}{Theorem}[section] 
\newtheorem{cor}[thm]{Corollary}
\newtheorem{lem}[thm]{Lemma} 
\newtheorem{prop}[thm]{Proposition}

\newtheorem {assu}[thm]{Assumption}

\usepackage{hyperref}

\definecolor{mydarkblue}{rgb}{0,0.08,0.45}
\definecolor{antiquegold}{RGB}{205,127,50} 
\definecolor{deepskyblue}{RGB}{0,191,255} 
\definecolor{deepgreen}{RGB}{34, 139, 34}
\definecolor{deepred}{RGB}{178, 34, 34}
\definecolor{customBlue}{RGB}{24, 116, 205}
\definecolor{SteelBlue}{RGB}{59, 107, 146}
\definecolor{LinkColor}{RGB}{102, 179, 194}

\definecolor{Adam_DA}{RGB}{24, 49, 92}
\definecolor{CAdam}{RGB}{5, 49, 84}
\definecolor{Sign}{RGB}{165, 13, 113}

\hypersetup{
  colorlinks=true,
  linkcolor=mydarkblue ,    % 链接颜色
  filecolor=mydarkblue ,  % 文件颜色
  urlcolor=mydarkblue ,     % URL颜色
  citecolor=SteelBlue       % 引用颜色
}

    \newcommand{\BC}{{\mathbb {C}}}

     \newcommand{\BR}{{\mathbb {R}}}

    \newcommand{\CA}{{\mathcal {A}}}

    \newcommand{\CI}{{\mathcal {I}}} \newcommand{\CJ}{{\mathcal {J}}}
    \newcommand{\CK}{{\mathcal {K}}} 
    \newcommand{\CM}{{\mathcal {M}}} 
    \newcommand{\CO}{{\mathcal {O}}} 
     \newcommand{\CR}{{\mathcal {R}}}
    \newcommand{\CS}{{\mathcal {S}}} 
    \newcommand{\CU}{{\mathcal {U}}} \newcommand{\CV}{{\mathcal {V}}}

    \newcommand{\tm}{{\bm{m}}} 
     
    \newcommand{\tu}{{\bm{u}}} \newcommand{\tv}{{\bm{v}}}
     \newcommand{\tx}{{\bm{x}}}
    \newcommand{\ty}{{\bm{y}}} \newcommand{\tz}{{\bm{z}}}

\newcommand{\diag}{\mathop{\mathrm{Diag}}}

\usepackage{microtype}
\usepackage{graphicx}
\usepackage{subcaption}
\usepackage{booktabs} % for professional tables

\usepackage{hyperref}

\usepackage[accepted]{icml2026}

\usepackage{amsmath}
\usepackage{amssymb}
\usepackage{mathtools}
\usepackage{amsthm}

\usepackage[capitalize,noabbrev]{cleveref}

\theoremstyle{plain}
\newtheorem{theorem}{Theorem}[section]

\theoremstyle{definition}

\theoremstyle{remark}
\newtheorem{remark}[theorem]{Remark}

\usepackage[textsize=tiny]{todonotes}

\icmltitlerunning{Understanding Dynamics of Adam in Zero-Sum Games: An ODE Approach}

\begin{document}

\twocolumn[
  \icmltitle{Understanding Dynamics of Adam in Zero-Sum Games: An ODE Approach}

  % It is OKAY to include author information, even for blind submissions: the
  % style file will automatically remove it for you unless you've provided
  % the [accepted] option to the icml2026 package.

  % List of affiliations: The first argument should be a (short) identifier you
  % will use later to specify author affiliations Academic affiliations
  % should list Department, University, City, Region, Country Industry
  % affiliations should list Company, City, Region, Country

  % You can specify symbols, otherwise they are numbered in order. Ideally, you
  % should not use this facility. Affiliations will be numbered in order of
  % appearance and this is the preferred way.
  \icmlsetsymbol{equal}{*}
  \begin{icmlauthorlist}
    \icmlauthor{Yi Feng}{equal,yyy}
    \icmlauthor{Weiming Ou}{equal,zzz}
    \icmlauthor{Xiao Wang}{equal,xxx}
    %\icmlauthor{}{sch}
    %\icmlauthor{}{sch}
  \end{icmlauthorlist}

%   \icmlaffiliation{xxx}{MoE Key Laboratory of Interdisciplinary Research of
% Computation and Economics, Shanghai University of Finance and Economics, Shanghai, China}
  \icmlaffiliation{yyy}{Aarhus University, Aarhus, Denmark.}
  \icmlaffiliation{xxx}{MoE Key Laboratory of Interdisciplinary Research of
Computation and Economics, Shanghai University of Finance and Economics, Shanghai, China}
\icmlaffiliation{zzz}{Shanghai University of Finance and Economics, Shanghai, China.}

  \icmlcorrespondingauthor{Xiao Wang}{wangxiao@sufe.edu.cn}

  % You may provide any keywords that you find helpful for describing your
  % paper; these are used to populate the "keywords" metadata in the PDF but
  % will not be shown in the document
  \icmlkeywords{Machine Learning, ICML}

  \vskip 0.3in
]

% this must go after the closing bracket ] following \twocolumn[ ...

% This command actually creates the footnote in the first column listing the
% affiliations and the copyright notice. The command takes one argument, which
% is text to display at the start of the footnote. The \icmlEqualContribution
% command is standard text for equal contribution. Remove it (just {}) if you
% do not need this facility.

% Use ONE of the following lines. DO NOT remove the command.
% If you have no special notice, KEEP empty braces:
\printAffiliationsAndNotice{\textsuperscript{*}Authors are listed in alphabetical order.}  % no special notice (required even if empty)
% Or, if applicable, use the standard equal contribution text:
% \printAffiliationsAndNotice{\icmlEqualContribution}

\begin{abstract}
The remarkable success of the Adam in training neural networks has naturally led to the widespread use of its descent-ascent counterpart, Adam-DA, for solving zero-sum games. Despite its popularity in practice, a rigorous theoretical understanding of Adam-DA still lags behind. In this paper, we derive ordinary differential equations (ODEs) that serve as continuous-time limits of the Adam-DA. These ODEs closely approximate the discrete-time dynamics of Adam-DA, providing a tractable analytical framework for understanding its behavior in zero-sum games. Using this ODE approach, we investigate two fundamental aspects of Adam-DA: local convergence and implicit gradient regularization. Our analysis reveals that the roles of the first- and second-order momentum parameters in zero-sum games are exactly the opposite of their well-documented effects in minimization problems. We validate these predictions through GAN experiments across multiple architectures and datasets, demonstrating the practical implications of this reversed momentum effect.
\end{abstract}

\section{Introduction}

Zero-sum games lie at the core of many modern machine learning tasks, including GANs \citep{goodfellow2014generative} and adversarial training \citep{goodfellow2014explaining}. While a rich body of theory has been developed for solving zero-sum games, particularly for Gradient Descent-Ascent (GDA) and its variants \citep{daskalakis2017training,mokhtari2020unified,fasoulakis2022forward}, in practice Adam Descent-Ascent (Adam-DA), the descent-ascent counterpart of Adam \citep{kingma2014adam}, remains the optimizer of choice. For instance, most GANs studies employ Adam-DA for training \citep{ arjovsky2017wasserstein,zhao2021improved,sauer2023stylegan}.  Despite this ubiquity, the theoretical understanding of Adam in zero-sum games remains far less developed compared to its  extensive studied behavior in minimization.

When studying momentum-based algorithms in games, such as Adam-DA, one might expect that, although some game-specific adaptations are necessary, the principles derived from theoretical analyses in minimization largely remain applicable to zero-sum games. For instance, the recent work of \cite{huang2022new}, \cite{bot2023relaxed} and \cite{lotidisaccelerated} study the benefits of incorporating Nesterov or Heavy-ball momentum, a core component of Adam, in game-solving algorithms. All of these works primarily focus on \textit{positive} momentum parameters, reflecting the standard choice in optimization algorithms for minimization. However, practical experience in GANs training--a prominent application of zero-sum games--suggests that \textit{negative} momentum often improves stability and performance \citep{gidel2019negative}. This gap between theory and practice motivates the following questions:

%However, previous studies suggest that this intuition may not hold. For instance, \cite{gidel2019negative} showed that Adam with a negative first-order momentum can achieve better training performance in GANs, whereas in minimization the first-order momentum of Adam is typically chosen to be nonnegative. Recently, \cite{feng2025continuoustime} discovered distinct behaviors of Heavy-ball momentum, a key component of Adam, in zero-sum games compared to minimization. These observations motivate the following questions:

\textit{How can we theoretically understand Adam’s behavior in zero-sum games? In particular, in which aspects does Adam behave differently in zero-sum games compared to standard minimization?}

In this work, we try to answer above questions by deriving ordinary differential equations (ODEs) which are the continuous time limit of Adam-DA. The derived ODEs accurately approximate the underlying algorithm. This approach of using ODEs to study optimization algorithms dates back to Polyak's seminal work
\citep{polyak1964some} for analyzing momentum method in minimization problems and has been further developed in recent years to analyze the dynamics of various
algorithms \citep{su2016differential,wibisono2016variational,shi2022understanding}. 

% In this work, we try to answer above questions in a \emph{deterministic} setting. Although Adam is
% typically used with stochastic gradients, deterministic analyses often yield valuable insights and act as an initial step. In
% minimization, they have uncovered phenomena such as the edge of stability
% \citep{cohen2021gradient,cohen2025understanding} and implicit regularization effects
% \citep{wang2022does,zhang2024implicit}. Since Adam in min--max games remains poorly understood, we take
% deterministic Adam as a first step. Methodologically, we focus on an ordinary differential equations (ODEs) approach, which yields tractable continuous-time ODEs that
% accurately approximate the underlying algorithms. This approach dates back to Polyak's seminal work
% \citep{polyak1964some} and has been further developed in recent years to analyze the dynamics of various
% algorithms in optimization \citep{su2016differential,wibisono2016variational}.

In particular, we study two aspects of Adam in zero-sum games: \textbf{local convergence} and \textbf{implicit gradient regularization}. These two aspects are natural targets: local convergence characterizes the behavior of the trajectory near an equilibrium, whereas implicit gradient regularization provides a more global perspective by describing how the trajectory interacts with the geometry (e.g., flatness) of the loss landscape. The two viewpoints are complementary and, taken together, offer a more complete picture of Adam’s dynamics in zero-sum games.
%Our  results are twofold, addressing two complementary aspects of Adam in zero-sum games: \textbf{local convergence} and \textbf{implicit gradient regularization}. Local convergence characterizes the local behavior of the algorithm when its trajectory is close to an equilibrium, whereas implicit gradient regularization takes a global perspective, describing how the algorithm's trajectories interact with the geometry of the loss landscape. Our main results and contribution can be summarized as follows.

\subsection{Summary of Contributions and Limitations}

\paragraph{Contributions on Local Convergence:} We provide quantitative local convergence results for both the ODEs of Adam-DA (Theorem~\ref{Corollary: Continuous Local convergence for general case}) and the Adam-DA algorithm (Theorem~\ref{Thm: Discrete Local convergence for general case}). We find that, in typical zero-sum games, a \textit{smaller} first-order momentum improves local convergence over a broader range of step sizes (Corollary~\ref{lmss}). This contrasts with Adam's behavior in minimization, where achieving a similar effect typically requires a \textit{larger} first-order momentum. As a further corollary, we show that Adam-DA always diverges on bilinear objectives, regardless of parameter choices (Corollary~\ref{Corollary: Continuous Local convergence for bilinear case}). This is again opposite to the minimization setting, where Adam can converge to a neighborhood of a local minimum under suitable parameters \citep{zhang2022adam}. These results substantially extend prior separations between game dynamics and minimization, which were previously established mainly for simpler methods without adaptivity and momentum \cite{bailey2018multiplicative,bailey2020finite}.

\paragraph{Contributions on Implicit Gradient Regularization:} We provide qualitative descriptions of how the parameters of algorithm influence the interaction between the algorithm's trajectories and the flatness of min-max loss landscapes through the analysis of purposed ODEs. We find that a \textit{smaller} first-order momentum $\beta$ and a \textit{larger} second-order momentum $\rho$ guide the trajectories toward flatter regions of the loss landscape. Again, this behavior is exactly the opposite of the implicit gradient regularization effect observed in minimization, where achieving a similar effect requires a \textit{larger} $\beta$ and a \textit{smaller} $\rho$ \citep{pmlr-v235-cattaneo24a}. We further validate these predictions empirically on GAN training across common architectures and datasets, a representative class of min--max games captured by our theory.

\paragraph{Limitations.} A limitation of this work is that we primarily study Adam’s dynamics in a deterministic setting. Nevertheless, our findings are informative for the stochastic regime. In particular, the experiments in Section~\ref{IGR} use stochastic Adam for GAN training, and the observed behavior is consistent with our theoretical predictions. We also note that, in minimization, deterministic analyses of Adam have been an important step and have yielded key insights, including the edge-of-stability phenomenon \citep{cohen2021gradient,cohen2025understanding} and implicit regularization effects \citep{wang2021implicit,wang2022does,zhang2024implicit,xie2024implicit}. Since Adam in zero-sum games is substantially less understood than in minimization, we view the deterministic analysis here as a natural first step toward a complete theory.

\subsection{Related Works}
%\textcolor{blue}{The most close works in the literature to the current work are \cite{rosca2021discretization} and \cite{feng2025continuoustime}, which studied the dynamics of (deterministic) gradient descent-ascent and heavy-ball momentum methods in zero-sum games through the lens of continuous-time analysis. The current work represent a nature extension of their results to the Adam algorithms. In particular, both of \cite{feng2025continuoustime} and the current work choose to focus on understanding the local convergence and implicit regularization aspects. Future compare between these works are provided in the Appendix \ref{Appendix: detailed comparison with feng}. In the following, we present more related works around the continuous-time analysis, local convergence in zero-sum games, and implicit gradient regularization.}

\paragraph{ODEs in Optimization.} 
ODEs method was introduced to study momentum in minimization problems, starting with the seminal work of \cite{polyak1964some}. Recently, inspired by optimization algorithms that combine momentum and adaptivity such as Adam, ODEs or their stochastic generalization has also been applied to study adaptive methods in minimization, e.g., \citep{gadat2022asymptotic,ma2022qualitative,compagnoniadaptive}. For zero-sum games,  \cite{suh2023continuous} employed this methodology to study the anchor acceleration methods.  \cite{compagnoni2024sdes} developed stochastic differential equations for algorithms like Extra-gradient in zero-sum games under a stochastic setting.  Two recent works, \cite{rosca2021discretization} and \cite{feng2025continuoustime} provide continuous-time analyses of vanilla Gradient Descent-Ascent and Heavy-ball momentum in zero-sum games, respectively. The current work generalizes these results to the Adam algorithm, which is more complex and remains poorly understood in game-theoretic settings. A more detailed comparison is provided in Appendix~\ref{Appendix: detailed comparison with feng}.

\paragraph{Local Convergence in zero-sum Games.} %Motivated  by observations that although global convergence usually fails for zero-sum algorithms in real-world applications like GANs training, local convergence is still achievable \citep{nagarajan2017gradient,mescheder2017numerics}, 
Local convergence of learning algorithms in zero-sum games has received great attention in recent years. \cite{pmlr-v89-liang19b} analyzed the local convergence of the Gradient Descent-Ascent (GDA) algorithm and its variants, showing the importance of the interaction between players on the dynamics of the algorithms. \cite{fiez2021local} provide a local convergence analysis of GDA under finite timescale separation. \cite{li2022convergence} studied the local convergence to a Stackelberg Equilibrium for GDA-based learning algorithms. \cite{pmlr-v151-zhang22e} studied the local convergence of alternating GDA methods and demonstrated its near-optimal property. Recently, \citet{wang2024local} proved the surprising fact that partial curvature generically suffices for the local convergence of GDA. 
%Local convergence results have also been extended to different zero-sum game settings, such as probability distribution spaces \citep{wang2022exponentially} and Riemannian manifolds \citep{zhang2025local}. 

%\textcolor{blue}{Besides this, recent work has explored adaptive time-scale strategies: \cite{litiada} proposed Tiada, a time-scale adaptive method for nonconvex minimax problems, and \cite{yang2022nest} developed nested adaptive schemes that achieve parameter-agnostic convergence—complementary approaches that adaptively tune per-variable step sizes and thereby mitigate the need for manual time-scale separation.}

\paragraph{Implicit Gradient Regularization.}

The implicit gradient regularization (IGR) effect was first developed for minimization problems. \citet{barrett2020implicit} studied IGR for the gradient descent algorithm. It was later extended to momentum methods by \citet{ghosh2023implicit}, which shows that large momentum parameters usually make algorithms find flatter minima. Recently, \citet{pmlr-v235-cattaneo24a} has extended these results to Adam.\footnote{The detailed discussion on Adam is in Appendix~\ref{Appendix: detailed comparison with feng}.} For zero-sum games, \cite{rosca2021discretization} first derived IGR for GDA algorithms, and \cite{feng2025continuoustime} extended this approach to momentum methods. Inspired by these findings, several recent algorithms that explicitly incorporate gradient regularization to enhance performance have been proposed \citep{zhang2023flatness,zhang2023gradient}.

%Implicit gradient regularization, a surprising phenomenon, which describes gradient-based algorithms implicitly regularize models by penalizing the trajectories that have large gradient norms in the loss surface. This drives the trajectories of gradient-based algorithms toward the local minimum in flatter regions which has smaller gradient norms. These flatter regions usually has lower test error and better generalization performance. To explain the implicit gradient regularization of Gradient Descent (GD), \cite{barrett2020implicit} employ backward error analysis in numerical integration to derive the modified gradient flow which induced by the loss function plus a gradient norm term. They proved that GD is closer than the modified gradient flow, rather than the original gradient flow. The gradient norm term in the modified continuous time flow provides an explanation for the implicit gradient regularization. \cite{ghosh2023implicit} and \cite{feng2025continuoustime} extended backward error analysis and investigated implicit gradient regularization in minimization and zero-sum games respectively. \cite{ma2022qualitative} proved Adam and RMSprop are closer to SignGD, and \cite{pmlr-v235-cattaneo24a} pointed out that Adam and RMSprop may implicitly regularize model by adding a $\ell_1$ norm term into the modified flows.

\section{Preliminaries}\label{PRE}
\textbf{Notations.} For matrix $\CM$, $\mathrm{Sp}(\CM)$ denotes the set of its eigenvalues in $\BC$. For $\lambda \in \mathrm{Sp}(\CM)$, $\Re(\lambda)$ and $\Im(\lambda)$ represent the real and imaginary parts of $\lambda$. The notation $\CM \preccurlyeq \textbf{0}$ or $\CM \succcurlyeq \textbf{0}$ means that $\CM$ is a negative or positive semi-definite matrix. We use $\mathrm{EigVec}(\CM)$ to denote the eigenspace of $\CM$, and $\mathrm{Ker}(\CM)$ to represent its kernel space, i.e., $\mathrm{Ker}(\CM) = \{ \tz \in \BC^d \mid \CM \tz = \textbf{0}\}$. $\CI_d$ denotes the $d$-dimension identity matrix. $\diag\{\tv\}$ denotes a diagonal matrix with diagonal elements $\{\tv_i\}_{i=1}^d$. Without specialization, we use the component-wise multiplication and division of vectors, as well as component-wise addition.
%In this section, we present some background of this paper. For two vectors $\bm{u}, \bm{v}$, operations such as $\bm{u}^2$, $\bm{u}/\tv$, and $\sqrt{\bm{u}}$ are understood to be element-wise.

\textbf{Zero-Sum Games.}  A zero-sum game with smooth loss function $f(\tx,\ty)$ can be formulated as
\begin{align}\label{minmaxprel}
    \min_{\tx \in \BR^{d_1}} \max_{\ty \in \BR^{d_2}} f(\tx,\ty)\tag{Zero-Sum Games}
\end{align}
If a pair of strategies $(\tx^*,\ty^*)$ satisfies $\forall \tx \in \CU,\ f(\tx,\ty^*) \ge f(\tx^*,\ty^*)$  and $\forall \ty \in \CV,\ f(\tx^*,\ty^*) \ge f(\tx^*,\ty)$ for some $\tx^*$'s neighborhood $\CU \subseteq \BR^{d_1}$ and $\ty^*$'s neighborhood $\CV \subseteq \BR^{d_2}$, then $(\tx^*,\ty^*)$ is called a \textit{local Nash equilibrium}. This is one of the most widely used solution concepts in zero-sum games, and is the focus of this work.

\textbf{Adam in Zero-Sum Games.} In zero-sum games, the $x$-player (resp. $y$-player) aims to minimize (resp. maximize) the objective function $f(\tx, \ty)$. Accordingly, the Adam algorithm must be adapted, which result in the following Adam Descent-Ascent (Adam-DA) algorithm. In particular, the x-player updates according to
\begin{align}
& \tilde{\tv}_{n+1} = \rho \tilde{\tv}_{n} + (1-\rho) \left(\nabla_x f(\tx_n,\ty_n)\right)^2,\nonumber \\
&\tilde{\tm}_{n+1} = \beta \tilde{\tm}_{n} + (1-\beta) \nabla_x f(\tx_n,\ty_n),\nonumber \\
&\tx_{n+1} = \tx_n - h \frac{\tilde{\tm}_{n+1}/(1-\beta^{n+1})}{\sqrt{\tilde{\tv}_{n+1}/(1-\rho^{n+1}) + \epsilon}},\nonumber 
\end{align}
and the y-players updates according to
\begin{align}\label{Adam}
& \hat{\tv}_{n+1} = \rho \hat{\tv}_{n} + (1-\rho) \left(\nabla_y f(\tx_n,\ty_n)\right)^2,\nonumber \\
&\hat{\tm}_{n+1} = \beta \hat{\tm}_{n} + (1-\beta) \nabla_y f(\tx_n,\ty_n), \nonumber\\
&\ty_{n+1} = \ty_n + h \frac{\hat{\tm}_{n+1}/(1-\beta^{n+1})}{\sqrt{\hat{\tv}_{n+1}/(1-\rho^{n+1}) + \epsilon}} \tag{\textcolor{Adam_DA}{Adam-DA}}
\end{align}
Here $h>0$ is the step size, $\epsilon>0$ is the numerical stability parameter, $\beta \in (-1,1)$ is the first-order momentum factor and $\rho \in (0,1)$ is the second-order momentum factor. 
%We use the term \textit{smaller momentum} to indicate that the value of $\beta$ is smaller, not its absolute value.

\textbf{Local Behaviors of Dynamical Systems.} For a system of differential equations $\dot{\tx}(t)=g(x)$ where $g:\mathbb{R}^d\rightarrow\mathbb{R}^d$ is a differentiable function, let $\tilde{\tx} \in \BR^d$ satisfy $g(\tilde{\tx}) = 0$. Then the local behavior of the system near $\tilde{\tx}$ is determined by the eigenvalues of Jacobian $\CJ_g(\tilde{\tx}) = \left( \frac{\partial g_i}{\partial x_j} (\tilde{\tx})\right)_{i,j}$: 
\begin{prop}\citep{khalil2002nonlinear}\label{Jacobiancr}
Suppose that $g$ is continuously differentiable. If $\alpha = \max_{\lambda \in \mathrm{Sp}(\CJ_g)} \Re(\lambda) < 0$, then there exist constants $\delta > 0$ and $C > 0$ such that for all initial conditions satisfying $\lVert \tx(0) - \tilde{\tx} \rVert \le \delta$, we have $\lVert \tx(t) - \tilde{\tx} \rVert \le C e^{t \alpha}, \forall t>0.$
\end{prop}

% \textbf{Local Behaviors of Discrete Dynamical Systems.}
% \begin{lem}\label{Lemma: local stability of dynamical system}[Corollary 4.35 in \cite{elaydi2005introduction} and Theorem II.1 in \cite{bock2021local}] Consider $\bar{T}: M\rightarrow M$ with a fixed point $w^*$ and $\bar{T}$ continuously differentiable in an open disk $B_{\delta}(x^*)\subset M$ with radius $\delta$. Assume $\varrho(\text{Jac}(\bar{T}_{w^*}))<1$, then there exists $0<\delta_0<\delta$ and $0\leq c<1$ such that for all $w_0$ with $\|w_0-w^*\|<\epsilon$ and for all $t\in \mathbb{N}$,
%     \[
%     \|w(t;w_0)-w^*\|\leq c^t\|w_0-w^*\|.
%     \]
        
%     \end{lem}

% \textbf{Notations.} For matrix $\CM$, $\mathrm{Sp}(\CM)$ denotes the set of its eigenvalues in $\BC$. For $\lambda \in \mathrm{Sp}(\CM)$, $\Re(\lambda)$ and $\Im(\lambda)$ represent the real and imaginary parts of $\lambda$. The notation $\CM \preccurlyeq \textbf{0}$ or $\CM \succcurlyeq \textbf{0}$ means that $\CM$ is a negative or positive semi-definite matrix. We use $\mathrm{EigVec}(\CM)$ to denote the eigenspace of $\CM$, and $\mathrm{Ker}(\CM)$ to represent its kernel space, i.e., $\mathrm{Ker}(\CM) = \{ \tz \in \BC^d \mid \CM \tz = \textbf{0}\}$. $\CI_d$ denotes the $d$-dimension identity matrix. $\diag\{\tv\}$ denotes a diagonal matrix with diagonal elements $\{\tv_i\}_{i=1}^d$. Without specialization, we use the component-wise multiplication and division of vectors, as well as component-wise addition of vectors.

\section{Continuous-Time Model}\label{CTM}
\begin{figure*}[t]
    \centering
    \subfigure{
        \includegraphics[width=2.0in]{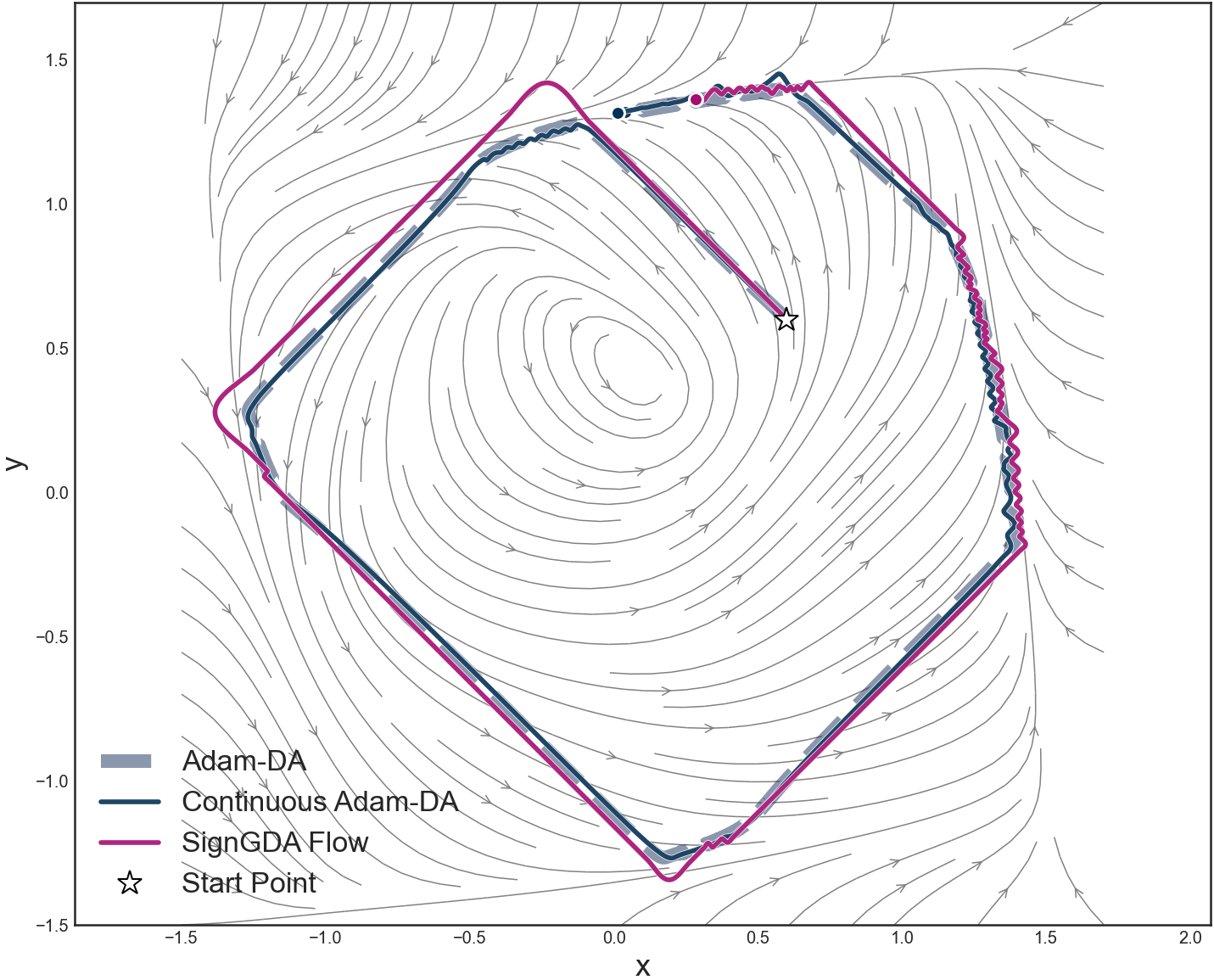}
        \label{p1}
    }
    \subfigure{
        \includegraphics[width=2.0in]{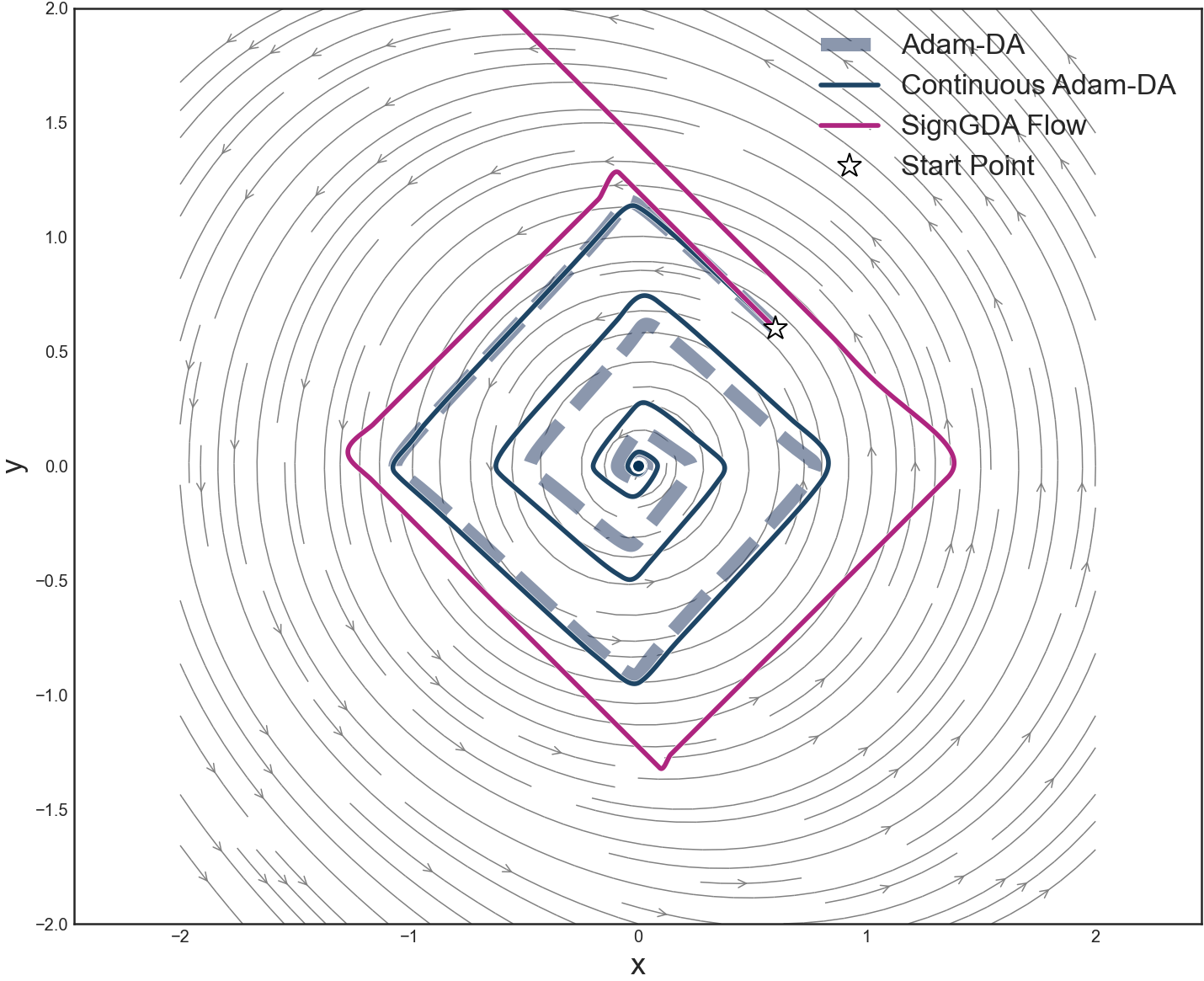}
        \label{p2}
    }
    \subfigure{
        \includegraphics[width=2.0in]{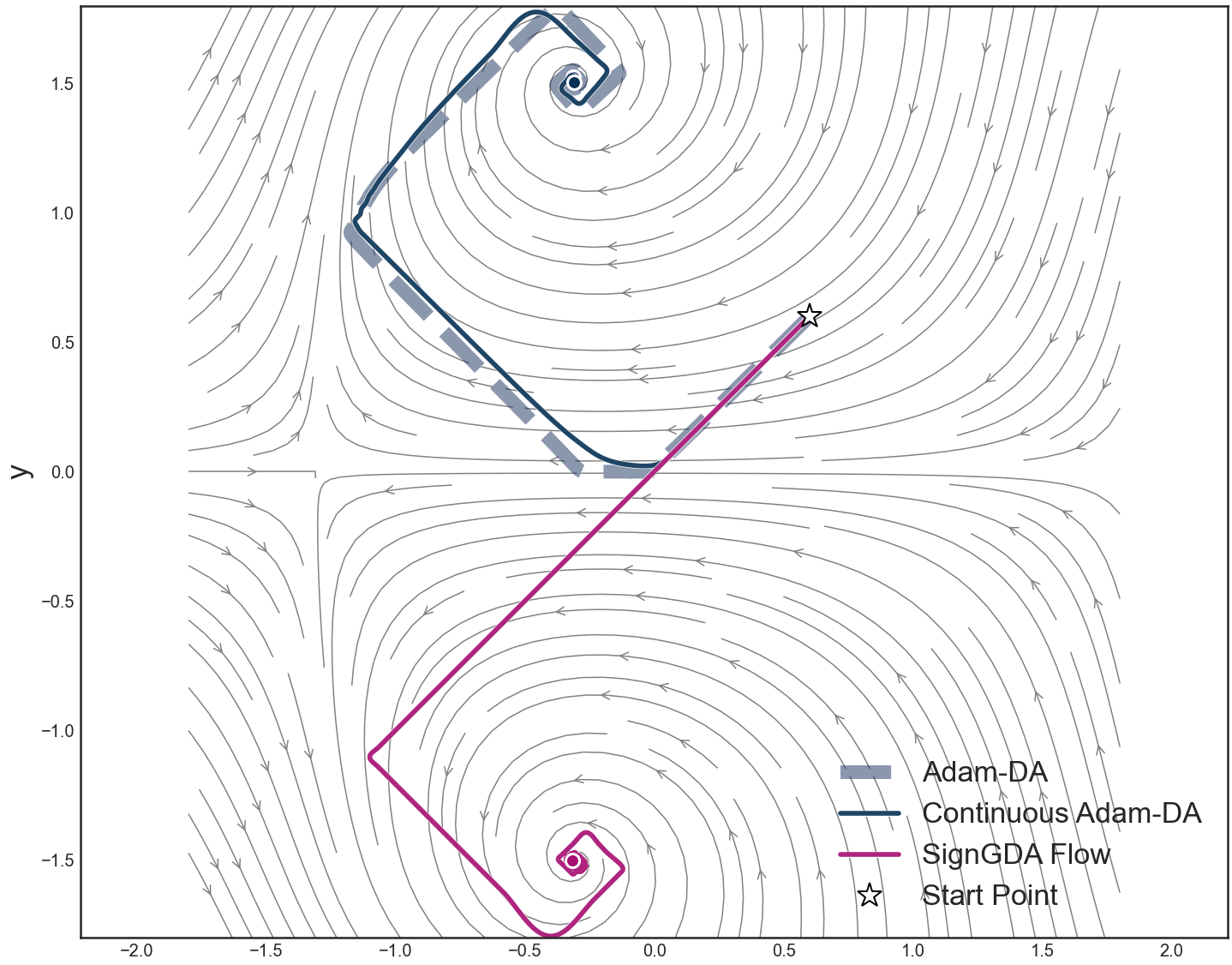}
        \label{p3}
    }
 %    \subfigure[Distance Curves of $f_1$]{
	% \includegraphics[width=1.6in]{Pictures/oder_dist.png}
 %        \label{p4}
 %    }
 %    \subfigure[Distance Curves of $f_2$]{
	% \includegraphics[width=1.6in]{Pictures/oder_dist2.png}
 %        \label{p5}
 %    }
 %    \subfigure[Distance Curves of $f_3$]{
	% \includegraphics[width=1.6in]{Pictures/oder_dist3.png}
 %        \label{p6}
 %    }
    \caption{\textit{Trajectories of \ref{Adam}, \ref{CADAM}, and \ref{sign} on three test functions from \cite{compagnoni2024sdes}.  \ref{CADAM} closely approximates \ref{Adam}. Especially in \ref{p2} and \ref{p3}, where \ref{sign} either diverges or approaches to a different equilibrium, while the trajectories of the other two methods remain similar. More details are provided in Appendix \ref{FDoF}.}}
    \label{comparewithsign}
\end{figure*}

%[Trajectories of $f_2$]

% Figures \ref{p4}, \ref{p5}, and \ref{p6} show the distances of two continuous-time models between Adam, with results averaged over 30 random initial conditions. Future details are provided in Appendix \ref{FDoF}.

% \begin{itemize}
% \item Theorem statement (give the equation, and assumption, and error bound $\CO(h^3)$, add pictures to compare the trajectories of our continuous-time model, $\CO(h^2)$-model \textcolor{red}{(SignGD?)} and discrete adam)

% % \item Compatiable with previous equations: Adam in minimization, Heavy-ball momentum in min-max games

% \item Argue the properties of continuous-time model are very related with the relation between $\epsilon$ and gradient norm. This motive us to divide the following discussion into two parts: 

% 1. when $\epsilon$ is smaller than gradient norm, we consider local convergence   2. when $\epsilon$ is larger than gradient norm, we consider implicit regularization
% \end{itemize}

%%%%%%%%%%%%%%%%%%%%%%%%%%%%%%%%%%%%%%%%%%%%%%%%%%%%%%%%%%%%%%%%%%%%%%%%%%%%%%%%%%%%%%%%%%%%%%%%%%%%%%%%%%%%%%%%%%%%%%%%%%%%%%%%%%%%%%%%%%%%%%%%%%%%%%%%%%%%%%

In this section, we present our continuous-time models for \ref{Adam}. First, we establish an error bound between the trajectories of discrete-time algorithms and continuous-time models. Then, we compare our model with SignGDA-flow, the min-max adaptation of the continuous-time model proposed by \cite{ma2022qualitative} for Adam in minimization. 

The continuous-time model we purposed for \ref{Adam} in zero-sum games is the following:
\begin{align}\label{CADAM}
 \dot{\tx}(t)
= & - \mu_{\epsilon}(\tx,\ty) \Biggl(
    \nabla_x f(\tx,\ty) + \frac{h}{2}\,\CM^{\mu}_{\beta,\rho,\epsilon}(\tx,\ty)\,\cdot\,\nonumber \\
&\nabla_{x}\Bigl(
  \lVert \nabla_x f(\tx,\ty) \rVert_{1,\epsilon}
        - \lVert \nabla_y f(\tx,\ty) \rVert_{1,\epsilon}
    \Bigr)
\Biggr), \nonumber\\
 \dot{\ty}(t)
=& \nu_{\epsilon}(\tx,\ty)\Biggl(
    \nabla_y f(\tx,\ty)  + \frac{h}{2}\,\CM^{\nu}_{\beta,\rho,\epsilon}(\tx,\ty)\,\cdot\,\nonumber \\
&  \nabla_{y}\Bigl(
    \lVert \nabla_x f(\tx,\ty) \rVert_{1,\epsilon}
    - \lVert \nabla_y f(\tx,\ty) \rVert_{1,\epsilon}
    \Bigr)
\Biggr).\tag{\textcolor{CAdam}{Continuous Adam-DA}}
\end{align}

Here the perturbed $\ell_1$ norm\footnote{Strictly speaking, $\lVert \cdot \rVert_{1,\epsilon}$ is not a norm. The terminology used here follows \cite{pmlr-v235-cattaneo24a}.} is defined as $\| \tv \|_{1,\epsilon} := \sum^d_{i=1} \sqrt{\tv^2_i + \epsilon}$.  Other terms are defined as:
\begin{itemize}
\item $\mu_{\epsilon}(\mathbf{x}, \mathbf{y}) := \mathrm{Diag} \left\{ \left( \| \partial_{x_j} f(\mathbf{x}, \mathbf{y}) \|_{1,\epsilon}^{-1} \right)_{j=1}^{d_1} \right\}$
\item $\mathcal{M}^{\mu}_{\beta,\rho,\epsilon} := \mathcal{K}(\beta,\rho) \mathcal{I}_{d_1} + \frac{\epsilon(1+\rho)}{1-\rho} \mu_{\epsilon}^2(\mathbf{x}, \mathbf{y})$   
\item $\nu_{\epsilon}(\mathbf{x}, \mathbf{y})  := \mathrm{Diag} \left\{ \left( \| \partial_{y_i} f(\mathbf{x}, \mathbf{y}) \|_{1,\epsilon}^{-1} \right)_{i=1}^{d_2} \right\}$
\item $\mathcal{M}^{\nu}_{\beta,\rho,\epsilon} := \mathcal{K}(\beta,\rho) \mathcal{I}_{d_2} + \frac{\epsilon(1+\rho)}{1-\rho} \nu_{\epsilon}^2(\mathbf{x}, \mathbf{y})$
\end{itemize}
where $ \CK(\beta,\rho) = (1+\beta)/(1-\beta) - (1+\rho)/(1-\rho).$
%Now we state the main theorem of this section:
\begin{thm}\label{error_ana} Let $f(\tx,\ty)$ be a smooth function with bounded  derivatives up to fourth order. Then for any given finite time horizon, the solution trajectories $\left(\tx(t),\ty(t) \right)$ of \ref{CADAM} is locally $\CO(h^3)$-close to the trajectories of \ref{Adam} after $ \max\{\frac{2\log h}{\log|\beta|}, \frac{2\log h}{\log\rho}\}$ steps. 
\end{thm}

\begin{remark}
The bounded derivatives condition in Theorem \ref{error_ana} is necessary to prove rigorous results about the errors between ODEs and algorithms. It is widely used in related literature \citep{rosca2021discretization,ghosh2023implicit,compagnoni2024adaptive}. Indeed, as remarked by \cite{pmlr-v235-cattaneo24a}, such assumptions are explicitly or implicitly present in all previous work on the backward error analysis of gradient-based machine learning algorithms.
\end{remark}

The proof of Theorem \ref{error_ana} relies on the backward error analysis from the numerical analysis literature \citep{haier2006geometric}, and is inspired by \cite{pmlr-v235-cattaneo24a}, which derived ODEs for the Adam algorithm in the minimization setting to study its implicit regularization effect. However, our equations differ significantly due to the presence of two interacting players inherent to zero-sum games. Detailed proof of Theorem \ref{error_ana} is provided in Appendix \ref{proofofthm1}.

Theorem \ref{error_ana} shows that \ref{CADAM} approximates \ref{Adam} with a $\mathcal{O}(h^3)$-local error. To highlight this advantage, we compare it with \ref{sign}, a min-max adaptation of the ODEs proposed by \cite{ma2022qualitative}, which approximates Adam with a $\mathcal{O}(h^2)$-local error:
\begin{align}\label{sign}
&\dot{\tx}(t) =  - \mu_{\epsilon}(\tx,\ty) \nabla_xf(\tx,\ty),\nonumber\\
&\dot{\ty}(t) = \nu_{\epsilon}(\tx,\ty) \nabla_yf(\tx,\ty)\tag{SignGDA-flow}
\end{align}
The name of \ref{sign} comes from the fact that when $\epsilon \thickapprox 0$, \ref{sign} depends only on the signs of the partial derivatives. In Figure \ref{comparewithsign}, we present numerical examples from \cite{compagnoni2024sdes} to compare these three methods. It can be observed that \ref{CADAM} better approximates \ref{Adam}, thus highlighting the benefits of $\CO(h^3)$-local error.

\section{Local Convergence}\label{LC}
Unlike minimization algorithms, which almost always converge to a local minimum \citep{LP19}, algorithms for zero-sum games often fail to converge and may exhibit complex behaviors, such as cycles and chaos \citep{bailey2020finite,CP2020}. However, recent studies show that if the initial strategy pair is sufficiently close to a local equilibrium, convergence can still be achieved \citep{li2022convergence,wang2024local,zhang2025local}, which is known as \textit{local convergence}. Building upon these findings, we are motivated to investigate the following question:

 \textit{How do the parameters in \ref{CADAM} and \ref{Adam} influence their local convergence?} 
 
In this section, we answer the above question. In Section \ref{SJ}, we describe the Jacobian structure of \ref{CADAM} at a local equilibrium, which is our main tool to study local convergence. The main results are presented in Section \ref{LC of continuous model}. Most proofs are deferred to Appendix \ref{Appendix_S4}.

\subsection{Jacobian Structure of Adam-DA}\label{SJ}
From dynamical system theory, the Jacobian matrix of a dynamical system completely determines its local behavior around an equilibrium. We first introduce the Jacobian of the following ODE for Gradient Decent-Ascent (GDA), which serves as a fundamental tool for analyzing \ref{CADAM} and \ref{Adam}:
\begin{align*}\label{GDA-Flow}
    &\dot{\tx}(t)=-\nabla_{x} f(\tx, \ty), \\
    &\dot{\ty}(t)=\nabla_{y}f(\tx, \ty)\tag{Continuous GDA}
\end{align*}
\ref{GDA-Flow} is the ODE for GDA \citep{wang2024local}. The Jacobian of \ref{GDA-Flow} at a local Nash equilibrium $(\tx^*,\ty^*)$ is defined by
\begin{align}\label{jacgda}
    \CJ=\begin{bmatrix}
    -\nabla_{x}^2f(\tx^*, \ty^*)& -\nabla_{xy}f(\tx^*, \ty^*)\\
    \\
    \nabla_{yx}f(\tx^*, \ty^*)& \nabla_{y}^2f(\tx^*, \ty^*)
\end{bmatrix} \tag{Jacobian}
\end{align}
Moreover, by the definition of local Nash equilibrium, we have
$\nabla^2_{x}f(\tx^*,\ty^*) \succcurlyeq \textbf{0},\  \nabla^2_{y}f(\tx^*,\ty^*) \preccurlyeq \textbf{0}.$

It may seem that \ref{GDA-Flow} is an over-simplified version of \ref{CADAM}. Surprisingly, their Jacobian are closely related, as shown in the following proposition:
\begin{prop}\label{jacobianadam}
Let $\CJ_{\mathrm{Adam}}$ be the Jacobian of \ref{CADAM} at the local equilibrium $(\tx^*,\ty^*)$, then $\CJ_{\text{Adam}}$ is given by a quadratic polynomial of $\CJ$:
\begin{align*}
    \CJ_{\mathrm{Adam}} = \frac{1}{\sqrt{\epsilon}} \left( \CI - \frac{h(1+\beta)}{2\sqrt{\epsilon}(1-\beta)}\CJ \right) \CJ.
\end{align*}
\end{prop}
This simple yet powerful relationship allows us to provide a detailed analysis of the local behavior of 
\ref{CADAM} in the following section.

Finally, we recall the work of \cite{JMLR:v20:19-008} introduced a decomposition of $\CJ$ as a summation of its symmetry part and anti-symmetry part $\CJ = \CS + \CA$, where 
\begin{align*}
\CS = \begin{bmatrix}
 -\nabla^2_{x}f(\tx^*,\ty^*) & \textbf{0}\\
 \\
 \textbf{0}^{\top} & \nabla^2_{y}f(\tx^*,\ty^*)
\end{bmatrix}
\end{align*}
and
\begin{align*}
\CA = \begin{bmatrix}
\textbf{0} & -\nabla_{xy}f(\tx^*,\ty^*)  \\
\\
\nabla_{yx}f(\tx^*,\ty^*)  & \textbf{0}^{\top}
\end{bmatrix}.
\end{align*}
In this decomposition, the symmetric part $\CS$ describes the part of the dynamics where players independently optimize their own loss functions. The anti-symmetric part $\CA$ describes the part where players interact adversarially. In typical zero-sum games, the magnitude of $\CA$ should dominate $\CS$; otherwise, the game reduces to a minimization problem where both players minimize their losses independently.

\subsection{Local Convergence of Adam in Zero-Sum Games}\label{LC of continuous model}
\begin{figure*}[t]
    \centering
    \subfigure[Continuous Adam-DA]{
        \includegraphics[width=2in]{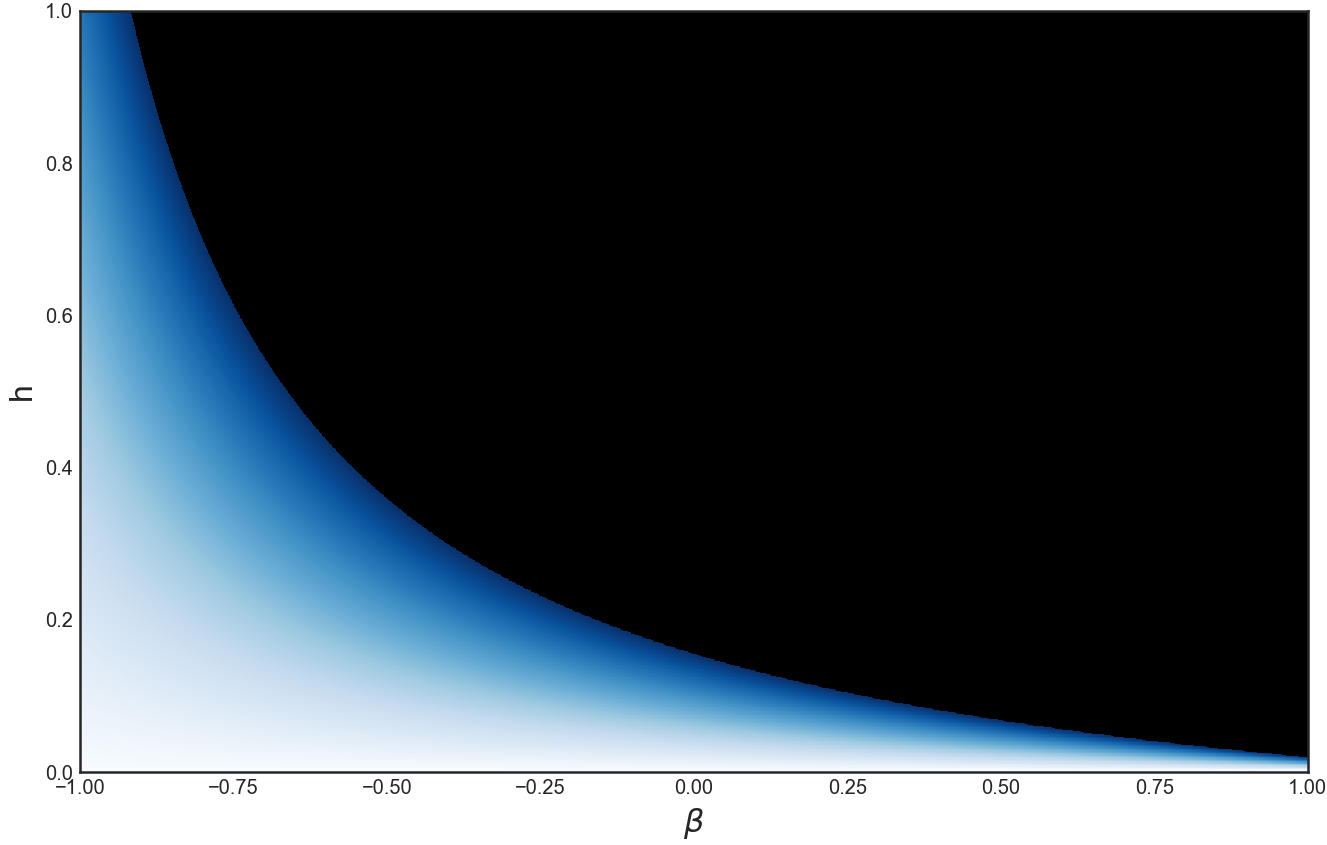}
        \label{lp1}
    }
    \subfigure[Adam-DA]{
        \includegraphics[width=2in]{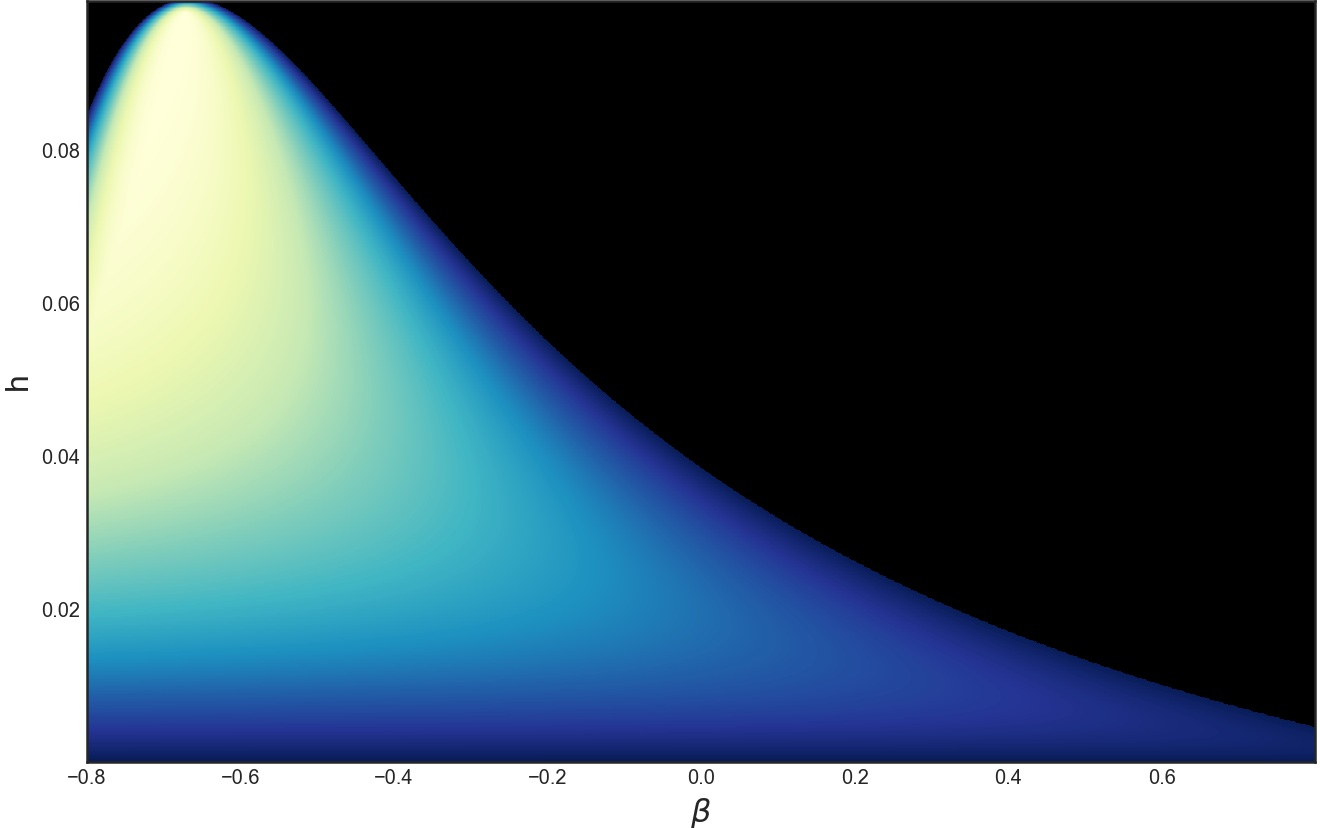}
        \label{lp2}
    }
    \subfigure[Adam in minimization]{
        \includegraphics[width=2in]{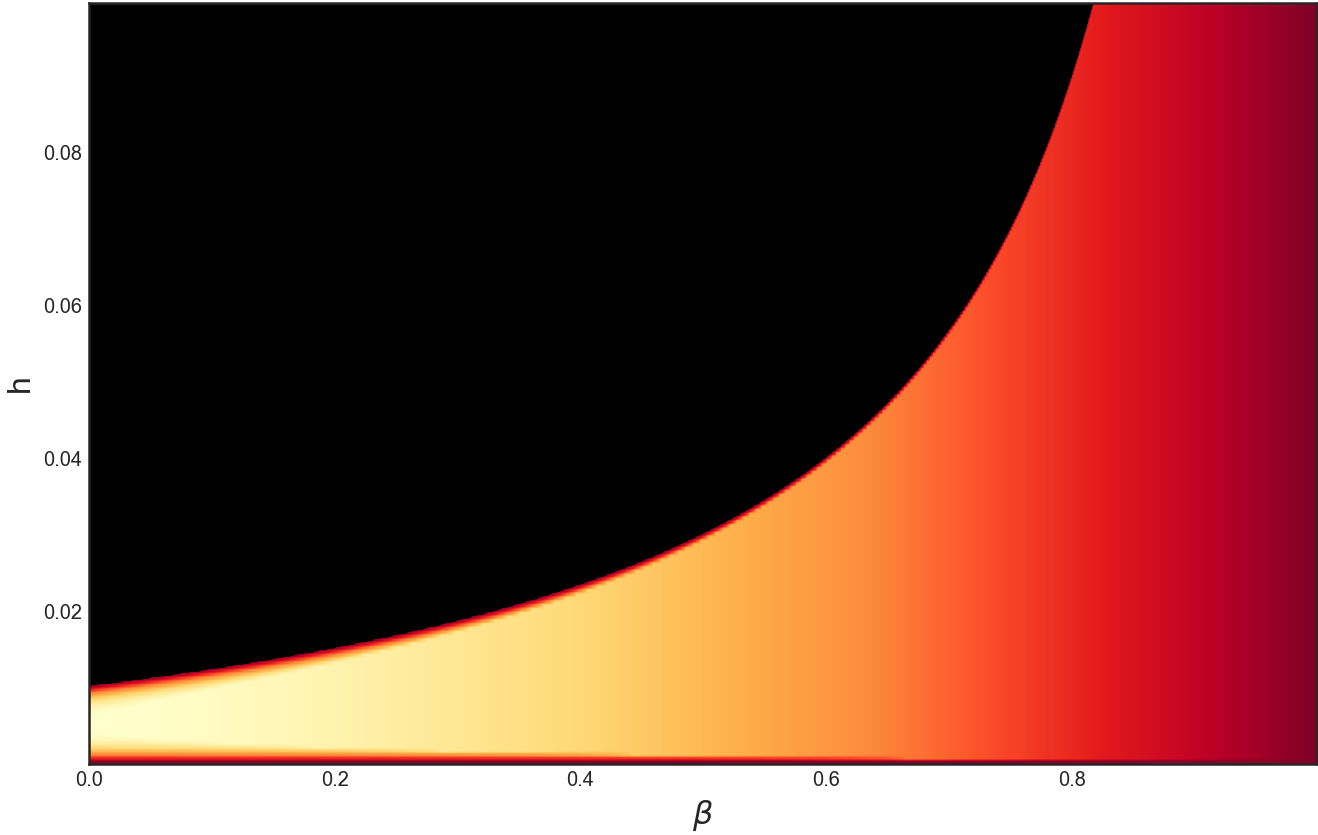}
        \label{lp3}
    }
 %    \subfigure[Continuous Adam-DA]{
 %        \includegraphics[width=1.7in]{Pictures/epsilon_h_cont2.png}
 %        \label{lp4}
 %    }
 %    \subfigure[Adam-DA]{
	% \includegraphics[width=1.7in]{Pictures/epsilon_h_disc2.png}
 %        \label{lp5}
 %    }
 %    \subfigure[Adam in minimization]{
	% \includegraphics[width=1.7in]{Pictures/epsilon_h_mini2.png}
 %        \label{lp6}
 %    }
    \caption{\textit{Numerical experiments on quadratic test functions for the local convergence of \ref{CADAM}, \ref{Adam} in zero-sum games, and Adam in minimization. The \textbf{black} regions indicate divergence, while lighter regions indicate faster convergence. The X-axis is $\beta$ and Y-axis is $h$. \textbf{Smaller} $\beta$ values allow \ref{CADAM} and \ref{Adam} to converge over a wider range of step sizes, while \textbf{larger} $\beta$ allow larger $h$ in minimization. These results support Corollary \ref{lmss} and highlight fundamental differences in the role of momentum $\beta$ between min-max and minimization problems. The non-decreasing behavior of $\beta$ for extremely small $\beta$ in \ref{lp2} aligns with Corollary \ref{lmss}.}}
    \label{lcpicture}
\end{figure*}

% In \ref{lp4}, \ref{lp5}, and \ref{lp6}, the X-axis is $\epsilon$ and Y-axis is $h$. Larger $\epsilon$ values expand the convergence range for all three methods, supporting Theorem \ref{Corollary: Continuous Local convergence for general case} and \ref{Thm: Discrete Local convergence for general case}. 

In this section, we present our findings regarding the local convergence. We offer results for both \ref{CADAM} and \ref{Adam}, demonstrating that they display similar behavior. First, we present the following assumption in the literature that characterizes the class of zero-sum games that we are interested in within this work.

\begin{assu}\label{ibr}\citep{wang2024local} There exist at least one eigenvalue $\lambda \in \mathrm{Sp}(\CJ)$ such that $\lvert\Im(\lambda)\lvert > \lvert\Re(\lambda)\lvert$, we denote the set of all such eigenvalues as $\widetilde{\mathrm{Sp}(\CJ)}$. We also assume $\mathrm{EigVec} \left( \CA \right)  \cap \mathrm{Ker} \left( \CS \right) = \{\textbf{0}\}.$
\end{assu}

Note that $\mathrm{EigVec} \left( \mathcal{A} \right) \cap \mathrm{Ker} \left( \mathcal{S} \right) = \{\mathbf{0}\}$ holds \textbf{generically} in the following sense: for any fixed $\mathcal{S} \ne \mathbf{0}$, if $\mathcal{A}$ is sampled independently from an absolutely continuous distribution, then this condition holds with probability one. Thus it is a mild assumption. 

The eigenvalue part of Assumption~\ref{ibr} is satisfied for zero-sum games in which the $\CA$ in \eqref{jacgda} dominates $\CS$ in magnitude. This regime corresponds to a strong coupling between the two players’ local dynamics near the equilibrium. Note that these are also the most representative zero-sum games in the literature \cite{pmlr-v89-liang19b, wang2024local}, in contrast to games that can essentially be treated as minimization problems, e.g., $\min_{\tx} \max_{\ty} f(\tx) + g(\ty)$, where players do not interact. 

\begin{thm}\label{Corollary: Continuous Local convergence for general case}
 Under Assumption \ref{ibr}, let $\beta\in (-1, 1)$ and $\rho\in(0,1)$. Then \ref{CADAM} achieves local convergence to $(\tx^*,\ty^*)$ with an exponential rate
 iff the step-size $h$ satisfies:
\begin{align*}
    0< h < \min_{\lambda \in \widetilde{\mathrm{Sp}(\CJ)}}\frac{2\sqrt{\epsilon} (1-\beta)\lvert \Re (\lambda) \lvert}{(1+\beta)(\Im(\lambda)^2 - \Re(\lambda)^2)}
\end{align*}
\end{thm}
Under Assumption \ref{ibr}, $\Im(\lambda)^2 - \Re(\lambda)^2 > 0$, so the upper bound is positive. By Proposition \ref{Jacobiancr}, the local convergence of \ref{CADAM} is governed by the eigenvalues of $\CJ_{\mathrm{Adam}}$, and the convergence rate equals the largest real part of these eigenvalues. The polynomial relationship between $\CJ_{\mathrm{Adam}}$ and $\CJ$ in Proposition \ref{jacobianadam} is key to the proof. The full proof is in Appendix \ref{Appendix_S42}.

% \begin{remark}
% Our proof of Theorem~\ref{Corollary: Continuous Local convergence for general case} also yields a bound for games violating Assumption~\ref{ibr}. Yet in such cases, \ref{CADAM} behaves more like a minimization algorithm than a zero-sum game solver, which contrary to this work’s core motivation. We therefore choose not to discuss them.
% \end{remark}

%The proof of Theorem \ref{Corollary: Continuous Local convergence for general case} relies on the analysis of the locations of eigenvalues of $\CJ_{\text{Adam}}$ . During the analysis, the polynomial relation between $\CJ_{\text{Adam}}$ and $\CJ$ in Proposition \ref{jacobianadam} plays an important role. The detailed proof of this theorem is deferred to Appendix \ref{Appendix_S42}.

It is interesting to ask how well the above theorem predicts the behavior of the original discrete-time \ref{Adam}. We establish the following local convergence result for \ref{Adam}, and the detailed proof is provided in Appendix \ref{Appendix :proof of local convergence of Adam}:  

\begin{thm}\label{Thm: Discrete Local convergence for general case}
Under Assumption \ref{ibr}, let $\beta\in (-1, 1)$ and $\rho\in(0,1)$.  Then \ref{Adam} achieves local convergence to the local equilibrium $(\tx^*,\ty^*)$ with an exponential rate iff the step-size $h$ satisfies: 
\begin{align}\label{h22}
    0<h<\min_{\lambda \in \mathrm{Sp}(\CJ)}\frac{2\sqrt{\epsilon}(1-\beta^2)|\Re(\lambda)|}{(1+\beta^2)|\lambda|^2+2\beta\left(\Im(\lambda)^2-\Re(\lambda)^2\right)}
\end{align}
\end{thm}

By comparing Theorem \ref{Corollary: Continuous Local convergence for general case} and Theorem \ref{Thm: Discrete Local convergence for general case}, we observe that they share several common properties. One particularly interesting property is the following corollary, which shows that in both theorems there exists a trade-off between the parameters $h$ and $\beta$.

\begin{cor}\label{lmss}
The upper bound in Theorem \ref{Corollary: Continuous Local convergence for general case} is a decreasing function of $\beta \in (-1,1)$. Similarly, the upper bound in Theorem \ref{Thm: Discrete Local convergence for general case} is also a decreasing function of $\beta$ if the minimizer of \eqref{h22} is achieved in $\widetilde{\mathrm{Sp}(\CJ)}$ and $\beta$ satisfies
% \begin{align}\label{localdescreasingbeta}
%     \max_{\lambda \in \mathrm{Sp}(\CJ)}\frac{|\Re(\lambda)|-|\Im(\lambda)|}{|\Re(\lambda)|+|\Im(\lambda)|}<\beta<1
% \end{align}
\begin{align}\label{localdescreasingbeta}
\max_{\lambda \in \widetilde{\mathrm{Sp}(\CJ)}}\frac{\left(|\Re(\lambda)|-|\Im(\lambda)|\right)}{ \left(|\Re(\lambda)|+|\Im(\lambda)|\right)}<\beta<1,
\end{align}
Therefore, within these regions, using a \textbf{smaller} $\beta$ allows \ref{CADAM} and \ref{Adam} to achieve local convergence across a broader range of step sizes h.
\end{cor}

The interval in (\ref{localdescreasingbeta}) can be very close to $(-1,1)$ when $|\Im(\lambda)|$ is much larger than $|\Re(\lambda)|$, as is the case in typical games such as bilinear games and GANs. Corollary \ref{lmss} reflects a fundamental difference in the role of $\beta$ between zero-sum games and minimization. In minimization problems, it is known that a larger $\beta$ can enable Adam to achieve local convergence with a wider range of step sizes \citep{o2015adaptive,bock2019proof}. However, as shown by Corollary \ref{lmss}, to achieve a similar effect, we need to use a smaller $\beta$ in zero-sum games. We illustrate this point in Figure \ref{lcpicture}. More experiments are presented in Appendix \ref{Appendix_S45}.

Finally, Theorem \ref{Corollary: Continuous Local convergence for general case} and \ref{Thm: Discrete Local convergence for general case} imply Adam-DA always diverge in the important class of bilinear games:

\begin{cor}\label{Corollary: Continuous Local convergence for bilinear case} For bilinear games $f(\tx, \ty)=\tx^{\top}A\ty$,  \ref{CADAM} and \ref{Adam} always diverge regardless of the choice of parameters $\beta$, $\rho$, $\epsilon$ and $h$.
\end{cor}
%Corollary \ref{Corollary: Continuous Local convergence for bilinear case} can be seen from the fact that for bilinear games, $\lvert \Re (\lambda) \lvert =0$ for all $\lambda \in \CJ$, thus the upper bounds in Theorem \ref{Corollary: Continuous Local convergence for general case} and Theorem \ref{Thm: Discrete Local convergence for general case} are always zero. 

Corollary \ref{Corollary: Continuous Local convergence for bilinear case} extends the previous results on the divergence of GDA and MWU in the bilinear zero-sum games \citep{bailey2018multiplicative,bailey2020finite} to the Adam algorithm. This corollary reflects another difference between Adam-DA in zero-sum games and Adam in minimization. In minimization, Adam can always converge to a neighborhood of a local minimum with carefully chosen parameters \citep{zhang2022adam}. However, Corollary \ref{Corollary: Continuous Local convergence for bilinear case} shows that Adam-DA always diverges in the fundamental class of bilinear games.

\section{Implicit Gradient Regularization}\label{IGR}
\begin{figure*}[t]
    \centering
    \subfigure[ResNet with different $\beta$,\ \  Data Set: CIFAR-10]{
        \includegraphics[width=1.5in]{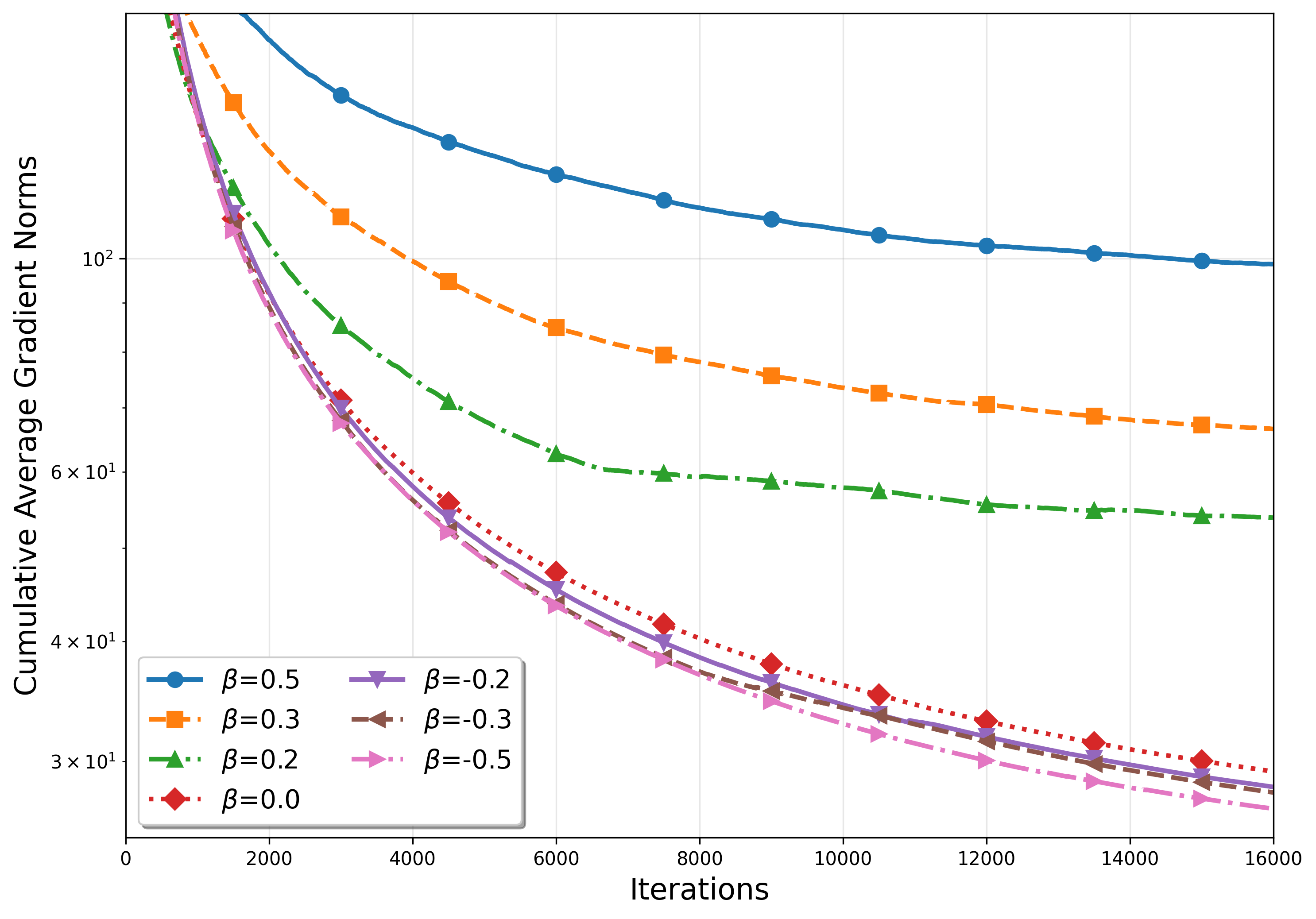}
        \label{p21}
    }
    \subfigure[ResNet with different $\rho$,\ \   Data Set: CIFAR-10]{
        \includegraphics[width=1.5in]{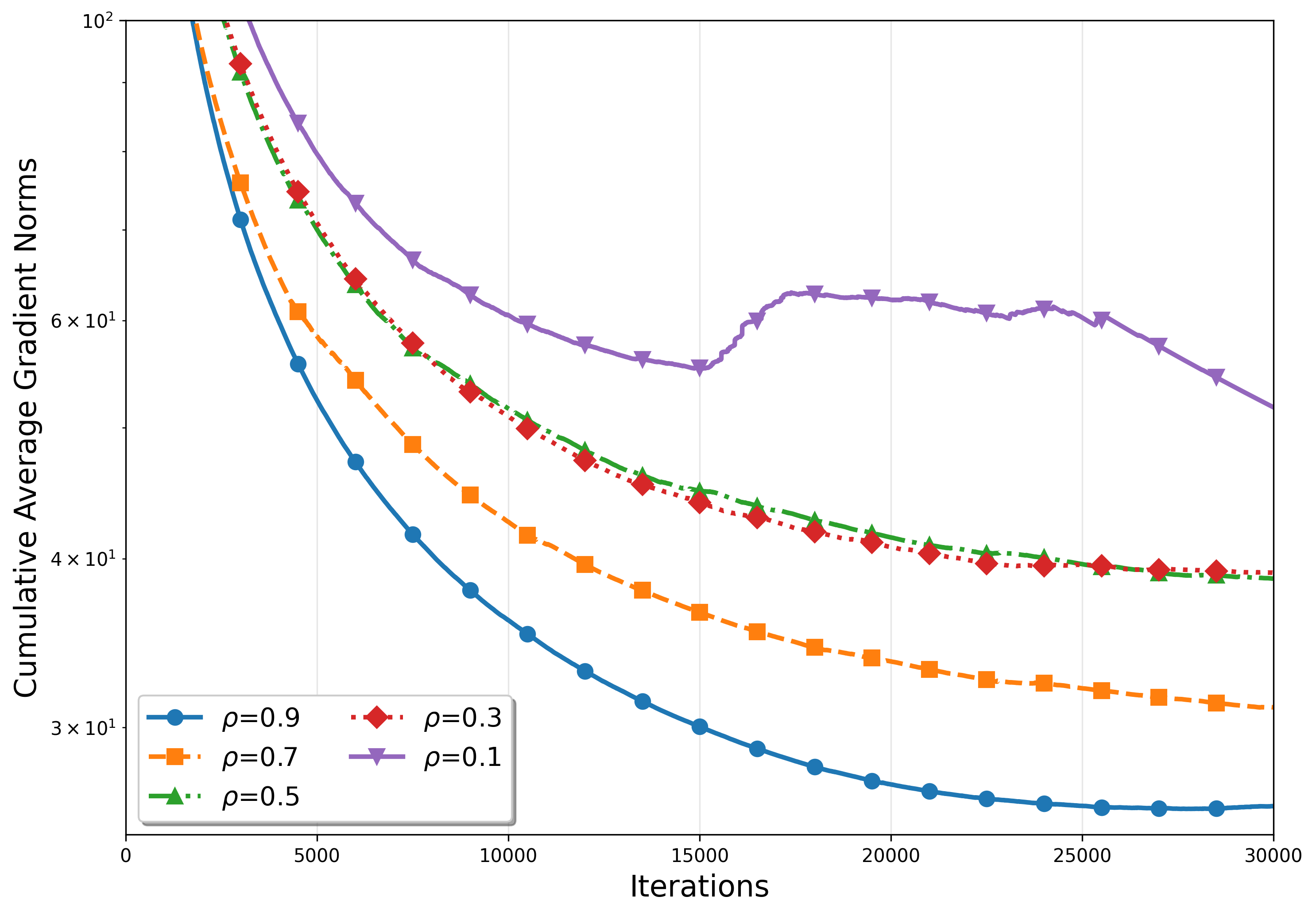}
        \label{p22}
    }
    \subfigure[ResNet with different $\beta$,\ \   Data Set: STL-10]{
        \includegraphics[width=1.5in]{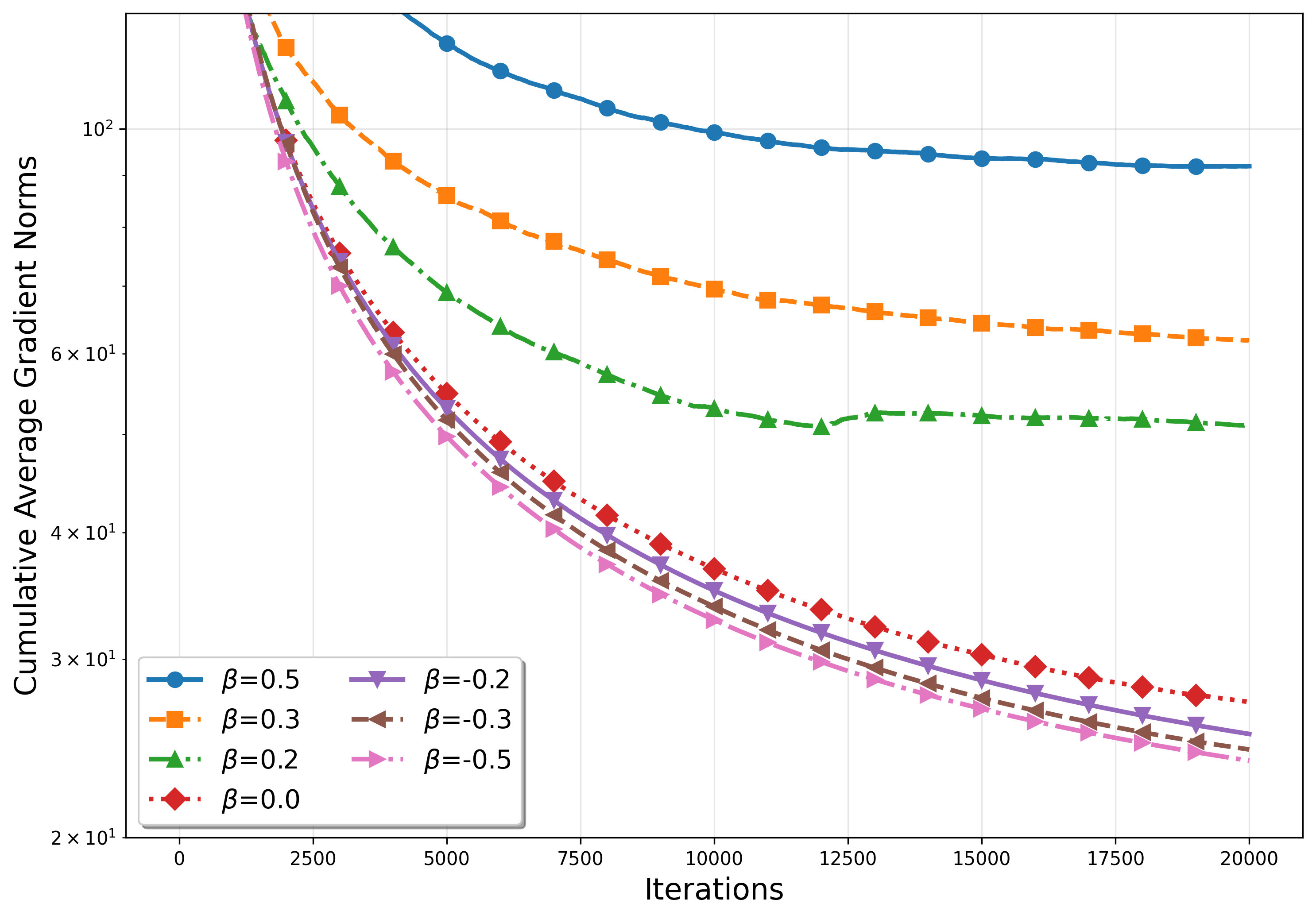}
        \label{p23}
    }
    \subfigure[ResNet with different $\rho$,\ \   Data Set: STL-10]{
	\includegraphics[width=1.5in]{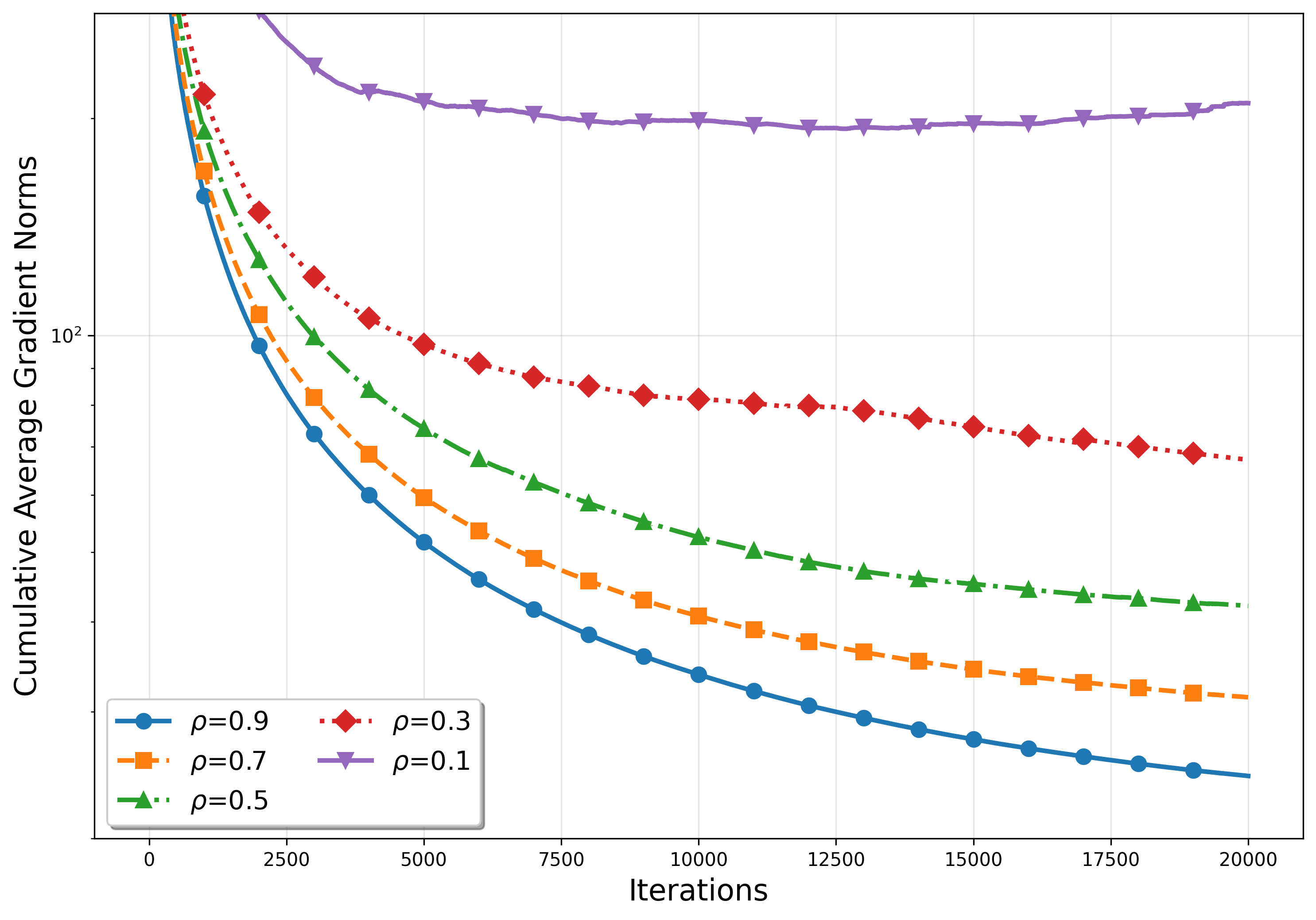}
        \label{p24}
    }
    \subfigure[CNN with different $\beta$,\ \ \ \ \ \ Data Set: CIFAR-10]{
	\includegraphics[width=1.5in]{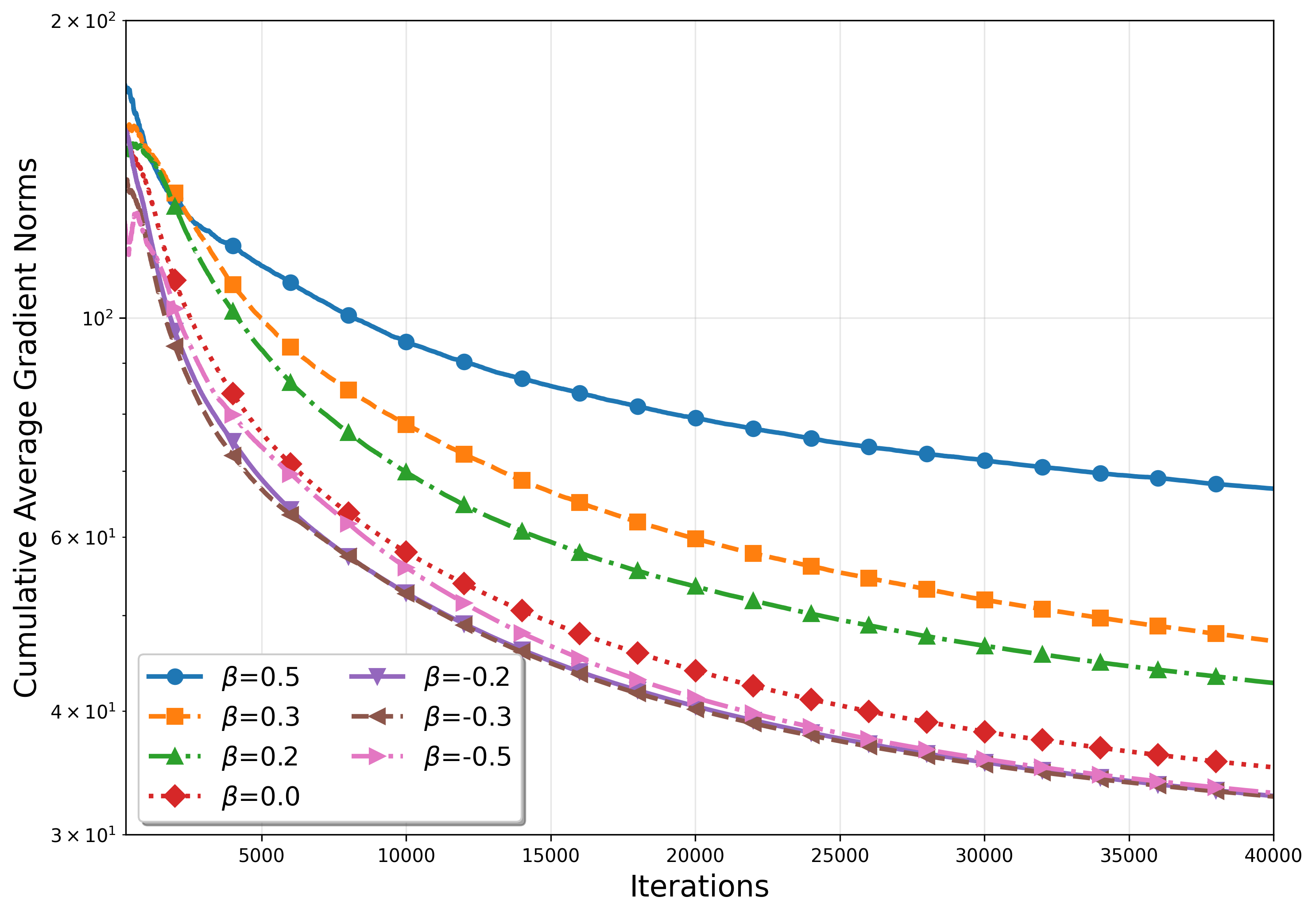}
        \label{p29}
    }
    \subfigure[CNN with different $\rho$,\ \ \ \ \ \ \ Data Set: CIFAR-10]{
	\includegraphics[width=1.5in]{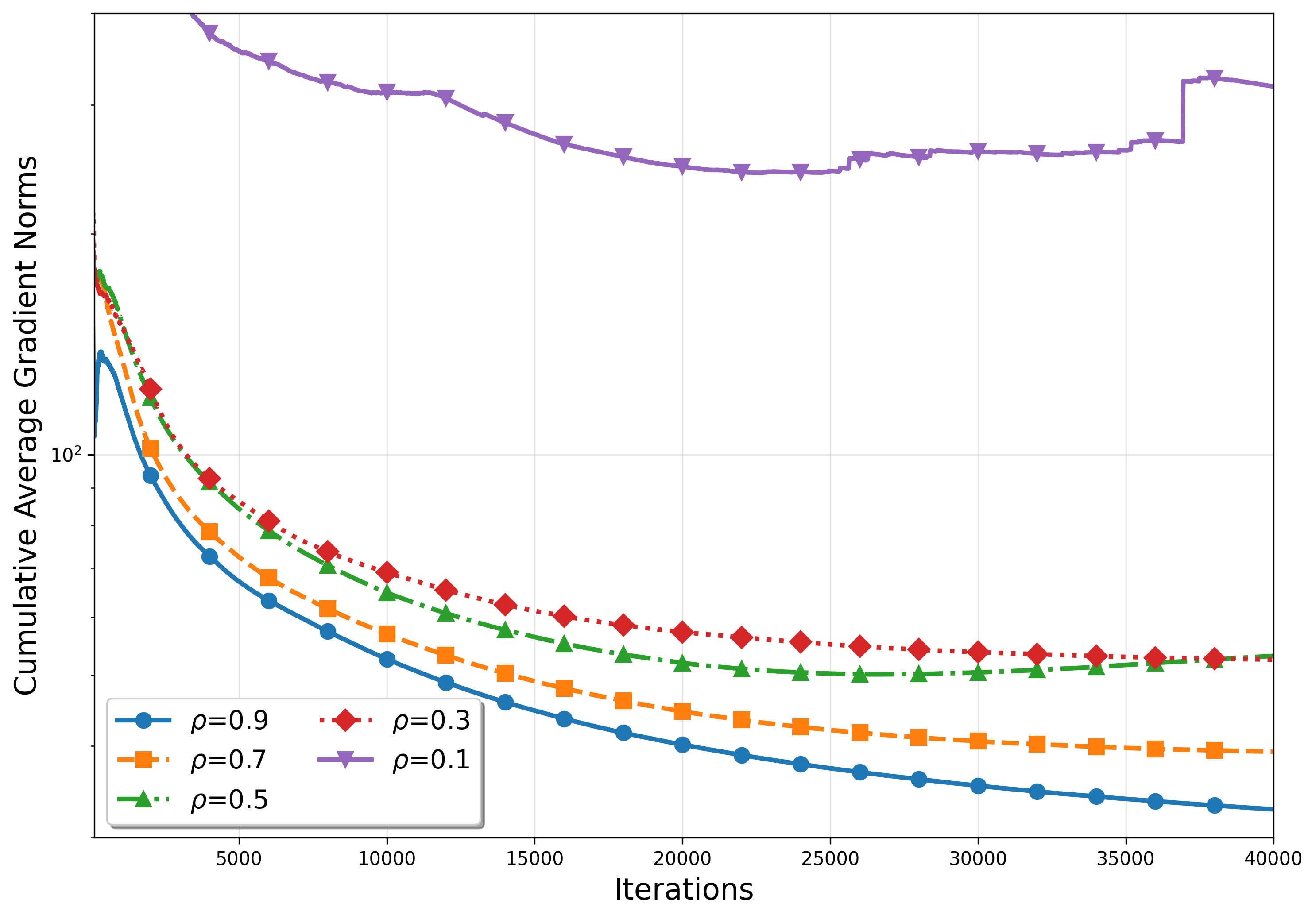}
        \label{p30}
    }
    \subfigure[CNN with different $\beta$,\ \ \ \ \ \ Data Set: STL-10]{
	\includegraphics[width=1.5in]{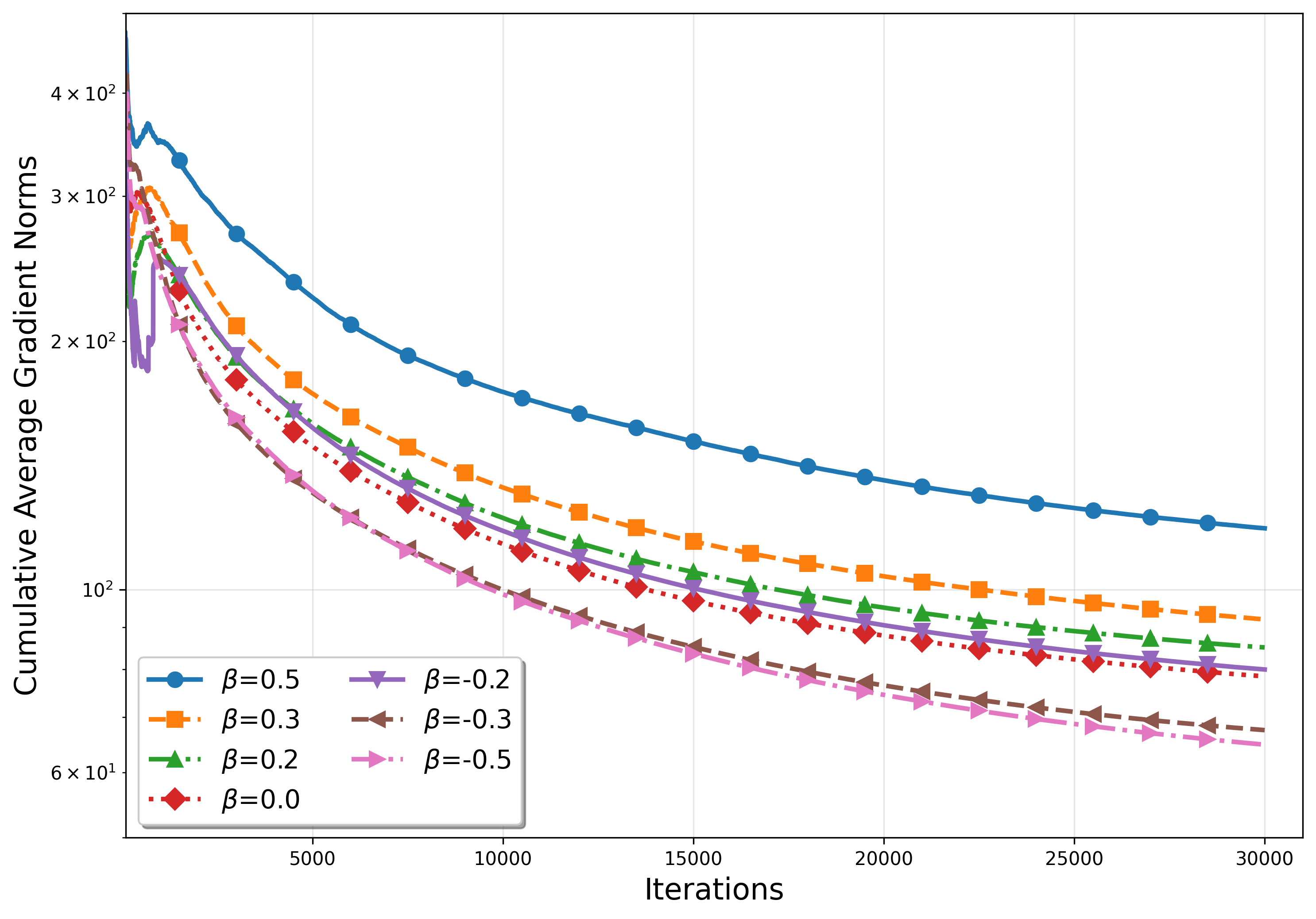}
        \label{p31}
    }
    \subfigure[CNN with different $\rho$,\ \ \ \ \ \ Data Set: STL-10]{
	\includegraphics[width=1.5in]{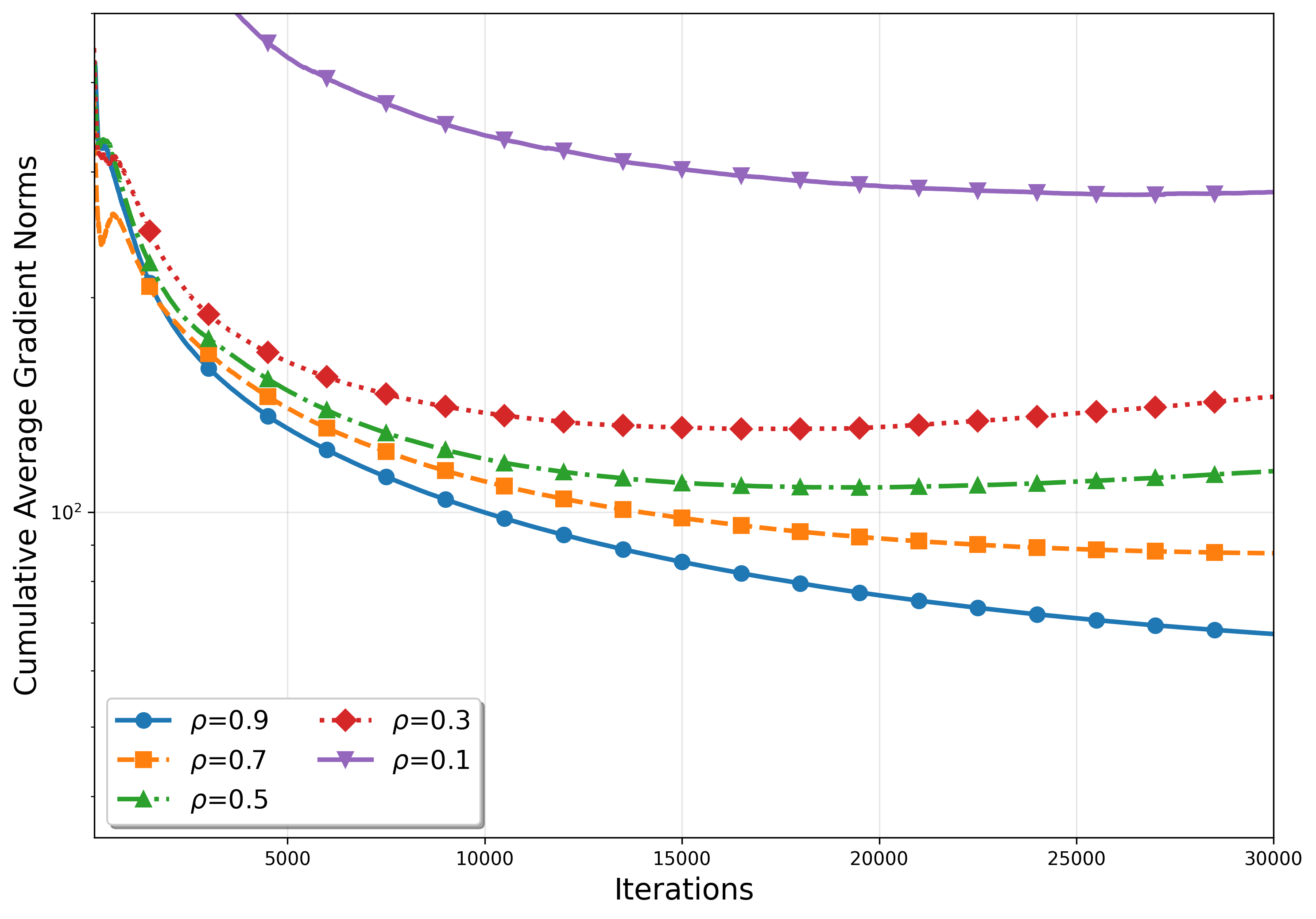}
        \label{p32}
    }
    \caption{\textit{The $\ell_1$ norm of gradients of \ref{Adam} with varying $\beta$ and $\rho$ during GANs training. Datasets: $\textit{CIFAR-10}$ and $\textit{STL-10}$. Architectures: ResNet and CNN. As shown in \ref{p21}, \ref{p23}, \ref{p29}, and \ref{p31}, smaller $\beta$ values result in smaller gradient norms. According to \ref{p22}, \ref{p24}, \ref{p30}, and \ref{p32}, larger $\rho$ values also lead to smaller gradient norms. Both findings support the thesis.}}
    \label{flatnesspic}
\end{figure*}

\begin{figure*}[h]
    \centering
    \subfigure[ResNet with different $\beta$,\ \ Data Set: CIFAR-10]{
	\includegraphics[width=1.5in]{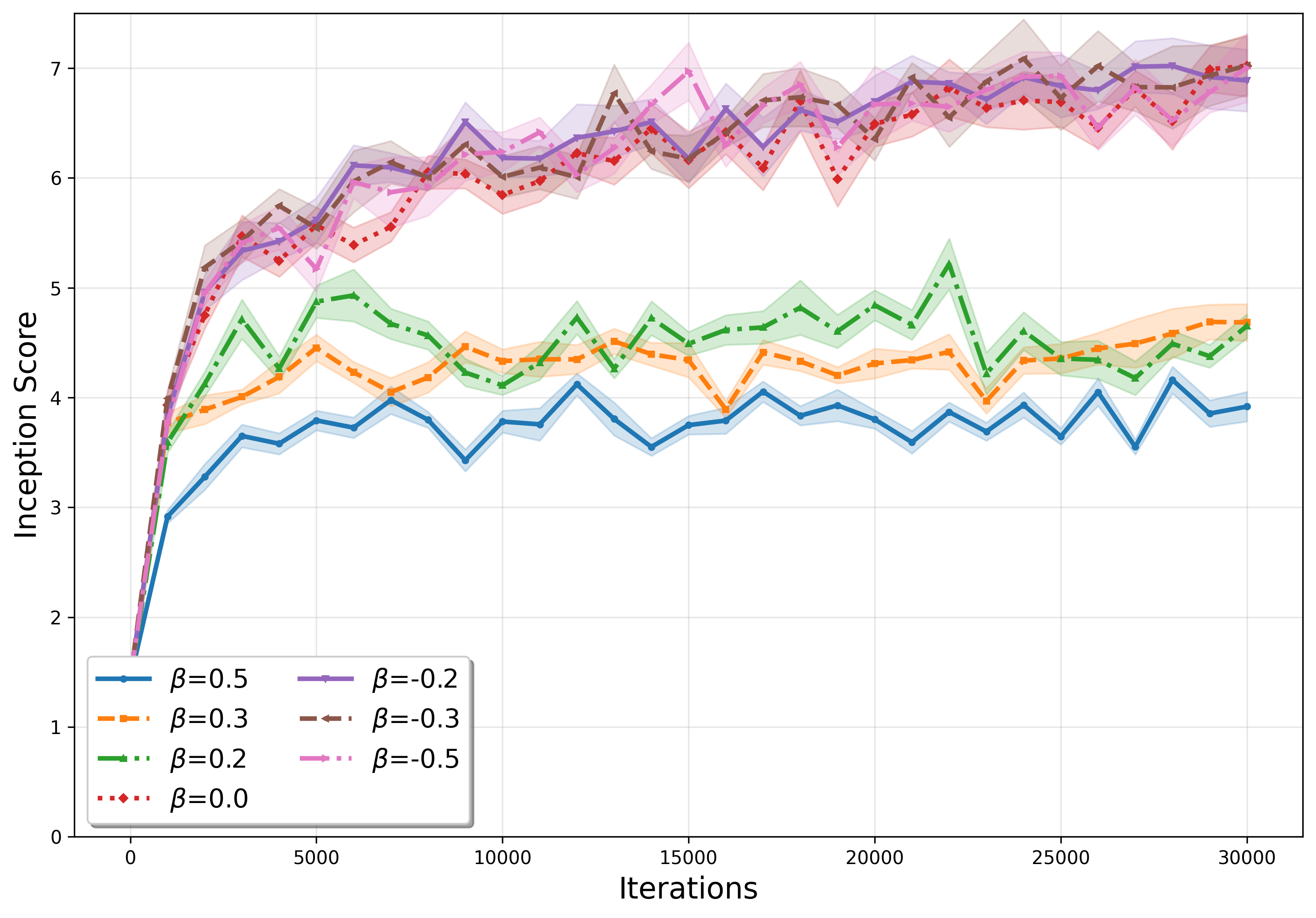}
        \label{p25}
    }
    \subfigure[ResNet with different $\rho$,\ \ Data Set: CIFAR-10]{
	\includegraphics[width=1.5in]{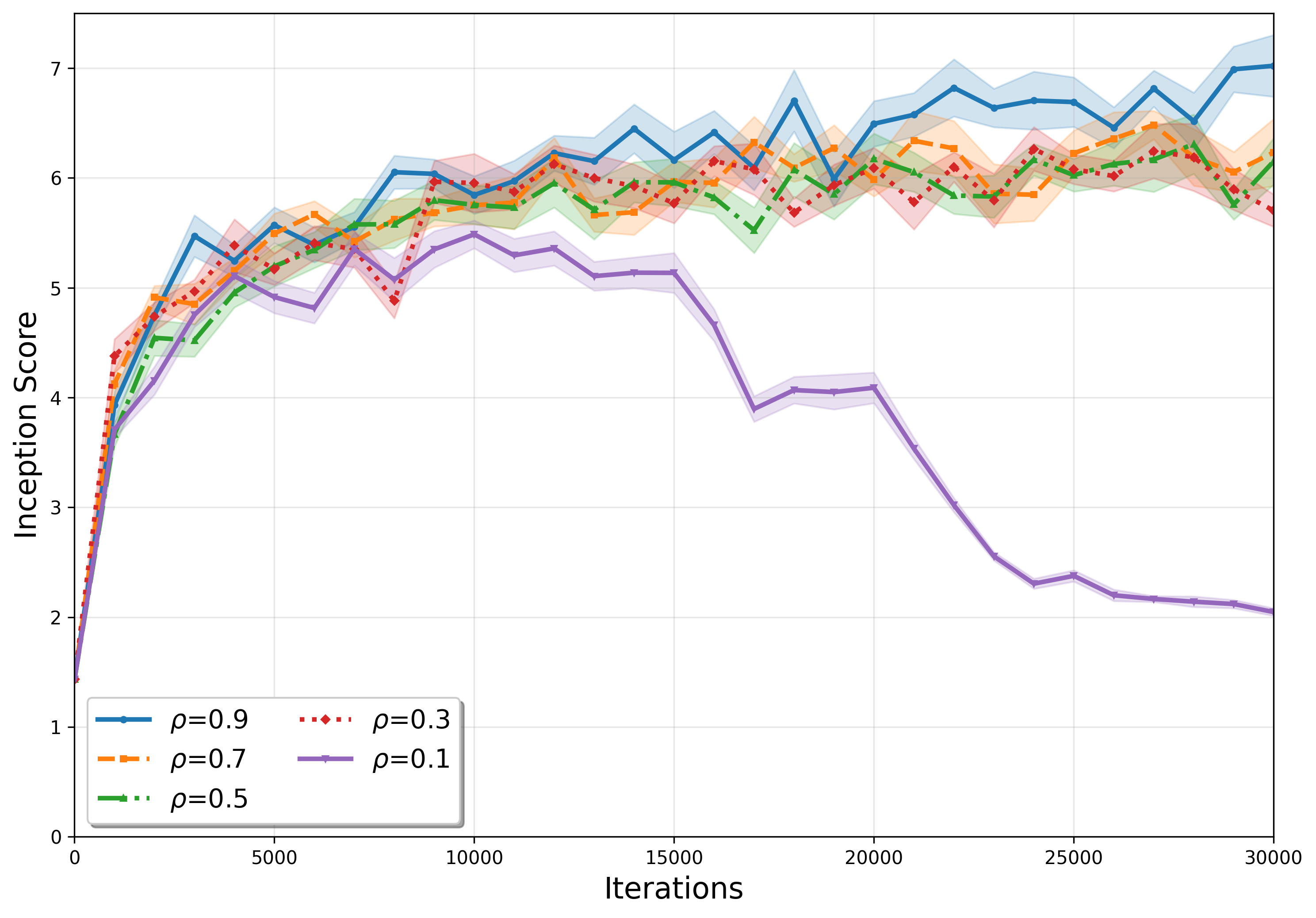}
        \label{p26}
    }
    \subfigure[ResNet with different $\beta$,\ \ Data Set: STL-10]{
	\includegraphics[width=1.5in]{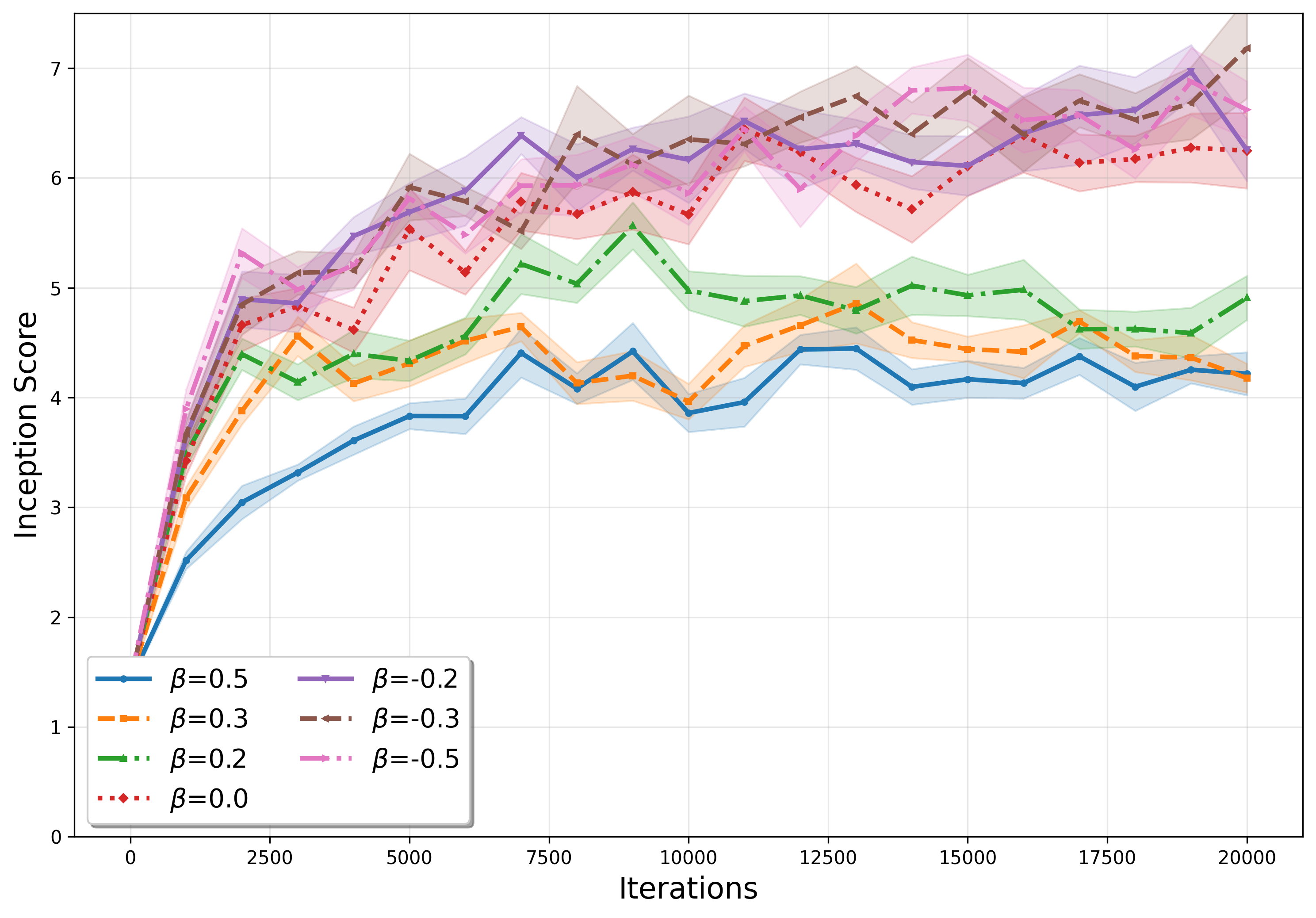}
        \label{p27}
    }
    \subfigure[ResNet with different $\rho$,\ \ Data Set: STL-10]{
	\includegraphics[width=1.5in]{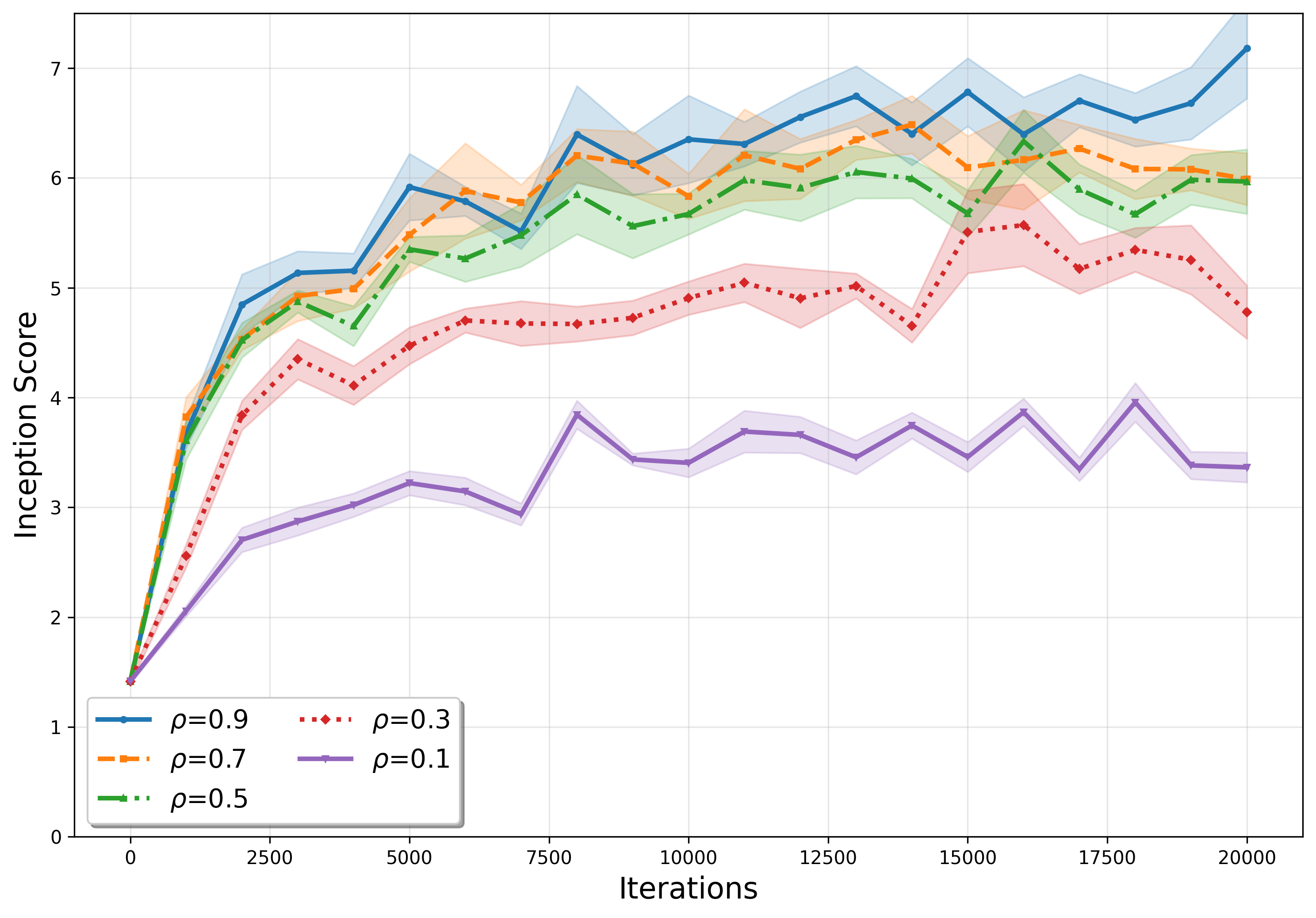}
        \label{p28}
    }
    \subfigure[CNN with different $\beta$,\ \ \ \ \ \ Data Set: CIFAR-10]{
	\includegraphics[width=1.5in]{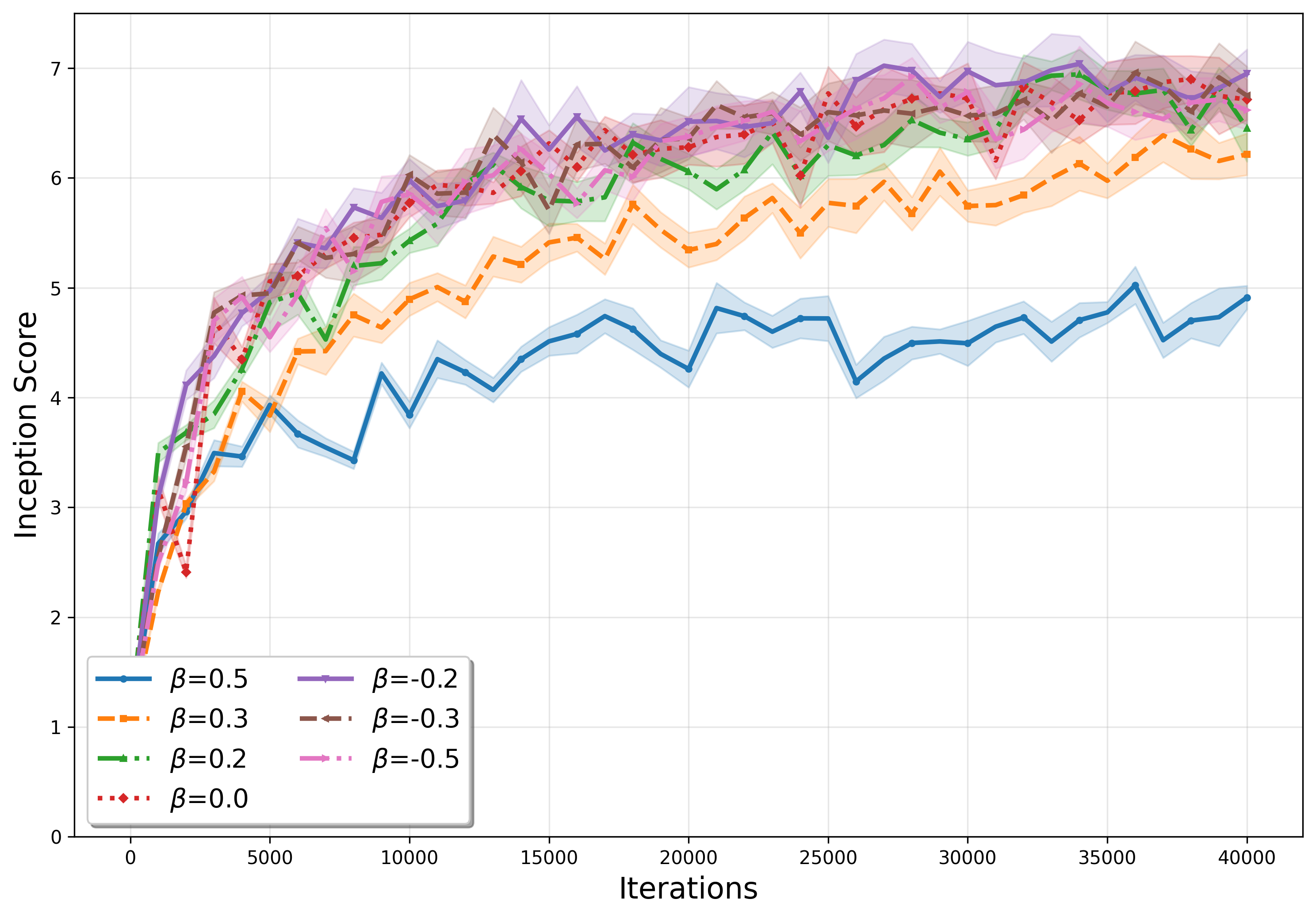}
        \label{p33}
    }
    \subfigure[CNN with different $\rho$,\ \ \ \ \ \ \ Data Set: CIFAR-10]{
	\includegraphics[width=1.5in]{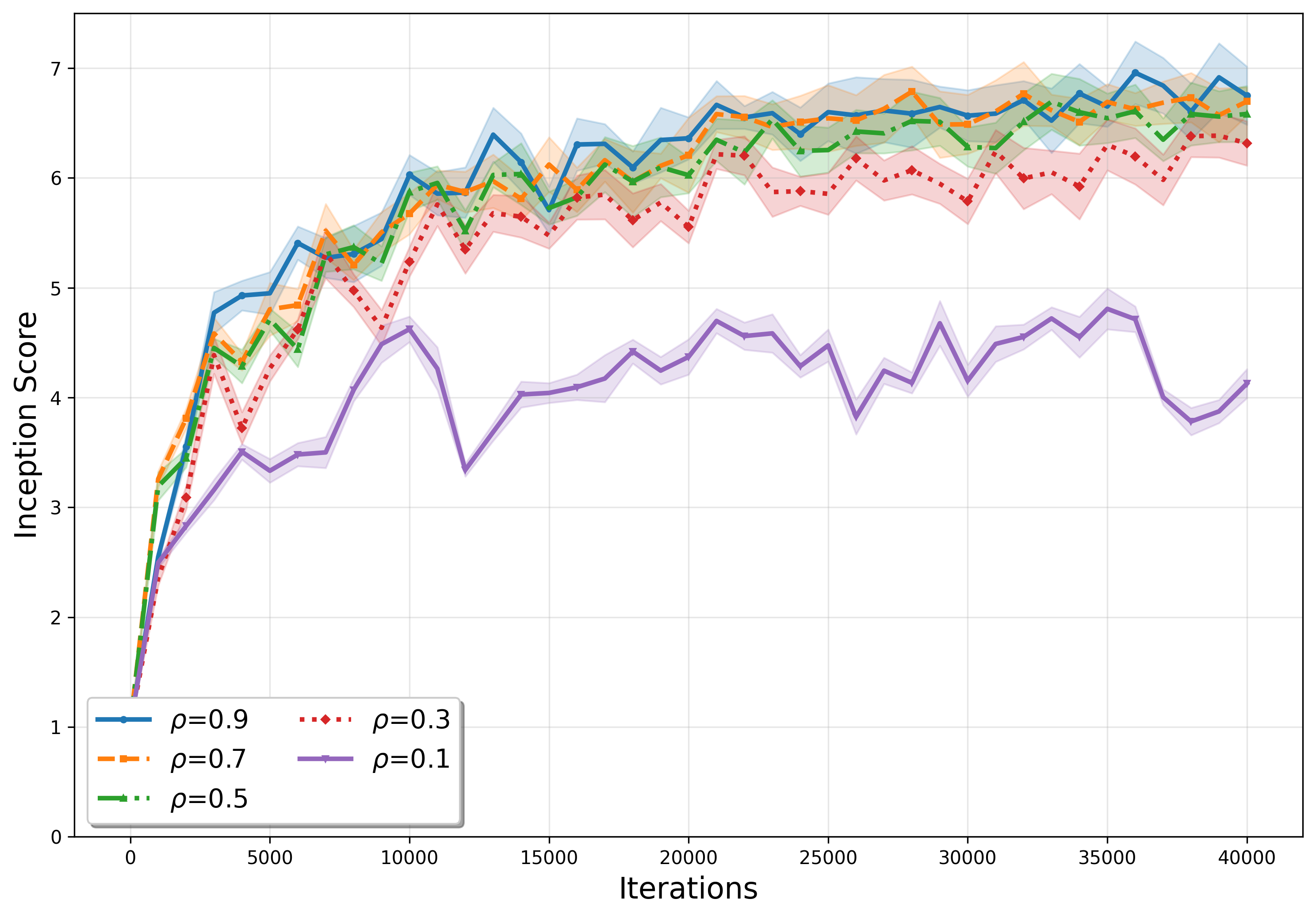}
        \label{p34}
    }
    \subfigure[CNN with different $\beta$,\ \ \ \ \ \  Data Set: STL-10]{
	\includegraphics[width=1.5in]{Pictures/IS_curve_STL.png}
        \label{35}
    }
    \subfigure[CNN with different $\rho$,\ \ \ \ \ \  Data Set: STL-10]{
	\includegraphics[width=1.5in]{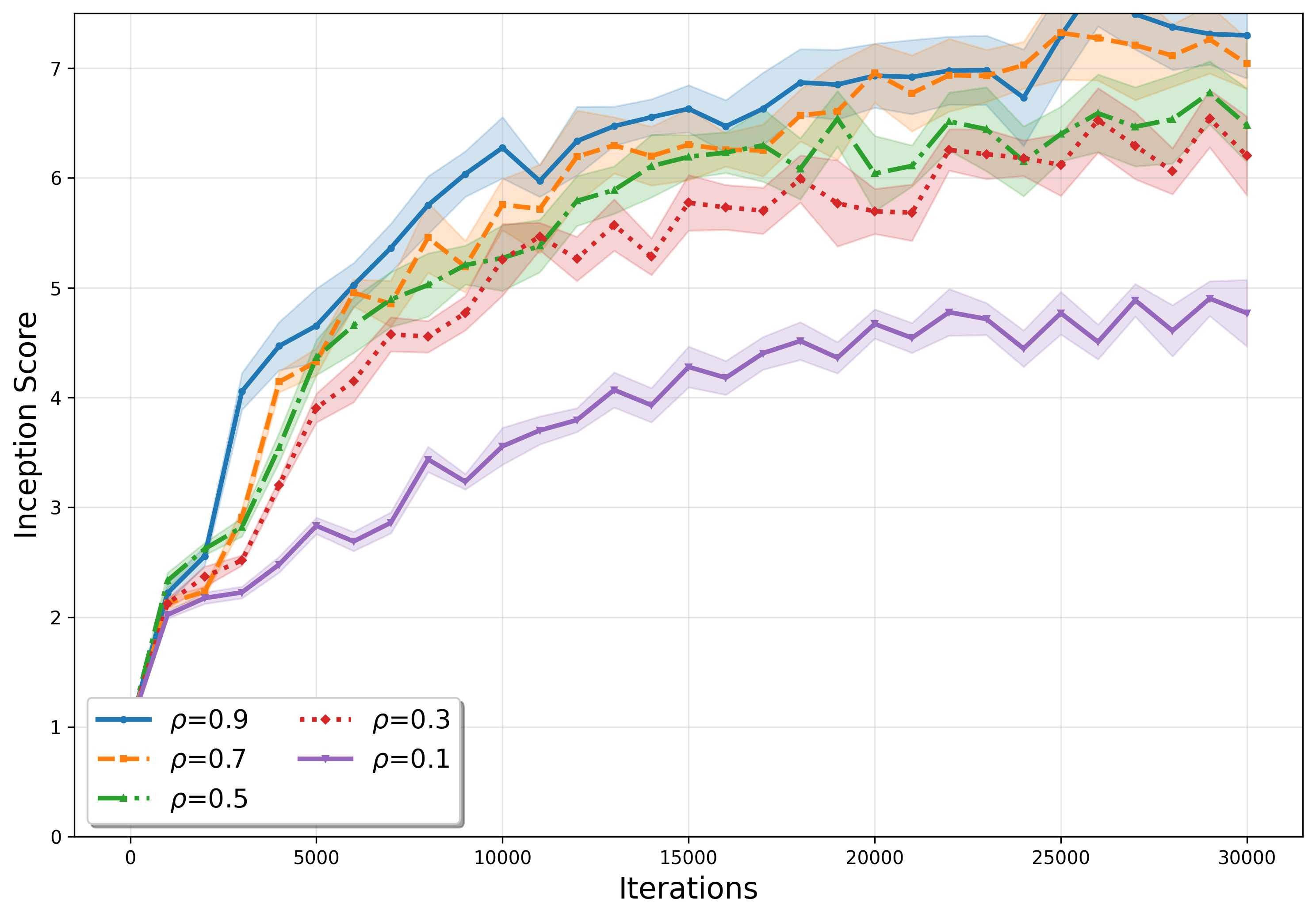}
        \label{p36}
    }
    \caption{\textit{Inception Score for the corresponding experimental settings in Figure \ref{flatnesspic}. A high score implies better training performance. Training algorithms with lower $\ell_1$ gradient norms also have better performance.}}
    \label{Inception_score}
\end{figure*}

Implicit gradient regularization (IGR) refers to the terms involving the gradients of loss functions that emerge from ODEs of the discrete-time algorithms. One example is that in minimization,  \citet{barrett2020implicit} derived the following $\CO(h^3)$-local error ODEs for the Gradient Descent algorithm (GD) $\tx_{t+1} = \tx_t - h \nabla_x f(\tx_t)$:
\begin{align}\label{continuous-GD}
    \dot{\tx}(t) = -\nabla_x f(\tx) - (h/4) \nabla_x\lVert \nabla_xf(\tx)\rVert^2_2,
\end{align}
The IGR term in (\ref{continuous-GD}) is $- (h/4)\nabla_x\lVert \nabla_xf(\tx)\rVert^2_2$, which reveals that besides minimizing function value $f(\tx)$, GD also implicitly minimizing $\lVert \nabla_xf(\tx)\rVert^2_2$, thereby guiding the optimization trajectories towards flatter region of the loss landscapes. Here we use the terminology \textit{flatter} regions to refer to regions in the loss landscape with smaller gradient norm, following \cite{barrett2020implicit}.  Recently, \cite{pmlr-v235-cattaneo24a} generalize this approach to the Adam algorithm in minimization. They found in minimization problems, the trajectories of Adam explore flatter regions of the loss landscape with a \textbf{larger}  $\beta$ and a \textbf{smaller}   $\rho$. Motivated by these findings, we ask the question:

\textit{How the parameters $\beta$ and $\rho$ affect the interaction between the trajectories of \ref{Adam} and the flatter regions of loss landscapes in zero-sum games?}

%Building on previous work in implicit gradient regularization \citep{barrett2020implicit,feng2025continuoustime}, we present our findings as a \textit{Thesis}. This emphasizes that our results stem from ODEss and are validated by experiments, not formal proofs. 

As in Section~\ref{LC}, here we restrict our attention to typical zero-sum games that are distinct from pure minimization problems. We focus on adversarial interaction-dominated games where $\nabla_{xy} f(\mathbf{x}, \mathbf{y})$—which characterizes the adversarial interaction between players' dynamics in \eqref{jacgda}—dominates the terms $\nabla_x^2 f(\mathbf{x}, \mathbf{y})$ and $\nabla_y^2 f(\mathbf{x}, \mathbf{y})$ in magnitude along the training trajectory. This property can be assessed by examining the Jacobian spectrum along the algorithm's iterates: in interaction-dominated games, the eigenvalues typically exhibit large imaginary parts, reflecting the antisymmetric structure induced by the coupling between the two players \cite{pmlr-v89-liang19b}. In the literature, we found that the GAN training problem most closely satisfies this interaction-dominated condition, where eigenvalues of the Jacobian with large imaginary parts are identified as a main cause of GAN training instability  \citep{nagarajan2017gradient}.

Our goal in this section is to argue the following thesis, which highlights the distinct behavior of IGR with Adam-DA in zero-sum games compared to Adam in minimization \citep{pmlr-v235-cattaneo24a}:

\underline{\textbf{\textit{Thesis.}}} In zero-sum games where players' interactions dominate the dynamics, \textit{smaller} $\beta$ and \textit{larger} $\rho$ guides trajectories of \ref{Adam} towards flatter regions of loss landscapes in terms of $\ell_1$ norm, i.e., regions with lower values of $\lVert \nabla_{x} f \rVert_1 + \lVert \nabla_{y} f \rVert_1$.

We present observations from \ref{CADAM} that motivate our thesis in Section \ref{ofc}, and verify it experimentally in Section \ref{expigr}.

\subsection{Observations from ODEs}\label{ofc}

% We introduce the following modified payoff function:
% \begin{align}\label{mpf}
%     \CF(\tx,\ty) = &f(\tx,\ty)
%     +  \frac{h}{2} \CK(\beta,\rho) \left( \lVert \nabla_x f(\tx,\ty) \lVert_{1,\epsilon} -  \lVert \nabla_y f(\tx,\ty) \lVert_{1,\epsilon} \right),
% \end{align}
% where  $ \CK(\beta,\rho) = (1+\beta)/(1-\beta) - (1+\rho)/(1-\rho)$ as in Section \ref{CTM}.

The interesting scenario in understanding implicit gradient regularization occurs when $\epsilon$ is much smaller than the gradient norms\footnote{The default value of $\epsilon$ in PyTorch is $1 \times 10^{-8}$. In practice, gradient norms are much larger than this.}. In this scenario, we take $\epsilon \to 0$, and \ref{CADAM} reduces to the following equations:
% \begin{align}\label{newcon}
% \dot{\tx}(t) = \mu_{\epsilon}(\tx,\ty) \cdot
% &\left[
% -\nabla_x f(\tx,\ty)
% - \frac{h}{2}\CK(\beta,\rho)\,\nabla_x \lVert \nabla_x f(\tx,\ty) \rVert_{1,\epsilon}
% \right. \nonumber\\
% &\left.\qquad
% + \frac{h}{2}\CK(\beta,\rho)\,\nabla_x \lVert \nabla_y f(\tx,\ty) \rVert_{1,\epsilon}
% \right] \nonumber\\
% \dot{\ty}(t) = \nu_{\epsilon}(\tx,\ty) \cdot
% &\left[
% \nabla_y f(\tx,\ty)
% + \frac{h}{2}\CK(\beta,\rho)\,\nabla_y \lVert \nabla_x f(\tx,\ty) \rVert_{1,\epsilon}
% \right. \nonumber\\
% &\left.\qquad
% - \frac{h}{2}\CK(\beta,\rho)\,\nabla_y \lVert \nabla_y f(\tx,\ty) \rVert_{1,\epsilon}
% \right],
% \end{align}
\begin{align*}
&\dot{\tx}(t)
= \mu_{\epsilon}(\tx,\ty)\cdot\Bigl[
-\nabla_x f(\tx,\ty)\nonumber \\
& + \frac{h}{2}\CK(\beta,\rho)\,
\nabla_x\!\Bigl(
\lVert \nabla_y f(\tx,\ty) \rVert_{1,\epsilon}
-\lVert \nabla_x f(\tx,\ty) \rVert_{1,\epsilon}
\Bigr)
\Bigr], \nonumber\\
&\dot{\ty}(t)
= \nu_{\epsilon}(\tx,\ty)\cdot\Bigl[
\nabla_y f(\tx,\ty)\nonumber \\
&+ \frac{h}{2}\CK(\beta,\rho)\,
\nabla_y\!\Bigl(
\lVert \nabla_x f(\tx,\ty) \rVert_{1,\epsilon}
-\lVert \nabla_y f(\tx,\ty) \rVert_{1,\epsilon}
\Bigr)
\Bigr].
\end{align*}

where  $ \CK(\beta,\rho) = (1+\beta)/(1-\beta) - (1+\rho)/(1-\rho)$ as in Section \ref{CTM}.
Here we continue to use $\ell_{1,\epsilon}$ and $\mu_{\epsilon}, \nu_{\epsilon}$ to avoid the non-differentiable problem of the $\ell_{1}$ norm and the zero denominator problem. 

We consider terms about $\lVert \nabla_x f(\tx,\ty) \lVert_{1,\epsilon}$. There are two factors affecting its evolution in above equations: The $x$-player's equation include a gradient \textit{descent} term on $ \CK(\beta,\rho)\lVert \nabla_x f(\tx,\ty) \lVert_{1,\epsilon}$, which aims to \textit{minimize} it:
\begin{align}\label{xplayerlim}
&- \frac{h}{2}\CK(\beta,\rho)  \nabla_x \lVert \nabla_x f(\tx,\ty) \lVert_{1,\epsilon} = \nonumber\\
& - \frac{h}{2}\CK(\beta,\rho) \left(\nabla^2_xf(\tx,\ty) \cdot \mu_{\epsilon}(\tx,\ty) \cdot \nabla_xf(\tx,\ty)\right),
\end{align}
while $y$-player's equation include a gradient \textit{ascent} term to  \textit{maximize} it:
\begin{align}\label{ypalyerlim}
&\frac{h}{2}\CK(\beta,\rho) \nabla_y \lVert \nabla_x f(\tx,\ty) \lVert_{1,\epsilon} =\nonumber \\
& \frac{h}{2}\CK(\beta,\rho) \left(\nabla_{yx} f(\tx,\ty) \cdot \mu_{\epsilon}(\tx,\ty) \cdot \nabla_xf(\tx,\ty)\right).
\end{align}

The proof of (\ref{xplayerlim}) and (\ref{ypalyerlim}) are provided in Appendix \ref{Appendix_S41}. By comparing (\ref{xplayerlim}) and (\ref{ypalyerlim}), we can predict that the evolution of $\lVert \nabla_x f(\tx,\ty) \rVert_{1,\epsilon}$ in the competition between the two players depends on the relative sizes of $\nabla^2_xf(\tx,\ty)$ and $\nabla_{yx}f(\tx,\ty)$. Moreover, under the condition that the players' interactions dominate the overall dynamics, we have $\nabla_{yx}f(\tx,\ty) \gg \nabla^2_xf(\tx,\ty)$, i.e., the matrix $\nabla_{yx}f(\tx,\ty)$ dominates  $\nabla^2_{x}f(\tx,\ty)$ in magnitude. This indicates that the $x$-player's impact on the evolution of $\lVert \nabla_x f(\tx,\ty) \rVert_{1,\epsilon}$ is negligible compared to the $y$-player’s, so we can focus on (\ref{ypalyerlim}) to understand how various parameters affect its evolution. From (\ref{ypalyerlim}), the IGR causes the $y$-player to perform gradient ascent on $\lVert \nabla_x f(\tx,\ty) \lVert_{1,\epsilon}$, effectively trying to maximize it. To encourage the trajectories of \ref{Adam} to explore flatter regions of the loss landscapes, this effect should be reduced. This can be achieved by selecting a smaller $\beta$ and a larger $\rho$, which lower $\CK(\beta,\rho)$; adjusting these parameters accordingly helps minimize its value. 

The above argument also applies to the evolution of $\lVert \nabla_y f(\tx,\ty) \rVert_{1,\epsilon}$, due to the symmetry between the two players. In particular, a smaller $\beta$ and a larger $\rho$ are expected to result in lower values of $\lVert \nabla_y f(\tx,\ty) \rVert_{1,\epsilon}$ along the trajectories. Thus, we obtain a justification of the thesis proposed above through \ref{CADAM}.

\subsection{Experiments}\label{expigr}

This section presents experimental results on GAN training to support our thesis. GAN training is a zero-sum game with strong adversarial interactions, aligning well with our thesis assumptions. Our goal here is \textbf{not} to outperform existing methods, but to validate the IGR effect of \ref{Adam}.

\textbf{Experimental Setting.}  Our setup generally follows the improved Wasserstein GANs framework \citep{gulrajani2017improved}. Data sets include CIFAR-10 ($32 \times 32$ resolution) and STL-10 ($64 \times 64$ resolution). Network architectures are ResNet and CNN. We train GANs using the \ref{Adam} with different $\beta$  and $\rho$. Both generator and discriminator use the learning rate $2 \times 10^{-4}$. The batch size is $64$. Future details are presented in Appendix \ref{Appendix_S5}. We use the following cumulative average gradient norms for visualization:
\begin{align*}
\mathrm{AvgS}_{\beta,\rho}(t) = \frac{1}{t}
\sum^{t}_{s=1} \left( \lVert \nabla_x f(\tx_s,\ty_s) \lVert_{1} +  \lVert \nabla_y f(\tx_s,\ty_s) \lVert_1  \right).
\end{align*}
Smaller $\mathrm{AvgS}_{\beta,\rho}$ indicate that the algorithms' trajectories are exploring flatter regions of loss landscapes.

\textbf{Experimental Results.} In Figure \ref{flatnesspic}, we show how the cumulative average gradient norms change during training. To evaluate the impact of $\beta$, we fix a value of $\rho = 0.9$ and plot the evolution of these norms for different $\beta$ values in Figures \ref{p21}, \ref{p23}, \ref{p29}, and \ref{p31}. We observe that smaller $\beta$ values lead to smaller gradient norms, supporting our thesis about the role of $\beta$. Similarly, to assess the effect of $\rho$, we fix $\beta = 0$ and plot the norms under different $\rho$ values in Figures \ref{p22}, \ref{p24}, \ref{p30}, and \ref{p32}. Larger $\rho$ values result in smaller gradient norms, which supports our thesis regarding $\rho$. The experimental results are consistency across different architectures and data sets.

% \begin{wraptable}{r}{0.6\textwidth}
%   \input{Tex/Table_2}
% \end{wraptable}
In Figure \ref{Inception_score}, we present the Inception Scores (IS) of various GANs trained by \ref{Adam}. IS is a widely used metric for evaluating the performance of GAN training, with higher scores indicating better performance \citep{salimans2016improved}. 
We find models that exhibit lower cumulative average gradient norms tend to achieve higher IS. This is analogous to the phenomena observed in minimization problems, where flatter regions tend to generalize better \citep{foret2020sharpness}. 

%It suggests a potential connection between the geometric properties such as flatness of regions in min-max loss landscapes and the enhanced performance of model parameters within those regions.

\section{Conclusions}
In this paper, we develop an ODE-based framework for analyzing the dynamics of Adam-DA in zero-sum games. We focus on how the algorithm's parameters shape two key aspects of the dynamics, namely local convergence and implicit gradient regularization. For both aspects, we show that these parameters influence Adam-DA in a way that is exactly opposite to Adam in minimization problems. This provides a theoretical explanation for empirical observations such as the benefits of negative momentum in GAN training \cite{gidel2019negative}, and it extends prior separation results between game dynamics and minimization dynamics from simpler methods without adaptivity or momentum \cite{bailey2018multiplicative,bailey2020finite} to the Adam-DA algorithm. An interesting direction in the future is to study dynamics of Adam in more complex and practical settings, such as the multi-player general-sum games, where the roles of momentum and adaptivity are far from well-understood.

%Several directions remain open for future works. First, our analysis is limited to the deterministic setting, and extending it to the stochastic case would be of great interest, where stochastic differential equations offer powerful tools for continuous-time analysis \citep{li2017stochastic,malladi2022sdes,compagnoni2024sdes}. Second, our experiments in Section \ref{IGR} suggest a potential link between the flatness of regions in min-max loss landscapes in terms of gradient norms and the improved performance of model parameters in those regions. Establishing a formal relation between these two, analogous to the connections already established in minimization problems \citep{haddouche2025a}, is an important problem for future study.

% \paragraph{Reproducibility Statement.} All theoretical results in this paper are accompanied by rigorous mathematical proofs, and the assumptions used are explicitly stated. To ensure reproducibility, the experimental code is provided in the supplementary materials.

% \paragraph{The Use of Large Language Models.} This paper utilized a large language model for grammar checking and polishing during its writing process. The code for some experiments in Section 5 of this paper was completed with the assistance of a large language model. The authors of this paper have reviewed all content generated by the large language model and confirm and take responsibility for it.
\section*{Acknowledgements}
This work was supported by Danish Data Science Academy, which is funded by the Novo Nordisk Foundation (NNF21SA0069429). The research of Xiao Wang is supported by National Key R\&D Program of China (2023YFA1009500), the Fundamental Research Funds for the Central Universities, Shanghai Sci-tech Co-research Program (25HB2703800). The research of Weiming Ou is supported by the Fundamental Research Funds for
the Central Universities. The authors thank Chang He and Ping Li for their feedback on early drafts of this work. They also thank the anonymous reviewer of ICML 2026 for their valuable suggestions.

\section*{Impact Statements}
This paper presents work whose goal is to advance the field of machine learning. There are many potential societal consequences of our work, none of which we feel must be specifically highlighted here.

\bibliography{Bibli}

@article{LP19,
	Author = {Jason D. Lee and Ioannis Panageas and Georgios Piliouras and Max Simchowitz and Michael I. Jordan and Benjamin Recht},
	Journal = {Math. Program.},
	Number = {1-2},
	Pages = {311--337},
	Title = {First-order methods almost always avoid strict saddle points},
	Volume = {176},
	Year = {2019}}

@article{daskalakis2017training,
  title={Training {GAN}s with optimism},
  author={Daskalakis, Constantinos and Ilyas, Andrew and Syrgkanis, Vasilis and Zeng, Haoyang},
  journal={arXiv preprint arXiv:1711.00141},
  year={2017}
}

@inproceedings{goodfellow2014generative,
author = {Goodfellow, Ian J. and Pouget-Abadie, Jean and Mirza, Mehdi and Xu, Bing and Warde-Farley, David and Ozair, Sherjil and Courville, Aaron and Bengio, Yoshua},
title = {Generative Adversarial Nets},
year = {2014},
publisher = {MIT Press},
address = {Cambridge, MA, USA},
booktitle = {Proceedings of the 27th International Conference on Neural Information Processing Systems - Volume 2},
pages = {2672–2680},
numpages = {9},
location = {Montreal, Canada},
series = {NIPS’14}
}

@inproceedings{CP2020,
	author    = {Yun Kuen Cheung and
	Georgios Piliouras},
	title     = {Chaos, Extremism and Optimism: Volume Analysis of Learning in Games},
	year      = {2020},
	booktitle      = {{NeurIPS} }
}

@inproceedings{bailey2018multiplicative,
  title={Multiplicative weights update in zero-sum games},
  author={Bailey, James P and Piliouras, Georgios},
  booktitle={Proceedings of the 2018 ACM Conference on Economics and Computation},
  pages={321--338},
  year={2018}
}

@inproceedings{hsieh2021limits,
  title={The limits of min-max optimization algorithms: Convergence to spurious non-critical sets},
  author={Hsieh, Ya-Ping and Mertikopoulos, Panayotis and Cevher, Volkan},
  booktitle={International Conference on Machine Learning},
  pages={4337--4348},
  year={2021},
  organization={PMLR}
}

@inproceedings{bailey2020finite,
  title={Finite regret and cycles with fixed step-size via alternating gradient descent-ascent},
  author={Bailey, James P and Gidel, Gauthier and Piliouras, Georgios},
  booktitle={Conference on Learning Theory},
  pages={391--407},
  year={2020},
  organization={PMLR}
}

@inproceedings{gidel2019negative,
  title={Negative momentum for improved game dynamics},
  author={Gidel, Gauthier and Hemmat, Reyhane Askari and Pezeshki, Mohammad and Le Priol, R{\'e}mi and Huang, Gabriel and Lacoste-Julien, Simon and Mitliagkas, Ioannis},
  booktitle={The 22nd International Conference on Artificial Intelligence and Statistics},
  year={2019},
}

@book{haier2006geometric,
  title={Geometric Numerical integration: structure-preserving algorithms for ordinary differential equations},
  author={Haier, Ernst and Lubich, Christian and Wanner, Gerhard},
  year={2006},
  publisher={Springer}
}

@inproceedings{fasoulakis2022forward,
  title={Forward looking best-response multiplicative weights update methods for bilinear zero-sum games},
  author={Fasoulakis, Michail and Markakis, Evangelos and Pantazis, Yannis and Varsos, Constantinos},
  booktitle={International Conference on Artificial Intelligence and Statistics},
  pages={11096--11117},
  year={2022},
  organization={PMLR}
}

@inproceedings{mokhtari2020unified,
  title={A unified analysis of extra-gradient and optimistic gradient methods for saddle point problems: Proximal point approach},
  author={Mokhtari, Aryan and Ozdaglar, Asuman and Pattathil, Sarath},
  booktitle={International Conference on Artificial Intelligence and Statistics},
  pages={1497--1507},
  year={2020},
  organization={PMLR}
}

@article{shi2022understanding,
  title={Understanding the acceleration phenomenon via high-resolution differential equations},
  author={Shi, Bin and Du, Simon S and Jordan, Michael I and Su, Weijie J},
  journal={Mathematical Programming},
  pages={1--70},
  year={2022},
  publisher={Springer}
}

@inproceedings{barrett2020implicit,
  title={Implicit Gradient Regularization},
  author={Barrett, David and Dherin, Benoit},
  booktitle={International Conference on Learning Representations},
  year={2021}
}

@article{ghosh2023implicit,
  title={Implicit regularization in heavy-ball momentum accelerated stochastic gradient descent},
  author={Ghosh, Avrajit and Lyu, He and Zhang, Xitong and Wang, Rongrong},
  journal={ICLR},
  year={2023}
}

@inproceedings{rosca2021discretization,
  title={Discretization drift in two-player games},
  author={Rosca, Mihaela C and Wu, Yan and Dherin, Benoit and Barrett, David},
  booktitle={International Conference on Machine Learning},
  pages={9064--9074},
  year={2021},
  organization={PMLR}
}

@article{nagarajan2017gradient,
  title={Gradient descent GAN optimization is locally stable},
  author={Nagarajan, Vaishnavh and Kolter, J Zico},
  journal={Advances in neural information processing systems},
  volume={30},
  year={2017}
}

@InProceedings{pmlr-v89-liang19b,
  title = 	 {Interaction Matters: A Note on Non-asymptotic Local Convergence of Generative Adversarial Networks},
  author =       {Liang, Tengyuan and Stokes, James},
  booktitle = 	 {Proceedings of the Twenty-Second International Conference on Artificial Intelligence and Statistics},
  year = 	 {2019},
}

@article{wang2024local,
  title={Local convergence of gradient methods for min-max games: partial curvature generically suffices},
  author={Wang, Guillaume and Chizat, L{\'e}na{\"\i}c},
  journal={Advances in Neural Information Processing Systems},
  volume={36},
  year={2024}
}

@InProceedings{pmlr-v151-zhang22e,
  title = 	 { Near-optimal Local Convergence of Alternating Gradient Descent-Ascent for Minimax Optimization },
  author =       {Zhang, Guodong and Wang, Yuanhao and Lessard, Laurent and Grosse, Roger B.},
  booktitle = 	 {International Conference on Artificial Intelligence and Statistics},
  pages = 	 {7659--7679},
  publisher =    {PMLR},
  year = 	 {2022},
}

@article{JMLR:v20:19-008,
  author  = {Alistair Letcher and David Balduzzi and S{{\'e}}bastien Racani{{\`e}}re and James Martens and Jakob Foerster and Karl Tuyls and Thore Graepel},
  title   = {Differentiable Game Mechanics},
  journal = {Journal of Machine Learning Research},
  year    = {2019},
  volume  = {20},
  number  = {84},
  pages   = {1--40},
}

@inproceedings{lotidisaccelerated,
  title={Accelerated Regularized Learning in Finite N-Person Games},
  author={Lotidis, Kyriakos and Giannou, Angeliki and Mertikopoulos, Panayotis and Bambos, Nicholas},
  booktitle={The Thirty-eighth Annual Conference on Neural Information Processing Systems},
  year    = {2024}
}

@article{wibisono2016variational,
  title={A variational perspective on accelerated methods in optimization},
  author={Wibisono, Andre and Wilson, Ashia C and Jordan, Michael I},
  journal={proceedings of the National Academy of Sciences},
  volume={113},
  number={47},
  pages={E7351--E7358},
  year={2016},
  publisher={National Acad Sciences}
}

@article{foret2020sharpness,
  title={Sharpness-aware minimization for efficiently improving generalization},
  author={Foret, Pierre and Kleiner, Ariel and Mobahi, Hossein and Neyshabur, Behnam},
  journal={ICLR},
  year={2021}
}

@article{gulrajani2017improved,
  title={Improved training of {W}asserstein {GAN}s},
  author={Gulrajani, Ishaan and Ahmed, Faruk and Arjovsky, Martin and Dumoulin, Vincent and Courville, Aaron C},
  journal={Advances in neural information processing systems},
  volume={30},
  year={2017}
}

@article{kingma2014adam,
  title={Adam: A method for stochastic optimization},
  author={Kingma, Diederik P. and Ba, Jimmy},
  journal={arXiv preprint arXiv:1412.6980},
  year={2014}
}

@book{graham2018matrix,
  title={Matrix Theory and Applications for Scientists and Engineers},
  author={Graham, Alexander},
  year={2018},
  publisher={Courier Dover Publications}
}

@article{o2015adaptive,
  title={Adaptive restart for accelerated gradient schemes},
  author={O’Donoghue, Brendan and Candes, Emmanuel},
  journal={Foundations of Computational Mathematics},
  volume={15},
  pages={715--732},
  year={2015},
  publisher={Springer}
}

@article{salimans2016improved,
  title={Improved techniques for training gans},
  author={Salimans, Tim and Goodfellow, Ian and Zaremba, Wojciech and Cheung, Vicki and Radford, Alec and Chen, Xi},
  journal={Advances in neural information processing systems},
  volume={29},
  year={2016}
}

@article{compagnoni2024adaptive,
  title={Adaptive Methods through the Lens of SDEs: Theoretical Insights on the Role of Noise},
  author={Compagnoni, Enea Monzio and Liu, Tianlin and Islamov, Rustem and Proske, Frank Norbert and Orvieto, Antonio and Lucchi, Aurelien},
  journal={arXiv preprint arXiv:2411.15958},
  year={2024}
}

@inproceedings{compagnoni2024sdes,
  title={SDEs for minimax optimization},
  author={Compagnoni, Enea Monzio and Orvieto, Antonio and Kersting, Hans and Proske, Frank and Lucchi, Aur{\'e}lien},
  booktitle={International Conference on Artificial Intelligence and Statistics},
  pages={4834--4842},
  year={2024},
  organization={PMLR}
}

@inproceedings{
feng2025continuoustime,
title={Continuous-Time Analysis of Heavy Ball Momentum in Min-Max Games},
author={Yi Feng and Kaito Fujii and Stratis Skoulakis and Xiao Wang and Volkan Cevher},
booktitle={Forty-second International Conference on Machine Learning},
year={2025}
}

@inproceedings{ma2022qualitative,
  title={A qualitative study of the dynamic behavior for adaptive gradient algorithms},
  author={Ma, Chao and Wu, Lei and others},
  booktitle={Mathematical and scientific machine learning},
  pages={671--692},
  year={2022},
  organization={PMLR}
}

@inproceedings{compagnoniadaptive,
  title={Adaptive Methods through the Lens of SDEs: Theoretical Insights on the Role of Noise},
  author={Compagnoni, Enea Monzio and Liu, Tianlin and Islamov, Rustem and Proske, Frank Norbert and Orvieto, Antonio and Lucchi, Aurelien},
  year={2025},
  booktitle={The Thirteenth International Conference on Learning Representations}
}

@InProceedings{pmlr-v235-cattaneo24a,
  title = 	 {On the Implicit Bias of {A}dam},
  author =       {Cattaneo, Matias D. and Klusowski, Jason Matthew and Shigida, Boris},
  booktitle = 	 {Proceedings of the 41st International Conference on Machine Learning},
  pages = 	 {5862--5906},
  year = 	 {2024},
  editor = 	 {Salakhutdinov, Ruslan and Kolter, Zico and Heller, Katherine and Weller, Adrian and Oliver, Nuria and Scarlett, Jonathan and Berkenkamp, Felix},
  volume = 	 {235},
  series = 	 {Proceedings of Machine Learning Research},
  month = 	 {21--27 Jul},
  publisher =    {PMLR},
}

@inproceedings{zhang2023gradient,
  title={Gradient norm aware minimization seeks first-order flatness and improves generalization},
  author={Zhang, Xingxuan and Xu, Renzhe and Yu, Han and Zou, Hao and Cui, Peng},
  booktitle={Proceedings of the IEEE/CVF Conference on Computer Vision and Pattern Recognition},
  pages={20247--20257},
  year={2023}
}

@inproceedings{zhang2023flatness,
  title={Flatness-aware minimization for domain generalization},
  author={Zhang, Xingxuan and Xu, Renzhe and Yu, Han and Dong, Yancheng and Tian, Pengfei and Cui, Peng},
  booktitle={Proceedings of the IEEE/CVF International Conference on Computer Vision},
  pages={5189--5202},
  year={2023}
}

@book{henrici1974acca,
  title={Applied and Computational Complex Analysis, Vol. I},
  author={Henrici, Peter},
  year={1974},
  publisher={Wiley-Interscience}
}

@book{elaydi2005introduction,
  title={An introduction to difference equations},
  author={Elaydi, Saber},
  year={2005},
  publisher={Springer}
}

@inproceedings{bock2019proof,
  title={A proof of local convergence for the Adam optimizer},
  author={Bock, Sebastian and Wei{\ss}, Martin},
  booktitle={2019 international joint conference on neural networks (IJCNN)},
  pages={1--8},
  year={2019},
  organization={IEEE}
}

@article{bock2021local,
  title={Local convergence of adaptive gradient descent optimizers},
  author={Bock, Sebastian and Wei{\ss}, Martin Georg},
  journal={arXiv preprint arXiv:2102.09804},
  year={2021}
}

@article{zhang2022adam,
  title={Adam can converge without any modification on update rules},
  author={Zhang, Yushun and Chen, Congliang and Shi, Naichen and Sun, Ruoyu and Luo, Zhi-Quan},
  journal={Advances in neural information processing systems},
  volume={35},
  pages={28386--28399},
  year={2022}
}

@article{su2016differential,
  title={A differential equation for modeling Nesterov's accelerated gradient method: Theory and insights},
  author={Su, Weijie and Boyd, Stephen and Candes, Emmanuel J},
  journal={Journal of Machine Learning Research},
  volume={17},
  number={153},
  pages={1--43},
  year={2016}
}

@article{polyak1964some,
  title={Some methods of speeding up the convergence of iteration methods},
  author={Polyak, Boris T},
  journal={Ussr computational mathematics and mathematical physics},
  volume={4},
  number={5},
  pages={1--17},
  year={1964},
  publisher={Elsevier}
}

@inproceedings{li2022convergence,
  title={On convergence of gradient descent ascent: A tight local analysis},
  author={Li, Haochuan and Farnia, Farzan and Das, Subhro and Jadbabaie, Ali},
  booktitle={International Conference on Machine Learning},
  pages={12717--12740},
  year={2022},
  organization={PMLR}
}

@inproceedings{
zhang2025local,
title={Local convergence of simultaneous min-max algorithms to differential equilibrium on Riemannian manifold},
author={Sixin Zhang},
booktitle={The Thirteenth International Conference on Learning Representations},
year={2025}
}

@article{goodfellow2014explaining,
  title={Explaining and harnessing adversarial examples},
  author={Goodfellow, Ian J and Shlens, Jonathon and Szegedy, Christian},
  journal={arXiv preprint arXiv:1412.6572},
  year={2014}
}

@inproceedings{arjovsky2017wasserstein,
  title={Wasserstein generative adversarial networks},
  author={Arjovsky, Martin and Chintala, Soumith and Bottou, L{\'e}on},
  booktitle={International conference on machine learning},
  pages={214--223},
  year={2017},
  organization={PMLR}
}

@inproceedings{sauer2023stylegan,
  title={Stylegan-t: Unlocking the power of gans for fast large-scale text-to-image synthesis},
  author={Sauer, Axel and Karras, Tero and Laine, Samuli and Geiger, Andreas and Aila, Timo},
  booktitle={International conference on machine learning},
  pages={30105--30118},
  year={2023},
  organization={PMLR}
}

@article{zhao2021improved,
  title={Improved transformer for high-resolution gans},
  author={Zhao, Long and Zhang, Zizhao and Chen, Ting and Metaxas, Dimitris and Zhang, Han},
  journal={Advances in Neural Information Processing Systems},
  volume={34},
  pages={18367--18380},
  year={2021}
}

@article{suh2023continuous,
  title={Continuous-time analysis of anchor acceleration},
  author={Suh, Jaewook and Park, Jisun and Ryu, Ernest},
  journal={Advances in Neural Information Processing Systems},
  volume={36},
  pages={32782--32866},
  year={2023}
}

@book{khalil2002nonlinear,
  title={Nonlinear systems},
  author={Khalil, Hassan K and Grizzle, Jessy W},
  volume={3},
  year={2002},
  publisher={Prentice hall Upper Saddle River, NJ}
}

@inproceedings{jin2020local,
  title={What is local optimality in nonconvex-nonconcave minimax optimization?},
  author={Jin, Chi and Netrapalli, Praneeth and Jordan, Michael},
  booktitle={International conference on machine learning},
  pages={4880--4889},
  year={2020},
  organization={PMLR}
}

@article{fiez2021global,
  title={Global convergence to local minmax equilibrium in classes of nonconvex zero-sum games},
  author={Fiez, Tanner and Ratliff, Lillian and Mazumdar, Eric and Faulkner, Evan and Narang, Adhyyan},
  journal={Advances in Neural Information Processing Systems},
  volume={34},
  pages={29049--29063},
  year={2021}
}

@inproceedings{fiez2021local,
  title={Local convergence analysis of gradient descent ascent with finite timescale separation},
  author={Fiez, Tanner and Ratliff, Lillian J},
  booktitle={Proceedings of the International Conference on Learning Representation},
  year={2021}
}

@inproceedings{litiada,
  title={TiAda: A Time-scale Adaptive Algorithm for Nonconvex Minimax Optimization},
  author={Li, Xiang and Junchi, YANG and He, Niao},
  booktitle={The Eleventh International Conference on Learning Representations},
  year={2022}
}

@article{yang2022nest,
  title={Nest your adaptive algorithm for parameter-agnostic nonconvex minimax optimization},
  author={Yang, Junchi and Li, Xiang and He, Niao},
  journal={Advances in Neural Information Processing Systems},
  volume={35},
  pages={11202--11216},
  year={2022}
}

@article{gadat2022asymptotic,
  title={Asymptotic study of stochastic adaptive algorithms in non-convex landscape},
  author={Gadat, S{\'e}bastien and Gavra, Ioana},
  journal={Journal of Machine Learning Research},
  volume={23},
  number={228},
  pages={1--54},
  year={2022}
}

@article{huang2022new,
  title={New first-order algorithms for stochastic variational inequalities},
  author={Huang, Kevin and Zhang, Shuzhong},
  journal={SIAM Journal on Optimization},
  volume={32},
  number={4},
  pages={2745--2772},
  year={2022},
  publisher={SIAM}
}

@article{bot2023relaxed,
  title={A relaxed inertial forward-backward-forward algorithm for solving monotone inclusions with application to GANs},
  author={Bot, Radu I and Sedlmayer, Michael and Vuong, Phan Tu},
  journal={Journal of Machine Learning Research},
  volume={24},
  number={8},
  pages={1--37},
  year={2023}
}

@inproceedings{
cohen2021gradient,
title={Gradient Descent on Neural Networks Typically Occurs at the Edge of Stability},
author={Jeremy Cohen and Simran Kaur and Yuanzhi Li and J Zico Kolter and Ameet Talwalkar},
booktitle={International Conference on Learning Representations},
year={2021},
url={https://openreview.net/forum?id=jh-rTtvkGeM}
}

@article{wang2022does,
  title={Does momentum change the implicit regularization on separable data?},
  author={Wang, Bohan and Meng, Qi and Zhang, Huishuai and Sun, Ruoyu and Chen, Wei and Ma, Zhi-Ming and Liu, Tie-Yan},
  journal={Advances in Neural Information Processing Systems},
  volume={35},
  pages={26764--26776},
  year={2022}
}

@article{zhang2024implicit,
  title={The implicit bias of adam on separable data},
  author={Zhang, Chenyang and Zou, Difan and Cao, Yuan},
  journal={Advances in Neural Information Processing Systems},
  volume={37},
  pages={23988--24021},
  year={2024}
}

@inproceedings{
cohen2025understanding,
title={Understanding Optimization in Deep Learning with Central Flows},
author={Jeremy Cohen and Alex Damian and Ameet Talwalkar and J Zico Kolter and Jason D. Lee},
booktitle={The Thirteenth International Conference on Learning Representations},
year={2025},
url={https://openreview.net/forum?id=sIE2rI3ZPs}
}

@inproceedings{wang2021implicit,
  title={The implicit bias for adaptive optimization algorithms on homogeneous neural networks},
  author={Wang, Bohan and Meng, Qi and Chen, Wei and Liu, Tie-Yan},
  booktitle={International Conference on Machine Learning},
  pages={10849--10858},
  year={2021},
  organization={PMLR}
}

@inproceedings{
xie2024implicit,
title={Implicit Bias of AdamW: \${\textbackslash}ell\_{\textbackslash}infty\$-Norm Constrained Optimization},
author={Shuo Xie and Zhiyuan Li},
booktitle={Forty-first International Conference on Machine Learning},
year={2024}
}

@article{hong2024convergence,
  title={On convergence of adam for stochastic optimization under relaxed assumptions},
  author={Hong, Yusu and Lin, Junhong},
  journal={Advances in Neural Information Processing Systems},
  volume={37},
  pages={10827--10877},
  year={2024}
}

@inproceedings{wang2024provable,
  title={Provable adaptivity of adam under non-uniform smoothness},
  author={Wang, Bohan and Zhang, Yushun and Zhang, Huishuai and Meng, Qi and Sun, Ruoyu and Ma, Zhi-Ming and Liu, Tie-Yan and Luo, Zhi-Quan and Chen, Wei},
  booktitle={Proceedings of the 30th ACM SIGKDD Conference on Knowledge Discovery and Data Mining},
  pages={2960--2969},
  year={2024}
}

@article{li2023convergence,
  title={Convergence of adam under relaxed assumptions},
  author={Li, Haochuan and Rakhlin, Alexander and Jadbabaie, Ali},
  journal={Advances in Neural Information Processing Systems},
  volume={36},
  pages={52166--52196},
  year={2023}
}

@article{zaheer2018adaptive,
  title={Adaptive methods for nonconvex optimization},
  author={Zaheer, Manzil and Reddi, Sashank and Sachan, Devendra and Kale, Satyen and Kumar, Sanjiv},
  journal={Advances in neural information processing systems},
  volume={31},
  year={2018}
}

@article{chen2018convergence,
  title={On the convergence of a class of adam-type algorithms for non-convex optimization},
  author={Chen, Xiangyi and Liu, Sijia and Sun, Ruoyu and Hong, Mingyi},
  journal={arXiv preprint arXiv:1808.02941},
  year={2018}
}

@inproceedings{zou2019sufficient,
  title={A sufficient condition for convergences of adam and rmsprop},
  author={Zou, Fangyu and Shen, Li and Jie, Zequn and Zhang, Weizhong and Liu, Wei},
  booktitle={Proceedings of the IEEE/CVF Conference on computer vision and pattern recognition},
  pages={11127--11135},
  year={2019}
}

@article{wang2023closing,
  title={Closing the gap between the upper bound and lower bound of Adam's iteration complexity},
  author={Wang, Bohan and Fu, Jingwen and Zhang, Huishuai and Zheng, Nanning and Chen, Wei},
  journal={Advances in Neural Information Processing Systems},
  volume={36},
  pages={39006--39032},
  year={2023}
}

@article{tieleman2012lecture,
  title={Lecture 6.5-rmsprop: Divide the gradient by a running average of its recent magnitude},
  author={Tieleman, Tijmen},
  journal={COURSERA: Neural networks for machine learning},
  volume={4},
  number={2},
  pages={26},
  year={2012}
}

@misc{KB14,
	author = {Diederik P. Kingma and Jimmy Ba},
	howpublished = {\url{https://arxiv.org/abs/1412.6980}},
	title = {Adam: A method for stochastic optimization},
	year = {2014}}

@inproceedings{RKK18,
	author = {Sashank J. Reddi and Satyen Kale and Sanjiv Kumar},
	booktitle = {ICLR '18: Proceedings of the 2018 International Conference on Learning Representations},
	title = {On the convergence of {Adam} and beyond},
	year = {2018}}

@article{duchi2011adaptive,
  title={Adaptive subgradient methods for online learning and stochastic optimization.},
  author={Duchi, John and Hazan, Elad and Singer, Yoram},
  journal={Journal of machine learning research},
  volume={12},
  number={7},
  year={2011}
}

@article{defossez2020simple,
  title={A simple convergence proof of adam and adagrad},
  author={D{\'e}fossez, Alexandre and Bottou, L{\'e}on and Bach, Francis and Usunier, Nicolas},
  journal={arXiv preprint arXiv:2003.02395},
  year={2020}
}
\bibliographystyle{icml2026}

% \bibliography{example_paper}
% \bibliographystyle{icml2026}

%%%%%%%%%%%%%%%%%%%%%%%%%%%%%%%%%%%%%%%%%%%%%%%%%%%%%%%%%%%%%%%%%%%%%%%%%%%%%%%
%%%%%%%%%%%%%%%%%%%%%%%%%%%%%%%%%%%%%%%%%%%%%%%%%%%%%%%%%%%%%%%%%%%%%%%%%%%%%%%
% APPENDIX
%%%%%%%%%%%%%%%%%%%%%%%%%%%%%%%%%%%%%%%%%%%%%%%%%%%%%%%%%%%%%%%%%%%%%%%%%%%%%%%
%%%%%%%%%%%%%%%%%%%%%%%%%%%%%%%%%%%%%%%%%%%%%%%%%%%%%%%%%%%%%%%%%%%%%%%%%%%%%%%
\newpage
\appendix
\onecolumn
\newpage
\section{More discussion on related works}\label{Appendix: detailed comparison with feng}

\paragraph{Comparison with \cite{rosca2021discretization} and \cite{feng2025continuoustime}.} 

We  make a detailed comparison with \cite{rosca2021discretization} and \cite{feng2025continuoustime}, which studied the dynamics of Gradient Descent-Ascent method and Heavy Ball method in zero-sum games through the lens of ODEs in the following. The work of \cite{rosca2021discretization} investigated implicit regularization of GDA in two-player general-sum games. However, their GDA algorithm involves neither adaptive stepsizes nor momentum parameters. Next, the work of \cite{feng2025continuoustime} studied relatively simple Heavy-ball momentum method, which involves only a non-adaptive momentum parameter. Adam’s adaptive structure substantially increases the complexity of both the ODEs and the discrete-time analysis. For instance, deriving an appropriate $\mathcal{O}(h^3)$ local error equation is significantly more involved than in the Heavy-ball case, which in turn makes the Jacobian analysis considerably more difficult. Our analysis yield new insights that cannot be inferred from previous works of \cite{rosca2021discretization,feng2025continuoustime}. In particular:
\begin{itemize}[leftmargin=*]
    \item Discrete-time vs. continuous-time behavior: We give a local convergence result for the continuous-time model and also analyze the original discrete-time Adam-DA algorithms. The discrete-time results show that the continuous-time model matches the algorithms well in usual settings. But when the momentum $\beta$ is very small, this match becomes weak, which highlight the potential failure of continuous-time models that do not appear in previous works.
    \item The role of $\rho$ : Heavy-ball momentum is non-adaptive, so the influence of $\rho$ is entirely absent in \cite{feng2025continuoustime}. In contrast, we show that in min–max games, larger $\rho$ amplifies Adam’s implicit regularization effect—in sharp contrast to the minimization setting, where \cite{pmlr-v235-cattaneo24a} shows that smaller $\rho$ strengthens implicit regularization. This is a fundamentally new phenomenon specific to the min–max setting.
    \item The role of $\epsilon$ : Traditionally, $\epsilon$ is viewed merely as a term preventing division by zero and is assumed to be extremely small. In fact, using large $\epsilon$ is often considered inconsistent with the design philosophy of adaptive gradient methods \citep{zhang2022adam}. However, Theorems 4.3 and 4.4 demonstrate that a large $\epsilon$ can enhance local convergence in min–max games—again revealing behavior entirely absent in the Heavy-ball analysis.
\end{itemize}
%Based on these points, we believe that our results provide new insights into the dynamics of Adam in min–max games that go well beyond the findings of \cite{feng2025continuoustime}.

\paragraph{More Related Works in Game Dynamics.} Besides converging to local nash equilibrium, other properties of zero-sum games has been researched. Early rigorous characterizations of local optimality and stability in this setting were developed by \cite{jin2020local}, who formalized notions of local min–max optimality in nonconvex–nonconcave landscapes. \cite{fiez2021global} investigated the non-asymptotic convergence rates to $\epsilon$-critical points in several classes of zero-sum games. At the same time, \cite{hsieh2021limits} exposed limitations of common min–max algorithms by showing possible convergence to spurious, non-critical sets, highlighting the delicate geometry of min–max flows. Parallel to these algorithmic and discrete-time analyses, recent work has explored adaptive time-scale strategies: \cite{litiada} proposed Tiada, a time-scale adaptive method for nonconvex minimax problems, and \cite{yang2022nest} developed nested adaptive schemes that achieve parameter-agnostic convergence—complementary approaches that adaptively tune per-variable step sizes and thereby mitigate the need for manual time-scale separation. 

\paragraph{Related work on Adam.} The evolution of adaptive optimization began with AdaGrad \cite{duchi2011adaptive}, which normalizes updates via the sum of squared gradients to optimize for sparse data. While theoretically sound, AdaGrad's monotonic accumulation leads to premature learning rate decay in non-convex regimes. To address this, RMSProp \cite{tieleman2012lecture} introduced an exponential moving average (EMA) to curb the rapid decline. The seminal Adam optimizer \cite{KB14} later integrated this adaptive scaling with momentum-based acceleration, tracking both first and second gradient moments to achieve superior stability in training deep neural networks.

Although Adam dominates empirically, theoretical convergence concerns led to the development of AMSGrad \cite{RKK18}, which imposes a non-decreasing constraint on the denominator. In terms of analysis, early studies relied on restrictive assumptions like bounded gradients \cite{zaheer2018adaptive,chen2018convergence,zou2019sufficient,defossez2020simple}. Contemporary research, however, has pivoted toward high-probability guarantees under looser conditions: \cite{li2023convergence} allows for sub-Gaussian noise, \cite{hong2024convergence} accommodates affine noise, and \cite{wang2023closing} has proven that Adam meets fundamental lower bounds. Extensions to finite-sum settings have also been established \cite{zhang2022adam,wang2024provable}.

% \paragraph{Continuous-time analysis.} Continuous-time (dynamical-systems) analyses of adaptive gradient methods—including Adam-like schemes—have been advanced in a series of works that predate some recent discrete-time studies: \cite{da2020general} proposed a general ODE system modeling first-order adaptive algorithms; \cite{barakat2021convergence} and \cite{barakat2021stochastic} analyzed convergence, fluctuations, and escape phenomena for Adam and momentum-based stochastic dynamics; and \cite{gadat2022asymptotic} provided asymptotic studies of stochastic adaptive algorithms in nonconvex landscapes. These continuous-time perspectives sit within a longer tradition of dynamical-systems approaches to stochastic approximation (e.g., \cite{benaim1996dynamical,borkar2008stochastic} and references therein) and collectively motivate the study of adaptive algorithms in min–max settings from both discrete- and continuous-time viewpoints.

\newpage
\section{Additional Materials for Section \ref{CTM}}

\subsection{Future Details of Figure \ref{comparewithsign}}\label{FDoF}

The  test functions in Figure \ref{comparewithsign} comes from \citep{compagnoni2024sdes}:
\begin{itemize}[leftmargin=*]
\item (Figure in the left:) $f_1(x,y) = x(y-0.45) + \phi(x) - \phi(y),\ \phi(z) = \frac{1}{4}z^2 - \frac{1}{2}z^4 + \frac{1}{6}z^6$. 

Initial point: $(x_0,y_0) = (0.6,0.6)$. $\beta=0,\ \rho = 0.5,\ \epsilon = 10^{-6}.$ Step size = $0.007.$
\item  (Figure in the middle:)$f_2(x,y) = xy - \frac{1}{10}(\frac{1}{2}y^2 - \frac{1}{4}y^4 )$

Initial point: $(x_0,y_0) = (0.6,0.6)$. $\beta=-0.3,\ \rho = 0.9,\ \epsilon = 10^{-3}.$ Step size = $0.002.$
\item  (Figure in the right:)$f_3(x,y) = \frac{1}{10}x^2 - \frac{1}{10}y^2 + \sin(x)\cos(y)$

Initial point: $(x_0,y_0) = (0.6,0.6)$. $\beta=0.3,\ \rho = 0.5,\ \epsilon = 10^{-4}.$ Step size = $0.005.$
\end{itemize}

In the following we also present $30$ random initial conditions, and draw the distance between the trajectories of \ref{CADAM}, \ref{sign} with \ref{Adam}. It can be observed that \ref{CADAM} can better approximate \ref{Adam} than \ref{sign}.

\begin{figure}[h]
    \centering
    % \subfigure{
    %     \includegraphics[width=1.6in]{Pictures/oder_traj.png}
    %     \label{ap1}
    % }
    % \subfigure{
    %     \includegraphics[width=1.6in]{Pictures/oder_traj2.png}
    %     \label{ap2}
    % }
    % \subfigure{
    %     \includegraphics[width=1.6in]{Pictures/oder_traj3.png}
    %     \label{ap3}
    % }
    \subfigure[Distance Curves of $f_1$]{
	\includegraphics[width=1.6in]{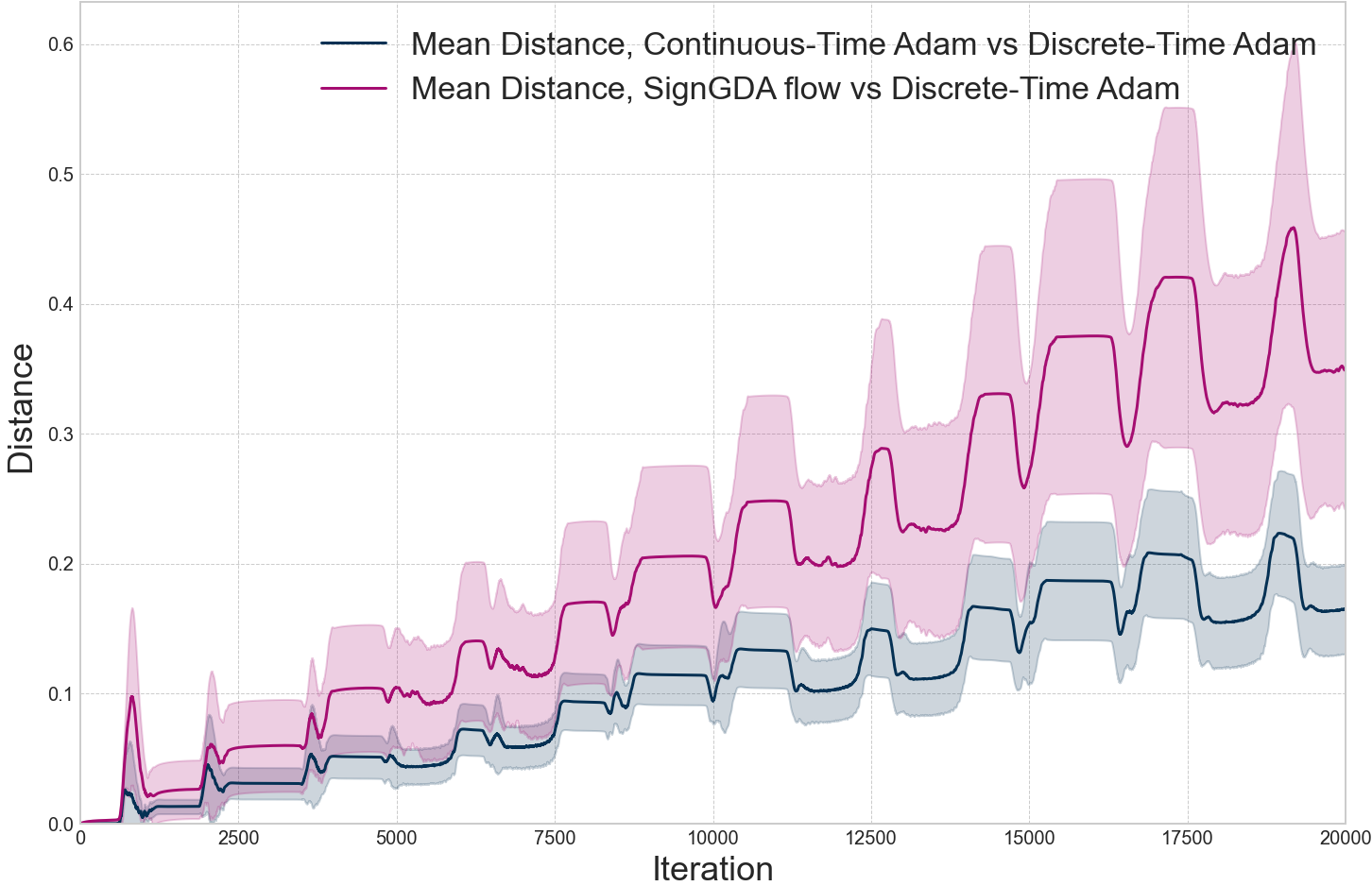}
        \label{ap4}
    }
    \subfigure[Distance Curves of $f_2$]{
	\includegraphics[width=1.6in]{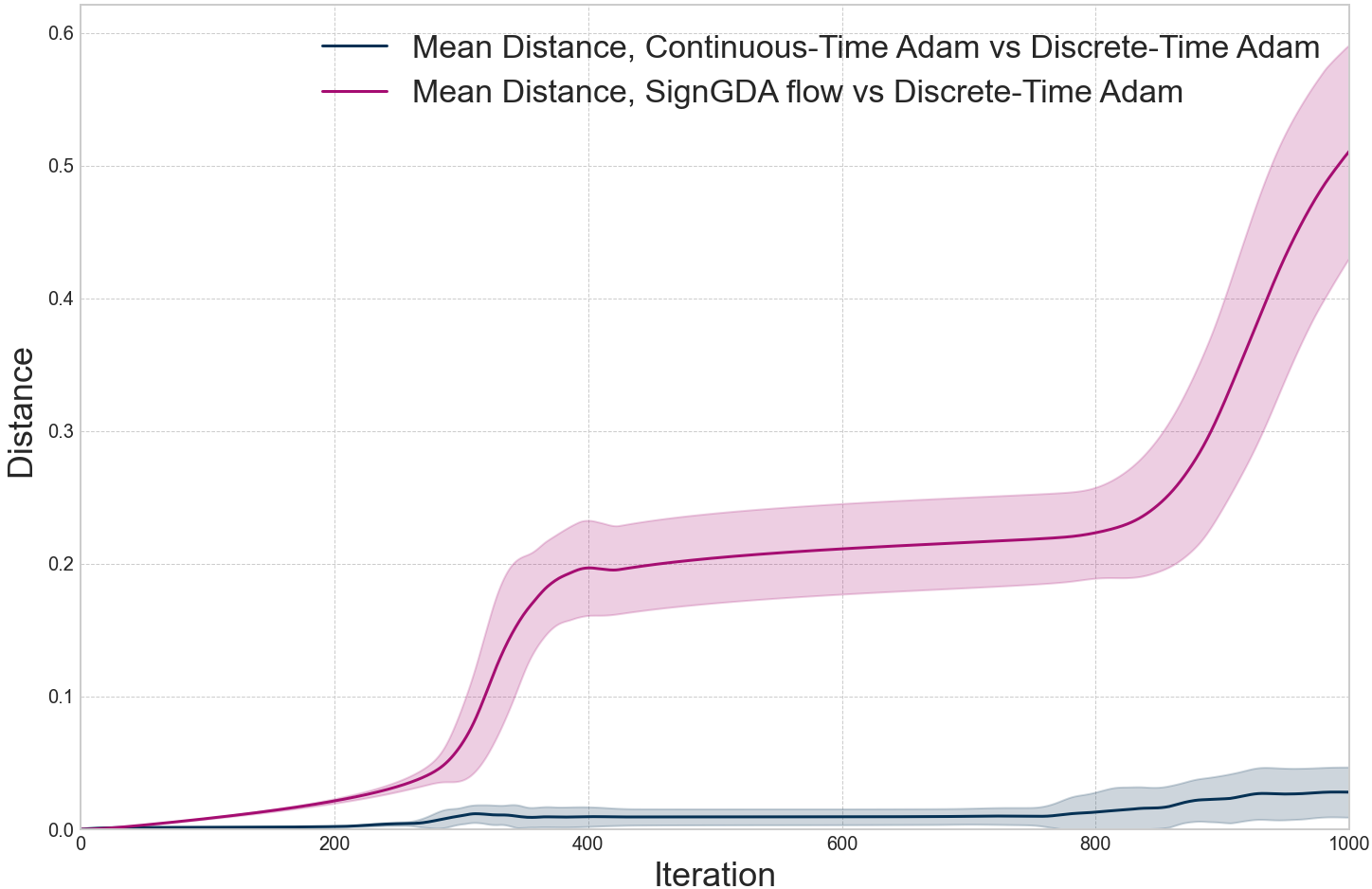}
        \label{ap5}
    }
    \subfigure[Distance Curves of $f_3$]{
	\includegraphics[width=1.6in]{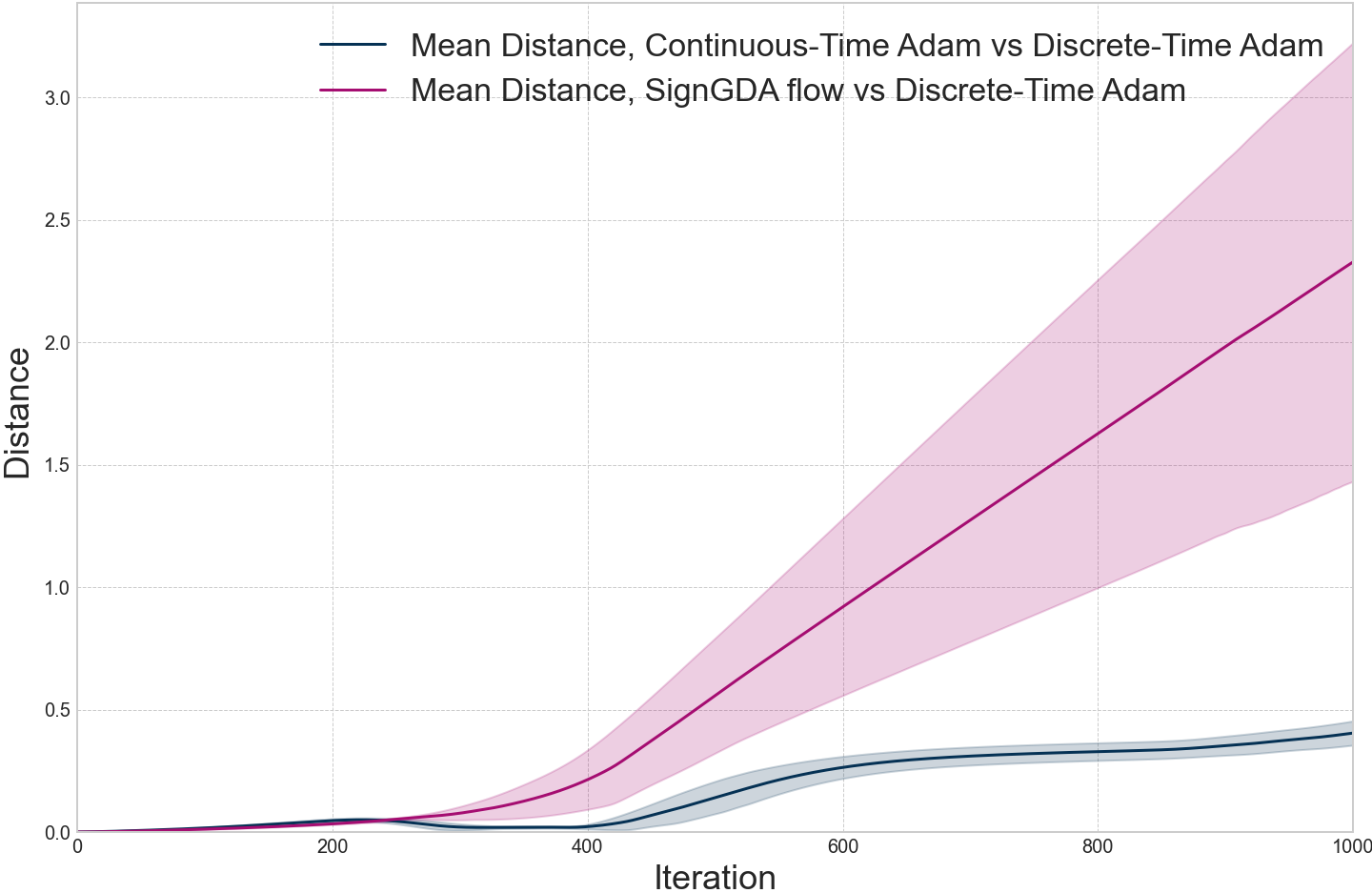}
        \label{ap6}
    }
    \caption{\textit{Figures \ref{ap4}, \ref{ap5}, and \ref{ap6} show the distances of two continuous-time models between Adam, with results averaged over 30 random initial conditions. }}
    \label{comparewithsign2}
\end{figure}

In the following Figure \ref{comparewithsign3}, we present additional experiments with different choice of parameters. The initial points are are chosen to be $(0.3,-0.3)$ in figures of trajectories. For the left figure, $\beta=0.5,\rho=0.8,\epsilon=10^{-6}, h=0.007$. For the middle figure,  $\beta=0.9,\rho=0.8,\epsilon=10^{-3}, h=0.002$. For the right figure, $\beta=-0.1,\rho=0.6,\epsilon=10^{-4}, h=0.005$

\begin{figure}[h]
    \centering
    \subfigure{
        \includegraphics[width=1.6in]{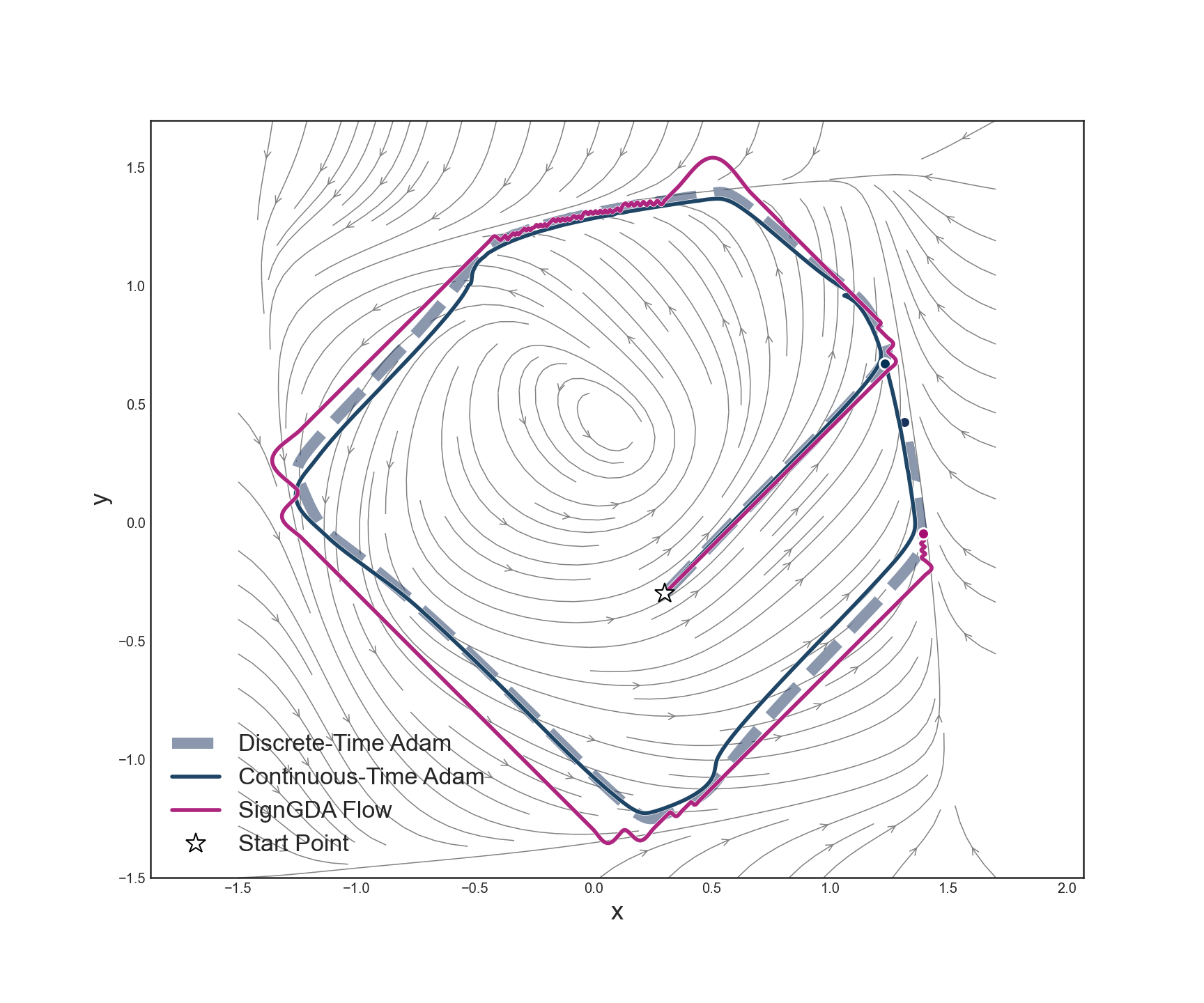}
        \label{ap11}
    }
    \subfigure{
        \includegraphics[width=1.6in]{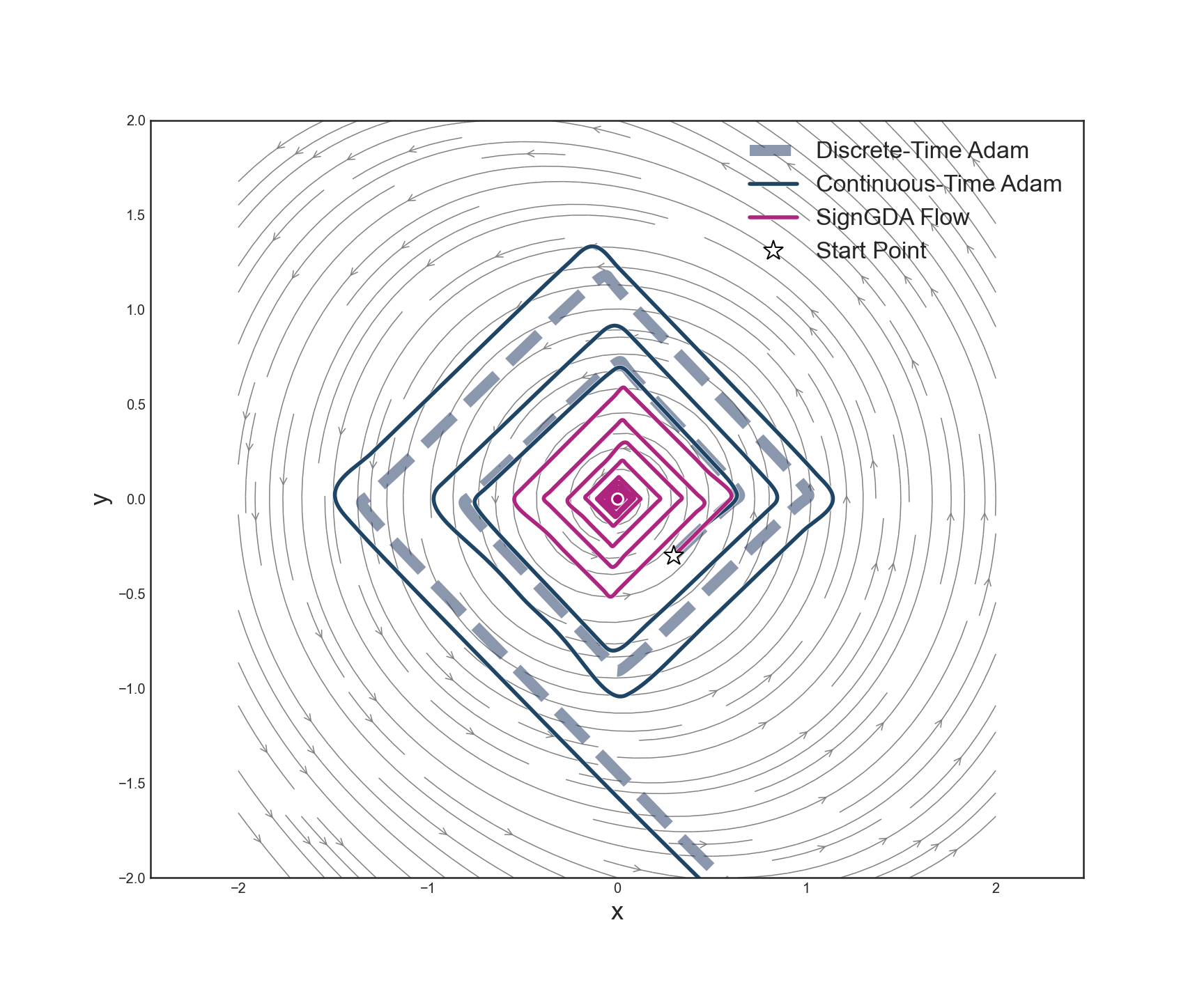}
        \label{ap21}
    }
    \subfigure{
        \includegraphics[width=1.6in]{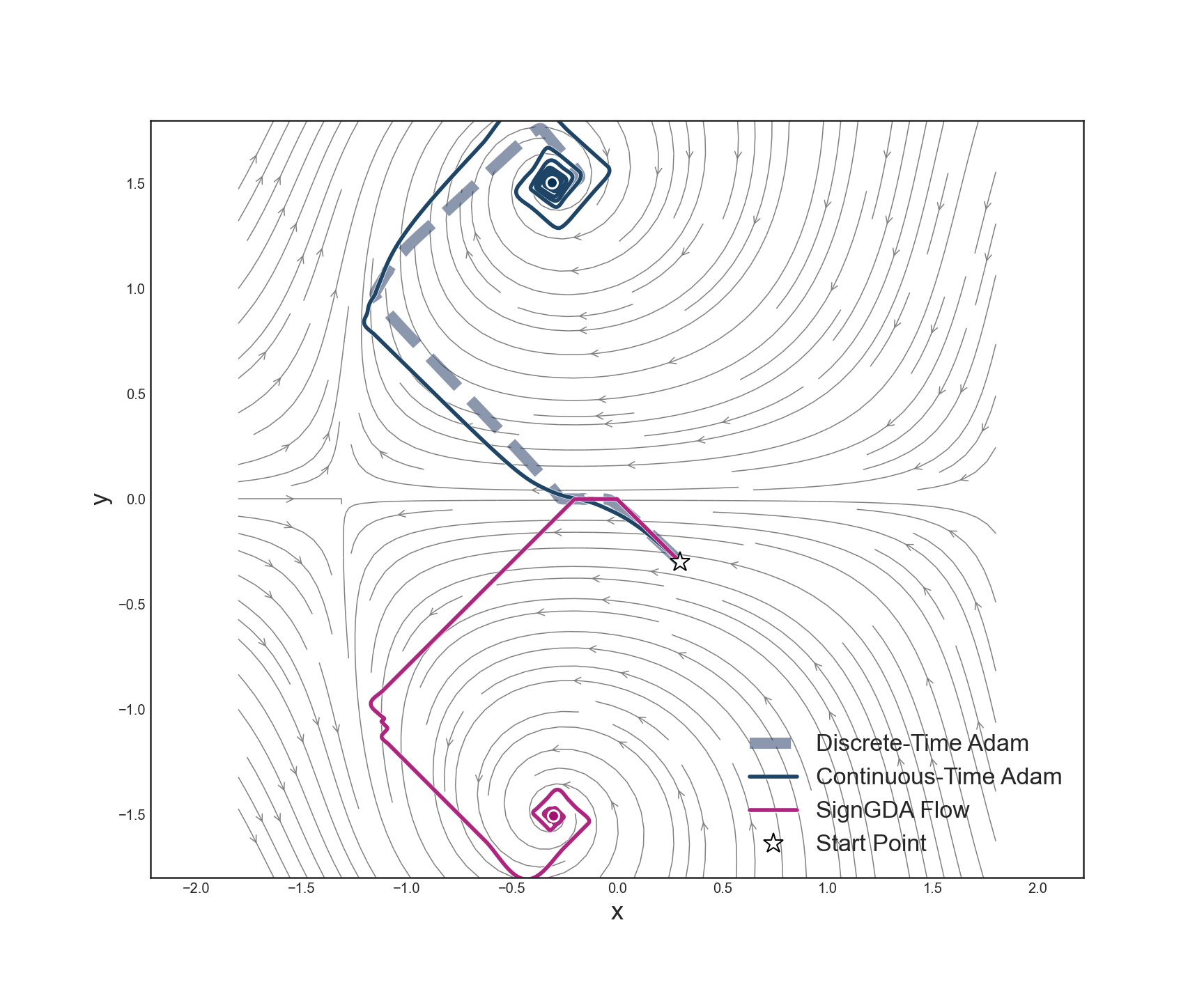}
        \label{ap31}
    }
    \subfigure[[Distance Curves of $f_1$]{
	\includegraphics[width=1.6in]{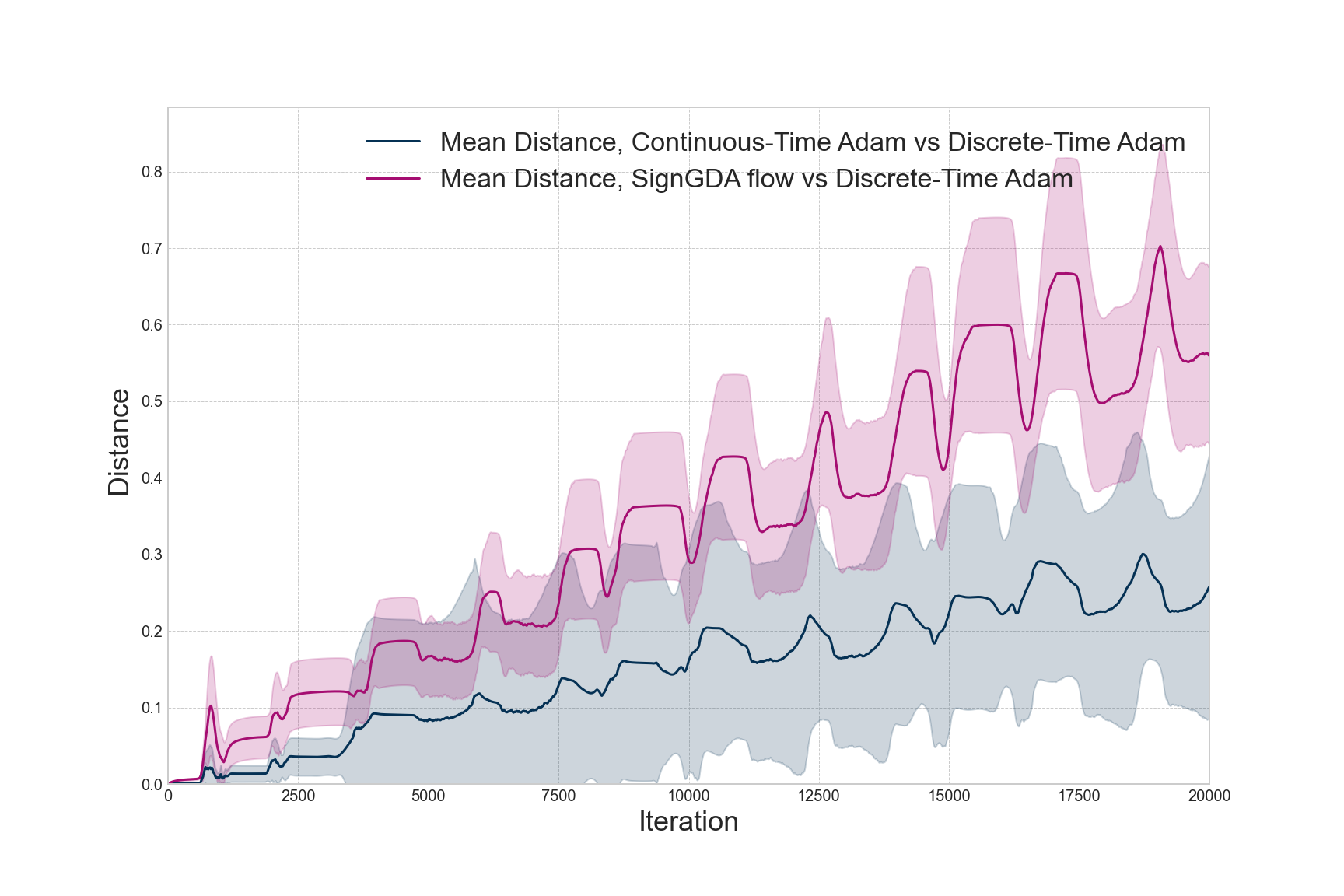}
        \label{ap41}
    }
    \subfigure[[Distance Curves of $f_2$]{
	\includegraphics[width=1.6in]{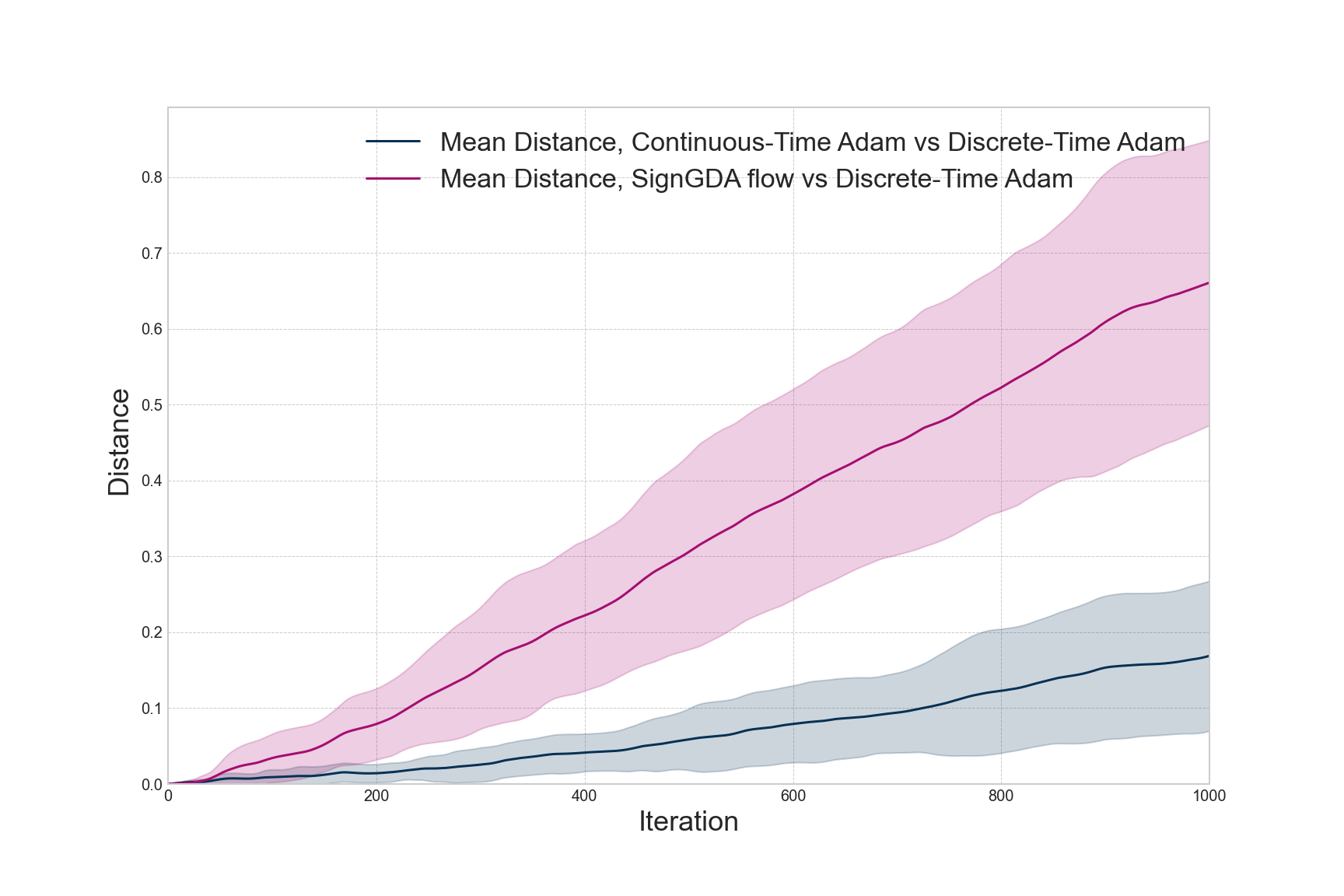}
        \label{ap51}
    }
    \subfigure[[Distance Curves of $f_3$]{
	\includegraphics[width=1.6in]{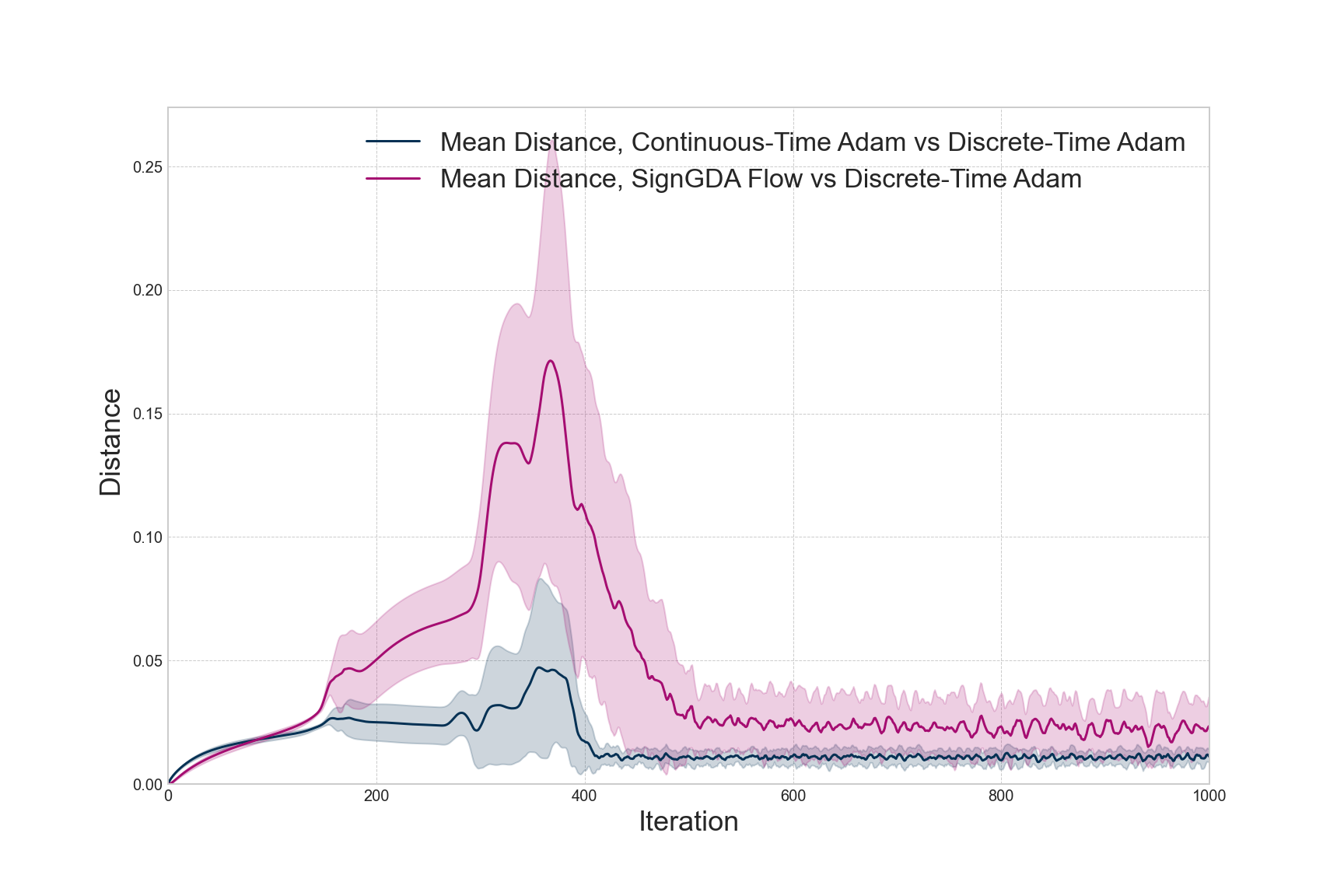}
        \label{ap61}
    }
    \caption{\textit{Additional experiments with different parameters.}}
    \label{comparewithsign3}
\end{figure}

\subsection{Proof of Theorem \ref{error_ana}}\label{proofofthm1}
\textbf{Notation Statement.} In Appendix \ref{proofofthm1}, component index and iteration index are used in many quantities. To avoid ambiguity, we use superscripts as coordinate indices (e.g., $\tx^{(j)}$) and subscripts as iteration indices (e.g., $A_n(\tx, \ty)$). The derivation of ODEs in Appendix \ref{proofofthm1} is written by coordinate-wise, which is the same as the matrix form in Theorem \ref{error_ana}.

Before providing the formal proof of Theorem \ref{error_ana}, we provide the sketched proof of Theorem \ref{error_ana} consists of several steps:
\begin{itemize}[leftmargin=*]
    \item Step 1: In Lemma \ref{lemma: ODE for Sim-Adam}, we derive a family of ODEs that each ODE corresponds to a time region $[t_n, t_{n+1})$ with length $h$ ($h$ is the stepsize of \ref{Adam}). The solutions of these ODEs achieves good approximation of \ref{Adam} with local error $\CO(h^3)$. It should be pointed out that these ODEs may be not differentiable at time nodes and they are dependent on $n$, which hinder the analysis.   
    \item Step 2: In Lemma \ref{lemma: local error between ODEs}, we prove that the family of ODEs in Step 1 will converge to the single equation \ref{CADAM} with an exponential rate. Due to the exponential rate, after a short time region, we can bound the local error between the family of ODEs and \ref{CADAM} by $\CO(h^3)$. The single ODE \ref{CADAM} is differentiable everywhere and independent of $n$, which can simplify the analysis of the parameters' behaviors in \ref{minmaxprel}.
\end{itemize}
    
% \end{proof}

% We proceed to present the following lemma. We will use Lemma \ref{lemma: 4th order bounded derivatives} to safely truncate the Taylor-expansion to get an $\CO(h^3)$ error.
% \begin{lem}\label{lemma: 4th order bounded derivatives}
%  Suppose that $f(\tx,\ty)$ has bounded derivatives up to the fourth order. Then for any bounded $h$, we can get $\dot{\tx}^{(j)}(t)$, $\ddot{\tx}^{(j)}(t)$, $\dddot{\tx}^{(j)}(t)$ and $\dot{\ty}^{(i)}(t)$, $\ddot{\ty}^{(i)}(t)$, $\dddot{\ty}^{(i)}(t)$ are bounded independent of $h$ for all $j=1,2,\cdots, d_1$ and $i=1,2,\cdots, d_2$.   
% \end{lem}
% \begin{proof}
%     The proof follows directly by the definition of $\dot{\tx}^{(j)}(t)$, $\ddot{\tx}^{(j)}(t)$, $\dddot{\tx}^{(j)}(t)$ and $\dot{\ty}^{(i)}(t)$, $\ddot{\ty}^{(i)}(t)$, $\dddot{\ty}^{(i)}(t)$ together with bounded derivatives of $f(x)$ up to the fourth order.
% \end{proof}

Define 
\begin{align}\label{equ: definition in full batch setting}
    \begin{split}
     &N_n^{\tx^{(j)}}(\tx(t_n), \ty(t_n))=\frac{\partial f(\tx(t_n), \ty(t_n))}{\partial \tx^{(j)}}, \  D_n^{\tx^{(j)}}(\tx(t_n), \ty(t_n))=\sqrt{\left(\frac{\partial f(\tx(t_n), \ty(t_n))}{\partial \tx^{(j)}}\right)^2+\epsilon},\\
     \\
     &N_n^{\ty^{(i)}}(\tx(t_n), \ty(t_n))=\frac{\partial f(\tx(t_n), \ty(t_n))}{\partial \ty^{(i)}},\  D_n^{\ty^{(i)}}(\tx(t_n), \ty(t_n))=\sqrt{\left(\frac{\partial f(\tx(t_n), \ty(t_n))}{\partial \ty^{(i)}}\right)^2+\epsilon},\\
     \\
     &Q_n^{\tx^{(j)}}(\tx(t_n), \ty(t_n))\\
     &=\left(\frac{\beta}{1-\beta}-\frac{(n+1)\beta^{n+1}}{1-\beta^{n+1}}\right)\left[\sum_{s=1}^{d_1}\frac{\partial^2 f(\tx(t_n), \ty(t_n))}{\partial \tx^{(j)} \partial \tx^{(s)}}\frac{N_n^{\tx^{(s)}}}{D_n^{\tx^{(s)}}}-\sum_{s=1}^{d_2}\frac{\partial^2 f(\tx(t_n), \ty(t_n))}{\partial \tx^{(j)} \partial \ty^{(s)}}\frac{N_n^{\ty^{(s)}}}{D_n^{\ty^{(s)}}}\right],\\
     \\
     &Q_n^{\ty^{(i)}}(\tx(t_n), \ty(t_n))\\
     &=\left(\frac{\beta}{1-\beta}-\frac{(n+1)\beta^{n+1}}{1-\beta^{n+1}}\right)\left[\sum_{s=1}^{d_1}\frac{\partial^2 f(\tx(t_n), \ty(t_n))}{\partial \ty^{(i)} \partial \tx^{(s)}}\frac{N_n^{\tx^{(s)}}}{D_n^{\tx^{(s)}}}-\sum_{s=1}^{d_2}\frac{\partial^2 f(\tx(t_n), \ty(t_n))}{\partial \ty^{(i)} \partial \ty^{(s)}}\frac{N_n^{\ty^{(s)}}}{D_n^{\ty^{(s)}}}\right],\\
     \\
     &P_n^{\tx^{(j)}}(\tx(t_n), \ty(t_n))=\left(\frac{\rho}{1-\rho}-\frac{(n+1)\rho^{n+1}}{1-\rho^{n+1}}\right)\left[\sum_{s=1}^{d_1}\frac{\partial^2 f(\tx(t_n), \ty(t_n))}{\partial \tx^{(j)} \partial \tx^{(s)}}\frac{\partial f(\tx(t_n), \ty(t_n))}{\partial \tx^{(j)}}\frac{N_n^{\tx^{(s)}}}{D_n^{\tx^{(s)}}} \right.\\&\left.
          \ \ \ \ \ \ \ \ \ \ \ \ \ \ \ \ \ \ \ \ \ \ \ \ \ \ \ \ \ \ \ \ \ \ \ \ \ \ \ \ \ \ \ \ \ \ \ \ \ \ \ \ \ \ \ \ \ \ \ \ \ \ \ \ \ \ \ \ \ \ \ \ \ \ \ \ \ \ \ \
          -\sum_{s=1}^{d_2}\frac{\partial^2 f(\tx(t_n), \ty(t_n))}{\partial \tx^{(j)} \partial \ty^{(s)}}\frac{\partial f(\tx(t_n), \ty(t_n))}{\partial \tx^{(j)}}\frac{N_n^{\ty^{(s)}}}{D_n^{\ty^{(s)}}}\right],\\
     \\
     % &N_n^{\ty^{(i)}}(\tx(t_n), \ty(t_n))=\frac{\partial f(\tx(t_n), \ty(t_n))}{\partial \ty^{(i)}},\\ 
     % &D_n^{\ty^{(i)}}(\tx(t_n), \ty(t_n))=\sqrt{\left(\frac{\partial f(\tx(t_n), \ty(t_n))}{\partial \ty^{(i)}}\right)^2+\epsilon},\\
     &P_n^{\ty^{(i)}}(\tx(t_n), \ty(t_n))=\left(\frac{\rho}{1-\rho}-\frac{(n+1)\rho^{n+1}}{1-\rho^{n+1}}\right)\left[\sum_{s=1}^{d_1}\frac{\partial^2 f(\tx(t_n), \ty(t_n))}{\partial \ty^{(i)} \partial \tx^{(s)}}\frac{\partial f(\tx(t_n), \ty(t_n))}{\partial \ty^{(i)}}\frac{N_n^{\tx^{(s)}}}{D_n^{\tx^{(s)}}} \right.\\&\left.
     \ \ \ \ \ \ \ \ \ \ \ \ \ \ \ \ \ \ \ \ \ \ \ \ \ \ \ \ \ \ \ \ \ \ \ \ \ \ \ \
     \ \ \ \ \ \ \ \ \ \ \ \ \ \ \ \ \ \ \ \ \ \ \ \ \ \ \ \ \ \ \ \ \ \ \ \ \ \ \ \
     -\sum_{s=1}^{d_2}\frac{\partial^2 f(\tx(t_n), \ty(t_n))}{\partial \ty^{(i)} \partial \ty^{(s)}}\frac{\partial f(\tx(t_n), \ty(t_n))}{\partial \ty^{(i)}}\frac{N_n^{\ty^{(s)}}}{D_n^{\ty^{(s)}}}\right].\\
    \end{split}
\end{align}

To simplify the notation, we denote $N_n^{\tx^{(j)}}(\tx(t_n), \ty(t_n))$ by $N_n^{\tx^{(j)}}$ in the following. Similarly, the same convention applies to $D_n^{\tx^{(j)}}$, $N_n^{\ty^{(i)}}$, $D_n^{\ty^{(i)}}$, $Q_n^{x^{(j)}}$, $P_n^{x^{(j)}}$, $Q_n^{y^{(i)}}$ and $P_n^{y^{(i)}}$.

%(resp. $D_n^{\tx^{(j)}}(\tx(t_n), \ty(t_n))$, $N_n^{\ty^{(i)}}(\tx(t_n), \ty(t_n))$ and $D_n^{\ty^{(i)}}(\tx(t_n), \ty(t_n))$) by $N_n^{\tx^{(j)}}$ (resp. $D_n^{\tx^{(j)}}$, $N_n^{\ty^{(i)}}$ and $D_n^{\ty^{(i)}}$).

We proceed to present the following lemma. We will use Lemma \ref{lemma: 4th order bounded derivatives} to safely truncate the Taylor-expansion to get an $\CO(h^3)$ error.
\begin{lem}\label{lemma: 4th order bounded derivatives}
 Suppose that $f(\tx,\ty)$ has bounded derivatives up to the fourth order. Then for any bounded $h$, we can get $\dot{\tx}^{(j)}(t)$, $\ddot{\tx}^{(j)}(t)$, $\dddot{\tx}^{(j)}(t)$ and $\dot{\ty}^{(i)}(t)$, $\ddot{\ty}^{(i)}(t)$, $\dddot{\ty}^{(i)}(t)$ are bounded independent of $h$ for all $j=1,2,\cdots, d_1$ and $i=1,2,\cdots, d_2$.   
\end{lem}
\begin{proof}
    The proof follows directly by the definition of $\dot{\tx}^{(j)}(t)$, $\ddot{\tx}^{(j)}(t)$, $\dddot{\tx}^{(j)}(t)$ and $\dot{\ty}^{(i)}(t)$, $\ddot{\ty}^{(i)}(t)$, $\dddot{\ty}^{(i)}(t)$ together with bounded derivatives of $f(x)$ up to the fourth order.
\end{proof}

\begin{lem}\label{lemma: ODE for Sim-Adam}
    Let $t_n=nh$ where $h$ is the stepsize of \ref{Adam}. Then in time region $[t_n, t_{n+1})$, the solutions of the following ODEs:
    \begin{align*}
    &\frac{d \tx^{(j)}(t)}{dt}\\
    &=\frac{-\frac{\partial f(\tx, \ty)}{\partial \tx^{(j)}}}{\sqrt{\left(\frac{\partial f(\tx, \ty)}{\partial \tx^{(j)}}\right)^2+\epsilon}}-\frac{h}{2\sqrt{\left(\frac{\partial f(\tx, \ty)}{\partial \tx^{(j)}}\right)^2+\epsilon}}\times\left(\frac{\partial}{\partial \tx^{(j)}}\|\nabla_{x} f(\tx, \ty)\|_{1,\epsilon}-\frac{\partial}{\partial \tx^{(j)}}\|\nabla_{y} f(\tx, \ty)\|_{1,\epsilon}\right)\\
     &\times\left[\frac{1+\beta}{1-\beta}-\frac{1+\rho}{1-\rho}+\frac{\epsilon}{\left(\frac{\partial f(\tx, \ty)}{\partial \tx^{(j)}}\right)^2+\epsilon}\left(\frac{1+\rho}{1-\rho}-\frac{(n+1)\rho^{n+1}}{1-\rho^{n+1}}\right)-\frac{(n+1)\beta^{n+1}}{1-\beta^{n+1}}+\frac{(n+1)\rho^{n+1}}{1-\rho^{n+1}}\right],\\
     \\
    &\frac{d \ty^{(i)}(t)}{dt}\\
    &=\frac{\frac{\partial f(\tx, \ty)}{\partial \ty^{(i)}}}{\sqrt{\left(\frac{\partial f(\tx, \ty)}{\partial \ty^{(i)}}\right)^2+\epsilon}}+\frac{h}{2\sqrt{\left(\frac{\partial f(\tx, \ty)}{\partial \ty^{(i)}}\right)^2+\epsilon}}\times\left(\frac{\partial}{\partial \ty^{(i)}}\|\nabla_{x} f(\tx, \ty)\|_{1,\epsilon}-\frac{\partial}{\partial \ty^{(i)}}\|\nabla_{y} f(\tx, \ty)\|_{1,\epsilon}\right)\\
    &\times\left[\frac{1+\beta}{1-\beta}-\frac{1+\rho}{1-\rho}+\frac{\epsilon}{\left(\frac{\partial f(\tx, \ty)}{\partial \ty^{(i)}}\right)^2+\epsilon}\left(\frac{1+\rho}{1-\rho}-\frac{(n+1)\rho^{n+1}}{1-\rho^{n+1}}\right)-\frac{(n+1)\beta^{n+1}}{1-\beta^{n+1}}+\frac{(n+1)\rho^{n+1}}{1-\rho^{n+1}}\right].
  \end{align*}
can approximate the trajectories of \ref{Adam} with an $\CO(h^3)$-local error.
\end{lem}

\begin{proof}
We rewrite \ref{Adam} as
\begin{align*}
    \begin{split}
    &\tx_{n+1}^{(j)}=\tx_{n}^{(j)}-h\frac{\sum_{k=0}^n\frac{\beta^{n-k}(1-\beta)}{1-\beta^{n+1}}\frac{\partial f(\tx_k, \ty_k)}{\partial \tx^{(j)}}}{\sqrt{\sum_{k=0}^n\frac{\rho^{n-k}(1-\rho)}{1-\rho^{n+1}}\left(\frac{\partial f(\tx_k, \ty_k)}{\partial \tx^{(j)}}\right)^2+\epsilon}}, \quad j=1,\cdots, d_1,\\
    \\&\ty^{(i)}_{n+1}=\ty_{n}^{(i)}+h\frac{\sum_{k=0}^n\frac{\beta^{n-k}(1-\beta)}{1-\beta^{n+1}}\frac{\partial f(\tx_k, \ty_k)}{\partial \ty^{(i)}}}{\sqrt{\sum_{k=0}^n\frac{\rho^{n-k}(1-\rho)}{1-\rho^{n+1}}\left(\frac{\partial f(\tx_k, \ty_k)}{\partial \ty^{(i)}}\right)^2+\epsilon}}, \quad i=1,\cdots, d_2.
    \end{split}
\end{align*}
In each time interval $[t_n, t_{n+1})$, we  define 
\begin{itemize}
\item $\nabla_{x} G_n(\tx, \ty)=\left(\frac{\partial G_n(\tx, \ty)}{\partial \tx^{(1)}}, \cdots,\frac{\partial G_n(\tx, \ty)}{\partial \tx^{(j)}},\cdots, \frac{\partial G_n(\tx, \ty)}{\partial \tx^{(d_1)}}\right)$
\item $\nabla_{y} F_n(\tx, \ty)=\left(\frac{\partial F_n(\tx, \ty)}{\partial \ty^{(1)}}, \cdots,\frac{\partial F_n(\tx, \ty)}{\partial \ty^{(i)}},\cdots, \frac{\partial F_n(\tx, \ty)}{\partial \ty^{(d_2)}}\right)$
\item $A_n(\tx, \ty)=\left(A_n^{(1)}(\tx, \ty),\cdots, A_n^{(j)}(\tx, \ty), \cdots, A_n^{(d_1)}(\tx, \ty)\right)$
\item $B_n(\tx, \ty)=\left(B_n^{(1)}(\tx, \ty),\cdots, B_n^{(i)}(\tx, \ty), \cdots, B_n^{(d_2)}(\tx, \ty)\right)$
\end{itemize}
where the terms $G_n(\tx, \ty), F_n(\tx, \ty), A_n^{(j)}(\tx, \ty)$ and $B_n^{(i)}(\tx, \ty)$ are assumed to satisfy the following equations in the time interval $[t_n, t_{n+1})$:
    \begin{align}
            &\frac{d \tx^{(j)}(t)}{dt}=-\frac{\partial G_n(\tx, \ty)}{\partial \tx^{(j)}}+hA_n^{(j)}(\tx, \ty), \label{equ: undetermined coefficients equation-1}\\
            &\frac{d \ty^{(i)}(t)}{dt}=\frac{\partial F_n(\tx, \ty)}{\partial \ty^{(i)}}+hB_n^{(i)}(\tx, \ty). \label{equ: undetermined coefficients equation-2}
    \end{align}
    Our goal is to find $G_n$, $F_n$, $A_n$ and $B_n$ such that 
 \begin{align}
         &\tx^{(j)}(t_{n+1})-\tx^{(j)}(t_n)=\frac{-h\sum_{k=0}^n\frac{\beta^{n-k}(1-\beta)}{1-\beta^{n+1}}\frac{\partial f(\tx(t_k), \ty(t_k))}{\partial \tx^{(j)}}}{\sqrt{\sum_{k=0}^n\frac{\rho^{n-k}(1-\rho)}{1-\rho^{n+1}}\left(\frac{\partial f(\tx(t_k), \ty(t_k))}{\partial \tx^{(j)}}\right)^2+\epsilon}}+\CO(h^3), \label{equ:goal-1}\\
         \nonumber\\
         &\ty^{(i)}(t_{n+1})-\ty^{(i)}(t_n)=\frac{h\sum_{k=0}^n\frac{\beta^{n-k}(1-\beta)}{1-\beta^{n+1}}\frac{\partial f(\tx(t_k), \ty(t_k))}{\partial \ty^{(i)}}}{\sqrt{\sum_{k=0}^n\frac{\rho^{n-k}(1-\rho)}{1-\rho^{n+1}}\left(\frac{\partial f(\tx(t_k), \ty(t_k))}{\partial \ty^{(i)}}\right)^2+\epsilon}}+\CO(h^3). \label{equ: goal-2}
 \end{align} 
 % Firstly, we solve $\frac{\partial G_n(\tx, \ty)}{\partial \tx^{(j)}}$ and $\frac{\partial F_n(\tx, \ty)}{\partial \ty^{(i)}}$.
 We firstly do the Taylor expansions of $\frac{\partial f(\tx(t_k), \ty(t_k))}{\partial \tx^{(j)}}$ as follows:
 \begin{equation}\label{equ:taylor expansion of partial x^{(j)}}
 \begin{split}
\frac{\partial f(\tx(t_k), \ty(t_k))}{\partial \tx^{(j)}} & =\frac{\partial f(\tx(t_n), \ty(t_n))}{\partial \tx^{(j)}}+ \sum_{s=1}^{d_1}\frac{\partial^2 f(\tx(t_n), \ty(t_n))}{\partial \tx^{(j)} \partial \tx^{(s)}}\left(\tx^{(s)}(t_k)-\tx^{(s)}(t_n)\right)\\
\\
& \ \ \ \ \ \ \ \ \ \ \ \ \ \ \ \ \ \ \ \ \ \ \ \ \ \ \ \ \ \ \ \ \ \ +\sum_{s=1}^{d_2}\frac{\partial^2 f(\tx(t_n), \ty(t_n))}{\partial \tx^{(j)} \partial \ty^{(s)}}\left(\ty^{(s)}(t_k)-\ty^{(s)}(t_n)\right)+\CO\left(h^2\right)\\
\\
&=\frac{\partial f(\tx(t_n), \ty(t_n))}{\partial \tx^{(j)}}-\sum_{s=1}^{d_1}\frac{\partial^2 f(\tx(t_n), \ty(t_n))}{\partial \tx^{(j)} \partial \tx^{(s)}}\dot{\tx}^{(s)}(t_n^+)\left(n-k\right)h\\
\\
& \ \ \ \ \ \ \ \ \ \ \ \ \ \ \ \ \ \ \ \ \ \ \ \ \ \ \ \ \ \ \ \ \ \ 
-\sum_{s=1}^{d_2}\frac{\partial^2 f(\tx(t_n), \ty(t_n))}{\partial \tx^{(j)} \partial \ty^{(s)}}\dot{\ty}^{(s)}(t_n^+)\left(n-k\right)h+\CO\left(h^2\right),\\
\\
&=\frac{\partial f(\tx(t_n), \ty(t_n))}{\partial \tx^{(j)}}+\sum_{s=1}^{d_1}\frac{\partial^2 f(\tx(t_n), \ty(t_n))}{\partial \tx^{(j)} \partial \tx^{(s)}}\frac{\partial G_n\left(\tx(t_n), \ty(t_n)\right)}{\partial \tx^{(s)}}\left(n-k\right)h\\
\\
& \ \ \ \ \ \ \ \ \ \ \ \ \ \ \ \ \ \ \ \ \ \ \ \ \ \ \ \ \ \ \ \ \ \ 
-\sum_{s=1}^{d_2}\frac{\partial^2 f(\tx(t_n), \ty(t_n))}{\partial \tx^{(j)} \partial \ty^{(s)}}\frac{\partial F_n\left(\tx(t_n), \ty(t_n)\right)}{\partial \ty^{(s)}}\left(n-k\right)h+\CO\left(h^2\right),
\end{split}    
\end{equation}
 then we obtain that
 \begin{equation}\label{equ:sim-substitute 1}
    \begin{split}
 \sum_{k=0}^n\frac{\beta^{n-k}(1-\beta)}{1-\beta^{n+1}}&\frac{\partial f(\tx(t_k), \ty(t_k))}{\partial \tx^{(j)}}=\sum_{k=0}^n\frac{\beta^{n-k}(1-\beta)}{1-\beta^{n+1}}\frac{\partial f(\tx(t_n), \ty(t_n))}{\partial \tx^{(j)}}\\
 \\
 & +h\sum_{k=0}^n\frac{\beta^{n-k}(1-\beta)\left(n-k\right)}{1-\beta^{n+1}}\sum_{s=1}^{d_1}\frac{\partial^2 f(\tx(t_n), \ty(t_n))}{\partial \tx^{(j)} \partial \tx^{(s)}}\frac{\partial G_n\left(\tx(t_n), \ty(t_n)\right)}{\partial \tx^{(s)}}\\
 \\
& -h\sum_{k=0}^n\frac{\beta^{n-k}(1-\beta)\left(n-k\right)}{1-\beta^{n+1}}\sum_{s=1}^{d_2}\frac{\partial^2 f(\tx(t_n), \ty(t_n))}{\partial \tx^{(j)} \partial \ty^{(s)}}\frac{\partial F_n\left(\tx(t_n), \ty(t_n)\right)}{\partial \ty^{(s)}}+\CO\left(h^2\right),
     \end{split}
 \end{equation}
 Executing square operation in the both sides of (\ref{equ:taylor expansion of partial x^{(j)}}), we can obtain
 \begin{equation}
\begin{split}
\left(\frac{\partial f(\tx(t_k), \ty(t_k))}{\partial \tx^{(j)}}\right)^2  &=\left(\frac{\partial f(\tx(t_n), \ty(t_n))}{\partial \tx^{(j)}}\right)^2\\
\\
& +2h\sum_{s=1}^{d_1}\left(n-k\right)\frac{\partial^2 f(\tx(t_n), \ty(t_n))}{\partial \tx^{(j)} \partial \tx^{(s)}}\frac{\partial f(\tx(t_n), \ty(t_n))}{\partial \tx^{(j)}}\frac{\partial G_n\left(\tx(t_n), \ty(t_n)\right)}{\partial \tx^{(s)}}\\
\\
&-2h\sum_{s=1}^{d_2}\left(n-k\right)\frac{\partial^2 f(\tx(t_n), \ty(t_n))}{\partial \tx^{(j)} \partial \ty^{(s)}}\frac{\partial f(\tx(t_n), \ty(t_n))}{\partial \tx^{(j)}}\frac{\partial F_n\left(\tx(t_n), \ty(t_n)\right)}{\partial \ty^{(s)}}+\CO\left(h^2\right).
     \end{split}
 \end{equation}
 Then we can  get
 \begin{equation}
     \begin{split}
  &\sum_{k=0}^n\frac{\rho^{n-k}(1-\rho)}{1-\rho^{n+1}}\left(\frac{\partial f(\tx(t_k), \ty(t_k))}{\partial \tx^{(j)}}\right)^2 \\
  \\
  &=\sum_{k=0}^n\frac{\rho^{n-k}(1-\rho)}{1-\rho^{n+1}}\left(\frac{\partial f(\tx(t_n), \ty(t_n))}{\partial \tx^{(j)}}\right)^2\\
  \\
  &\ \ \ \ \ \ \ \ \ +2h\sum_{k=0}^n\frac{\rho^{n-k}(1-\rho)\left(n-k\right)}{1-\rho^{n+1}}\sum_{s=1}^{d_1}\frac{\partial^2 f(\tx(t_n), \ty(t_n))}{\partial \tx^{(j)} \partial \tx^{(s)}}\frac{\partial f(\tx(t_n), \ty(t_n))}{\partial \tx^{(j)}}\frac{\partial G_n\left(\tx(t_n), \ty(t_n)\right)}{\partial \tx^{(s)}}\\
  \\
&\ \ \ \ \ \ \ \ \ -2h\sum_{k=0}^n\frac{\rho^{n-k}(1-\rho)\left(n-k\right)}{1-\rho^{n+1}}\sum_{s=1}^{d_2}\frac{\partial^2 f(\tx(t_n), \ty(t_n))}{\partial \tx^{(j)} \partial \ty^{(s)}}\frac{\partial f(\tx(t_n), \ty(t_n))}{\partial \tx^{(j)}}\frac{\partial F_n\left(\tx(t_n), \ty(t_n)\right)}{\partial \ty^{(s)}}\\
\\
&\ \ \ \ \ \ \ \ \ +\CO\left(h^2\right).   
     \end{split}
 \end{equation}
 Then we use the fact that $\left(\sum_{k=0}^na_kh^k\right)^{-\frac{1}{2}}=\frac{1}{\sqrt{a_0}}-\frac{a_1}{2(\sqrt{a_0})^3}h+O(h^2)$ to get
 \begin{equation}\label{equ: sim-substitute-2}
     \begin{split}
    &\frac{1}{\sqrt{\sum_{k=0}^n\frac{\rho^{n-k}(1-\rho)}{1-\rho^{n+1}}\left(\frac{\partial f(\tx(t_k), \ty(t_k))}{\partial \tx^{(j)}}\right)^2+\epsilon}}=\frac{1}{D_n^{\tx^{(j)}}}+\frac{hS_n^{\tx^{(j)}}}{\left(D_n^{\tx^{(j)}}\right)^3}+\CO\left(h^2\right),     
\end{split}
 \end{equation}
 where
 \begin{align*}
     S_n^{\tx^{(j)}}=\sum_{k=0}^n\frac{\rho^{n-k}(1-\rho)\left(n-k\right)}{1-\rho^{n+1}}\left(\sum_{s=1}^{d_1}\frac{\partial^2 f(\tx(t_n), \ty(t_n))}{\partial \tx^{(j)} \partial \tx^{(s)}}\frac{\partial f(\tx(t_n), \ty(t_n))}{\partial \tx^{(j)}}\frac{\partial G_n\left(\tx(t_n), \ty(t_n)\right)}{\partial \tx^{(s)}}\right. \\ \\ \left.-\sum_{s=1}^{d_2}\frac{\partial^2 f(\tx(t_k), \ty(t_k))}{\partial \tx^{(j)} \partial \ty^{(s)}}\frac{\partial f(\tx(t_k), \ty(t_k))}{\partial \tx^{(j)}}\frac{\partial F_n\left(\tx(t_n), \ty(t_n)\right)}{\partial \ty^{(s)}}\right),
 \end{align*} and we have used the equation
 \[
 \sum_{k=0}^n\frac{\rho^{n-k}(1-\rho)}{1-\rho^{n+1}}\left(\frac{\partial f(\tx(t_n), \ty(t_n))}{\partial \tx^{(j)}}\right)^2=\left(\frac{\partial f(\tx(t_n), \ty(t_n))}{\partial \tx^{(j)}}\right)^2.
 \]
Firstly, we solve $\frac{\partial G_n(\tx, \ty)}{\partial \tx^{(j)}}$ and $\frac{\partial F_n(\tx, \ty)}{\partial \ty^{(i)}}$.
Substituting (\ref{equ:sim-substitute 1}) and (\ref{equ: sim-substitute-2}) into (\ref{equ:goal-1}), we can get
\begin{equation}\label{equ: sim-Adam final form}
    \begin{split}
&\tx^{(j)}(t_{n+1})-\tx^{(j)}(t_n)=-h\frac{N_n^{\tx^{(j)}}}{D_n^{\tx^{(j)}}}+\CO(h^2)        
    \end{split}
\end{equation}
 
 Since $\dot{\tx}^{(j)}(t_n^+)$, $\ddot{\tx}^{(j)}(t_n^+)$ and $\dddot{\tx}^{(j)}(t_n^+)$ are bounded independent of $h$ (by Lemma \ref{lemma: 4th order bounded derivatives}), we can obtain the Taylor expansion at time $t_n$ (recall that $t_n=nh$) 
 \begin{align}\label{equ: taylor expansion}
     \tx^{(j)}(t_{n+1})-\tx^{(j)}(t_n)&=\tx^{(j)}(t_{n}+h)-\tx^{(j)}(t_n)\nonumber\\
     &=h\dot{\tx}^{(j)}(t_n^+)+\frac{h^2}{2}\ddot{\tx}^{(j)}(t_n^+)+\CO(h^3),
 \end{align}
 where $\dot{\tx}^{(j)}(t_n^+)$ (resp. $\ddot{\tx}^{(j)}(t_n^+)$) is the right first-order derivative (resp. the right second-order derivative) at time $t_n$. 

 Applying the chain rule for (\ref{equ: undetermined coefficients equation-1}), we have
 \begin{equation}\label{equ: second-order right derivative}
     \begin{split}
 &\ddot{\tx}^{(j)}(t_n^+)\\
 \\
     &=-\sum_{s=1}^{d_1}\frac{\partial^2G_n(\tx(t_n), \ty(t_n))}{\partial \tx^{(s)}\partial \tx^{(j)}}\dot{\tx}^{(s)}(t_n^+)-\sum_{s=1}^{d_2}\frac{\partial^2G_n(\tx(t_n), \ty(t_n))}{\partial \tx^{(j)} \partial \ty^{(s)}}\dot{\ty}^{(s)}(t_n^+)+\CO(h)\\
    \\ &=\sum_{s=1}^{d_1}\frac{\partial^2G_n(\tx(t_n), \ty(t_n))}{\partial \tx^{(s)}\partial \tx^{(j)}}\frac{\partial G_n(\tx(t_n), \ty(t_n))}{\partial \tx^{(s)}}\\
    \\
     &\ \ \ \ \ \ \ \ \ \ \ \ \ -\sum_{s=1}^{d_2}\frac{\partial^2G_n(\tx(t_n), \ty(t_n))}{\partial \tx^{(j)} \partial \ty^{(s)}}\frac{\partial F_n(\tx(t_n), \ty(t_n))}{\partial \ty^{(s)}}+\CO(h).        
     \end{split}
 \end{equation}
 Also note that from (\ref{equ: undetermined coefficients equation-1}), we have
 \begin{align}\label{equ: first-order right derivative}
    \dot{\tx}^{(j)}(t_n^+)=-\frac{\partial G_n(\tx(t_n), \ty(t_n))}{\partial \tx^{(j)}}+hA_n^{(j)}(\tx(t_n), \ty(t_n)),
 \end{align}
Substituting (\ref{equ: second-order right derivative}) and (\ref{equ: first-order right derivative}) into (\ref{equ: taylor expansion}), we can get
\begin{equation}\label{equ:final taylor expansion}
    \begin{split}
   \tx^{(j)}(t_{n+1})-\tx^{(j)}(t_n)&=-h\frac{\partial G_n(\tx(t_n), \ty(t_n))}{\partial \tx^{(j)}}+h^2A_n^{(j)}(\tx(t_n), \ty(t_n))\\
   \\
   &+\frac{h^2}{2}\sum_{s=1}^{d_1}\frac{\partial^2G_n(\tx(t_n), \ty(t_n))}{\partial \tx^{(s)}\partial \tx^{(j)}}\frac{\partial G_n(\tx(t_n), \ty(t_n))}{\partial \tx^{(s)}}\\
   \\
   &-\frac{h^2}{2}\sum_{s=1}^{d_2}\frac{\partial^2G_n(\tx(t_n), \ty(t_n))}{\partial \tx^{(j)} \partial \ty^{(s)}}\frac{\partial F_n(\tx(t_n), \ty(t_n))}{\partial \ty^{(s)}}+\CO(h^3).    
    \end{split}
\end{equation}
Comparing the coefficients of term $h$ in (\ref{equ: sim-Adam final form}) and (\ref{equ:final taylor expansion}), we can get
\begin{align}\label{equ: G_n-sim}
\frac{\partial G_n(\tx(t_n), \ty(t_n))}{\partial \tx^{(j)}}=\frac{N_n^{\tx^{(j)}}}{D_n^{\tx^{(j)}}}. 
\end{align}
Repeating the similar argument for $y$-player's equation, we can get
\begin{align}\label{equ: F_n-sim}
\frac{\partial F_n(\tx(t_n), \ty(t_n))}{\partial \ty^{(i)}}=\frac{N_n^{\ty^{(i)}}}{D_n^{\ty^{(i)}}}.    
\end{align}
With (\ref{equ: G_n-sim}) and (\ref{equ: F_n-sim}) in hand, we can substitute (\ref{equ:sim-substitute 1}) and (\ref{equ: sim-substitute-2}) into (\ref{equ:goal-1}) to get 
\begin{equation}\label{equ: sim-Adam final form-A_n}
    \begin{split}
&\tx^{(j)}(t_{n+1})-\tx^{(j)}(t_n)=-h\frac{N_n^{\tx^{(j)}}}{D_n^{\tx^{(j)}}}-h^2\left(\frac{Q_n^{\tx^{(j)}}}{D_n^{\tx^{(j)}}}-\frac{P_n^{\tx^{(j)}}N_n^{\tx^{(j)}}}{\left(D_n^{\tx^{(j)}}\right)^3}\right)+\CO(h^3)        
    \end{split}
\end{equation}
Comparing the coefficients of term $h^2$ in (\ref{equ: sim-Adam final form-A_n}) and (\ref{equ:final taylor expansion}), we can get
\begin{equation}\label{equ: explicit of A_n}
\begin{split}
&A_n^{(j)}(\tx(t_n), \ty(t_n))\\
\\
&=\frac{1}{2}\sum_{s=1}^{d_2}\frac{\partial^2G_n(\tx(t_n), \ty(t_n))}{\partial \tx^{(j)} \partial \ty^{(s)}}\frac{\partial F_n(\tx(t_n), \ty(t_n))}{\partial \ty^{(s)}} \\ 
\\
&\ \ \ \ \ \ \ \ \ \ -\frac{1}{2}\sum_{s=1}^{d_1}\frac{\partial^2G_n(\tx(t_n), \ty(t_n))}{\partial \tx^{(s)}\partial \tx^{(j)}}\frac{\partial G_n(\tx(t_n), \ty(t_n))}{\partial \tx^{(s)}}-\left(\frac{Q_n^{\tx^{(j)}}}{D_n^{\tx^{(j)}}}-\frac{P_n^{\tx^{(j)}}N_n^{\tx^{(j)}}}{\left(D_n^{\tx^{(j)}}\right)^3}\right) \\
\\
&=-\left(\sum_{s=1}^{d_1}\frac{\frac{\partial N_n^{\tx^{(j)}}}{\partial \tx^{(s)}}D_n^{\tx^{(j)}}-\frac{\partial D_n^{\tx^{(j)}}}{\partial \tx^{(s)}}N_n^{\tx^{(j)}}}{2\left(D_n^{\tx^{(j)}}\right)^2}\frac{N_n^{\tx^{(s)}}}{D_n^{\tx^{(s)}}}-\sum_{s=1}^{d_2}\frac{\frac{\partial N_n^{\tx^{(j)}}}{\partial \ty^{(s)}}D_n^{\tx^{(j)}}-\frac{\partial D_n^{\tx^{(j)}}}{\partial \ty^{(s)}}N_n^{\tx^{(j)}}}{2\left(D_n^{\tx^{(j)}}\right)^2}\frac{N_n^{\ty^{(s)}}}{D_n^{\ty^{(s)}}}\right) \\
\\
&\ \ \ \ \ \ \ \ \ \ -\left(\frac{Q_n^{\tx^{(j)}}}{D_n^{\tx^{(j)}}}-\frac{P_n^{\tx^{(j)}}N_n^{\tx^{(j)}}}{\left(D_n^{\tx^{(j)}}\right)^3}\right).
\end{split}
\end{equation}
Substituting $D_n^{\tx^{(j)}}$, $N_n^{\tx^{(j)}}$, $P_n^{\tx^{(j)}}$ and $Q_n^{\tx^{(j)}}$ in (\ref{equ: definition in full batch setting}) into (\ref{equ: explicit of A_n}), we can get the desired equation.

Repeating the above similar argument for $y$-player's equation, we can get
\begin{align}\label{equ: explicit equation of F_n^i}
&\frac{\partial F_n(\tx, \ty)}{\partial \ty^{(i)}}=\frac{N_n^{\ty^{(i)}}}{D_n^{\ty^{(i)}}}, \\ \nonumber \\ \nonumber\label{equ:B_n^i}
&B_n^{(i)}(\tx(t_n), \ty(t_n))=\sum_{s=1}^{d_1}\frac{\frac{\partial N_n^{\ty^{(i)}}}{\partial \tx^{(s)}}D_n^{\ty^{(i)}}-\frac{\partial D_n^{\ty^{(i)}}}{\partial \tx^{(s)}}N_n^{\ty^{(i)}}}{2\left(D_n^{\ty^{(i)}}\right)^2}\frac{N_n^{\tx^{(s)}}}{D_n^{\tx^{(s)}}}\\\nonumber \\
&\ \ \ \ \ \ \ \ \ \ \ \ \ \ \ \ -\sum_{s=1}^{d_2}\frac{\frac{\partial N_n^{\ty^{(i)}}}{\partial \ty^{(s)}}D_n^{\ty^{(i)}}-\frac{\partial D_n^{\ty^{(i)}}}{\partial \ty^{(s)}}N_n^{\ty^{(i)}}}{2\left(D_n^{\ty^{(i)}}\right)^2}\frac{N_n^{\ty^{(s)}}}{D_n^{\ty^{(s)}}}+\left(\frac{Q_n^{\ty^{(i)}}}{D_n^{\ty^{(i)}}}-\frac{P_n^{\ty^{(i)}}N_n^{\ty^{(i)}}}{\left(D_n^{\ty^{(i)}}\right)^3}\right).
\end{align}
Substituting $D_n^{\ty^{(i)}}$, $N_n^{\ty^{(i)}}$, $P_n^{\ty^{(i)}}$ and $Q_n^{\ty^{(i)}}$ in (\ref{equ: definition in full batch setting}) into (\ref{equ: explicit equation of F_n^i}) and (\ref{equ:B_n^i}), we can get the desired equation.
\end{proof}

\begin{cor}\label{Corollary: ODEs when n tends to infty}
When $n$ tends to infinity, the ODEs become
\begin{equation*}
    \begin{split}
        &\frac{d \tx^{(j)}(t)}{dt}=-\frac{\frac{\partial f(\tx, \ty)}{\partial \tx^{(j)}}}{\sqrt{\left(\frac{\partial f(\tx, \ty)}{\partial \tx^{(j)}}\right)^2+\epsilon}}-\frac{h}{2\sqrt{\left(\frac{\partial f(\tx, \ty)}{\partial \tx^{(j)}}\right)^2+\epsilon}}\times\\
        &\left(\frac{1+\beta}{1-\beta}-\frac{1+\rho}{1-\rho}+\frac{\epsilon}{\left(\frac{\partial f(\tx, \ty)}{\partial \tx^{(j)}}\right)^2+\epsilon}\frac{1+\rho}{1-\rho}\right)\left(\frac{\partial}{\partial \tx^{(j)}}\|\nabla_{x} f(\tx, \ty)\|_{1,\epsilon}-\frac{\partial}{\partial \tx^{(j)}}\|\nabla_{y} f(\tx, \ty)\|_{1,\epsilon}\right),
    \end{split}
\end{equation*}
\begin{equation*}
    \begin{split}
        &\frac{d \ty^{(i)}(t)}{dt}=\frac{\frac{\partial f(\tx, \ty)}{\partial \ty^{(i)}}}{\sqrt{\left(\frac{\partial f(\tx, \ty)}{\partial \ty^{(i)}}\right)^2+\epsilon}}+\frac{h}{2\sqrt{\left(\frac{\partial f(\tx, \ty)}{\partial \ty^{(i)}}\right)^2+\epsilon}}\times\\
        &\left(\frac{1+\beta}{1-\beta}-\frac{1+\rho}{1-\rho}+\frac{\epsilon}{\left(\frac{\partial f(\tx, \ty)}{\partial \ty^{(i)}}\right)^2+\epsilon}\frac{1+\rho}{1-\rho}\right)\left(\frac{\partial}{\partial \ty^{(i)}}\|\nabla_{x} f(\tx, \ty)\|_{1,\epsilon}-\frac{\partial}{\partial \ty^{(i)}}\|\nabla_{y} f(\tx, \ty)\|_{1,\epsilon}\right).
    \end{split}
\end{equation*}
\end{cor}
\textbf{Remark.} Recall that the derivation of ODEs in Appendix \ref{proofofthm1} is written by coordinate-wise and we transform it into matrix form in Section \ref{CTM}.

\begin{lem}\label{lemma: local error between ODEs}
    Suppose that $f(\tx)$ has bounded derivatives up to the fourth order. Let $T_0$ be a fixed time interval. For all $t\in [0, T_0]$. Let $\left(\tilde{\tx}(t), \tilde{\ty}(t)\right)$ be the solution trajectory of ODEs defined in Lemma \ref{lemma: ODE for Sim-Adam} and $\left(\tx(t), \ty(t)\right)$ be the solution trajectory of ODEs for \ref{CADAM} defined in Corollary \ref{Corollary: ODEs when n tends to infty}. Suppose that at time $t_n=nh$, we have\[\left(\tilde{\tx}(t), \tilde{\ty}(t)\right)=\left(\tx(t), \ty(t)\right).\] If we select $n>\max\{\frac{2\log h}{\log|\beta|}, \frac{2\log h}{\log\rho}\}$, then at time $t_{n+1}=(n+1)h$, we have
    \[
    \|\left(\tilde{\tx}(t_{n+1}), \tilde{\ty}(t_{n+1})\right)-\left(\tx(t_{n+1}), \ty(t_{n+1})\right)\|=\CO(h^3).
    \]
\end{lem}
\begin{proof} We only need to prove
\begin{align*}
 &|\tilde{\tx}^{(j)}(t_{n+1})-\tx^{(j)}(t_{n+1})|=\CO(h^3),\,\,\, j=1,2,\cdots, d_1, \\
 &|\tilde{\ty}^{(i)}(t_{n+1})-\ty^{(i)}(t_{n+1})|=\CO(h^3),\,\,\, i=1,2,\cdots, d_2.
\end{align*}
    Firstly, we prove $|\tilde{\tx}^{(j)}(t_{n+1})-\tx^{(j)}(t_{n+1})|=\CO(h^3)$. By the Taylor expansion, we have
    \begin{align*}
       &\tilde{\tx}^{(j)}(t_{n+1})=\tilde{\tx}^{(j)}(nh+h)= \tilde{\tx}^{(j)}(nh)+h\dot{\tilde{\tx}}^{(j)}(nh^+)+\frac{h^2}{2}\ddot{\tilde{\tx}}^{(j)}(nh^+)+\CO(h^3),\\
       &\tx^{(j)}(t_{n+1})=\tx^{(j)}(nh+h)=\tx^{(j)}(nh)+h\dot{\tx}^{(j)}(nh)+\frac{h^2}{2}\ddot{\tx}^{(j)}(nh)+\CO(h^3),
    \end{align*}
     Recall that $\left(\tx(t_n), \ty(t_n)\right)=\left(\tilde{\tx}(t_n), \tilde{\ty}(t_n)\right)$, then we can get
     \[
     |\tilde{\tx}^{(j)}(t_{n+1})-\tx^{(j)}(t_{n+1})|=h|\dot{\tilde{\tx}}^{(j)}(nh^+)-\dot{\tx}^{(j)}(nh)|+\frac{h^2}{2}|\ddot{\tilde{\tx}}^{(j)}(nh^+)-\ddot{\tx}^{(j)}(nh)|+\CO(h^3).
     \]
     By Lemma \ref{lemma: ODE for Sim-Adam} and Corollary \ref{Corollary: ODEs when n tends to infty}, we have
     \begin{align*}
         &|h\dot{\tilde{\tx}}^{(j)}(nh^+)-h\dot{\tx}^{(j)}(nh)|\\
         &=h^2\left|\frac{\epsilon}{\left(\frac{\partial f(\tx(t), \ty(t))}{\partial \ty^{(i)}}\right)^2+\epsilon}\left(-\frac{(n+1)\rho^{n+1}}{1-\rho^{n+1}}\right)-\frac{(n+1)\beta^{n+1}}{1-\beta^{n+1}}+\frac{(n+1)\rho^{n+1}}{1-\rho^{n+1}}\right|\\
         &\ \ \ \ \ \ \ \ \times \left|\frac{\partial}{\partial \tx^{(j)}}\|\nabla_x f(x(t), y(t))\|_{1,\epsilon}-\frac{\partial}{\partial \tx^{(j)}}\|\nabla_y f(x(t), y(t))\|_{1,\epsilon}\right|\\
         \\
         &\leq h^2\left(\left|\frac{\epsilon}{\left(\frac{\partial f(x(t), y(t))}{\partial \ty^{(i)}}\right)^2+\epsilon}\left(-\frac{(n+1)\rho^{n+1}}{1-\rho^{n+1}}\right)\right|+\left|-\frac{(n+1)\beta^{n+1}}{1-\beta^{n+1}}\right|+\left|\frac{(n+1)\rho^{n+1}}{1-\rho^{n+1}}\right|\right)\\
         &\ \ \ \ \ \ \ \ \times\left|\frac{\partial}{\partial \tx^{(j)}}\|\nabla_x f(x(t), y(t))\|_{1,\epsilon}-\frac{\partial}{\partial \tx^{(j)}}\|\nabla_y f(x(t), y(t))\|_{1,\epsilon}\right|\\
         \\
         &\leq h^2\left(\left|\frac{(n+1)\rho^{n+1}}{1-\rho}\right|+\left|\frac{(n+1)\beta^{n+1}}{1-\beta}\right|+\left|\frac{(n+1)\rho^{n+1}}{1-\rho}\right|\right)\\
         &\ \ \ \ \ \ \ \ \times\left|\frac{\partial}{\partial \tx^{(j)}}\|\nabla_x f(x(t), y(t))\|_{1,\epsilon}-\frac{\partial}{\partial \tx^{(j)}}\|\nabla_y f(x(t), y(t))\|_{1,\epsilon}\right|
     \end{align*}
     Actually, the term
     $$\left|\frac{\partial}{\partial \tx^{(j)}}\|\nabla_x f(\tx(t), \ty(t))\|_{1,\epsilon}-\frac{\partial}{\partial \tx^{(j)}}\|\nabla_y f(\tx(t), \ty(t))\|_{1,\epsilon}\right|$$
     is bounded since the first and second derivatives of $f$ are bounded. It suffices to prove
     \[
     \left|\frac{(n+1)\rho^{n+1}}{1-\rho}\right|+\left|\frac{(n+1)\beta^{n+1}}{1-\beta}\right|+\left|\frac{(n+1)\rho^{n+1}}{1-\rho}\right|=\CO(h). 
     \]
     If we select $n>\max\{\frac{2\log h}{\log|\beta|}, \frac{2\log h}{\log\rho}\}$, i.e., $|\beta|^n<h^2$ and $|\rho|^n<h^2$, then we have
     \[
     |(n+1)\beta^{n+1}|<\frac{T_0}{h}|\beta|^{n+1}<T_0|\beta|h, \quad |(n+1)\rho^{n+1}|<\frac{T_0}{h}|\rho|^{n+1}<T_0|\rho|h.
     \]
     Repeating the similar argument, we can prove $$h^2|\ddot{\tilde{\tx}}^{(j)}(nh^+)-\ddot{\tx}^{(j)}(nh)|=\CO(h^3)$$ 
     if $n>\max\{\frac{2\log h}{\log|\beta|}, \frac{2\log h}{\log\rho}\}.$ We can also prove $|\tilde{\ty}^{(i)}(t_{n+1})-\ty^{(i)}(t_{n+1})|=\CO(h^3)$ through similarly arguments.
\end{proof}
The proof of Theorem \ref{error_ana} follows from Lemma \ref{lemma: ODE for Sim-Adam}, Corollary \ref{Corollary: ODEs when n tends to infty} and Lemma \ref{lemma: local error between ODEs}.

\newpage
\section{Additional Materials for Section \ref{LC}.}\label{Appendix_S4}

\subsection{Proof of Proposition \ref{jacobianadam}.}\label{Appendix_S41} 

\begin{lem}\label{Lemma1Appendix_S42}
    We have
    \begin{itemize}
        \item $\nabla_x \lVert \nabla_xf(\tx,\ty) \lVert_{1,\epsilon} = \nabla^2_xf(\tx,\ty) \cdot \mu_{\epsilon}(\tx,\ty) \cdot \nabla_xf(\tx,\ty)$
        \item $\nabla_y \lVert \nabla_xf(\tx,\ty) \lVert_{1,\epsilon} = \nabla_{yx}f(\tx,\ty) \cdot \mu_{\epsilon}(x,y) \cdot \nabla_xf(\tx,\ty)$
        \item $\nabla_y \lVert \nabla_yf(\tx,\ty) \lVert_{1,\epsilon} = \nabla^2_{y}f(\tx,\ty) \cdot \nu_{\epsilon}(x,y) \cdot \nabla_yf(\tx,\ty)$
        \item $\nabla_x \lVert \nabla_yf(\tx,\ty) \lVert_{1,\epsilon} = \nabla_{xy}f(\tx,\ty) \cdot \nu_{\epsilon}(x,y) \cdot \nabla_yf(\tx,\ty)$
    \end{itemize}
\end{lem}

\begin{proof}
    Here we only calculate $\nabla_x \lVert \nabla_xf(\tx,\ty) \lVert_{1,\epsilon}$, other items are similar.

    Let $g(\tx,\ty) = \nabla_xf(\tx,\ty)$, then we have
    \begin{align*}
        \nabla_x \lVert \nabla_xf(\tx,\ty) \lVert_{1,\epsilon}  = \nabla_x \lVert g(\tx,\ty) \lVert_{1,\epsilon} = \nabla_x g(\tx,\ty) \cdot  \nabla_g \lVert g (\tx,\ty)\lVert_{1,\epsilon} =  \nabla^2_xf(\tx,\ty) \cdot  \nabla_g \lVert g (\tx,\ty)\lVert_{1,\epsilon}
    \end{align*}
    For the term $\nabla_g \lVert g (\tx,\ty)\lVert_{1,\epsilon}$, we have
    \begin{align*}
        \nabla_g \lVert g (\tx,\ty)\lVert_{1,\epsilon} =   \nabla_g \left( \sum^n_{i=1} \sqrt{g^2_i+\epsilon} \right)
        =\left( \frac{g_i}{\sqrt{g^2_i+\epsilon}}\right)_{i}
        = \mu_{\epsilon}(\tx,\ty) \cdot \nabla_x f(\tx,\ty)
    \end{align*}
    Thus combine above, we get
    \begin{align*}
      \nabla_x \lVert \nabla_xf(\tx,\ty) \lVert_{1,\epsilon} =  \nabla^2_xf(\tx,\ty) \cdot \mu_{\epsilon}(\tx,\ty) \cdot \nabla_x f(\tx,\ty)
    \end{align*}
\end{proof}

\begin{lem}\label{fghepir} For \ref{CADAM}, let $(\tx^*,\ty^*)$ be a local Nash equilibrium. Denote the constant
\begin{align*}
 \gamma := \frac{h(1+\beta)}{2\sqrt{\epsilon}(1-\beta)},
\end{align*}
then we have
\begin{align*}
    &\nabla_x \left( \frac{d\tx}{dt} \right)|_{(\tx^*,\ty^*)}  = -\frac{1}{\sqrt{\epsilon}} \left( \nabla^2_xf(\tx^*,\ty^*)\ + \gamma \left( \nabla^2_xf\tx^*,\ty^*) \cdot  \nabla^2_xf(\tx^*,\ty^*) - \nabla_{xy}f(\tx^*,\ty^*) \cdot \nabla_{yx}f(\tx^*,\ty^*) \right)\right) \\
   &\nabla_y \left( \frac{d\tx}{dt} \right)|_{(\tx^*,\ty^*)}  = -\frac{1}{\sqrt{\epsilon}} \left( \nabla_{xy}f(\tx^*,\ty^*) + \gamma \left( \nabla^2_{x}f(\tx^*,\ty^*) \cdot  \nabla_{xy}f(\tx^*,\ty^*) - \nabla_{xy}f(\tx^*,\ty^*) \cdot \nabla^2_{y}f(\tx^*,\ty^*) \right)\right)\\
   & \nabla_x \left( \frac{dy}{dt} \right)|_{(\tx^*,\ty^*)}  = \frac{1}{\sqrt{\epsilon}} \left( \nabla_{yx}f(\tx^*,\ty^*) + \gamma \left(\nabla_{yx}f(\tx^*,\ty^*) \cdot  \nabla^2_{x}f(\tx^*,\ty^*) - \nabla^2_{y}f(\tx^*,\ty^*) \cdot \nabla_{yx}f(\tx^*,\ty^*) \right)\right)\\
   & \nabla_y \left( \frac{dy}{dt} \right)|_{(\tx^*,\ty^*)}  = \frac{1}{\sqrt{\epsilon}} \left( \nabla^2_{y}f(\tx^*,\ty^*) + \gamma \left(\nabla_{yx}f(\tx^*,\ty^*) \cdot  \nabla_{xy}f(\tx^*,\ty^*) - \nabla^2_{y}f(\tx^*,\ty^*)\cdot \nabla^2_{y}f(\tx^*,\ty^*) \right)\right)
\end{align*}    
\end{lem}

\begin{proof}
Here we only present the detailed proof for $\nabla_x \left( \frac{d\tx}{dt} \right)|_{(\tx^*,\ty^*)}$, other terms can be proved through a similar calculation.

Recall that in \ref{CADAM}, we have
\begin{align*}
    \dot{\tx}(t) &= - \mu_{\epsilon}(\tx,\ty) \left( \nabla_x f(\tx,\ty) + \frac{h}{2}\CM^{\mu}_{\beta,\rho,\epsilon}(\tx,\ty) \cdot  \nabla_{x} \left( \lVert \nabla_x f(\tx,\ty) \lVert_{1,\epsilon} -  \lVert \nabla_y f(\tx,\ty) \lVert_{1,\epsilon} \right)  \right),
\end{align*}
thus by the product rule, we have
\begin{align}
& \nabla_x \left( \frac{d\tx}{dt} \right)|_{(\tx^*,\ty^*)}  =  - \nabla_x  \mu_{\epsilon}(\tx^*,\ty^*) \cdot \left( \nabla_x f(\tx^*,\ty^*) + \frac{h}{2}\CM^{\mu}_{\beta,\rho,\epsilon}(\tx^*,\ty^*) \cdot  \nabla_{x} \left( \lVert \nabla_x f(\tx^*,\ty^*) \lVert_{1,\epsilon} -  \lVert \nabla_y f(\tx^*,\ty^*) \lVert_{1,\epsilon} \right)  \right)\nonumber \\
&\ \ \ \ \ \ \ \ \ \ \ \ - \mu_{\epsilon}(\tx^*,\ty^*) \cdot 
\nabla_x \left( \nabla_x f(\tx^*,\ty^*) + \frac{h}{2}\CM^{\mu}_{\beta,\rho,\epsilon}(\tx^*,\ty^*) \cdot  \nabla_{x} \left( \lVert \nabla_x f(\tx^*,\ty^*) \lVert_{1,\epsilon} -  \lVert \nabla_y f(\tx^*,\ty^*) \lVert_{1,\epsilon} \right)  \right)\label{gradient_xdxdt}
\end{align}
By the definition of local Nash equilibrium, we have
\begin{align*}
    \nabla_x f(\tx^*,\ty^*) = \textbf{0},\ \  \nabla_y f(\tx^*,\ty^*) = \textbf{0}
\end{align*}
and
\begin{align*}
    \mu_{\epsilon}(\tx^*,\ty^*) := 
    \begin{bmatrix}
    \frac{1}{\sqrt{\epsilon}} & & &  \\
    & \frac{1}{\sqrt{\epsilon}} & &  \\
    & & \ddots & \\
    & & & \frac{1}{\sqrt{\epsilon}}
    \end{bmatrix}_{d_1 \times d_1},\ \ 
    \ \CM^{\mu}_{\beta,\rho,\epsilon}(\tx^*,\ty^*) = \CK(\beta,\rho) \CI_{d_1} + \frac{\epsilon(1+\rho)}{1-\rho} \mu_{\epsilon}^2(\tx^*,\ty^*) .
\end{align*}
Moreover, future simplification gives
\begin{align*}
    \CM^{\mu}_{\beta,\rho,\epsilon}(\tx^*,\ty^*) & = \CK(\beta,\rho) \CI_{d_1} + \frac{\epsilon(1+\rho)}{1-\rho} \mu_{\epsilon}^2(\tx^*,\ty^*) \\
    & = \CK(\beta,\rho) \CI_{d_1} + \frac{1+\rho}{1-\rho} \CI_{d_1} \\
    & = \frac{1+\beta}{1-\beta}\CI_{d_1},\ \ \  \mathrm{since}\ \   \CK(\beta,\rho) = (1+\beta)/(1-\beta) + (1+\rho)/(1-\rho).
\end{align*}

Thus the first term on the right hand side of \eqref{gradient_xdxdt} is $\textbf{0}$, and we have
\begin{align}
     \nabla_x & \left( \frac{d\tx}{dt} \right)|_{(\tx^*,\ty^*)}\nonumber    \\
     & = - \mu_{\epsilon}(\tx^*,\ty^*) \cdot 
\nabla_x \left( \nabla_x f(\tx^*,\ty^*) + \frac{h}{2}\CM^{\mu}_{\beta,\rho,\epsilon}(\tx^*,\ty^*) \cdot  \nabla_{x} \left( \lVert \nabla_x f(\tx^*,\ty^*) \lVert_{1,\epsilon} -  \lVert \nabla_y f(\tx^*,\ty^*) \lVert_{1,\epsilon} \right)  \right)\nonumber \\
\nonumber \\
& = - \frac{1}{\sqrt{\epsilon}} \cdot \left(
\nabla^2_x f(\tx^*,\ty^*) + \frac{h(1+\beta)}{2(1-\beta)} \cdot \nabla^2_x \left( \lVert \nabla_x f(\tx^*,\ty^*) \lVert_{1,\epsilon} -  \lVert \nabla_y f(\tx^*,\ty^*) \lVert_{1,\epsilon} \right) 
\right)\label{asdfadifj}
\end{align}
Recall from Lemma \ref{Lemma1Appendix_S42}, we have 

\begin{align*}
\nabla^2_x & \lVert \nabla_xf(\tx,\ty) \lVert_{1,\epsilon} \\
& =\nabla_x \left(\nabla_x^2f(\tx,\ty) \cdot \mu_{\epsilon}(\tx,\ty)\right) \cdot \nabla_xf(\tx,\ty) + \nabla_x^2f(\tx,\ty) \cdot \mu_{\epsilon}(\tx,\ty) \cdot  \nabla^2_xf(\tx,\ty)
\end{align*}

and 

\begin{align*}
\nabla^2_x & \lVert \nabla_yf(\tx,\ty) \lVert_{1,\epsilon} \\
& =\nabla_x \left(\nabla_{xy}f(\tx,\ty) \cdot \nu_{\epsilon}(\tx,\ty)\right) \cdot \nabla_yf(\tx,\ty) + \nabla_{xy}f(\tx,\ty) \cdot \mu_{\epsilon}(\tx,\ty) \cdot  \nabla_{yx}f(\tx,\ty)
\end{align*}

Take $(\tx^*,\ty^*)$ into above two equalities, and use the fact that $\nabla_yf(\tx^*,\ty^*),\nabla_yf(\tx^*,\ty^*) = \textbf{0}$, we get

\begin{align*}
    \nabla^2_x \lVert \nabla_xf(\tx,\ty) \lVert_{1,\epsilon} &= \nabla_x^2f(\tx^*,\ty^*) \cdot \mu_{\epsilon}(\tx^*,\ty^*) \cdot  \nabla^2_xf(\tx^*,\ty^*) \\
    & = \frac{1}{\sqrt{\epsilon}} \nabla_x^2f(\tx^*,\ty^*) \cdot  \nabla^2_xf(\tx^*,\ty^*)
\end{align*}
and
\begin{align*}
   \nabla^2_x \lVert \nabla_yf(\tx,\ty) \lVert_{1,\epsilon}  & = \nabla_{xy}f(\tx^*,\ty^*) \cdot \mu_{\epsilon}(\tx^*,\ty^*) \cdot  \nabla_{yx}f(\tx^*,\ty^*) \\
   & = \frac{1}{\sqrt{\epsilon}} \nabla_{xy}f(\tx^*,\ty^*) \cdot \nabla_{yx}f(\tx^*,\ty^*)
\end{align*}
Take above two equalities into \eqref{asdfadifj}, we get
\begin{align*}
    \nabla_x \left( \frac{d\tx}{dt} \right)|_{(\tx^*,\ty^*)}  = -\frac{1}{\sqrt{\epsilon}} \left( \nabla^2_xf(\tx^*,\ty^*)\ + \gamma \left( \nabla^2_xf(\tx^*,\ty^*) \cdot  \nabla^2_xf(\tx^*,\ty^*) - \nabla_{xy}f(\tx^*,\ty^*) \cdot \nabla_{yx}f(\tx^*,\ty^*) \right)\right),
\end{align*}
this completes the proof for the terms $\nabla_x \left( \frac{d\tx}{dt} \right)|_{(\tx^*,\ty^*)}$.
\end{proof}

Now we are ready to proof Proposition \ref{jacobianadam}.

\begin{proof}[Proof of Proposition \ref{jacobianadam} ]
By definition, the Jacobian $\CJ_{\mathrm{Adam}}$ of \ref{CADAM} at $(\tx^*,\ty^*)$ is written as
\begin{align*}
    \CJ_{\mathrm{Adam}} = 
    \begin{bmatrix}
        \nabla_x\left( \frac{d\tx}{dt} \right)|_{(\tx^*,\ty^*)} & \nabla_y \left( \frac{d\tx}{dt} \right)|_{(\tx^*,\ty^*)}\\
        \\
        \nabla_x \left( \frac{d\ty}{dt} \right)|_{(\tx^*,\ty^*)} & \nabla_y \left( \frac{d\ty}{dt} \right)|_{(\tx^*,\ty^*)}
    \end{bmatrix}
\end{align*}
Take the terms of $\nabla_x\left( \frac{d\tx}{dt} \right)|_{(\tx^*,\ty^*)}, \nabla_y \left( \frac{d\tx}{dt} \right)|_{(\tx^*,\ty^*)}, \nabla_x \left( \frac{d\ty}{dt} \right)|_{(\tx^*,\ty^*)} $ and $\nabla_y \left( \frac{d\ty}{dt} \right)|_{(\tx^*,\ty^*)}$ in Lemma \ref{fghepir} into above, we get

\begin{align*}
\CJ_{\mathrm{Adam}} = \frac{1}{\sqrt{\epsilon}} \left( \CI_{d_1} - \frac{h(1+\beta)}{2\sqrt{\epsilon}(1-\beta)}\CJ \right) \CJ, \ \CJ =  \begin{bmatrix}
        -\nabla^2_{x}f(\tx^*,\ty^*) & -\nabla_{xy}f(\tx^*,\ty^*) \\
        \\
        \nabla_{yx}f(\tx^*,\ty^*) & \nabla^2_{y}f(\tx^*,\ty^*)
    \end{bmatrix}.
\end{align*}
\end{proof}

\subsection{Proof of Theorem \ref{Corollary: Continuous Local convergence for general case}.}\label{Appendix_S42}

We first introduce the following lemma, which describe the eigenvalues of matrix polynomial
\begin{lem}\label{grahammatrix}\cite{graham2018matrix} If $p(x)$ is a polynomial and $A \in \BR^{n \times n}$, then every eigenvalue of $p(A)$ can be represented by $p(\lambda)$, where $\lambda \in \mathrm{Sp}(\CJ)$. 
\end{lem}
% \begin{proof}
%     The following proof comes from  \cite{graham2018matrix}. A key point here is that we are considering eigenvalues of matrices over complex number field $\BC$, and every $n$-degree polynomial over $\BC$ has $n$ roots. Let $\mu \in \BC$ be an eigenvalue of $p(A)$, and we consider the polynomial $p(x)-\mu$. Over the complex number field $\BC$, the polynomial $p(x)-\mu$ can be factored as
%     \begin{align}\label{polyfact}
%         p(x)-\mu = a \prod^n_{i=1} (x - \lambda_i)
%     \end{align}
%     Moreover, since $\mu$ is an eigenvalue of $p(A)$, we have $\det(p(A) - \mu \CI) = 0$. Thus, from (\ref{polyfact}), as least one of $A - \lambda_i \CI$ is not invertible, i.e., $\exists v \in \BC^n$, such that $\left(A - \lambda_i \CI \right) v = 0$. This implies 
%     $\mu = p(\lambda_i)$ and finishes the proof.
% \end{proof}

In the following, we state a corollary of Assumption \ref{ibr}, which is proved in \citep{wang2024local}.

\begin{lem}\label{wangresult}[Theorem 2.1 in \citet{wang2024local}]
    Under Assumption \ref{ibr}, we have $\Re(\lambda) < 0$ for every $\lambda \in \mathrm{Sp}(\CJ)$.
\end{lem}

\begin{proof}[Proof of Theorem~\ref{Corollary: Continuous Local convergence for general case}.] 

According to Proposition \ref{jacobianadam} and Lemma \ref{grahammatrix}, every eigenvalue of $\CJ_{\mathrm{Adam}}$ can be represented by a quadratic polynomial
\begin{align}\label{eigenvaluesimgf}
    \frac{1}{\sqrt{\epsilon}} \left( 1 - \frac{h(1+\beta)}{2\sqrt{\epsilon}(1-\beta)}\lambda \right) \lambda, \ 
\end{align}
where $\lambda \in \mathrm{Sp}(\CJ)$ is an eigenvalue of \ref{jacgda}, and any number represented by \eqref{eigenvaluesimgf} is an eigenvalue of $\CJ_{\mathrm{Adam}}$.

Thus, to ensure the local convergence of \ref{CADAM}, we need 
\begin{align}\label{eigenvaluesimgf2}
    &\Re \left( \left[1 - \frac{h(1+\beta)}{2\sqrt{\epsilon}(1-\beta)}\lambda \right]  \cdot \lambda \right) \nonumber \\
    & \qquad =  \Re(\lambda) - \frac{h(1+\beta)}{2\sqrt{\epsilon}(1-\beta)} \left( \Re(\lambda)^2 - \Im(\lambda)^2 \right)
    < 0, \ \ \forall \lambda \in \mathrm{Sp}(\CJ).
\end{align}

From Lemma~\ref{wangresult}, $\Re(\lambda)$  is a negative number, and from Assumption~\ref{ibr}, the term $\Re(\lambda)^2 - \Im(\lambda)^2$ in \eqref{eigenvaluesimgf2} is a negative number if $\lambda \in \widetilde{\mathrm{Sp}(\CJ)}$. Moreover, for $\lambda \in \mathrm{Sp}(\CJ)/\widetilde{\mathrm{Sp}(\CJ)}$, \eqref{eigenvaluesimgf2} always satisfied. Thus, for some fixed parameter $\beta,\epsilon$ and $\forall \lambda \in  \widetilde{\mathrm{Sp}(\CJ)}$, we need the step size $h$ in \eqref{eigenvaluesimgf2} satisfies
\begin{align}
    h <\min_{\lambda \in   \widetilde{\mathrm{Sp}(\CJ)}} \frac{2\sqrt{\epsilon} (1-\beta)}{(1+\beta)}  \frac{\lvert\Re(\lambda)\lvert}{(\Im(\lambda)^2 - \Re(\lambda)^2)},
\end{align}
and this finishes the proof of Theorem~\ref{Corollary: Continuous Local convergence for general case}.
\end{proof}

\subsection{Proof of Theorem \ref{Thm: Discrete Local convergence for general case}.}\label{Appendix :proof of local convergence of Adam}
%\subsubsection{Outline of Proof of Theorem }
\paragraph{Additional Notations.} Denote the Jacobian matrix of any matrix $A$ by $\text{Jac}(A)$. Denote the spectral radius of any matrix $A$ by $\varrho(A)$.

Define $\tz_n=\left(\tilde{\tm}_n, \tilde{\tv}_n,  \tx_n, \hat{\tm}_n, \hat{\tv}_n, \ty_n\right)^{\top}$. We can rewrite \ref{Adam} as a time-dependent discrete-time dynamical system $\tz_{n+1}=T(n, \tz_n)$, which can be split into an autonomous dynamical system $\bar{T}(\tz_n)$ and a non-autonomous one $\CR(n, \tz_n)$ in the following form
    \begin{equation}\label{equ: non-autonomous system}
      \tz_{n+1}=T(n, \tz_n)=\bar{T}(\tz_n)+\CR(n, \tz_n), \tag{Non-Autonomous System}
    \end{equation}
    where
    \[
    {T}(n,\tz_n)=\begin{bmatrix}
        &\beta \tilde{\tm}_n+(1-\beta)\nabla_{x} f(\tx_n, \ty_n)\\
        \\
        &\rho \tilde{\tv}_{n}+(1-\rho)\left( \nabla_{x} f(\tx_n, \ty_n)\right)^2\\
        \\
    &\tx_n-h\frac{\tilde{\tm}_{n+1}/(1-\beta^{n+1})}{\sqrt{\tilde{\tv}_{n+1}/(1-\rho^{n+1})+\epsilon}} \\
    \\
    &\beta \hat{\tm}_n+(1-\beta)\nabla_{y} f(\tx_n, \ty_n)\\
    \\
    &\rho \hat{\tv}_n+(1-\rho)\left(\nabla_{y} f(\tx_n, \ty_n)\right)^2\\
    \\
    &\ty_n+h\frac{\hat{\tm}_{n+1}/(1-\beta^{n+1})}{\sqrt{\hat{\tv}_{n+1}/(1-\rho^{n+1})+\epsilon}}
    \end{bmatrix},\ \ 
%    \]
%    \[
    \bar{T}(\tz_n)=\begin{bmatrix}
        &\beta \tilde{\tm}_n+(1-\beta)\nabla_{x} f(\tx_n, \ty_n)\\
        \\
        &\rho \tilde{\tv}_{n}+(1-\rho)\left( \nabla_{x} f(\tx_n, \ty_n)\right)^2\\
        \\
    &\tx_n-h\frac{\tilde{\tm}_{n+1}}{\sqrt{\tilde{\tv}_{n+1}+\epsilon}} \\
    \\
    &\beta \hat{\tm}_n+(1-\beta)\nabla_{y} f(\tx_n, \ty_n)\\
    \\
     &\rho \hat{\tv}_n+(1-\rho)\left(\nabla_{y} f(\tx_n, \ty_n)\right)^2\\
     \\
    &\ty_n+h\frac{\hat{\tm}_{n+1}}{\sqrt{\hat{\tv}_{n+1}+\epsilon}}
    \end{bmatrix},
    \]
    and
    \begin{align*}
    \CR(n, \tz_n)&=\begin{bmatrix}
        &\textbf{0}\\
        \\
        &\textbf{0}\\
        \\
    &\frac{h\tilde{\tm}_{n+1}}{\sqrt{\tilde{\tv}_{n+1}+\epsilon}}-\frac{h\tilde{\tm}_{n+1}/(1-\beta^{n+1})}{\sqrt{\tilde{\tv}_{n+1}/(1-\rho^{n+1})+\epsilon}} \\
    \\
    &\textbf{0}\\
    \\
    &\textbf{0}\\
    \\
    &\frac{-h\hat{\tm}_{n+1}}{\sqrt{\hat{\tv}_{n+1}+\epsilon}}+\frac{h\hat{\tm}_{n+1}/(1-\beta^{n+1})}{\sqrt{\hat{\tv}_{n+1}/(1-\rho^{n+1})+\epsilon}}
    \end{bmatrix}=\begin{bmatrix}
        &\textbf{0}\\
        \\
        &\textbf{0}\\
        \\
    &\frac{h\tilde{\tm}_{n+1}}{\sqrt{\tilde{\tv}_{n+1}+\epsilon}}-\frac{h\sqrt{1-\rho^{n+1}}}{1-\beta^{n+1}}\frac{\tilde{\tm}_{n+1}}{\sqrt{\tilde{\tv}_{n+1}+(1-\rho^{n+1})\epsilon}} \\
    \\
    &\textbf{0}\\
    \\
    &\textbf{0}\\
    \\
    &\frac{-h\hat{\tm}_{n+1}}{\sqrt{\hat{\tv}_{n+1}+\epsilon}}+\frac{h\sqrt{1-\rho^{n+1}}}{1-\beta^{n+1}}\frac{\hat{\tm}_{n+1}}{\sqrt{\hat{\tv}_{n+1}+(1-\rho^{n+1})\epsilon}}
    \end{bmatrix}.
    \end{align*}
With this split in hand, the proof of local convergence of \ref{Adam} is transformed into the proof of local convergence of $\bar{T}(\tz_n)$ and $\CR(n, \tz_n)$. The sketched proof of Theorem \ref{Thm: Discrete Local convergence for general case} follows by two steps:
\begin{itemize}[leftmargin=*]
    \item Step 1: We start by proving that finding the Local Nash Equilibrium $(\tx^*, \ty^*)$ is equivalent to finding the fixed point $(\textbf{0},\textbf{0}, \tx^*, \textbf{0}, \textbf{0}, \ty^*)$ of $\bar{T}$ in Lemma \ref{lemma: fixed point=nash}. We next compute the characteristic polynomial of the Jacobian matrix of $\bar{T}$ at the fixed point. Then by Lemma \ref{Lemma: complex-4th-order polynomial}, we can select the parameters $h$, $\beta$ and $\epsilon$ to ensure that the spectral radius of the Jacobian matrix of $\bar{T}$ is less than $1$. Therefore, we can conclude  $\bar{T}$ converges locally near the local Nash equilibrium by Lemma \ref{Lemma: local stability of dynamical system}.
    \item Step 2: We prove that $\|\CR(n, \tz_n)\|$, the perturbation term of \ref{equ: non-autonomous system}, converges locally with an exponential rate by direct algebra computation, i.e., the perturbation term of \ref{equ: non-autonomous system} vanish sufficiently fast with an exponential rate. Intuitively, the exponentially vanishing perturbation hardly influence the local convergence of \ref{equ: non-autonomous system}.
    \item Step 3: We prove that (\ref{equ: non-autonomous system}) consists of an autonomous system $\bar{T}$ with $\varrho\left(\text{Jac}(\bar{T})\right)<1$ and an exponentially vanishing perturbation will locally converge with an exponential rate in Lemma \ref{lemma:combine two parts}.
\end{itemize}

% \subsection{Proof of Theorem \ref{Thm: Discrete Local convergence for general case}}
    
    \begin{lem}\label{lemma: fixed point=nash}
         $(\tx^*, \ty^*)$ is a local Nash equilibrium if and only if $\tz^*=(\textbf{0},\textbf{0}, \tx^*, \textbf{0},\textbf{0}, \ty^*)^{\top}$ is the fixed point of $T$ and $\bar
         {T}$, and $\nabla_{x}^2f(\tx^*, \ty^*)\succeq \textbf{0}$, $\nabla_{y}^2f(\tx^*, \ty^*)\preceq \textbf{0}$. 
    \end{lem}
\begin{proof}
    On the one hand, if $\tz^*=(\textbf{0},\textbf{0}, \tx^*, \textbf{0},\textbf{0}, \ty^*)$ is the fixed point of $\bar{T}$, consider the following equation
    \begin{align*}
        \bar{T}\left((\textbf{0},\textbf{0}, \tx^*, \textbf{0},\textbf{0}, \ty^*)^{\top}\right)=\begin{bmatrix}
        &(1-\beta)\nabla_{x} f(\tx^*, \ty^*)\\
        \\
        &(1-\rho)\left( \nabla_{x} f(\tx^*, \ty^*)\right)^2\\
        \\
    &\tx^* \\
    \\
    &(1-\beta)\nabla_{y} f(\tx^*, \ty^*)\\
    \\
    &(1-\rho)\left(\nabla_{y} f(\tx^*, \ty^*)\right)^2\\
    \\
    &\ty^*
    \end{bmatrix}=\begin{bmatrix}
        &\textbf{0}\\
        \\
        &\textbf{0}\\
        \\
    &\tx^* \\
    \\
    &\textbf{0}\\
    \\
    &\textbf{0}\\
    \\
    &\ty^*
    \end{bmatrix},
    \end{align*}
    we can get $\nabla_{x} f(\tx^*, \ty^*)=\textbf{0}$ and $\nabla_{y} f(\tx^*, \ty^*)=\textbf{0}$, i.e., $(\tx^*, \ty^*)$ is a local Nash equilibrium.

    On the other hand, if $\nabla_{x} f(\tx^*, \ty^*)=\textbf{0}$ and $\nabla_{y} f(\tx^*, \ty^*)=\textbf{0}$, we have
    \begin{align*}
        \bar{T}\left((\tilde{\tm}^*,\tilde{\tv}^*, \tx^*, \hat{\tm}^*,\hat{\tv}^*, \ty^*)^{\top}\right)=\begin{bmatrix}
        &\beta \tilde{\tm}^*\\
        \\
        &\rho \tilde{\tv}^*\\
        \\
    &\tx^*-h \frac{\tilde{\tm}^*}{\sqrt{\tilde{\tv}^*+\epsilon}}\\
    \\
    &\beta \hat{\tm}^*\\
    \\
    &\rho \hat{\tv}^*\\
    \\
    &\ty^*+h\frac{\hat{\tm}^*}{\sqrt{\hat{\tv}^*+\epsilon}}
    \end{bmatrix}=\begin{bmatrix}
        &\tilde{\tm}^*\\
        \\
        &\tilde{\tv}^*\\
        \\
    &\tx^* \\
    \\
    &\hat{\tm}^*\\
    \\
    &\hat{\tv}^*\\
    \\
    &\ty^*
    \end{bmatrix},
    \end{align*}
    solve this fixed point equation, we can get $\tilde{\tm}^*=\textbf{0}$, $\tilde{\tv}^*=\textbf{0}$, $\hat{\tm}^*=\textbf{0}$ and $\hat{\tv}^*=\textbf{0}$, which means $(\textbf{0},\textbf{0}, \tx^*, \textbf{0},\textbf{0}, \ty^*)$ is the fixed point of $\bar{T}$. We can implement the similar argument for $T$.
\end{proof}

\begin{lem}\label{Lemma: local stability of dynamical system}[Corollary 4.35 in \cite{elaydi2005introduction} and Theorem II.1 in \cite{bock2021local}] Consider $\bar{T}: M\rightarrow M$ with a fixed point $w^*$ and $\bar{T}$ continuously differentiable in an open disk $B_{\delta}(w^*)\subset M$ with radius $\delta$. Assume $$\varrho(\text{Jac}(\bar{T}_{w^*}))<1,$$ then there exists $0<\delta_0<\delta$ and $0\leq c<1$ such that for all $w_0$ with $\|w_0-w^*\|<\delta_0$ and for all $t\in \mathbb{N}$,
    \[
    \|w(t;w_0)-w^*\|\leq c^t\|w_0-w^*\|.
    \]
        
    \end{lem}
    
    % The Jacobian matrix of $\bar{T}$ at $\tz^*=(0,0,x^*,0,0,y^*)$ is
    %  \begin{align*}
    %     \CM_{\CS}=\begin{bmatrix}
    %    \beta I& 0 & (1-\beta) \nabla^2_xf(\tx^*, \ty^*) & 0 &0 & (1-\beta)\nabla_{xy}f(\tx^*, \ty^*) \\
    %    0& \rho I & 0& 0 &0 & 0 \\
    %    -\frac{h \beta}{\sqrt{\epsilon}}I& 0 & I-\frac{h(1-\beta)}{\sqrt{\epsilon}}  \nabla^2_xf(\tx^*, \ty^*) & 0 &0 & -\frac{h(1-\beta)}{\sqrt{\epsilon}}\nabla_{xy}f(\tx^*, \ty^*) \\
    %    \\
    %    0& 0 & (1-\beta)\nabla_{yx}f(\tx^*, \ty^*) & \beta I &0 & (1-\beta) \nabla^2_yf(\tx^*, \ty^*) \\
    %    0& 0 & 0 & 0 & \rho I & 0 \\
    %    0& 0 & \frac{h(1-\beta)}{\sqrt{\epsilon}}\nabla_{yx}f(\tx^*, \ty^*) & \frac{h \beta}{\sqrt{\epsilon}}I & 0 & I+\frac{h(1-\beta)}{\sqrt{\epsilon}}  \nabla^2_yf(\tx^*, \ty^*)  \\
    % \end{bmatrix}
    % \end{align*}
    % For the convenience of notation, we denote $\nabla^2f(\tx^*, \ty^*)=\nabla^2f(\tx^*, \ty^*)$.

\begin{lem}\label{Lemma: complex-4th-order polynomial}[Theorem 6.8(b) in \cite{henrici1974acca}]
    For a 2rd-order polynomial $p(\lambda)=\lambda^2+a\lambda+ b$, where $a, b\in \mathbb{C}$, its roots all lie within the open unit disk of the complex plane if and only if $$|b|<1$$ and $$|a-b\bar{a}|<1-|b|^2,$$ where $\bar{a}$ is the complex conjugate of $a$.   
    \end{lem}
% Suppose that $\varrho\left(\text{Jac}(\bar{T})\right)<1$.
Lemma \ref{lem:norm}, which describes how to construct an equivalent norm of matrix $J$ satisfying the contraction property when $\varrho(J) < 1$, is the key to proving Lemma \ref{lemma:combine two parts}.
\begin{lem}[Equivalent norm construction \cite{elaydi2005introduction}]\label{lem:norm}
Suppose that the matrix $J$ satisfies $\varrho(J) < 1$. Then there exists a positive definite matrix $P$ defined by
$$P = \sum_{k=0}^\infty (J^\top)^k J^k,$$

such that the induced norm
\[
  \|\tx\|_P = \sqrt{\tx^\top P \tx}
\]
satisfies
\[
  \|J\tx\|_P \leq \gamma \|\tx\|_P, \quad \forall \tx\in\mathbb{R}^d
\]
for some $\gamma \in (0,1)$.  
In other words, $J$ is a contraction in the $\|\cdot\|_P$-norm.
\end{lem}
\begin{proof} The following proof comes from standard techniques in matrix calculation, e.g., \cite{elaydi2005introduction}. We include it here for the completeness of the proof.  The key of the proof is to verify $P = \sum_{k=0}^\infty (J^\top)^k J^k$ is well defined, i.e., $P = \sum_{k=0}^\infty (J^\top)^k J^k,$ converges in component-wise. The proof of $P = \sum_{k=0}^\infty (J^\top)^k J^k$ converging in component-wise can be decomposed into the following 5 steps:

\noindent\textbf{Step 1. Make a Jordan decomposition for $J$.}
Over complex field $\mathbb{C}$, there exists an invertible matrix $V$ such that 
\[
J = V \, \operatorname{diag}(J_1,\dots,J_t)\, V^{-1},
\]
where each $J_i$ with size $r_i$ is a Jordan block of the form
\[
J_i = \lambda_i \CI_{r_i} + N_i, 
\]
with $N_i$ strictly upper triangular and $N_i^{\,r_i}=\textbf{0}$. Let
\[
m := \max_i r_i
\]
be the size of the largest Jordan block. Then
\[
\|J^k\| \;\le\; \|V\|\,\|V^{-1}\| \cdot \max_i \|J_i^k\|.
\]
\noindent\textbf{Step 2. Make binomial expansion of each single Jordan block.}
Fix one block $J_i=\lambda_i \CI_{r_i} + N_i$. By the binomial expansion,
\[
J_i^k = \sum_{s=0}^{{r_i}-1} \binom{k}{s}\,\lambda_i^{\,k-s}\, N_i^s,
\qquad k\ge 0,
\]
since $N_i^{r_i}=0$. Hence
\[
\|J_i^k\| \;\le\; \sum_{s=0}^{{r_i}-1} \binom{k}{s}\,|\lambda_i|^{\,k-s}\, \|N_i\|^s.
\]
\noindent\textbf{Step 3. Estimate the combinatorial numbers.}
First, $\binom{k}{s}\le k^s/s!$ for $0\le s\le k$.  
Second, choose $\mu$ with $\rho(J)<\mu<1$. Then $|\lambda_i|\le \rho(J)<\mu$, so
\[
|\lambda_i|^{\,k-s}\;\le\; \mu^{\,k-s} \;=\; \mu^k \mu^{-s}.
\]
Therefore,
\[
\|J_i^k\|
\;\le\; \mu^k \sum_{s=0}^{{r_i}-1} \frac{k^s}{s!}\,(\mu^{-1}\|N_i\|)^s.
\]

\medskip
\noindent\textbf{Step 4. Extract the polynomial factors.}
Since $s\le {r_i}-1$, for $k\ge 1$ we have $k^s\le k^{{r_i}-1}$. Hence
\[
\|J_i^k\| \;\le\; \mu^k\, k^{{r_i}-1}\, \sum_{s=0}^{{r_i}-1}\frac{(\mu^{-1}\|N_i\|)^s}{s!}.
\]
Define the constant
\[
C_i := \sum_{s=0}^{{r_i}-1}\frac{(\mu^{-1}\|N_i\|)^s}{s!}.
\]
Thus,
\[
\|J_i^k\| \;\le\; C_i\, k^{{r_i}-1}\, \mu^k.
\]

\medskip
\noindent\textbf{Step 5. Combine all blocks.}
Taking the maximum over all Jordan blocks yields (Recall that $m = \max_i r_i$.)
\[
\|J^k\| \;\le\; \|V\|\,\|V^{-1}\| \cdot \max_i \|J_i^k\|
\;\le\; {\Big(\|V\|\,\|V^{-1}\|\, \max_i C_i\Big)}\, 
k^{\,m-1}\, \mu^k=C_{\max}\ k^{\,m-1}\, \mu^k,
\]
where we define
\[
C_{\max}=\|V\|\,\|V^{-1}\|\, \max_i C_i.
\]
\noindent Therefore, we have 
\begin{align*}
\sum_{k=0}^{\infty}\|(J^\top)^k J^k\|\leq \sum_{k=0}^{\infty}\| J^k\|^2\leq \sum_{k=0}^{\infty}C_{\max}^2k^{\,2(m-1)}\, \mu^{2k}<\infty,
\end{align*}
i.e., $P = \sum_{k=0}^\infty (J^\top)^k J^k$ converges absolutely, thus $P$ converges in component-wise, which means $P$ is well defined. We have proved that $P = \sum_{k=0}^\infty (J^\top)^k J^k$ is well defined. Next, we procced to complete the remaining proof of Lemma \ref{lem:norm}. 

Note that for all $\tx\in \mathbb{R}^d$,
\[
\tx^{\top}P\tx=\sum_{k=0}^{\infty}\tx^{\top}(J^{\top})^k J^k\tx=\sum_{k=0}^{\infty}\| J^kx\|^2,
\]
we can easily get $P$ is positive definite.

Besides, we can verify
\[
J^{\top}PJ=J^{\top}\left(\sum_{k=0}^\infty (J^\top)^k J^k\right)J=\sum_{k=0}^\infty (J^\top)^{k+1} J^{k+1}=\sum_{k=1}^\infty (J^\top)^k J^k.
\]
Then we have
\begin{equation}\label{equ:stein equation}
P-J^{\top}PJ=\sum_{k=0}^\infty (J^\top)^k J^k-\sum_{k=1}^\infty (J^\top)^k J^k=(J^{\top})^0J^0=\CI.
\end{equation}
From (\ref{equ:stein equation}), for all $\mathbb{R}^d$, we have
\[
\tx^{\top}J^{\top}PJ\tx=\tx^{\top}(P-\CI)\tx=\tx^{\top}P\tx-\|\tx\|^2.
\]
Suppose $\lambda_{\min}$ is the smallest eigenvalue of $P^{-1}$, then we have
\begin{align*}
\tx^{\top}J^{\top}PJ\tx&=\tx^{\top}P\tx-\|\tx\|^2\\
&=\tx^{\top}P\tx-\tx^{\top}\CI\tx\\
&=\tx^{\top}P\tx-\tx^{\top}P^{\frac{1}{2}}P^{-1}P^{\frac{1}{2}}\tx\\
&\leq \tx^{\top}P\tx-\lambda_{\min}\,\tx^{\top}P\tx\\
&=(1-\lambda_{\min}) \tx^{\top}P\tx.
\end{align*}
Therefore $$J^{\top}PJ\preceq (1-\lambda_{\min}) P.$$
Since $P$ is positive definite, then $J^{\top}PJ$ is also positive definite, thus $1-\lambda_{\min}>0$. Also, $P^{-1}$ is positive definite implies $\lambda_{\min}>0$, thus $0<1-\lambda_{\min}<1$.

Then the operator norm induced by the matrix $P$, $\|\cdot\|_{P}$ will satisfy (Denote $\gamma^2=1-\lambda_{\min}\in (0,1)$)
\begin{align*}
    \|J\|_{P}^2:=\sup_{\tx\ne 0}\frac{\|J\tx\|_P}{\|\tx\|_P}=\sup_{\tx\ne 0}\frac{\tx^{\top}J^{\top}PJ\tx}{\tx^{\top}P\tx}\leq 1-\lambda_{\min}=\gamma^2,
\end{align*}
i.e., 
\[
  \|J\tx\|_P \leq \gamma \|\tx\|_P, \quad \forall \tx\in\mathbb{R}^d.
\]
\end{proof}

\begin{lem}
\label{lemma:combine two parts}
Consider (\ref{equ: non-autonomous system}). Let $\tz^*$ be a fixed point of $\bar{T}$, and suppose:
\begin{enumerate}[label=(\roman*)]
  \item The Jacobian $\text{Jac}(\bar{T}(\tz^*))$ satisfies $\varrho\left(\text{Jac}(\bar{T}(\tz^*))\right) < 1$.
  \item The perturbation satisfies $\|\CR(n, \tz)\| \leq C^{\prime} \rho^n \|\tz - \tz^*\|$ for some $C^{\prime}>0$, $\rho \in (0,1)$, uniformly for $\tz$ in a neighborhood of $\tz^*$.
\end{enumerate}
Then there exists a neighborhood $U$ of $z^*$ and constant $0<\tilde{\gamma}<1$ such that for any $z_0 \in U$, the iterates of $\{\tz_n\}$ satisfy
\[
  \|\tz_n -\tz^*\|_2 =\CO\left(\tilde{\gamma}^n\|\tz_1-\tz^*\|_2\right).
\]
    % Let $\bar{T}: M\rightarrow M$ defined on any compact set $M\subset \BR^d$ and $\tz^*$ is the unique fixed point of $\bar{T}$. Consider the following non-autonomous discrete-time system 
    % $$
    % \tz_{n+1}=T(n,\tz_n)=\bar{T}(\tz_n)+\CR(n, \tz_n).
    % $$
    % Suppose $\varrho\left(\text{Jac}(\bar{T})\right)<1$, and there exists $D>0$ and $0<Q<1$ such that
    % $$\|\CR(n, \tz_n)\|\leq D\,Q^n\|\tz_n-\tz^*\|$$ for all $\tz_n\in M$. Then there exists a norm $\|\cdot\|_{*}$ and $\epsilon>0$ with $\|\tz_0-\tz^*\|_{*}<\epsilon$ such that $\tz_n$ converges to $\tz^*$ with an exponential rate.
\end{lem}
\begin{proof}
By Lemma~\ref{lem:norm}, there exists a norm $\|\cdot\|_P$ and $\gamma \in (0,1)$ such that
\[
  \|\bar{T}(\tz) - \tz^*\|_P=\|\bar{T}(\tz) -\bar{T}(\tz^*)\|_P=\|\bar{T}(\tz-\tz^*)\|_P \leq \gamma \|\tz - \tz^*\|_P
\]
for $\tz$ close to $\tz^*$. Now write the iteration as
\[
  \tz_{n+1} - \tz^* = \bar{T}(\tz_n) - \bar{T}(\tz^*) + \CR(n, \tz_n).
\]
Taking the $\|\cdot\|_P$ norm,
\[
  \|\tz_{n+1} - \tz^*\|_P \leq \gamma \|\tz_n - \tz^*\|_P + \|\CR(n, \tz_n)\|_P.
\]
By assumption, $\|R(n, z_n)\|_P \leq C' \rho^n \|z_n - z^*\|_P$ for some $C' > 0$. Thus
\[
  \|\tz_{n+1} - \tz^*\|_P \leq \left(\gamma + C'\rho^n\right) \|\tz_n - \tz^*\|_P.
\]
Since $\rho^n \to 0$, for sufficiently large $n$ we have $\gamma + C'\rho^n < \tilde{\gamma} < 1$.  
Therefore the sequence contracts at rate $\tilde{\gamma} < 1$, implying local convergence:
\[
  \|\tz_n - \tz^*\|_P \leq \tilde{\gamma}^{n-1}\|\tz_1-\tz^*\|_P.
\]
By the Equivalence of norms in finite dimensions, we have there exists $c_1, c_2>0$ such that 
\[
c_1\|\tz\|_2\leq \|\tz\|_P\leq c_2\|\tz\|_2, \,\, \, \forall \tz.
\]
This gives the local convergence with an exponential rate in the Euclidean norm as well, i.e.,
\[
\|\tz_n - \tz^*\|_2 \leq \frac{c_2}{c_1}\tilde{\gamma}^{n-1}\|\tz_1-\tz^*\|_2.
\]
% Since $\varrho\left(\text{Jac}(\bar{T})\right)<1$
    % Recall that $\bar{T}(x^*)=x^*$. Simple algebras yields
    % \begin{align*}
    %     \|\tz_{n+1}-\tz^*\|_{*}&=\|T(n,\tz_n)-\tz^*\|_{*}\\
    %     &=\|\bar{T}(\tz_n)+\CR(n, \tz_n)-\tz^*\|_{*}\\
    %     &=\|\bar{T}(\tz_n)-\bar{T}(z^*)+\CR(n, \tz_n)\|_{*}\\
    %     &=\|\bar{T}(\tz_n)-\bar{T}(z^*)\|_{*}+\|\CR(n, \tz_n)\|_{*}\\
    %     &\leq \sup_{z\in U}\|\text{Jac}(\bar{T}(z))\|_{*}\|\tz_n-\tz^*\|_{*}+\|\CR(n, \tz_n)\|_{*}
    % \end{align*}
\end{proof}
    
\begin{lem}\label{lemma: local convergence of autonomous system}
    Suppose that $f(\tx,\ty)$ satisfies Assumption \ref{ibr}. Let $\beta\in (-1, 1)$ and $0<\rho<1$.  Set $h$, $\epsilon$ and $\beta$ such that $$ h<\min_{\lambda \in \mathrm{Sp}(\CJ)}\frac{2\sqrt{\epsilon}(1-\beta^2)|\Re(\lambda)|}{(1+\beta^2)|\lambda|^2+2\beta\left(|\Im(\lambda)|^2-|\Re(\lambda)|^2\right)}.$$ Then $\bar{T}$ converges locally with an exponential rate.
\end{lem}
\begin{proof}According to Lemma \ref{Lemma: local stability of dynamical system}, we complete the proof by proving $\varrho\left(\text{Jac}(\bar{T})\right)<1$.

     By direct computation, we can get the Jacobian matrix of $\bar{T}$ at $\tz^*=(\textbf{0},\textbf{0},\tx^*,\textbf{0},\textbf{0},\ty^*)$ is
     
     \begin{align*}
        \CM_{\CS}=\begin{bmatrix}
       \beta \CI& \textbf{0} & (1-\beta) \nabla^2_{\tx}f(\tx^*, \ty^*) & \textbf{0} &\textbf{0} & (1-\beta)\nabla_{xy}f(\tx^*, \ty^*) \\
       \\
       \textbf{0}& \rho \CI & \textbf{0}& \textbf{0} &\textbf{0} & \textbf{0} \\
       \\
       -\frac{h \beta}{\sqrt{\epsilon}}\CI& \textbf{0} & \CI-\frac{h(1-\beta)}{\sqrt{\epsilon}}  \nabla^2_{\tx}f(\tx^*, \ty^*) & \textbf{0} & \textbf{0}& -\frac{h(1-\beta)}{\sqrt{\epsilon}}\nabla_{xy}f(\tx^*, \ty^*) \\
       \\
       \textbf{0}& \textbf{0} & (1-\beta)\nabla_{yx}f(\tx^*, \ty^*) & \beta \CI & \textbf{0}& (1-\beta) \nabla^2_{\ty}f(\tx^*, \ty^*) \\
       \\
       \textbf{0}& \textbf{0} & \textbf{0} & \textbf{0} & \rho \CI & \textbf{0} \\
       \\
       \textbf{0}& \textbf{0} & \frac{h(1-\beta)}{\sqrt{\epsilon}}\nabla_{yx}f(\tx^*, \ty^*) & \frac{h \beta}{\sqrt{\epsilon}}\CI & \textbf{0} & \CI+\frac{h(1-\beta)}{\sqrt{\epsilon}}  \nabla^2_{\ty}f(\tx^*, \ty^*)  \\
    \end{bmatrix}.
    \end{align*}
    
     Obviously, $\CM_S$ has $d_1+d_2$ eigenvalues $\rho$. Next, we consider the following matrix
     
    \begin{align*}
        \CM_\CS^{(1)}=\begin{bmatrix}
       \beta \CI&  (1-\beta) \nabla^2_{\tx}f(\tx^*, \ty^*) & \textbf{0}  & (1-\beta)\nabla_{xy}f(\tx^*, \ty^*) \\
       \\
       -\frac{h \beta}{\sqrt{\epsilon}}\CI & \CI-\frac{h(1-\beta)}{\sqrt{\epsilon}}  \nabla^2_{\tx}f(\tx^*, \ty^*) & \textbf{0}  & -\frac{h(1-\beta)}{\sqrt{\epsilon}}\nabla_{xy}f(\tx^*, \ty^*) \\
       \\
       \textbf{0} & (1-\beta)\nabla_{yx}f(\tx^*, \ty^*) & \beta \CI  & (1-\beta) \nabla^2_{\ty}f(\tx^*, \ty^*)\\
       \\
       \textbf{0} & \frac{h(1-\beta)}{\sqrt{\epsilon}}\nabla_{yx}f(\tx^*, \ty^*) & \frac{h \beta}{\sqrt{\epsilon}}\CI  & \CI+\frac{h(1-\beta)}{\sqrt{\epsilon}}  \nabla^2_{\ty}f(\tx^*, \ty^*)  \\
    \end{bmatrix}.
    \end{align*}
    
Since $0<\rho<1$, if we want to prove $\varrho\left(\text{Jac}(\bar{T})\right)<1$, we only need to ensure $\varrho\left(\CM_\CS^{(1)}\right)<1$.
    
    Exchange the 1st and 4th rows, as well as 1st and 4th columns of $\CM_\CS^{(1)}$, we can get 
    
    \[
    \CM_\CS^{(2)}=\begin{bmatrix}
    \CI+\frac{h(1-\beta)}{\sqrt{\epsilon}}\nabla_{y}^2f(\tx^*, \ty^*)&\frac{h(1-\beta)}{\sqrt{\epsilon}} \nabla_{yx}f(\tx^*, \ty^*)& \frac{h\beta}{\sqrt{\epsilon}} \CI & \textbf{0}\\
    \\
    -\frac{h(1-\beta)}{\sqrt{\epsilon}} \nabla_{xy}f(\tx^*, \ty^*) & \CI-\frac{h(1-\beta)}{\sqrt{\epsilon}}\nabla_{x}^2f(\tx^*, \ty^*) & \textbf{0} &-\frac{h\beta}{\sqrt{\epsilon}} \CI \\
    \\
    (1-\beta)\nabla_{y}^2f(\tx^*, \ty^*) & (1-\beta) \nabla_{yx}f(\tx^*, \ty^*)& \beta \CI & \textbf{0}\\
    \\
   (1-\beta) \nabla_{xy}f(\tx^*, \ty^*) & (1-\beta)\nabla_{x}^2f(\tx^*, \ty^*) & \textbf{0} & \beta \CI
    \end{bmatrix}.
    \]
    
    Then we only need to ensure $\varrho\left(\CM_\CS^{(2)}\right)<1$. Calculate its characteristic polynomial:
    
     \begin{align*}
    \begin{split}
    &\det(\hat{\mu} \CI-\CM_\CS^{(2)})\\
    &=\det\left(\begin{bmatrix}
         (\hat{\mu}-1)\CI-\frac{h(1-\beta)}{\sqrt{\epsilon}}\nabla_{y}^2f(\tx^*, \ty^*)&-\frac{h(1-\beta)}{\sqrt{\epsilon}} \nabla_{yx}f(\tx^*, \ty^*)& -\frac{h\beta}{\sqrt{\epsilon}} \CI & \textbf{0}\\
         \\
    \frac{h(1-\beta)}{\sqrt{\epsilon}} \nabla_{xy}f(\tx^*, \ty^*) & (\hat{\mu}-1)\CI+\frac{h(1-\beta)}{\sqrt{\epsilon}}\nabla_{x}^2f(\tx^*, \ty^*) & \textbf{0} &\frac{h\beta}{\sqrt{\epsilon}} \CI \\
    \\
    -(1-\beta)\nabla_{y}^2f(\tx^*, \ty^*) & -(1-\beta) \nabla_{yx}f(\tx^*, \ty^*)& (\hat{\mu}-\beta) \CI & \textbf{0}\\
    \\
   -(1-\beta) \nabla_{xy}f(\tx^*, \ty^*) & -(1-\beta)\nabla_{x}^2f(\tx^*, \ty^*) & \textbf{0} & (\hat{\mu}-\beta) \CI   
        \end{bmatrix}\right).
        \end{split}
    \end{align*}
    
    Let 
    
    $$A=\begin{bmatrix}
      (\hat{\mu}-1)\CI-\frac{h(1-\beta)}{\sqrt{\epsilon}}\nabla_{y}^2f(\tx^*, \ty^*)&-\frac{h(1-\beta)}{\sqrt{\epsilon}} \nabla_{yx}f(\tx^*, \ty^*)\\
      \\
      \\
      \frac{h(1-\beta)}{\sqrt{\epsilon}} \nabla_{xy}f(\tx^*, \ty^*) & (\hat{\mu}-1)\CI+\frac{h(1-\beta)}{\sqrt{\epsilon}}\nabla_{x}^2f(\tx^*, \ty^*)
    \end{bmatrix}, B=\begin{bmatrix}
       -\frac{h\beta}{\sqrt{\epsilon}} \CI & \textbf{0}\\
       \\
        \textbf{0} &\frac{h\beta}{\sqrt{\epsilon}} \CI
    \end{bmatrix},$$ 
    
    and
    $$C=\begin{bmatrix}
      -(1-\beta)\nabla_{y}^2f(\tx^*, \ty^*) & -(1-\beta) \nabla_{yx}f(\tx^*, \ty^*)\\
      \\
      \\
      -(1-\beta) \nabla_{xy}f(\tx^*, \ty^*) & -(1-\beta)\nabla_{x}^2f(\tx^*, \ty^*)
    \end{bmatrix},\ \  D=\begin{bmatrix}
        (\hat{\mu}-\beta) \CI & \textbf{0}\\
        \\
        \textbf{0} & (\hat{\mu}-\beta) \CI
    \end{bmatrix}.$$ 
    Without loss of generality, we assume $\hat{\mu}\ne \beta$. (If $\hat{\mu}=\beta$, we can easily verify $\beta$ is the unique eigenvalue of $\CM_\CS^{(2)}$. We only need to set $\beta\in (-1, 1)$ to ensure $\varrho\left(\CM_\CS^{(2)}\right)<1$.)
    
    We state a fact that if $D$ is invertible, then $\det\left(\begin{bmatrix}
        A&B\\
        C&D
    \end{bmatrix}\right)=\det(D)\,\det(A-BD^{-1}C)$.
    
    According to this fact, we can get
    \begin{equation*}
\begin{split}
        &\det(\hat{\mu} \CI-\CM_\CS^{(2)})\\
        &=(\hat{\mu}-\beta)^{d_1+d_2}\,\det\left(\begin{bmatrix}
            (\hat{\mu}-1)\CI-\frac{h(1-\beta)}{\sqrt{\epsilon}}\nabla_{y}^2f(\tx^*, \ty^*)&-\frac{h(1-\beta)}{\sqrt{\epsilon}} \nabla_{yx}f(\tx^*, \ty^*)\\
            \\
            \\
      \frac{h(1-\beta)}{\sqrt{\epsilon}} \nabla_{xy}f(\tx^*, \ty^*) & (\hat{\mu}-1)\CI+\frac{h(1-\beta)}{\sqrt{\epsilon}}\nabla_{x}^2f(\tx^*, \ty^*)
        \end{bmatrix}\right.\\ 
        \\
        & \ \ \ \ \ \ \ \ \ \ \ \ \ \ \ \ -\left.\begin{bmatrix}
            -\frac{h\beta}{\sqrt{\epsilon}}\CI& \textbf{0}\\
            \\
            \textbf{0}& \frac{h\beta}{\sqrt{\epsilon}}\CI
        \end{bmatrix}\begin{bmatrix}
            \frac{1}{\hat{\mu}-\beta}\CI&\textbf{0}\\
            \\
            \textbf{0}&\frac{1}{\hat{\mu}-\beta}\CI
        \end{bmatrix}\begin{bmatrix}
            -(1-\beta)\nabla_{y}^2f(\tx^*, \ty^*)&-(1-\beta)\nabla_{yx}^2f(\tx^*, \ty^*)\\
            \\
            \\
            -(1-\beta)\nabla_{xy}^2f(\tx^*, \ty^*)&-(1-\beta)\nabla_{x}^2f(\tx^*, \ty^*)
        \end{bmatrix} \right)\\
        \\
        &=(\hat{\mu}-\beta)^{d_1+d_2}\,\det\left(\begin{bmatrix}
            (\hat{\mu}-1)\CI-\frac{h(1-\beta)\hat{\mu}}{\sqrt{\epsilon}(\hat{\mu}-\beta)}\nabla_{y}^2f(\tx^*, \ty^*)&-\frac{h(1-\beta)\hat{\mu}}{\sqrt{\epsilon}(\hat{\mu}-\beta)} \nabla_{yx}f(\tx^*, \ty^*)\\
            \\
            \\
      \frac{h(1-\beta)\hat{\mu}}{\sqrt{\epsilon}(\hat{\mu}-\beta)} \nabla_{xy}f(\tx^*, \ty^*) & (\hat{\mu}-1)\CI+\frac{h(1-\beta)\hat{\mu}}{\sqrt{\epsilon}(\hat{\mu}-\beta)}\nabla_{x}^2f(\tx^*, \ty^*)
        \end{bmatrix}\right) \\
        \\
        &=(\hat{\mu}-\beta)^{d_1+d_2}\,\det\left(\begin{bmatrix}
            \hat{\mu} \CI&\textbf{0}\\
            \\
            \textbf{0}&\hat{\mu}\CI
        \end{bmatrix}-\left(\begin{bmatrix}
             \CI&\textbf{0}\\
             \\
            \textbf{0}&\CI
        \end{bmatrix}-\frac{h(1-\beta)\hat{\mu}}{\sqrt{\epsilon}(\hat{\mu}-\beta)}\begin{bmatrix}
            -\nabla_{y}^2f(\tx^*, \ty^*)&-\nabla_{yx}f(\tx^*, \ty^*)\\
            \\
            \\
            \nabla_{xy}f(\tx^*, \ty^*)&\nabla_{x}^2f(\tx^*, \ty^*)
        \end{bmatrix}\right)\right).
       \end{split}
\end{equation*}
    Let $\det(\hat{\mu} \CI-\CM_\CS^{(2)})=0$. Obviously, $\beta\in(-1,1)$ have ensured that the $d_1+d_2$ roots $\hat{\mu}=\beta$ lie in the unit disk, we only need to consider 
    \begin{align*}
    \nonumber\\
        \det\left(\begin{bmatrix}
            \hat{\mu} \CI&\textbf{0}\\
            \\
            \textbf{0}&\hat{\mu}\CI
        \end{bmatrix}-\left(\begin{bmatrix}
             \CI&\textbf{0}\\
             \\
            \textbf{0}&\CI
        \end{bmatrix}-\frac{h(1-\beta)\hat{\mu}}{\sqrt{\epsilon}(\hat{\mu}-\beta)}\begin{bmatrix}
            -\nabla_{y}^2f(\tx^*, \ty^*)&-\nabla_{yx}f(\tx^*, \ty^*)\\
            \\
            \nabla_{xy}f(\tx^*, \ty^*)&\nabla_{x}^2f(\tx^*, \ty^*)
        \end{bmatrix}\right)\right)=\textbf{0},
        \nonumber\\\
    \end{align*}
    which means $\hat{\mu}$ is the eigenvalue of the following matrix
    \[
    N:=\begin{bmatrix}
             \CI&\textbf{0}\\
             \\
            \textbf{0}&\CI
        \end{bmatrix}-\frac{h(1-\beta)\hat{\mu}}{\sqrt{\epsilon}(\hat{\mu}-\beta)}\begin{bmatrix}
            -\nabla_{y}^2f(\tx^*, \ty^*)&-\nabla_{yx}f(\tx^*, \ty^*)\\
            \\
            \nabla_{xy}f(\tx^*, \ty^*)&\nabla_{x}^2f(\tx^*, \ty^*)
        \end{bmatrix}.
    \]
Note that 
\begin{equation*}\label{equ: relation between Spec(J) and the spectral of another matrix}
    \begin{split}
   &\mathrm{Sp}\left(\begin{bmatrix}
        -\nabla_{y}^2f(\tx^*, \ty^*)& -\nabla_{yx}f(\tx^*, \ty^*)\\
        \\
        \nabla_{xy}f(\tx^*, \ty^*)& \nabla_{x}^2f(\tx^*, \ty^*)
    \end{bmatrix}\right)  =\mathrm{Sp}\left(\begin{bmatrix}
        \nabla_{x}^2f(\tx^*, \ty^*)& \nabla_{xy}f(\tx^*, \ty^*)\\
        \\
        -\nabla_{yx}f(\tx^*, \ty^*)& -\nabla_{y}^2f(\tx^*, \ty^*)
    \end{bmatrix}\right)\\
    \\
    &=-\mathrm{Sp}\left(\begin{bmatrix}
        -\nabla_{x}^2f(\tx^*, \ty^*)& -\nabla_{xy}f(\tx^*, \ty^*)\\
        \\
        \nabla_{yx}f(\tx^*, \ty^*)& \nabla_{y}^2f(\tx^*, \ty^*)
    \end{bmatrix}\right)=-\mathrm{Sp}\left(\CJ\right).
    \end{split}
    \end{equation*} 
    
    Let $\{\lambda_i\}_{i=1}^{d_1+d_2}$ be the eigenvalues of the matrix $\CJ$.
    % $$\begin{bmatrix}
    %     -\nabla_{y}^2f(\tx^*, \ty^*)& -\nabla_{yx}f(\tx^*, \ty^*)\\
    %     \\
    %     \nabla_{xy}f(\tx^*, \ty^*)& \nabla_{x}^2f(\tx^*, \ty^*)
    % \end{bmatrix}.$$ 
    Then $\left\{1-\frac{h(1-\beta)\hat{\mu}}{\sqrt{\epsilon}(\hat{\mu}-\beta)}(-\lambda_i)\right\}_{i=1}^{d_1+d_2}$ are the eigenvalues of $N$. Then we can get 
    \begin{equation}\label{equation111}
        \hat{\mu}=1-\frac{h(1-\beta)\hat{\mu}}{\sqrt{\epsilon}(\hat{\mu}-\beta)}(-\lambda_i), \quad i=1,2,\cdots, d_1+d_2.
    \end{equation}
    Solving the equation (\ref{equation111}), we can get all the eigenvalues of $\CM_\CS^{(2)}$.

    Rewriting (\ref{equation111}), we can get 
    \begin{equation}\label{equ: charateristic polynomial}
    \hat{\mu}^2-\left(\beta+1+\frac{h(1-\beta)}{\sqrt{\epsilon}}\lambda_i\right)\hat{\mu}+\beta=0.
    \end{equation}

% Note that 
% \begin{equation}\label{equ: relation between Spec(J) and the spectral of another matrix}
%     \begin{split}
%     \{\sigma_i\}_{i=1}^{d_1+d_2}&=\mathrm{Sp}\left(\begin{bmatrix}
%         -\nabla_{y}^2f(\tx^*, \ty^*)& -\nabla_{yx}f(\tx^*, \ty^*)\\
%         \\
%         \nabla_{xy}f(\tx^*, \ty^*)& \nabla_{x}^2f(\tx^*, \ty^*)
%     \end{bmatrix}\right)\\
%     &=\mathrm{Sp}\left(\begin{bmatrix}
%         \nabla_{x}^2f(\tx^*, \ty^*)& \nabla_{xy}f(\tx^*, \ty^*)\\
%         \\
%         -\nabla_{yx}f(\tx^*, \ty^*)& -\nabla_{y}^2f(\tx^*, \ty^*)
%     \end{bmatrix}\right)\\
%     &=-\mathrm{Sp}\left(\begin{bmatrix}
%         -\nabla_{x}^2f(\tx^*, \ty^*)& -\nabla_{xy}f(\tx^*, \ty^*)\\
%         \\
%         \nabla_{yx}f(\tx^*, \ty^*)& \nabla_{y}^2f(\tx^*, \ty^*)
%     \end{bmatrix}\right)\\
%     &=-\mathrm{Sp}\left(\CJ\right)=\{-\lambda_i\}_{i=1}^{d_1+d_2},
%     \end{split}
%     \end{equation} 
%     where $\{\lambda_i\}_{i=1}^{d_1+d_2}$ are the eigenvalues of $\CJ$.
%     It should be pointed out that Assumption \ref{ibr} implies $\Re(\lambda_i)<0$, thus $\Re(\sigma_i)>0$.
    
    Applying Lemma \ref{Lemma: complex-4th-order polynomial} for (\ref{equ: charateristic polynomial}) with
    \[
    a=-\left(\beta+1-\frac{h(1-\beta)}{\sqrt{\epsilon}}\lambda_i\right),\,\,\,\, b=\beta,
    \]
    and solving the resulting inequalities, we can get  \begin{align}
    h&<\min_i\frac{-2\sqrt{\epsilon}(1-\beta^2)\Re(\lambda_i)}{(1+\beta^2)|\lambda_i|^2+2\beta\left(|\Im(\lambda_i)|^2-|\Re(\lambda_i)|^2\right)}\\
    \nonumber\\
    &=\min_{\lambda \in \mathrm{Sp}(\CJ)}\frac{2\sqrt{\epsilon}(1-\beta^2)|\Re(\lambda)|}{(1+\beta^2)|\lambda|^2+2\beta\left(|\Im(\lambda)|^2-|\Re(\lambda)|^2\right)},\label{equ: upper bound of h}
    \end{align}
    where we use the fact that Assumption \ref{ibr} implies $\Re(\lambda_i)<0$ in the equality.
    
    This means all roots of (\ref{equ: charateristic polynomial}) lie within the open unit disk of the complex plane, which means the spectral radius of $\bar{T}$ is less than $1$, i.e., $\varrho(\text{Jac}(\bar{T}))<1$. 
    % By Lemma \ref{Lemma: local stability of dynamical system}, we conclude $\bar{T}$ converges locally.
    \end{proof}

     \begin{lem}\label{lemma: non-autonomous system}
        Assume $f(x)$ is $C^2$. Then the non-autonomous system $\CR(n, \tz_n)$ converges locally with an exponential rate, i.e., there exists $0<Q<1$ such that $$\|\CR(n, \tz_n)\|=\CO(Q^n\|\tz_n-\tz^*\|).$$
    \end{lem}
    \begin{proof}For the non-autonomous system, we have
        \begin{equation*}
            \begin{split}
                \|\CR(n, \tz_n)\|^2&=\underbrace{h^2\left\|\frac{\tilde{\tm}_{n+1}}{\sqrt{\tilde{\tv}_{n+1}+\epsilon}}-\frac{\sqrt{1-\rho^{n+1}}}{1-\beta^{n+1}}\frac{\tilde{\tm}_{n+1}}{\sqrt{\tilde{\tv}_{n+1}+(1-\rho^{n+1})\epsilon}}\right\|^2}_{\mathrm{Term 1}}\\
                &+\underbrace{h^2\left\|\frac{\hat{\tm}_{n+1}}{\sqrt{\hat{\tv}_{n+1}+\epsilon}}-\frac{\sqrt{1-\rho^{n+1}}}{1-\beta^{n+1}}\frac{\hat{\tm}_{n+1}}{\sqrt{\hat{\tv}_{n+1}+(1-\rho^{n+1})\epsilon}}\right\|^2}_{\mathrm{Term 2}}.
            \end{split}
        \end{equation*}
        For  Term 1, we have
        \begin{equation*}
            \begin{split}
                &\mathrm{Term 1}\\
                &=h^2\left\|\frac{\tilde{\tm}_{n+1}}{\sqrt{\tilde{\tv}_{n+1}+\epsilon}}-\frac{\sqrt{1-\rho^{n+1}}}{1-\beta^{n+1}}\frac{\tilde{\tm}_{n+1}}{\sqrt{\tilde{\tv}_{n+1}+\epsilon}}+\frac{\sqrt{1-\rho^{n+1}}}{1-\beta^{n+1}}\left(\frac{\tilde{\tm}_{n+1}}{\sqrt{\tilde{\tv}_{n+1}+\epsilon}}-\frac{\tilde{\tm}_{n+1}}{\sqrt{\tilde{\tv}_{n+1}+(1-\rho^{n+1})\epsilon}}\right)\right\|^2\\
                \\
                &\leq 2h^2\left\|\frac{\tilde{\tm}_{n+1}}{\sqrt{\tilde{\tv}_{n+1}+\epsilon}}-\frac{\sqrt{1-\rho^{n+1}}}{1-\beta^{n+1}}\frac{\tilde{\tm}_{n+1}}{\sqrt{\tilde{\tv}_{n+1}+\epsilon}}\right\|^2+2h^2\left\|\frac{\sqrt{1-\rho^{n+1}}}{1-\beta^{n+1}}\left(\frac{\tilde{\tm}_{n+1}}{\sqrt{\tilde{\tv}_{n+1}+\epsilon}}-\frac{\tilde{\tm}_{n+1}}{\sqrt{\tilde{\tv}_{n+1}+(1-\rho^{n+1})\epsilon}}\right)\right\|^2\\
                \\
                &=2h^2\left(1-\frac{\sqrt{1-\rho^{n+1}}}{1-\beta^{n+1}}\right)^2\frac{\|\tilde{\tm}_{n+1}\|^2}{\tilde{\tv}_{n+1}+\epsilon}+2h^2\left(\frac{\sqrt{1-\rho^{n+1}}}{1-\beta^{n+1}}\right)^2\left\|\frac{\tilde{\tm}_{n+1}}{\sqrt{\tilde{\tv}_{n+1}+\epsilon}}-\frac{\tilde{\tm}_{n+1}}{\sqrt{\tilde{\tv}_{n+1}+(1-\rho^{n+1})\epsilon}}\right\|^2\\
                \\
                &\leq \frac{2h^2\|\tilde{\tm}_{n+1}\|^2}{\epsilon}\left(\frac{\sqrt{1-\rho^{n+1}}-(1-\beta^{n+1})}{1-\beta^{n+1}}\right)^2+\frac{2h^2\|\tilde{\tm}_{n+1}\|^2}{(1-\beta)^2}\left(\frac{1}{\sqrt{\tilde{\tv}_{n+1}+\epsilon}}-\frac{1}{\sqrt{\tilde{\tv}_{n+1}+(1-\rho^{n+1})\epsilon}}\right)^2,
            \end{split}
        \end{equation*}
        where we use the fact $\|\tu+\tv\|^2\leq 2(\|\tu\|^2+\|\tv\|^2)$ in the first inequality.
        
         Direct computation yields
         \begin{align*}
           \left(\frac{\sqrt{1-\rho^{n+1}}-(1-\beta^{n+1})}{1-\beta^{n+1}}\right)^2&=\left(\frac{1-\rho^{n+1}-(1-\beta^{n+1})^2}{(1-\beta^{n+1})\left(\sqrt{1-\rho^{n+1}}+(1-\beta^{n+1})\right)}\right)^2  \\
           \\
           &\leq \left(\frac{-\rho^{n+1}+2\beta^{n+1}-\left(\rho^{2}\right)^{n+1}}{(1-\beta)\left(\sqrt{1-\rho}+1-\beta\right)}\right)^2\\
           \\
           &\leq \left(\frac{|-\rho^{n+1}|+2|\beta^{n+1}|+|-\left(\rho^{2}\right)^{n+1}|}{(1-\beta)\left(\sqrt{1-\rho}+1-\beta\right)}\right)^2\\
           \\
           &\leq \left(\frac{4Q^{n+1}}{(1-\beta)\left(\sqrt{1-\rho}+1-\beta\right)}\right)^2 = C_1Q^{2n+2}
         \end{align*}
         Here
         $$Q=\max\{|\rho|, |\beta|^2, |\rho^2|\},\ \ \ C_1=\frac{16}{(1-\beta)^2(\sqrt{1-\rho}+1-\beta)^2},$$ and we use the fact $$|a+b+c|^2\leq (|a|+|b|+|c|)^2$$ in the second inequality.

         Simple algebra operation yields
         \begin{align*}
          &\left(\frac{1}{\sqrt{\tilde{\tv}_{n+1}+\epsilon}}-\frac{1}{\sqrt{\tilde{\tv}_{n+1}+(1-\rho^{n+1})\epsilon}}\right)^2\\
          &= \left(\frac{\rho^{n+1}\epsilon}{\sqrt{\tilde{\tv}_{n+1}+\epsilon}\sqrt{\tilde{\tv}_{n+1}+(1-\rho^{n+1})\epsilon}\left(\sqrt{\tilde{\tv}_{n+1}+\epsilon}+\sqrt{\tilde{\tv}_{n+1}+(1-\rho^{n+1})\epsilon}\right)}\right)^2\\
          &\leq \left(\frac{\rho^{n+1}\epsilon}{\sqrt{\epsilon}\sqrt{(1-\rho)\epsilon}\left(\sqrt{\epsilon}+\sqrt{(1-\rho)\epsilon}\right)}\right)^2\\
          &=\left(\frac{\rho^{n+1}}{\sqrt{\epsilon}\sqrt{(1-\rho)}\left(1+\sqrt{1-\rho}\right)}\right)^2 = C_2\rho^{2n+2},
         \end{align*}
         where $$C_2=\frac{1}{\epsilon(1-\rho)(1+\sqrt{1-\rho})^2}.$$
         
         Since $f(x)$ is $C^2$, $\|\nabla^2f(\tx, \ty)\|$ is bounded in any bounded set, which means $\nabla_{x}f(\tx,\ty)$ is locally Lipschitz in any bounded set, i.e., there exists $L>0$ such that for all $(\tx_1, \ty_1), (\tx_2, \ty_2)$ in the neighborhood $B(\tx^*, \ty^*)$ of $(\tx^*, \ty^*)$, $$\|\nabla_{x}f(\tx_1, \ty_1)-\nabla_{x}f(\tx_2, \ty_2)\|\leq L\|(\tx_1, \ty_1)-(\tx_2,\ty_2)\|.$$ 
         
         Recall that the fixed point $\left(\tilde{\tm}^*, \tilde{\tv}^*, \tx^*, \hat{\tm}^*, \hat{\tv}^*, \ty^*\right)=(\textbf{0},\textbf{0}, \tx^*, \textbf{0}, \textbf{0}, \ty^*)$.
    \begin{equation}\label{equ:39}
    \begin{split}
\|\tilde{\tm}_{n+1}\|&=\|\tilde{\tm}_{n+1}-\tilde{\tm}^*\|\\
&\leq \beta\|\tilde{\tm}_n-\tilde{\tm}^*\|+(1-\beta)\|\nabla_{x}f^{(n)}(\tx_n, \ty_n)-\nabla_{x}f^{(n)}(\tx^*, \ty^*)\|\\
    &\leq \beta\|\tilde{\tm}_n-\tilde{\tm}^*\|+(1-\beta)L\|(\tx_n, \ty_n)-(\tx^*, \ty^*)\|.
    \end{split}
    \end{equation}
    In fact, $$\beta\|\tilde{\tm}-\tilde{\tm}^*\|+(1-\beta)L\|(\tilde{\tx}, \tilde{\ty})-(\tx^*, \ty^*)\|$$ corresponds to a norm  
    \begin{equation}
        \label{equ:40}\|(\tilde{\tm}, \tilde{\tx}, \tilde{\ty})\|_*:=\beta\|\tilde{\tm}\|+(1-\beta)L\|(\tilde{\tx}, \tilde{\ty})\|.
        \end{equation}
    Next, we verify that $\|\cdot\|_{*}$ is exactly a norm:
    \begin{itemize}
        \item \textbf{Positive definiteness.} Obviously, we have
        $$\|(\tilde{\tm}, \tilde{\tx}, \tilde{\ty})\|_*=\beta\left\|\tilde{\tm}\right\|+(1-\beta)L\left\|(\tilde{\tx}, \tilde{\ty})\right\|\geq 0$$ and $\|(\tilde{\tm}, \tilde{\tx}, \tilde{\ty})\|_*=\textbf{0}$ iff $(\tilde{\tm}, \tilde{\tx}, \tilde{\ty})=(\textbf{0}, \textbf{0}, \textbf{0})$, since $0<\beta<1$ and $L>0$.
        \item \textbf{Absolute homogeneity.} For all $a$ and all $(\tilde{\tm}, \tilde{\tx}, \tilde{\ty})$, we have
        \begin{align*}
        \left\|a\ (\tilde{\tm}, \tilde{\tx}, \tilde{\ty})\right\|_*&=\beta\left\||a|\ \tilde{\tm}\right\|+(1-\beta)L\left\||a|\ (\tilde{\tx}, \tilde{\ty})\right\|\\
        &=|a|\ \beta \left\|\tilde{\tm}\right\|+ |a|\ (1-\beta)L\left\|(\tilde{\tx}, \tilde{\ty})\right\|\\
        &= |a|\left(\beta\left\|\tilde{\tm}\right\|+(1-\beta)L\left\|(\tilde{\tx}, \tilde{\ty})\right\|\right)\\
        &=|a|\ \left\|\ (\tilde{\tm}, \tilde{\tx}, \tilde{\ty})\right\|_*
        \end{align*}
        \item \textbf{Triangle inequality.} For all $(\tilde{\tm}_1, \tilde{\tx}_1, \tilde{\ty}_1)$ and $(\tilde{\tm}_2, \tilde{\tx}_2, \tilde{\ty}_2)$, we have
        \begin{align*}
        &\left\|(\tilde{\tm}_1, \tilde{\tx}_1, \tilde{\ty}_1)+ (\tilde{\tm}_2, \tilde{\tx}_2, \tilde{\ty}_2) \right\|_*\\
        &=\left\|\left(\tilde{\tm}_1+\tilde{\tm}_2, \tilde{\tx}_1+\tilde{\tx}_2, \tilde{\ty}_1+\tilde{\ty}_2\right)  \right\|_*\\
        &=\beta\left\|\tilde{\tm}_1+\tilde{\tm}_2\right\|+(1-\beta)L\left\|\left( \tilde{\tx}_1+\tilde{\tx}_2, \tilde{\ty}_1+\tilde{\ty}_2\right)\right\|\\
        &\leq \beta \left(\|\tilde{\tm}_1\|+\|\tilde{\tm}_2\|\right)+(1-\beta)L\left(\left\|\left(\tilde{\tx}_1, \tilde{\ty}_1\right)\right\|+\left\|\left(\tilde{\tx}_2, \tilde{\ty}_2\right)\right\|\right)\\
        &=\beta\left\|\tilde{\tm}_1\right\|+(1-\beta)L\left\|(\tilde{\tx}_1, \tilde{\ty}_1)\right\|+\beta\left\|\tilde{\tm}_2\right\|+(1-\beta)L\left\|(\tilde{\tx}_2, \tilde{\ty}_2)\right\|\\
        &=\left\|(\tilde{\tm}_1, \tilde{\tx}_1, \tilde{\ty}_1)\right\|_*+\left\|(\tilde{\tm}_2, \tilde{\tx}_2, \tilde{\ty}_2)\right\|_*.
        \end{align*}
    \end{itemize}
    We have verified $\|\cdot\|_{*}$ is exactly a norm.
    
    By the equivalence of norms, there exists $\tilde{C}>0$ such that 
    \begin{equation}\label{equ:41}
        \|(\tilde{\tm}, \tilde{\tx}, \tilde{\ty})\|_*\leq \tilde{C}\|(\tilde{\tm}, \tilde{\tx}, \tilde{\ty})\|.
        \end{equation}
    Combining (\ref{equ:39}), (\ref{equ:40}) and (\ref{equ:41}), we have
    \begin{align*}
    \|\tilde{\tm}^{(n+1)}\|=\|\tilde{\tm}^{(n+1)}-\tilde{\tm}^*\|&\leq \tilde{C}\|(\tilde{\tm}_n-\tilde{\tm}^*, {\tx}_n-\tx^*, {\ty}_n-\ty^*)\|\leq     \tilde{C}\|\tz_n-\tz^*\|,
    \end{align*}
    since $\tilde{\tm}^* = \textbf{0}$ according to Lemma \ref{lemma: fixed point=nash}.
    
    Putting all the above facts together, we have
    \begin{align*}
        \mathrm{Term\ 1}&\leq \frac{2h^2C_1\tilde{C}^2\|\tz_n-\tz^*\|^2}{\epsilon}Q^{2n+2}+\frac{2h^2C_2\tilde{C}^2\|\tz_n-\tz^*\|^2}{(1-\beta)^2}\rho^{2n+2}\\
        \\
        &\leq \left(\frac{2h^2C_1\tilde{C}^2\|\tz_n-\tz^*\|^2}{\epsilon}+\frac{2h^2C_2\tilde{C}^2\|\tz_n-\tz^*\|^2}{(1-\beta)^2}\right)Q^{2n+2}\\
        \\
        &=D_1 Q^{2n+2}\|\tz_n-\tz^*\|^2,
    \end{align*}
    where $$D_1=\frac{2h^2C_1\tilde{C}^2}{\epsilon}+\frac{2h^2C_2\tilde{C}^2}{(1-\beta)^2}.$$
    
    As for $\mathrm{Term\ 2}$, by a similar argument of $\mathrm{Term\ 1}$, we have
    \begin{align*}
      \mathrm{Term\ 2}\leq D_2Q^{2n+2}\|\tz_n-\tz^*\|^2.
    \end{align*}
    Eventually, we can get $\|\CR(n, \tz_n)\|=\CO(Q^{n}\|\tz_n-\tz^*\|)$, where $Q<1$.
    \end{proof}

    % \begin{lem}
    %      $(\tx^*, \ty^*)$ is a local Nash equilibrium if and only if $\tz^*=(0,0, x^*, 0,0, y^*)$ is the fixed point of $\bar
    %      {T}$. 
    % \end{lem}

    Lastly, putting Lemma \ref{lemma: local convergence of autonomous system}, Lemma \ref{lemma: non-autonomous system} together with Lemma \ref{lemma:combine two parts}, we can get \ref{Adam} converges locally with an exponential rate and this completes the proof of Theorem \ref{Thm: Discrete Local convergence for general case}.

\newpage
\subsection{Future Details of Figure \ref{lcpicture}.}\label{Appendix_S45}

The details of Figure \ref{lcpicture} are:
\begin{itemize}[leftmargin=*]
\item For the \ref{Adam} and \ref{CADAM}, the test function is defined as
\begin{align*}
    \min_x\max_y 0.2x^2 - xy + 0.2y^2,
\end{align*}
which is a convex-concave function. We fix $\epsilon = 10^{-3}$.
\item For Adam in minimization problem, the test function is defined as
\begin{align*}
    \min_{(x,y)}\ x^2 + y^2
\end{align*}
which is a convex function. We fix $\epsilon = 10^{-4}$.
\end{itemize}

In Figure \ref{lcpicture2}, we present how the range of $h$ changes with $\epsilon$; the x-axis represents $\epsilon$ and the y-axis represents $h$. We can observe that larger $\epsilon$ values expand the convergence range for all three methods, supporting Theorem \ref{Corollary: Continuous Local convergence for general case} and \ref{Thm: Discrete Local convergence for general case}.

\begin{figure}[h]
    \centering
    % \subfigure[Continuous Adam-DA]{
    %     \includegraphics[width=1.7in]{Pictures/beta_h_cont2.png}
    %     \label{lp1}
    % }
    % \subfigure[Adam-DA]{
    %     \includegraphics[width=1.7in]{Pictures/beta_h_disc2.png}
    %     \label{lp2}
    % }
    % \subfigure[Adam in minimization]{
    %     \includegraphics[width=1.7in]{Pictures/beta_h_mini2.png}
    %     \label{lp3}
    % }
    \subfigure[Continuous Adam-DA]{
        \includegraphics[width=1.7in]{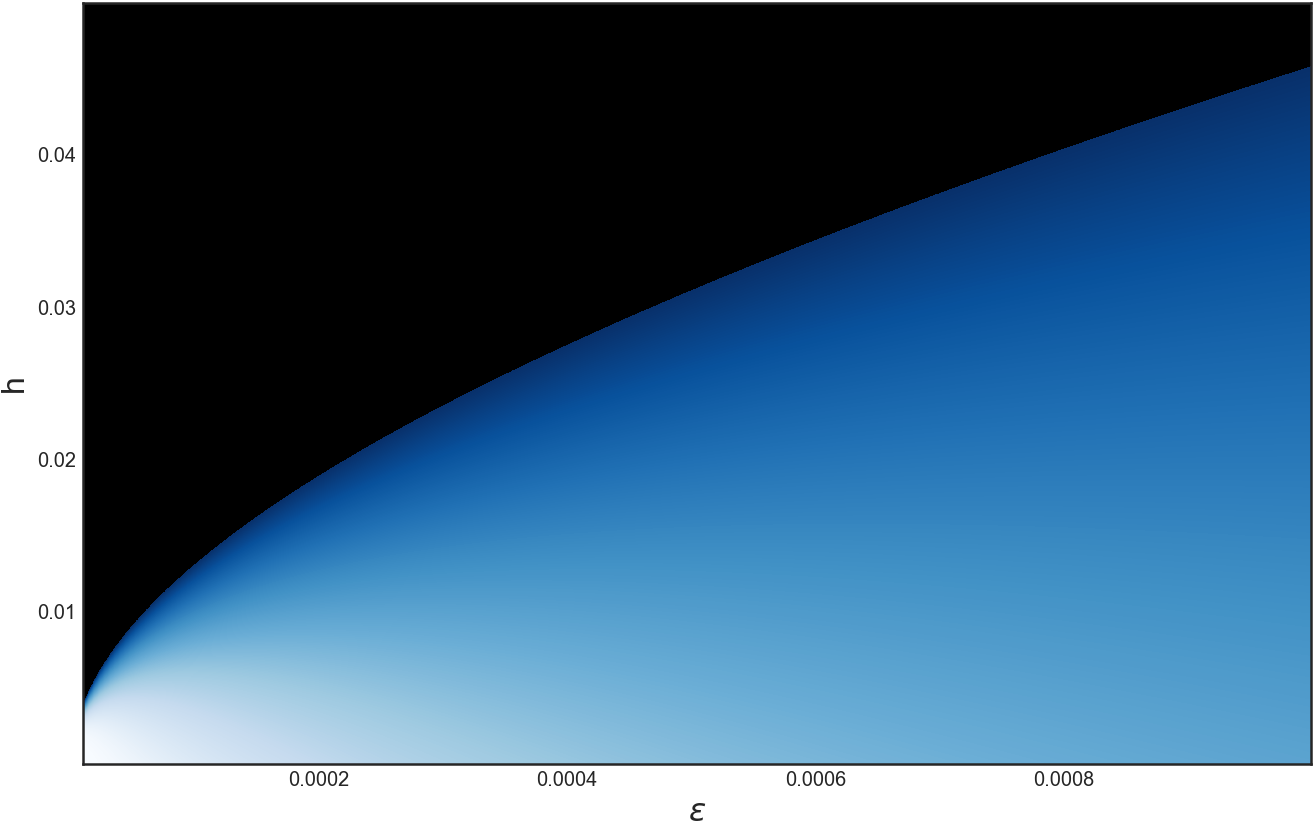}
        \label{lp4}
    }
    \subfigure[Adam-DA]{
	\includegraphics[width=1.7in]{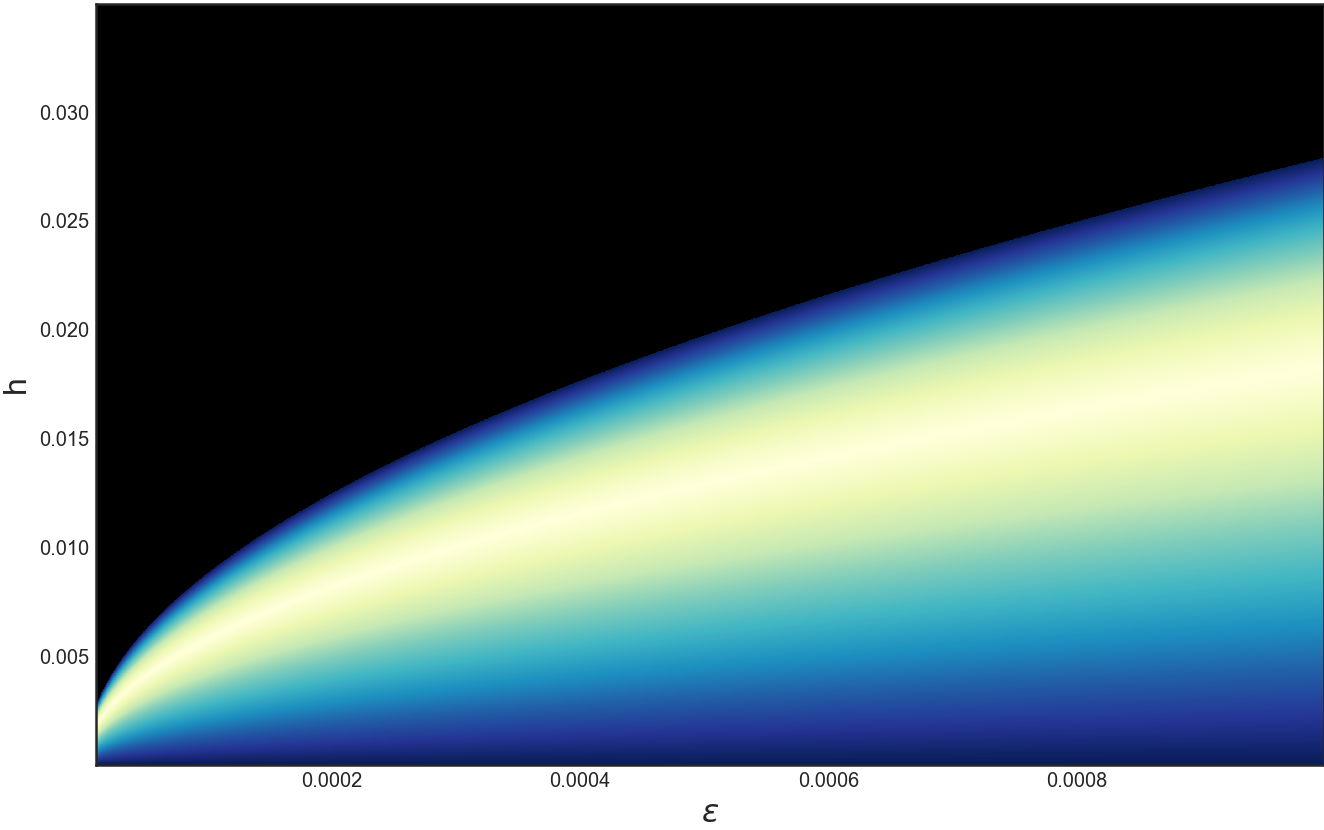}
        \label{lp5}
    }
    \subfigure[Adam in minimization]{
	\includegraphics[width=1.7in]{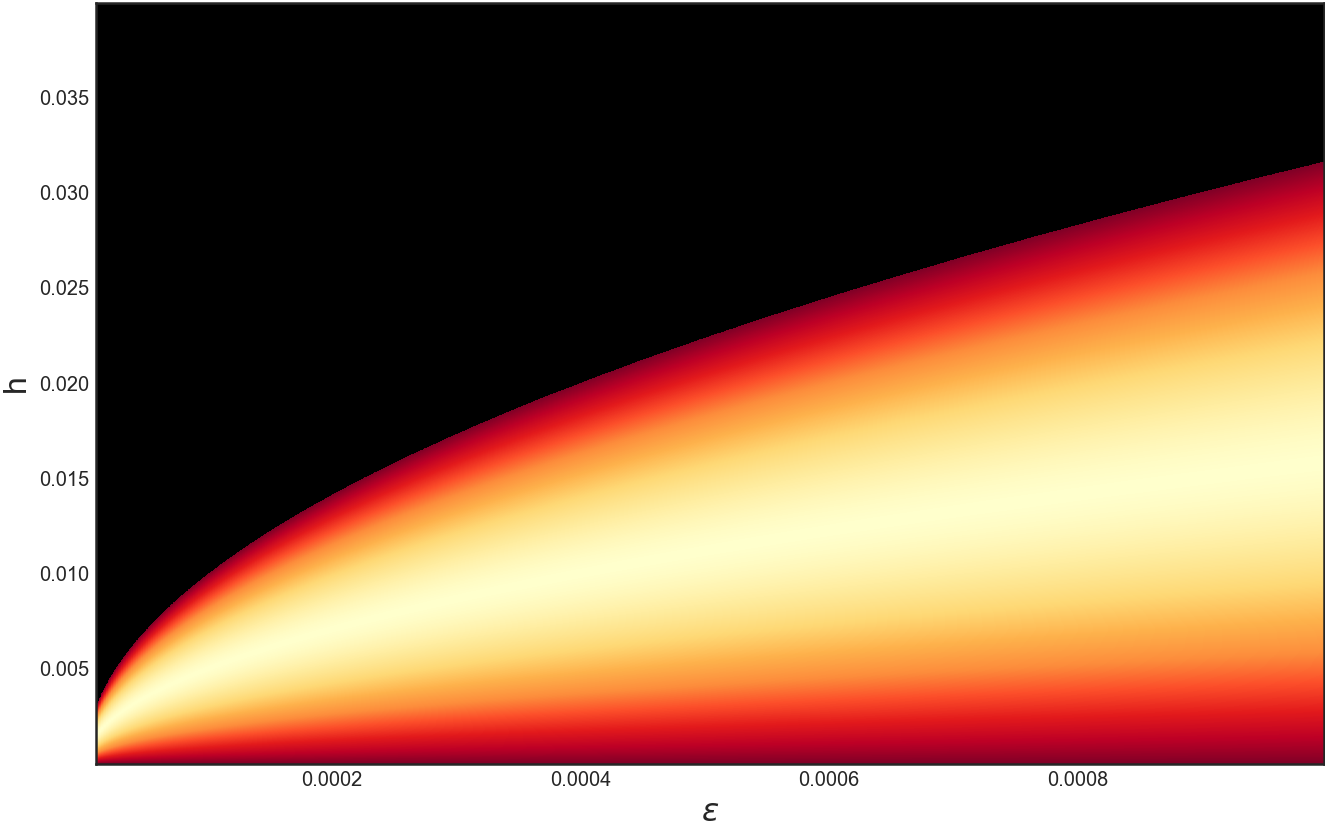}
        \label{lp6}
    }
    \caption{\textit{Effect of $\epsilon$.}}
    \label{lcpicture2}
\end{figure}

\newpage
\section{Additional Materials for Section \ref{IGR}}\label{Appendix_S5}
\subsection{Additional Experiments}

\begin{figure}[htbp]
    \centering

    \subfigure[Cumulative average gradient norms, different $\beta$\label{fig:sa_a}]{
        \includegraphics[width=0.48\textwidth]{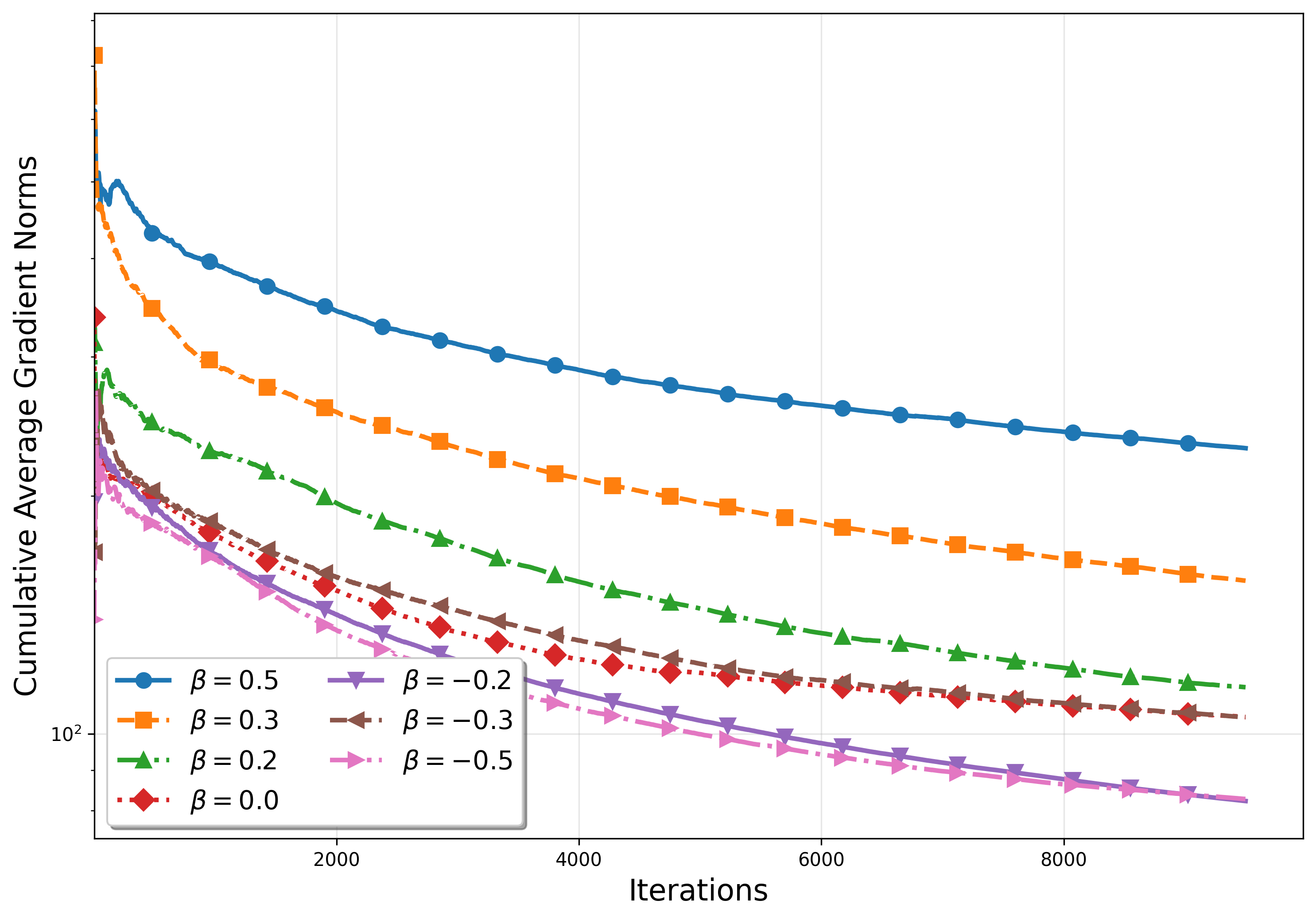}
    }
    \hfill
    \subfigure[FID, different $\beta$ and Optimistic Adaptive method.\label{fig:sa_b}]{
        \includegraphics[width=0.48\textwidth]{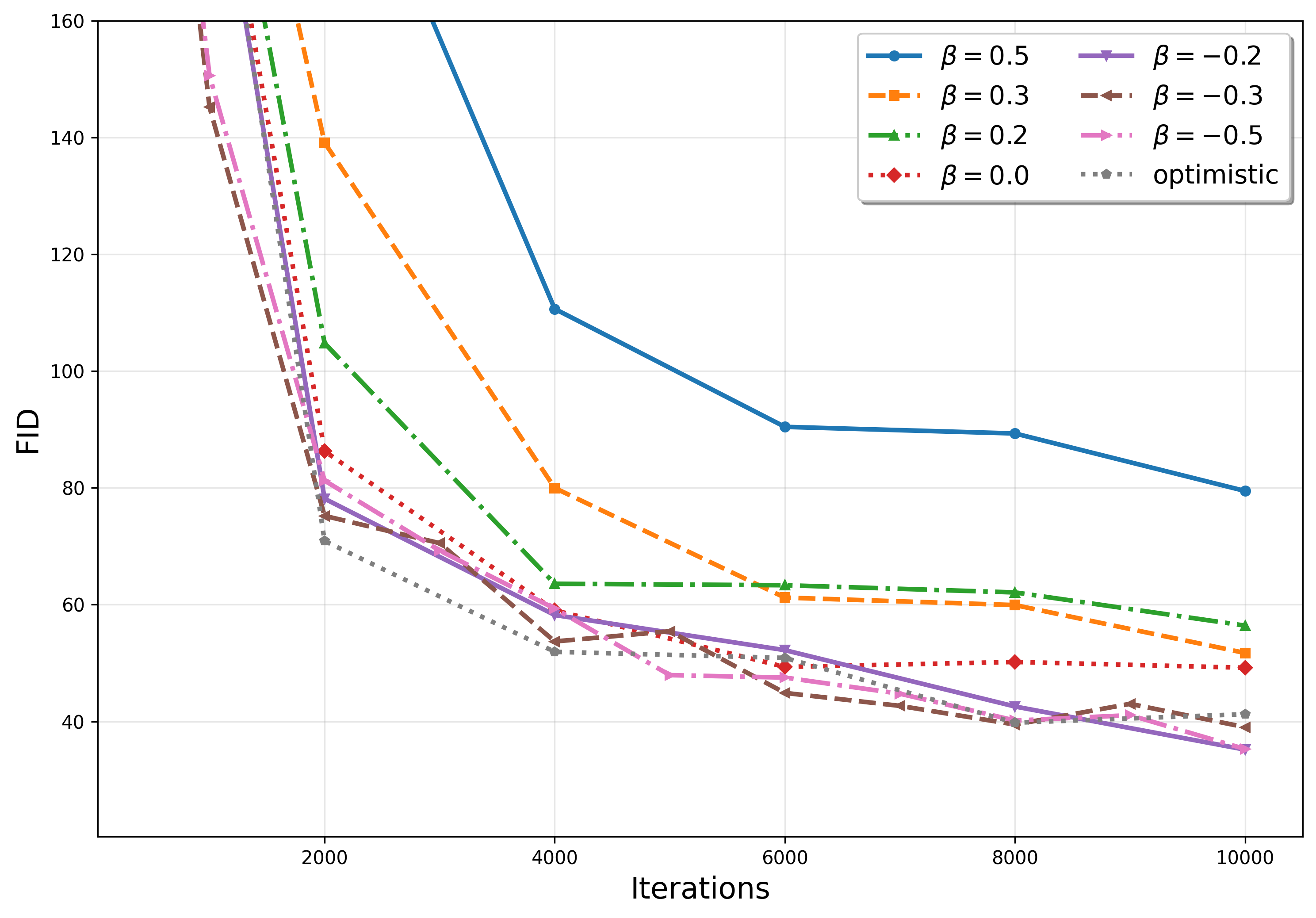}
    }

    \vspace{0.5em}

    \subfigure[Cumulative average gradient norms, different $\rho$\label{fig:sa_c}]{
        \includegraphics[width=0.48\textwidth]{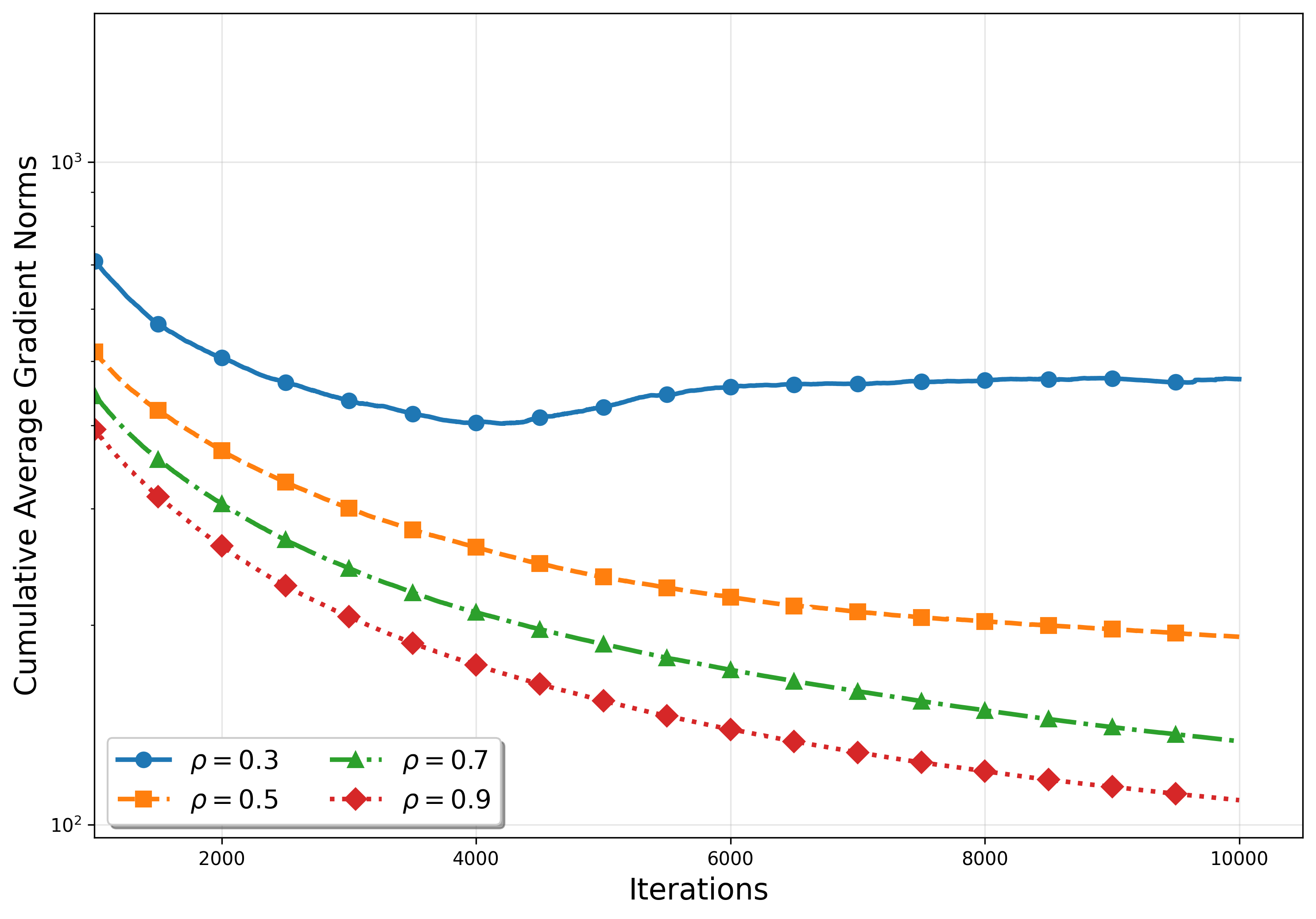}
    }
    \hfill
    \subfigure[FID, different $\rho$.\label{fig:sa_d}]{
        \includegraphics[width=0.48\textwidth]{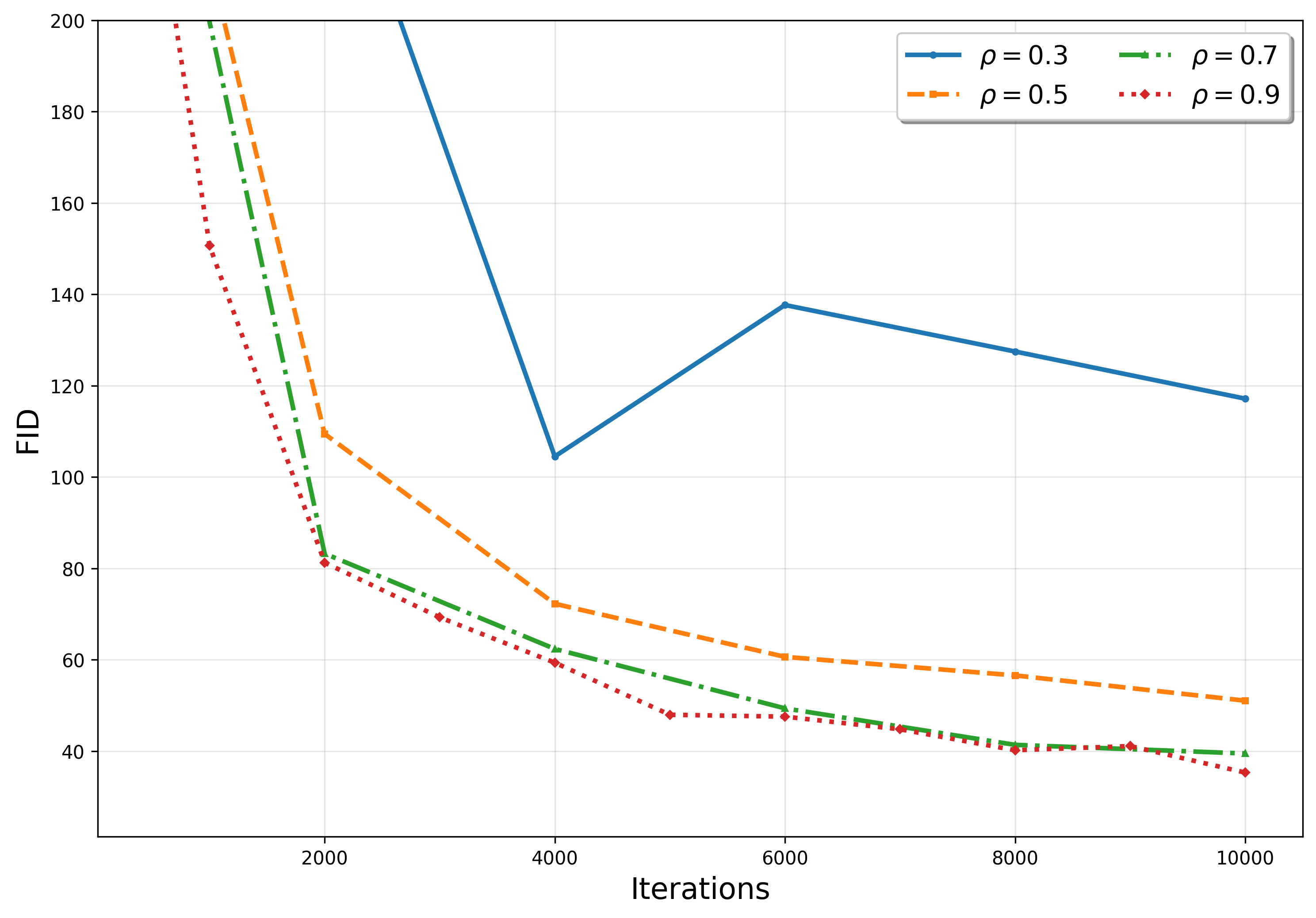}
    }

    \caption{Self-Attention GAN Experiments on CelebA, Evaluated by FID. In Figure 1(a), we present the evolution of the cumulative average gradient norms for $\rho=0.9$ and different values of $\beta$. In Figure 1(b), we compare the FID scores for these parameter settings with those of the \textit{optimistic} adaptive method proposed by Daskalakis et al. We observe that, in general, smaller $\beta$ leads to lower gradient norms and lower FID, indicating better training performance. Moreover, by comparing the optimistic method with Adam, we find that \textbf{although the optimistic method performs better than Adam with $\beta \ge 0$, it does not outperform Adam with negative $\beta$}. In Figures 1(c) and 1(d), we present the evolution of the cumulative average gradient norms and the FID scores for $\beta=-0.5$ and different values of $\rho$. We also observe that larger $\rho$ leads to lower gradient norms and lower FID, indicating better training performance. These results suggest that our findings in Section~5 also hold in this experimental setting.}
    \label{FF1}
\end{figure}

\newpage

\begin{figure}[htbp]
    \centering

    \subfigure[\label{fig:cnn_a}]{
        \includegraphics[width=0.48\textwidth]{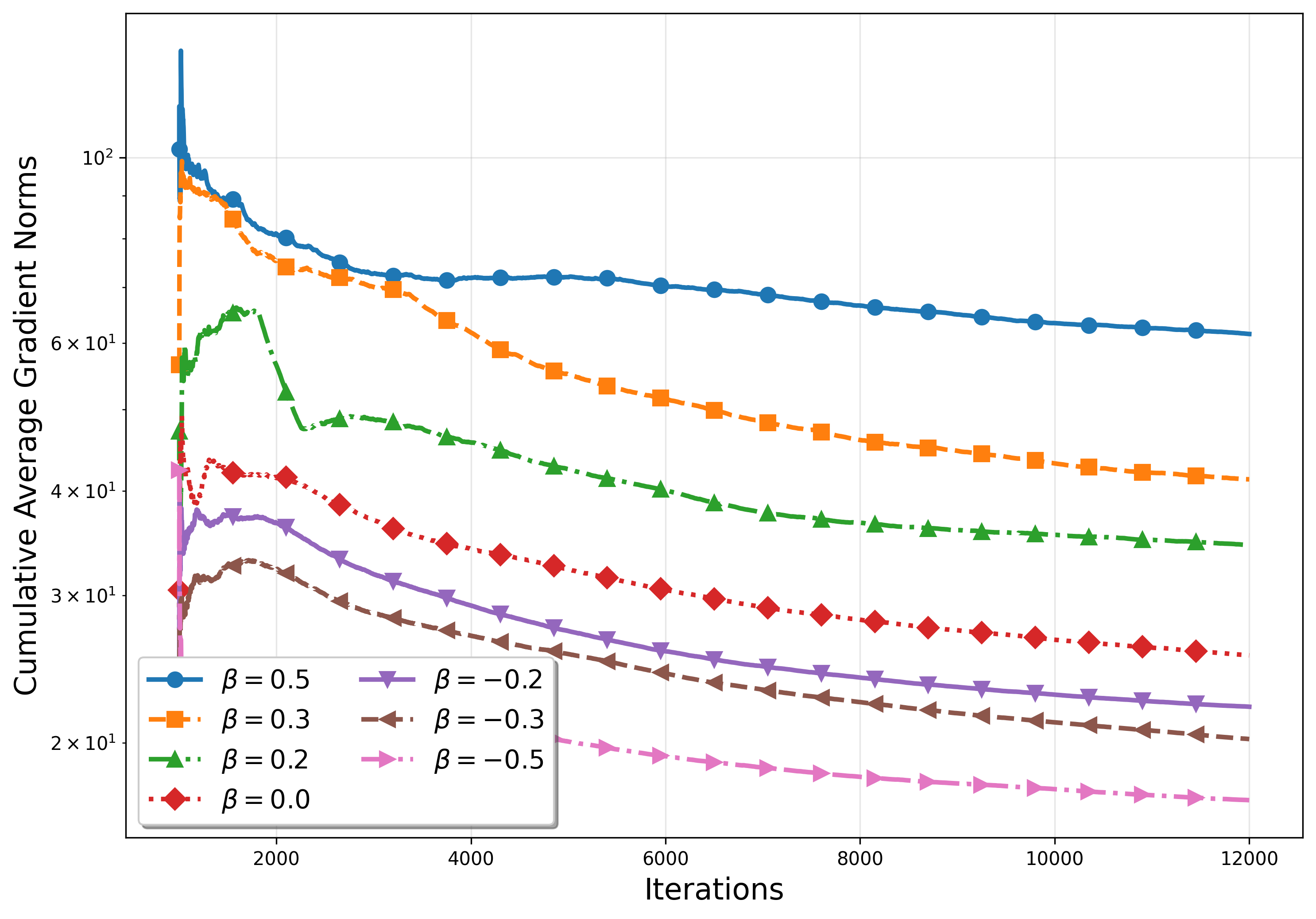}
    }
    \hfill
    \subfigure[\label{fig:cnn_b}]{
        \includegraphics[width=0.48\textwidth]{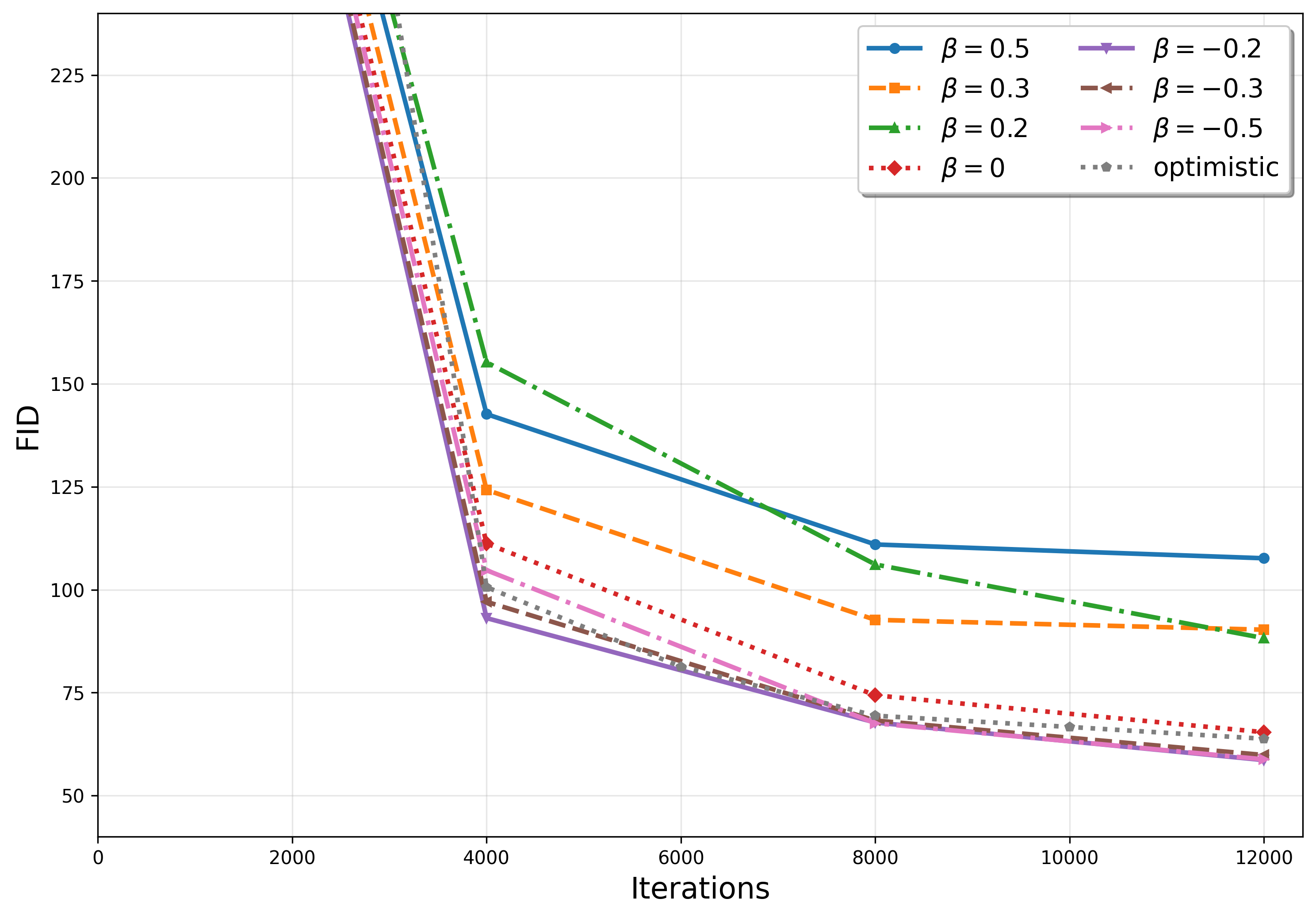}
    }

    \vspace{0.5em}

    \subfigure[\label{fig:cnn_c}]{
        \includegraphics[width=0.48\textwidth]{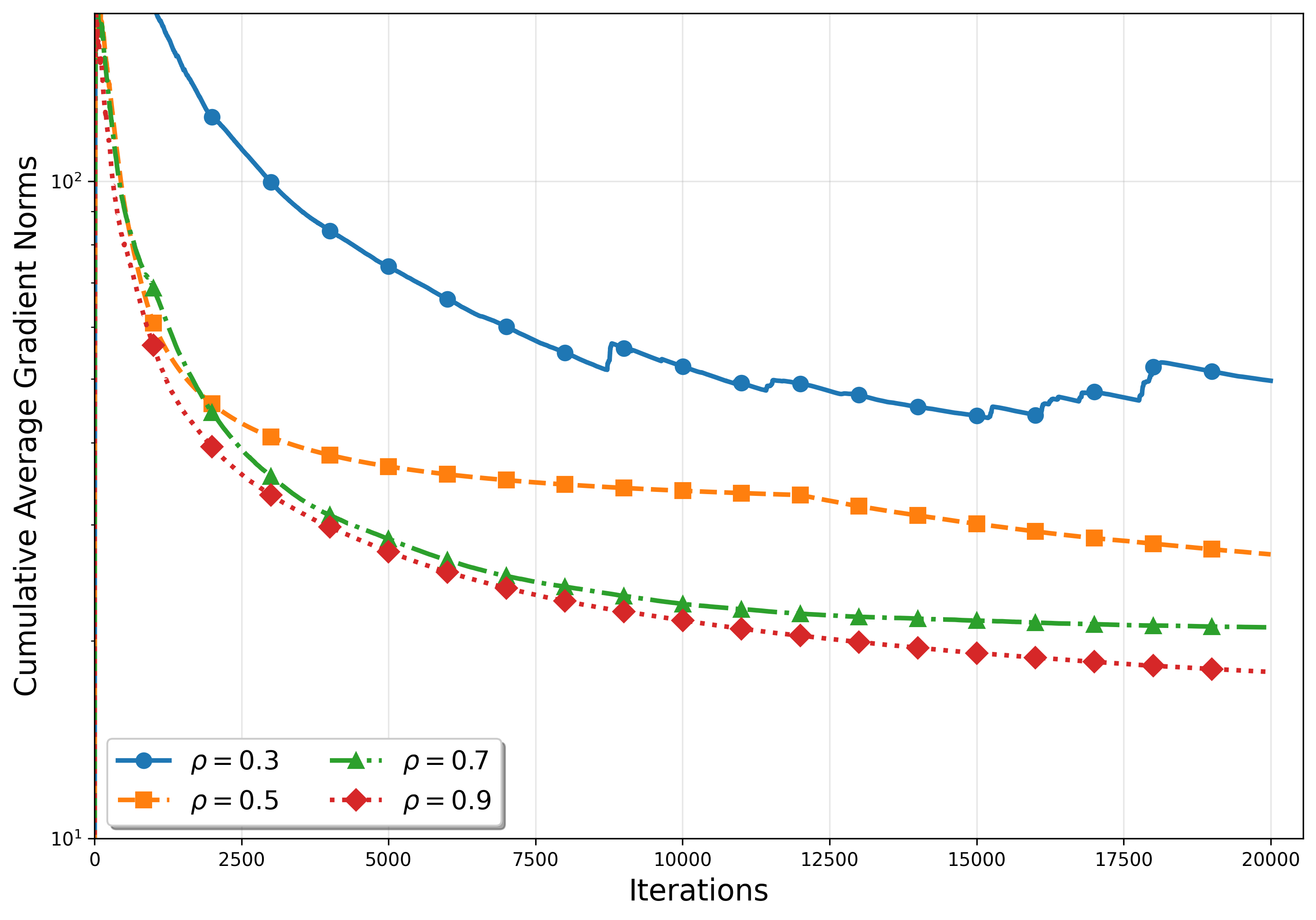}
    }
    \hfill
    \subfigure[\label{fig:cnn_d}]{
        \includegraphics[width=0.48\textwidth]{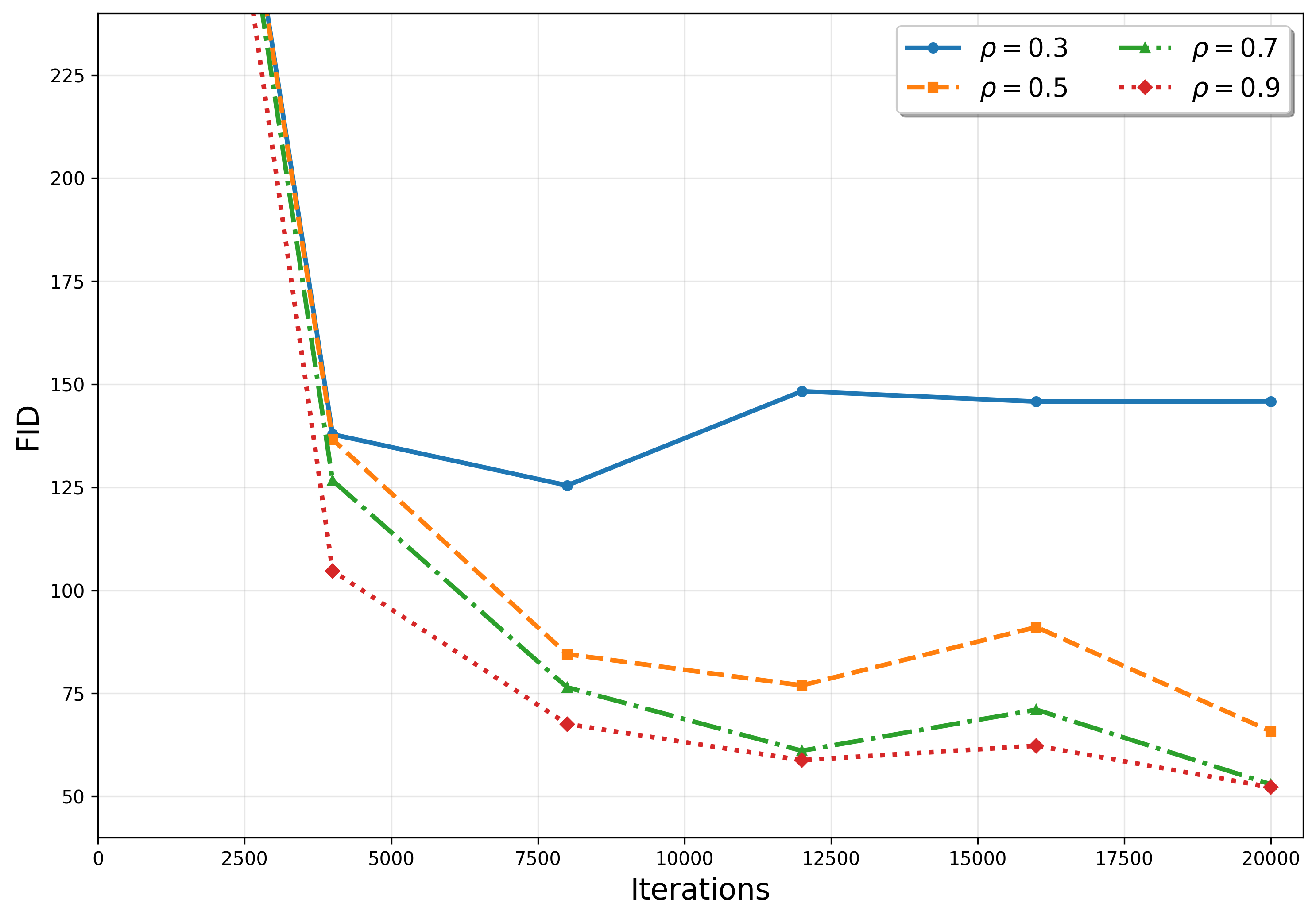}
    }

    \caption{IIn this figure, we reproduce one set of the experimental results from Section~5 of the submission on CNN GANs trained on the CIFAR-10 dataset. We evaluate performance using FID and include a comparison with the \textit{optimistic} adaptive method. The conclusion is the same as that in Figure~\ref{FF1} and is consistent with the results reported in Section~5 of the submission.}
\end{figure}

\newpage

\begin{figure}[htbp]
    \centering

    \subfigure[Self-attention GANs on CelebA.\label{fig:heatmap_a}]{
        \includegraphics[width=0.3\textwidth]{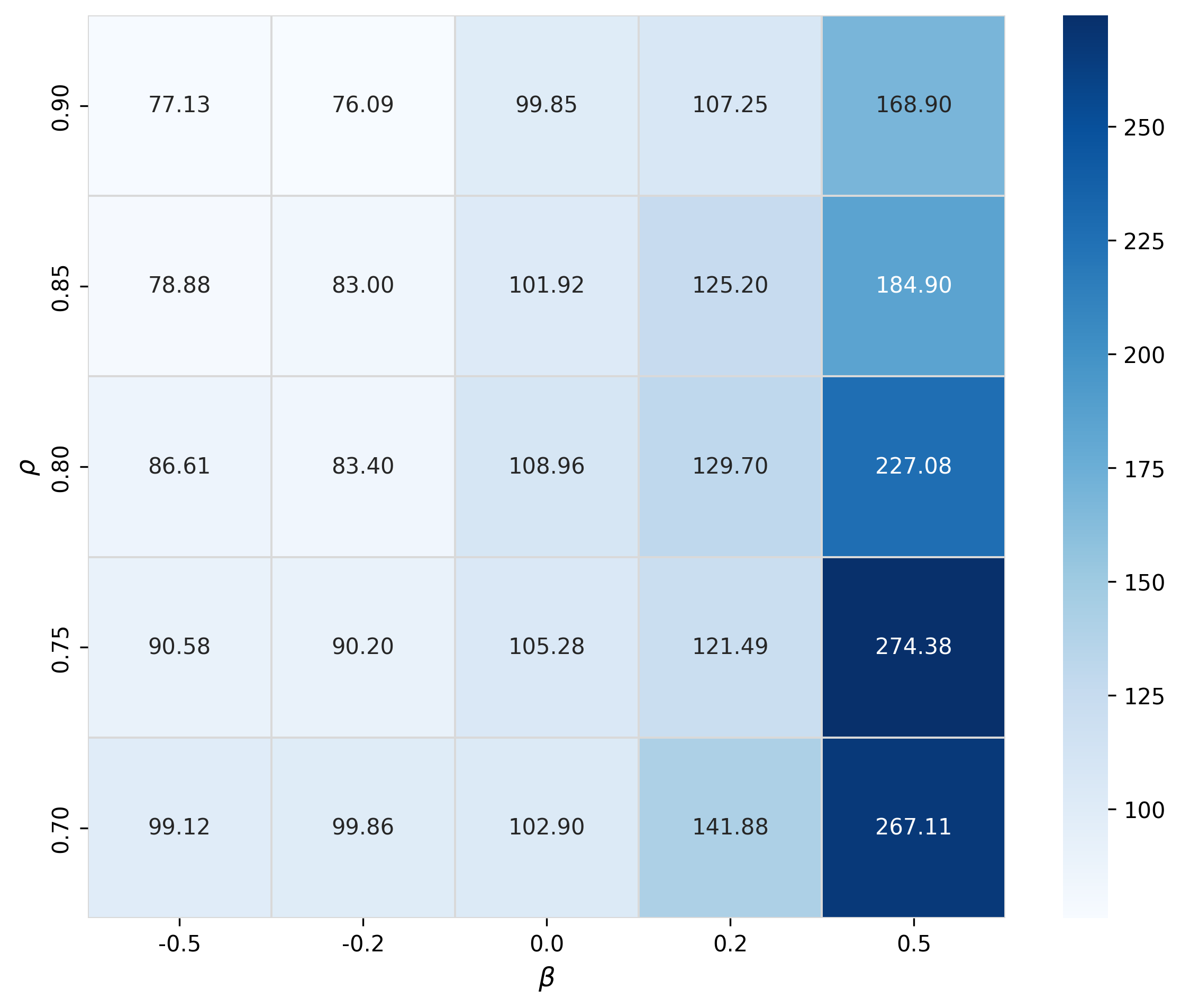}
    }
    \hfill
    \subfigure[CNN-based GANs on CIFAR-10.\label{fig:heatmap_b}]{
        \includegraphics[width=0.3\textwidth]{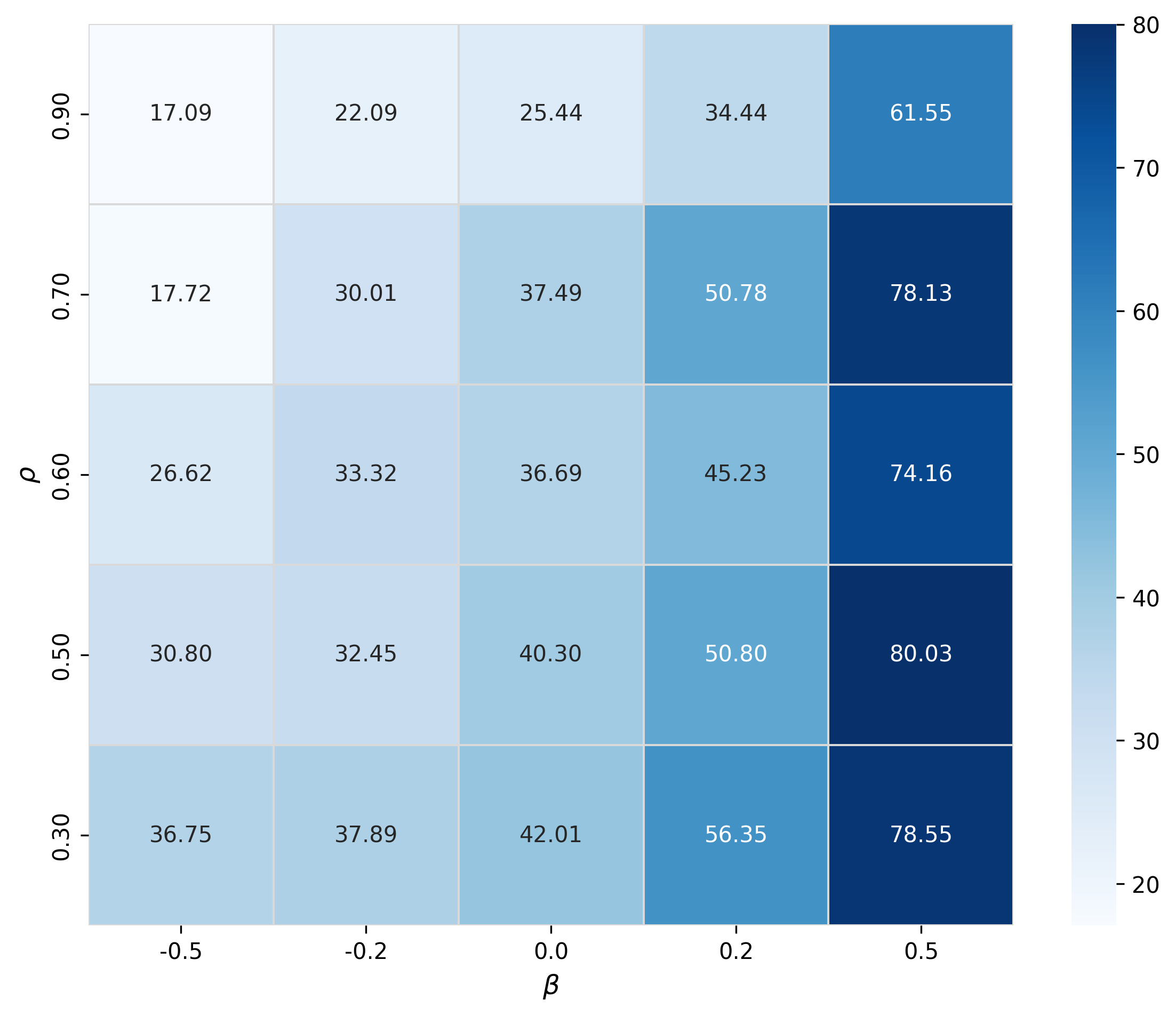}
    }
    \hfill
    \subfigure[ResNet-based GANs on STL-10.\label{fig:heatmap_c}]{
        \includegraphics[width=0.3\textwidth]{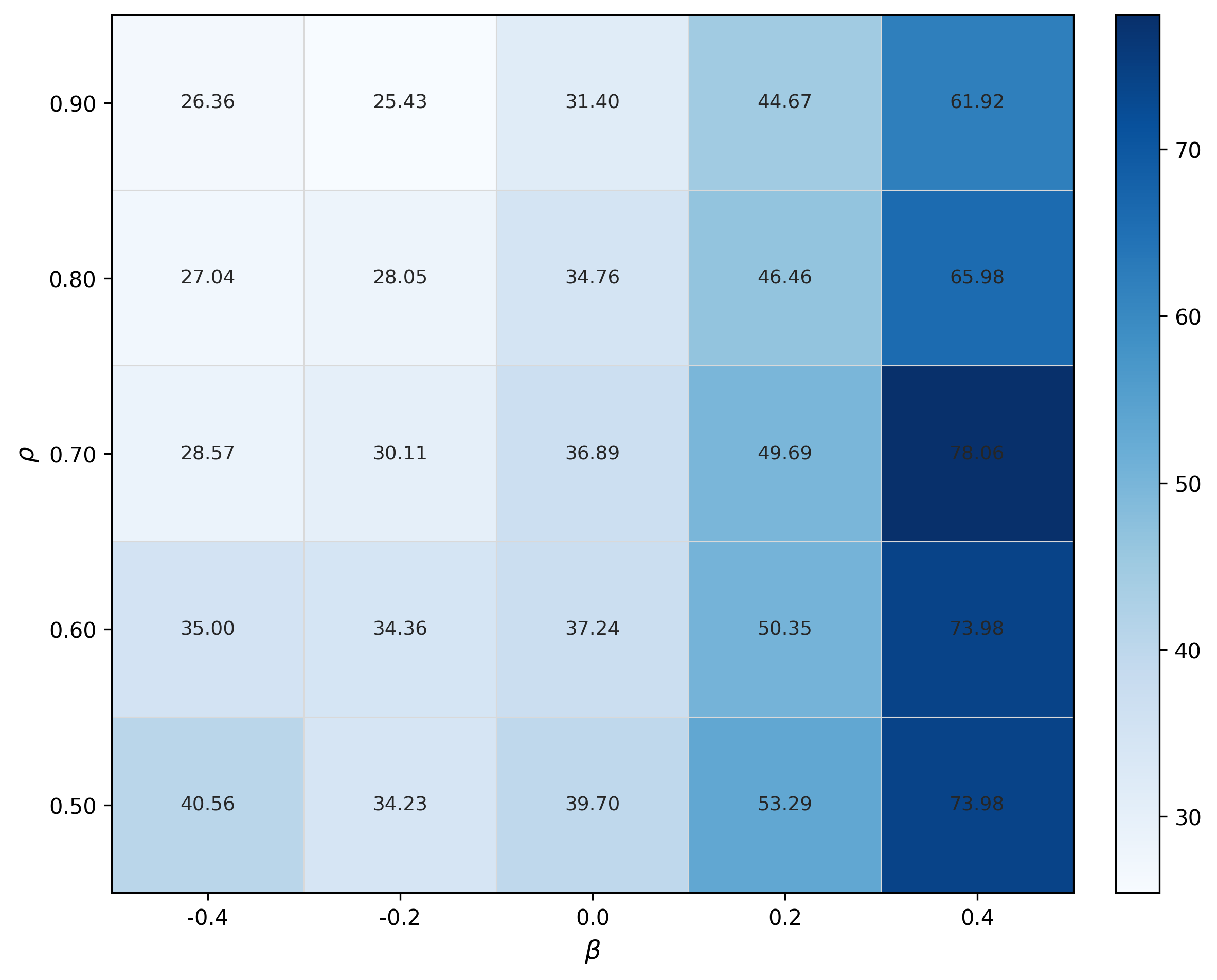}
    }

    \subfigure[FID for Self-attention GANs.\label{fig:heatmap_fid_sa}]{
        \includegraphics[width=0.3\textwidth]{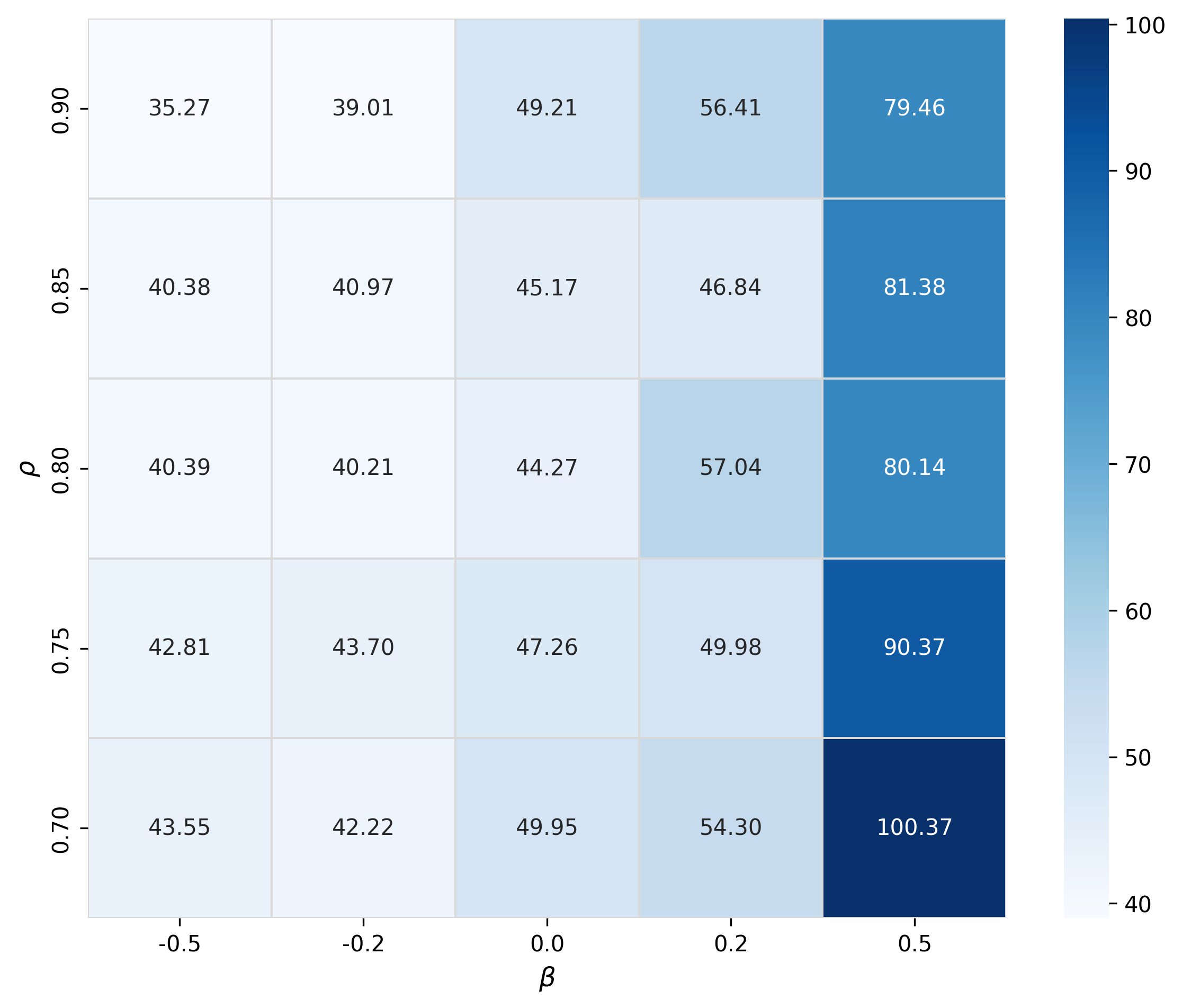}
    }
    \hfill
    \subfigure[FID for CNN-based GANs.\label{fig:heatmap_fid_cnn}]{
        \includegraphics[width=0.3\textwidth]{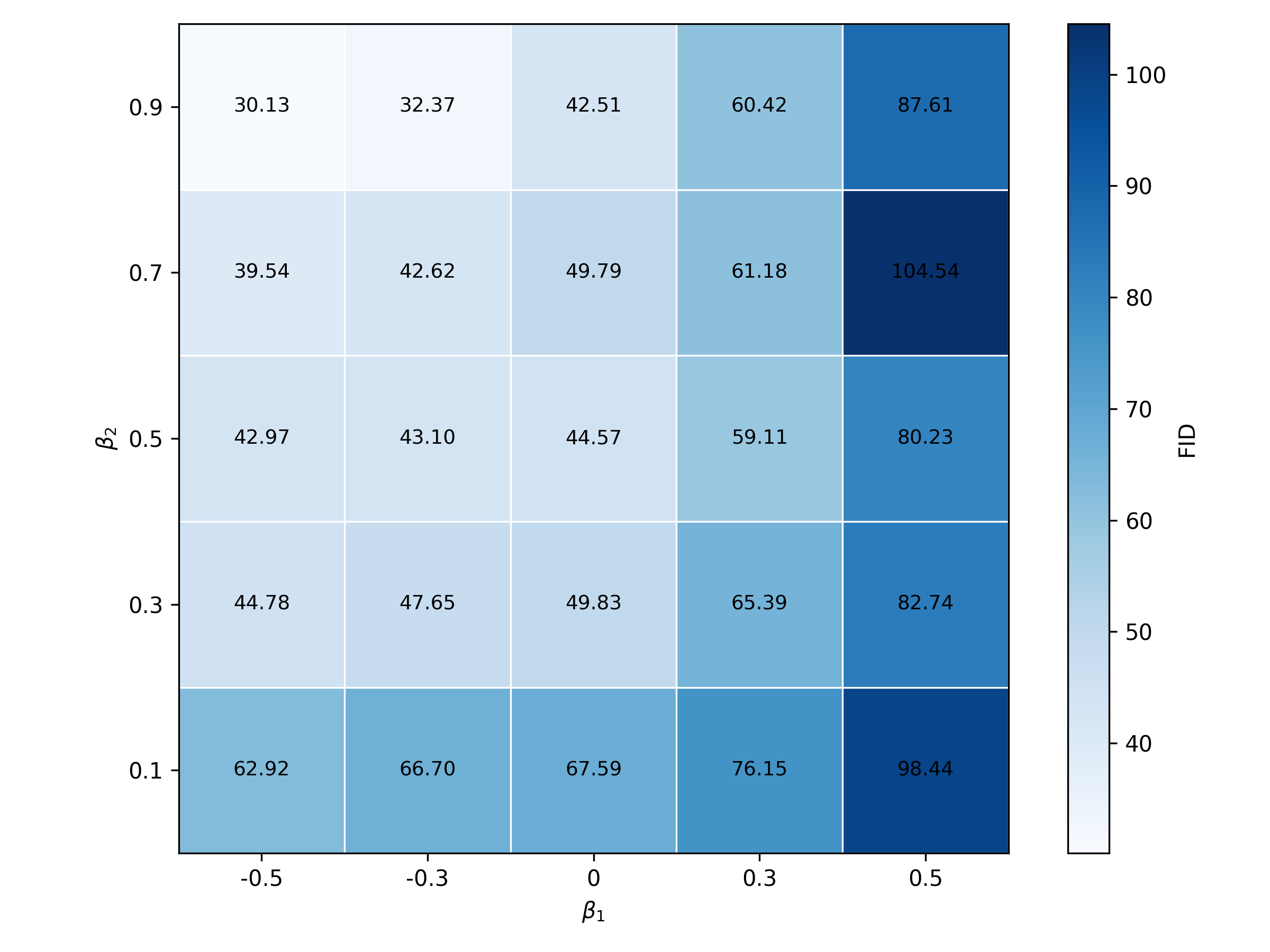}
    }
    \hfill
    \subfigure[FID for ResNet-based GANs.\label{fig:heatmap_fid_res}]{
        \includegraphics[width=0.3\textwidth]{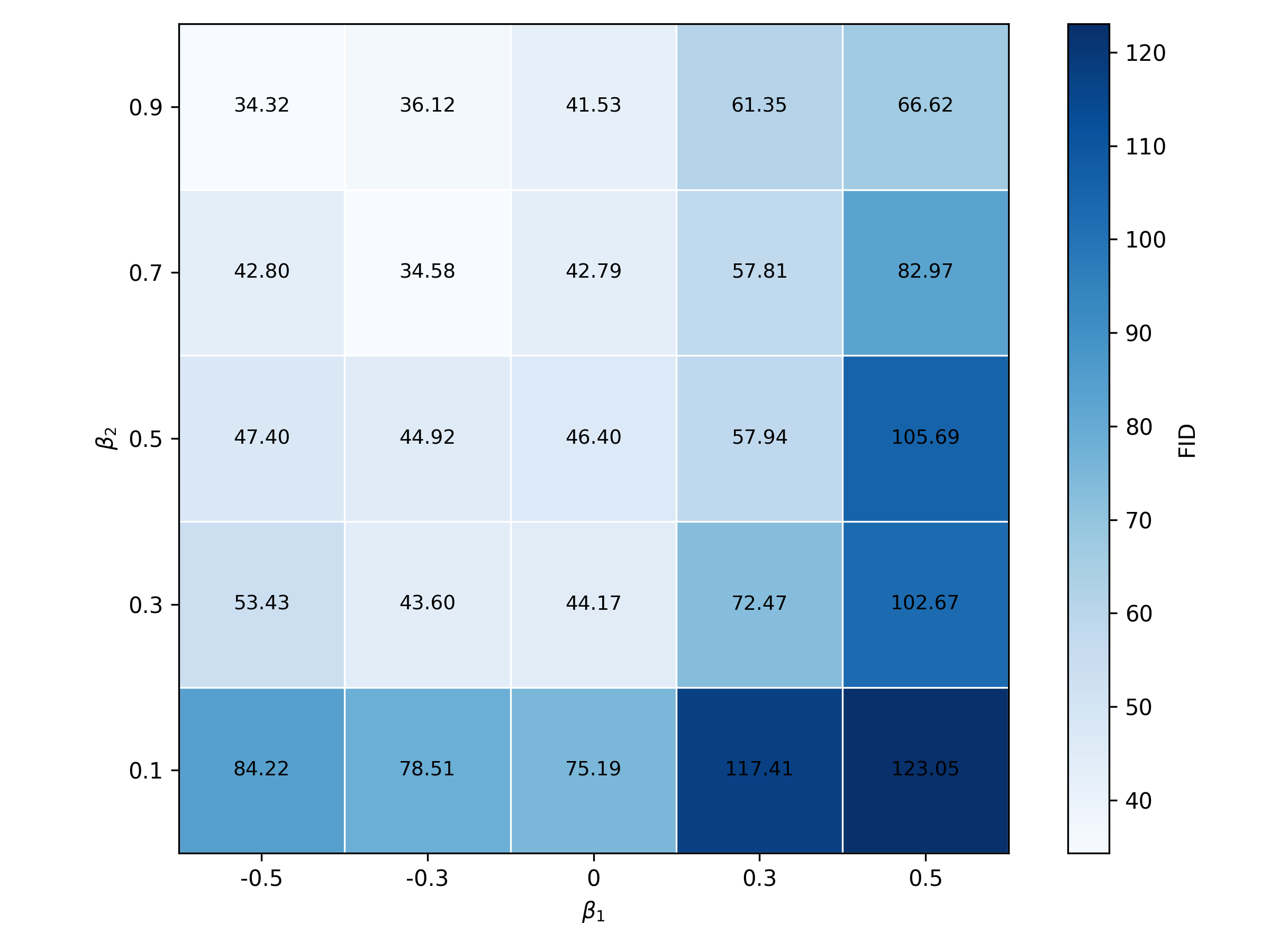}
    }

    \caption{2D sweep over $(\beta, \rho)$ jointly. Each figure represents the final cumulative average gradient norms on 25 GANs training. Each figure shows the final cumulative average gradient norms over 25 GAN training runs. We observe that the \textbf{upper-left corner of each figure exhibits smaller gradient norms than the lower-right corner}, indicating that smaller $\beta$ and larger $\rho$ guide the optimization trajectories toward flatter regions of the loss landscape. This is consistent with our findings in Section 5. Similarly, upper-left corner of each figure also exhibits smaller FID than the lower-right corner.}
\end{figure}

\newpage

\begin{figure}[htbp]
    \centering
    \subfigure[Trained by Adam, $\beta = -0.5, \rho = 0.9$\label{fig:adam_sample}]{
        \includegraphics[width=0.48\textwidth]{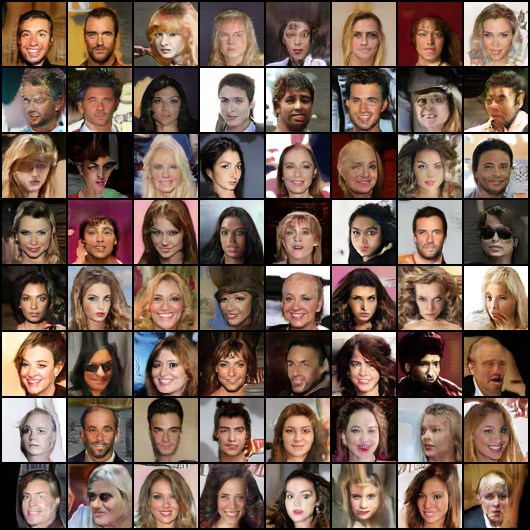}
    }
    \hfill
    \subfigure[Trained by Optimistic method\label{fig:optimistic_sample}]{
        \includegraphics[width=0.48\textwidth]{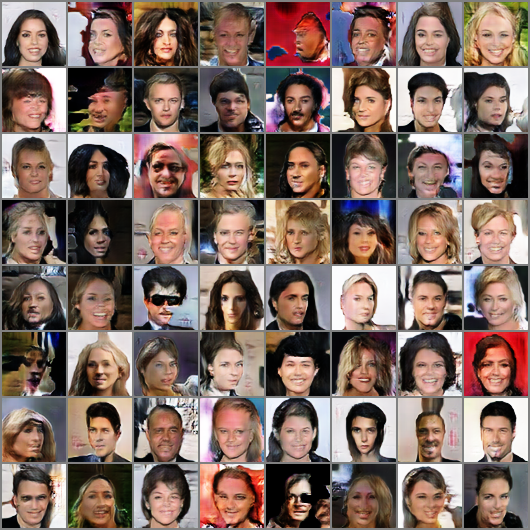}
    }
    \caption{Sample images generated by the models trained in Figure~\ref{FF1} trained for $10^4$ steps. Note that $10^4$ is fewer than the number of training steps typically required to achieve state-of-the-art GAN results. We present these images to compare Adam and the optimistic method, rather than to reproduce state-of-the-art GAN training results.}
\end{figure}

\subsection{Inception Scores}

We report the final Inception Scores from the experiments in Section~\ref{IGR}. Table~\ref{tab1} summarizes the scores across architectures and datasets when training with \ref{Adam} under different choices of $\beta$. Table~\ref{tab2} reports the corresponding results under different choices of $\rho$. Together, these tables provide the numerical values underlying Figure~\ref{Inception_score}.
\begin{table*}[h]
\centering
\caption{Inception scores under different $\rho$.}
\label{tab2}
\begin{tabular}{lcccc}
\toprule
Choice of $\rho$ & 0.9 & 0.7 & 0.5 & 0.3 \\
\midrule
ResNet, CIFAR-10 & 
\makecell{$\textbf{7.087}$ \\ {\scriptsize$\pm0.358$}} &
\makecell{$6.483$ \\ {\scriptsize$\pm0.129$}} &
\makecell{$6.308$ \\ {\scriptsize$\pm0.270$}} &
\makecell{$6.265$ \\ {\scriptsize$\pm0.198$}}\\
% &
% \makecell{$5.487$ \\ {\scriptsize$\pm0.127$}} \\
\midrule
ResNet, STL-10 & 
\makecell{$\textbf{7.187}$ \\ {\scriptsize$\pm0.457$}} &
\makecell{$6.486$ \\ {\scriptsize$\pm0.263$}} &
\makecell{$6.335$ \\ {\scriptsize$\pm0.286$}} &
\makecell{$5.571$ \\ {\scriptsize$\pm0.347$}} \\
% &
% \makecell{$3.957$ \\ {\scriptsize$\pm0.177$}} \\
\midrule
CNN, CIFAR-10 & 
\makecell{$\textbf{7.010}$ \\ {\scriptsize$\pm0.178$}} &
\makecell{$6.809$ \\ {\scriptsize$\pm0.171$}} &
\makecell{$6.685$ \\ {\scriptsize$\pm0.186$}} &
\makecell{$6.280$ \\ {\scriptsize$\pm0.135$}} \\
% &
% \makecell{$4.120$ \\ {\scriptsize$\pm0.075$}} \\
\midrule
CNN, STL-10 & 
\makecell{$\textbf{7.791}$ \\ {\scriptsize$\pm0.410$}} &
\makecell{$7.332$ \\ {\scriptsize$\pm0.426$}} &
\makecell{$6.775$ \\ {\scriptsize$\pm0.288$}} &
\makecell{$6.541$ \\ {\scriptsize$\pm0.262$}} \\
% &
% \makecell{$4.902$ \\ {\scriptsize$\pm0.158$}} \\
\bottomrule
\end{tabular}
\end{table*}

\begin{table*}[h]
\centering
\caption{Inception scores under different $\beta$ values. }
\label{tab1}
\begin{tabular}{lccccccc}
\toprule
Choice of $\beta$ & 0.5 & 0.3 & 0.2 & 0.0 & -0.2 & -0.3 & -0.5 \\
\midrule
ResNet, CIFAR-10 & 
\makecell{$4.160$ \\ {\ \ \scriptsize$\pm0.122$}} &
\makecell{$4.698$ \\ {\ \ \scriptsize$\pm0.158$}} &
\makecell{$5.217$ \\ {\ \ \scriptsize$\pm0.232$}} &
\makecell{$7.022$ \\ {\ \ \scriptsize$\pm0.282$}} &
\makecell{$7.020$ \\ {\ \ \scriptsize$\pm0.255$}} &
\makecell{$\textbf{7.087}$ \\ {\ \ \scriptsize$\pm0.358$}} &
\makecell{$7.002$ \\ {\ \ \scriptsize$\pm0.313$}} \\
\midrule
ResNet, STL-10 & 
\makecell{$4.447$ \\ {\ \ \scriptsize$\pm0.193$}} &
\makecell{$4.858$ \\ {\ \ \scriptsize$\pm0.363$}} &
\makecell{$5.565$ \\ {\ \ \scriptsize$\pm0.214$}} &
\makecell{$6.445$ \\ {\ \ \scriptsize$\pm0.289$}} &
\makecell{$6.969$ \\ {\ \ \scriptsize$\pm0.243$}} &
\makecell{$\textbf{7.181}$ \\ {\ \ \scriptsize$\pm0.457$}} &
\makecell{$6.878$ \\ {\ \ \scriptsize$\pm0.308$}} \\
\midrule
CNN, CIFAR-10 & 
\makecell{$4.942$ \\ {\ \ \scriptsize$\pm0.141$}} &
\makecell{$6.322$ \\ {\ \ \scriptsize$\pm0.169$}} &
\makecell{$6.519$ \\ {\ \ \scriptsize$\pm0.104$}} &
\makecell{$6.804$ \\ {\scriptsize$\pm0.197$}} &
\makecell{$\textbf{7.062}$ \\ {\ \ \scriptsize$\pm0.153$}} &
\makecell{$7.010$ \\ {\ \ \scriptsize$\pm0.178$}} &
\makecell{$6.761$ \\ {\ \ \scriptsize$\pm0.195$}} \\
\midrule
CNN, STL-10 & 
\makecell{$6.775$ \\ {\ \ \scriptsize$\pm0.288$}} &
\makecell{$7.178$ \\ {\ \ \scriptsize$\pm0.256$}} &
\makecell{$7.302$ \\ {\ \ \scriptsize$\pm0.227$}} &
\makecell{$7.594$ \\ {\ \ \scriptsize$\pm0.346$}} &
\makecell{$7.383$ \\ {\ \ \scriptsize$\pm0.257$}} &
\makecell{$\textbf{7.791}$ \\ {\ \ \scriptsize$\pm0.410$}} &
\makecell{$7.520$ \\ {\scriptsize$\pm0.319$}} \\
\bottomrule
\end{tabular}
\end{table*}

% \newpage
% \subsection{Evolution of Gradient Norms}
% In the main text, we use cumulative average gradient norms to improve the visualization; here, we present the evolution of gradient norms in each iteration.
% \input{Tex/figure_gradient_norm_evo}

\newpage
\subsection{Sample Images}

\begin{figure}[h]
    \centering
    \subfigure[$\beta =- 0.3, \rho=0.9$]{
        \includegraphics[width=0.25\textwidth]{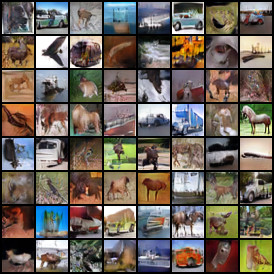}
    }
    \subfigure[$\beta =- 0.2, \rho=0.9$]{
        \includegraphics[width=0.25\textwidth]{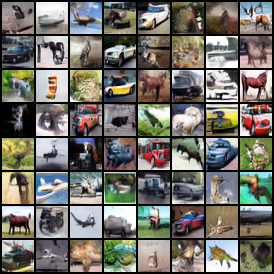}
    }
    \subfigure[$\beta = 0.0, \rho=0.9$]{
	\includegraphics[width=0.25\textwidth]{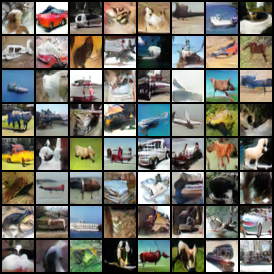}
    }\hfill
    \subfigure[$\beta = 0.2, \rho=0.9$]{
	\includegraphics[width=0.25\textwidth]{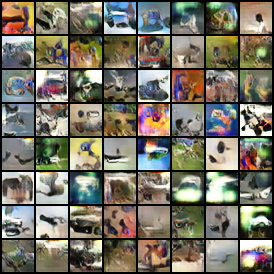}
    }
    \subfigure[$\beta = 0.3, \rho=0.9$]{
	\includegraphics[width=0.25\textwidth]{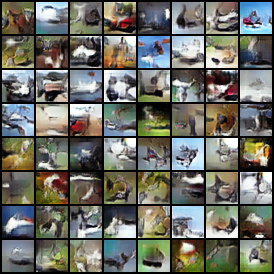}
    }
    \subfigure[$\beta = 0.5, \rho=0.9$]{
	\includegraphics[width=0.25\textwidth]{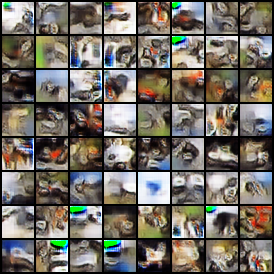}
    }
    \caption{\textit{Sample images for different $\beta$. Architecture: ResNet. Data Set: CIFAR-10. }}
    \label{betaResCIFAR}
\end{figure}

\begin{figure}[h]
    \centering
    \subfigure[$\beta = 0.0, \rho=0.5$]{
	\includegraphics[width=0.25\textwidth]{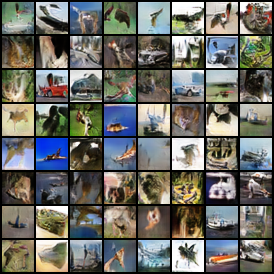}
    }
    \subfigure[$\beta = 0.0, \rho=0.7$]{
	\includegraphics[width=0.25\textwidth]{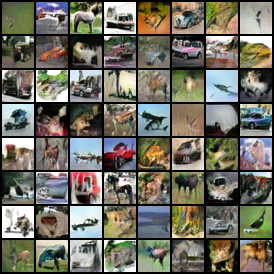}
    }
    \subfigure[$\beta = 0.0, \rho=0.9$]{
	\includegraphics[width=0.25\textwidth]{Pictures/beta00_09_res32_cifar10.png}
    }
    \caption{\textit{Sample images for different $\rho$. Architecture: ResNet. Data Set: CIFAR-10. }}
    \label{betaResCIFAR2}
\end{figure}

\begin{figure}[t]
    \centering
    \subfigure[$\beta =- 0.3, \rho=0.9$]{
        \includegraphics[width=0.25\textwidth]{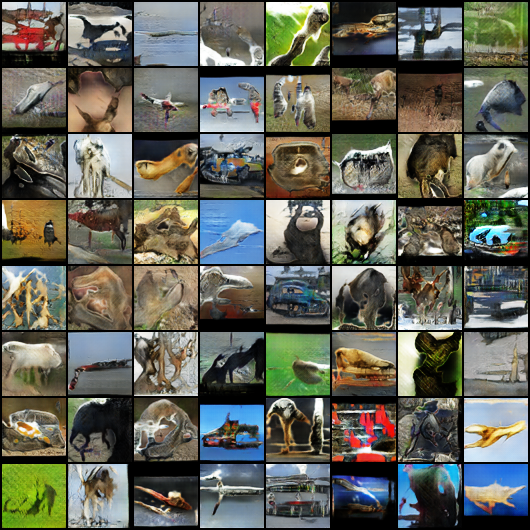}
    }
    \subfigure[$\beta =- 0.2, \rho=0.9$]{
        \includegraphics[width=0.25\textwidth]{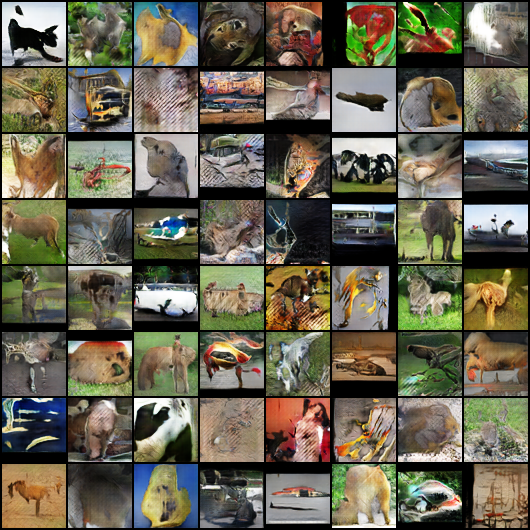}
    }
    \subfigure[$\beta = 0.0, \rho=0.9$]{
	\includegraphics[width=0.25\textwidth]{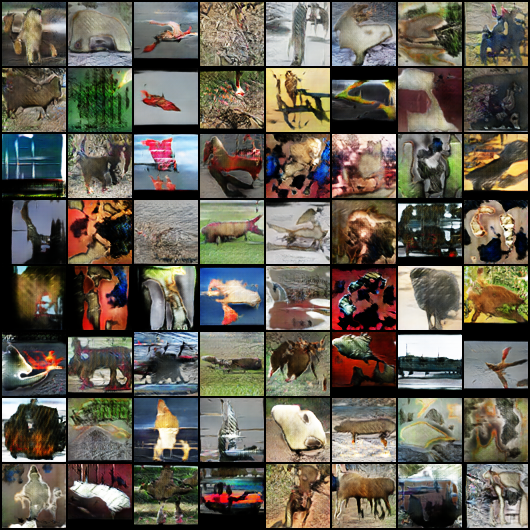}
    }\hfill
    \subfigure[$\beta = 0.2, \rho=0.9$]{
	\includegraphics[width=0.25\textwidth]{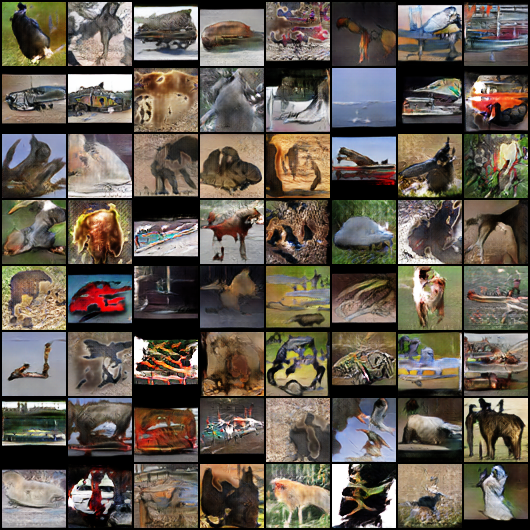}
    }
    \subfigure[$\beta = 0.3, \rho=0.9$]{
	\includegraphics[width=0.25\textwidth]{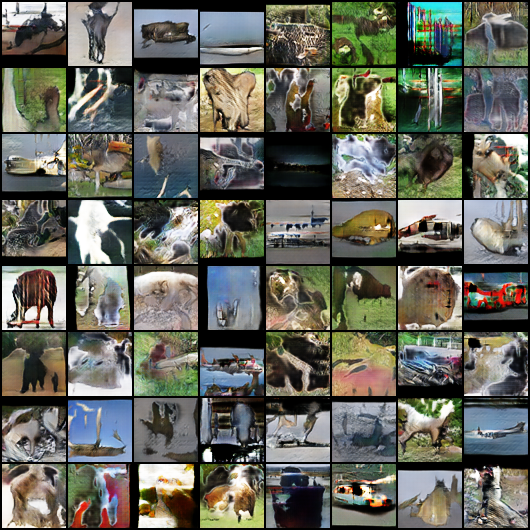}
    }
    \subfigure[$\beta = 0.5, \rho=0.9$]{
	\includegraphics[width=0.25\textwidth]{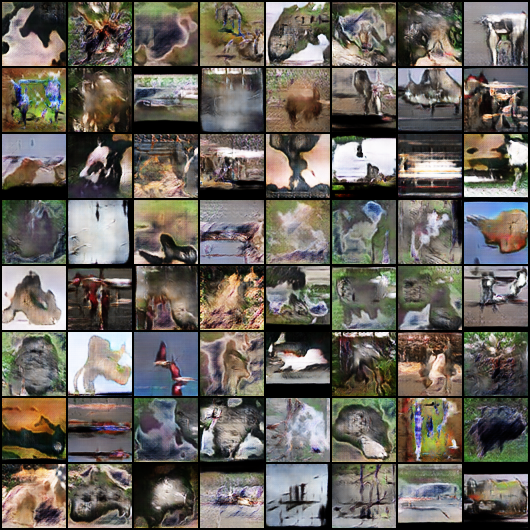}
    }
    \caption{\textit{Sample images for different $\beta$. Architecture: CNN. Data Set: STL-10. }}
    \label{betaCNNSTL}
\end{figure}
\begin{figure}[h]
    \centering
    \subfigure[$\beta = -0.3, \rho=0.5$]{
	\includegraphics[width=0.25\textwidth]{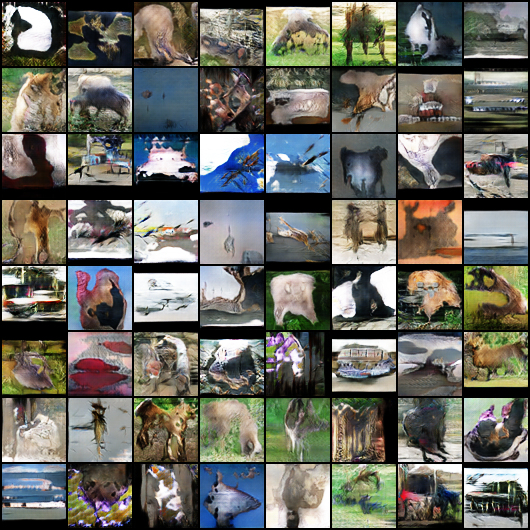}
    }
    \subfigure[$\beta = -0.3, \rho=0.7$]{
	\includegraphics[width=0.25\textwidth]{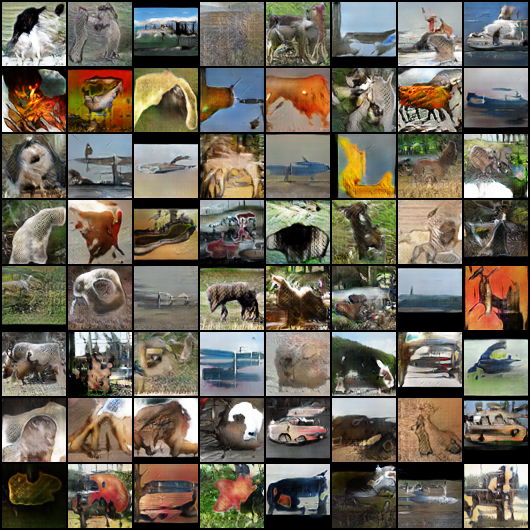}
    }
    \subfigure[$\beta = -0.3, \rho=0.9$]{
	\includegraphics[width=0.25\textwidth]{Pictures/betaneg03_09_cnn32_stl10.png}
    }
    \caption{\textit{Sample images for different $\rho$. Architecture: CNN. Data Set: STL-10. }}
    \label{betaCNNSTL2}
\end{figure}

%\input{Tex/Ap6}

%%%%%%%%%%%%%%%%%%%%%%%%%%%%%%%%%%%%%%%%%%%%%%%%%%%%%%%%%%%%%%%%%%%%%%%%%%%%%%%
%%%%%%%%%%%%%%%%%%%%%%%%%%%%%%%%%%%%%%%%%%%%%%%%%%%%%%%%%%%%%%%%%%%%%%%%%%%%%%%

\end{document}